\documentclass{article}

\usepackage[nonatbib, final]{neurips_2021}

\usepackage[utf8]{inputenc} 
\usepackage[T1]{fontenc}    
\usepackage{hyperref}       
\usepackage{url}            
\usepackage{booktabs}       
\usepackage{amsfonts}       
\usepackage{nicefrac}       
\usepackage{microtype}      
\usepackage{booktabs}       
\usepackage{xcolor}         
\usepackage{graphicx}
\usepackage[maxnames = 6]{biblatex}
\usepackage{hyperref}
\usepackage{amsmath,amsfonts,amssymb}
\usepackage{mathtools}
\usepackage{dsfont}
\usepackage[ruled,vlined,linesnumbered,noend]{algorithm2e}
\usepackage{authblk}      
\usepackage{minitoc}      
\usepackage{multirow}

\usepackage{amsthm}
    {
        \theoremstyle{plain}
        \newtheorem{assumption}{Assumption}
        \newtheorem{thm}{Theorem}[section]
        \newtheorem{lem}[thm]{Lemma}
        \newtheorem{prop}[thm]{Proposition}
        
        \newtheorem{definition}{Definition}
        
        \newtheorem{remark}{Remark}
        
        \theoremstyle{plain}
    }

\newtheorem*{rep@prop}{\rep@title}
\newcommand{\newrepprop}[2]{%
\newenvironment{rep#1}[1]{%
    \def\rep@title{#2 \ref{##1}}%
        \begin{rep@prop}
    }%
    {\end{rep@prop}}
}

\newtheorem*{rep@thm}{\rep@title}
\newcommand{\newrepthm}[2]{%
\newenvironment{rep#1}[1]{%
    \def\rep@title{#2 \ref{##1}}%
        \begin{rep@thm}
    }%
    {\end{rep@thm}}
}

\newrepprop{prop}{Proposition}
\newrepthm{thm}{Theorem}

\makeatletter
\newcommand{\neutralize}[1]{\expandafter\let\csname c@#1\endcsname\count@}
\makeatother

\newenvironment{assumptionbis}[1]
  {\renewcommand{\theassumption}{\ref*{#1}$'$}%
   \neutralize{assumption}\phantomsection
   \begin{assumption}}
  {\end{assumption}}
  
\newenvironment{thmbis}[1]
  {\renewcommand{\thethm}{\ref*{#1}$'$}%
   \neutralize{thm}\phantomsection%
   \begin{thm}}
  {\end{thm}}
  
\theoremstyle{definition}
\newtheorem{dfn}{Definition}[section]

\newtheorem*{rep@definition}{\rep@title}
\newcommand{\newrepdefinition}[2]{%
\newenvironment{rep#1}[1]{%
    \def\rep@title{#2 \ref{##1}}%
        \begin{rep@definition}
    }%
    {\end{rep@definition}}
}
\newrepdefinition{definition}{Definition}

\renewcommand{\vec}[1]{\mathbf{#1}}                                                             
\newcommand{\FedEM}{\texttt{{FedEM}}}
\newcommand{\DEM}{\texttt{D-FedEM}}           
\newcommand{\FedAvg}{\texttt{FedAvg}}
\newcommand{\FedProx}{\texttt{FedProx}}
\newcommand{\pFedMe}{\texttt{pFedMe}}

\DeclareMathOperator{\tr}{tr}

\DeclarePairedDelimiter\floor{\lfloor}{\rfloor}

\DeclareMathOperator*{\argmin}{arg\,min}

\DeclareMathOperator*{\minimize}{minimize}
\DeclareMathOperator*{\E}{\mathbb{E}}
\DeclareMathOperator*{\dd}{d}
\DeclareMathOperator{\prox}{prox}

\DeclareUnicodeCharacter{0301}{\'{e}}

\usepackage{todonotes}

\allowdisplaybreaks

\title{Federated Multi-Task Learning \\ under a Mixture of Distributions}

\author[1,3]{Othmane Marfoq}
\author[1]{Giovanni Neglia}
\author[2]{Aur\'elien Bellet}
\author[3]{Laetitia Kameni}
\author[3]{Richard Vidal}
\affil[1]{Inria,~Universit\'e C\^ote d'Azur,~France,~\{othmane.marfoq,~giovanni.neglia\}@inria.fr}
\affil[2]{Inria,~Universit\'e de Lille,~France,~aurelien.bellet@inria.fr}
\affil[3]{Accenture Labs,~France,~\{richard.vidal,~laetitia.kameni\}@accenture.com}

\begin{document}

\doparttoc 
\faketableofcontents 

\part{} 

\maketitle

\begin{abstract}

The increasing size of data generated by smartphones and IoT devices motivated the development of 
\emph{Federated Learning} (FL), a framework for on-device collaborative training of machine learning models. First efforts in FL focused on learning a single global model with good average performance across clients, but the global model may be arbitrarily bad for a given client, due to the inherent heterogeneity of local data distributions. 
Federated \emph{multi-task learning} (MTL) approaches can learn \emph{personalized models} by formulating an opportune penalized optimization problem.
The penalization term can capture complex relations among personalized models, but eschews clear statistical assumptions about local data distributions.

In this work, we propose to study federated MTL under the flexible assumption that each local data distribution is a \emph{mixture of unknown underlying distributions}.
This assumption encompasses most of the existing personalized FL approaches and leads to federated EM-like algorithms for both client-server and fully decentralized settings.
Moreover, it  provides a principled way to serve personalized models to clients not seen at training time. The algorithms' convergence is analyzed through a novel federated surrogate optimization framework, which can be of general interest.
Experimental results on FL benchmarks show that our approach provides models with higher accuracy and fairness than state-of-the-art methods.

\end{abstract}

\section{Introduction}
\label{sec:introduction}
Federated Learning (FL) \cite{kairouz2019advances} allows a set of clients to collaboratively train models without sharing their local data.
Standard FL approaches train a unique model for all clients~\cite{mcmahan2017communication,konevcny2016federated, Sahu2018OnTC, karimireddy2020scaffold, mohri2019agnostic}.
However, as discussed in~\cite{sattler2020clustered}, the existence of such a global model suited for all clients is at odds with the statistical heterogeneity observed across different clients~\cite{li2020federated,kairouz2019advances}.
Indeed, clients can have non-iid data and \emph{varying preferences}. 
Consider for example a language modeling task: given the sequence of tokens ``\textit{I love eating},'' the next word can be arbitrarily different from one client to another.
Thus, having personalized models for each client is a necessity in many FL applications.

\textbf{Previous work on personalized FL.} A naive approach for FL personalization consists in learning first a global model and then fine-tuning its parameters at each client via a few iterations of stochastic gradient descent~\cite{sim2019investigation}.
In this case, the global model plays the role of a meta-model to be used as initialization for few-shot adaptation at each client.
In particular, the connection between FL and Model Agnostic Meta Learning (MAML)~\cite{jiang2019improving} has been studied in \cite{fallah2020personalized, khodak2019adaptive, acar2021debiasing} in order to build a more suitable meta-model for local personalization.
Unfortunately, these methods can fail to build a model with low generalization error (as exemplified by LEAF synthetic dataset~\cite[App.~1]{caldas2018leaf}). 
An alternative approach is to jointly train a global model and one local model per client and then let each client build a personalized model by interpolating them~\cite{deng2020adaptive,corinzia2019variational,mansour2020three}.
However, if local distributions are far from the average distribution, a relevant global model does not exist and this approach boils down to every client learning only on its own local data. This issue is formally captured by the generalization bound in~\cite[Theorem~1]{deng2020adaptive}. 

Clustered FL~\cite{sattler2020clustered, ghosh2020efficient, mansour2020three} addresses the potential lack of a global model by assuming that clients can be partitioned into several clusters.
Clients belonging to the same cluster share the same optimal model, but those models can be arbitrarily different across clusters (see \cite[Assumption~2]{sattler2020clustered} for a rigorous formulation).
During training, clients learn the cluster to which they belong as well as the cluster model. 
The Clustered FL assumption is also quite limiting, as no knowledge transfer is possible across clusters.
In the extreme case where each client has its own optimal local model (recall the example on language modeling), the number of clusters  coincides with the number of clients and no federated learning is possible.

Multi-Task Learning (MTL) has recently emerged as an alternative approach to learn personalized models in the federated setting and allows for more nuanced relations among clients' models~\cite{smith2017federated, Vanhaesebrouck2017a, pmlr-v108-zantedeschi20a, hanzely2020federated, dinh2020personalized}. 
The authors of \cite{smith2017federated,Vanhaesebrouck2017a} were the first to frame FL personalization as a MTL problem. 
In particular, they defined federated MTL as a penalized optimization problem, where the penalization term models relationships among tasks (clients).
The work \cite{smith2017federated} proposed the \textsc{Mocha} algorithm for the client-server scenario, while \cite{Vanhaesebrouck2017a,pmlr-v108-zantedeschi20a} presented decentralized algorithms for the same problem.
Unfortunately, these algorithms can only learn simple models (linear models or linear combination of pre-trained models), because of the complex penalization term.
Other MTL-based approaches~\cite{hanzely2020federated, hanzely2020lower, dinh2020personalized, huang2021personalized, li2021ditto} are able to train more general models at the cost of considering simpler penalization terms (e.g., the distance to the average model), thereby losing the capability to capture complex relations among tasks.
Moreover, a general limitation of this line of work is that the penalization term is justified qualitatively and not on the basis of clear statistical assumptions on local data distributions.

More recently, \cite{shamsian2021personalized} proposed \texttt{pFedHN}.
\texttt{pFedHN} feeds local clients' representations to a global (across clients) hypernetwork, which can output personalized heterogeneous models.
Unfortunately, the hypernetwork has a large memory footprint already for small clients' models (e.g., the hypernetwork in the experiments in~\cite{shamsian2021personalized} has $100$ more parameters than the output model). 
Hence, it is not clear if  \texttt{pFedHN} can scale to more complex models.
Moreover, \texttt{pFedHN} requires each client to communicate multiple times for the server to learn meaningful representations.
Therefore, its performance is likely to deteriorate when clients participate only once (or few times) to training, as it is the case for large-scale cross-device FL training.
Furthermore, even once the hypernetwork parameters have been learned, training personalized models for new clients still requires multiple client-server communication rounds.
More similar to our approach, \texttt{FedFOMO} \cite{zhang2020personalized} lets each client interpolate other clients' local models with opportune weights learned during training.
However, this method lacks both theoretical justifications for such linear combinations and convergence guarantees.
Moreover, \texttt{FedFOMO} requires the presence of a powerful server able to 1) store all individual local models and 2) learn for each client---through repeated interactions---which other clients' local models may be useful.
Therefore, \texttt{FedFOMO} is not suited for cross-device FL where the number of clients may be very large (e.g.,~$10^{5}$--$10^{7}$ participating clients \cite[Table~2]{kairouz2019advances}) and a given client may only participate in a single training round.

Overall, although current personalization approaches can lead to superior empirical performance in comparison to a shared global model or individually trained local models, it is still not well understood whether and under which conditions clients are guaranteed to benefit from collaboration.

\textbf{Our contributions.} In this work, we first show that federated learning is impossible without assumptions on local data distributions.
Motivated by this negative result, we formulate a general and flexible assumption: \emph{the data distribution of each client is a mixture of $M$ underlying distributions.} The proposed formulation has the advantage that each client can benefit from knowledge distilled from all other clients' datasets (even if any two clients can be arbitrarily different from each other).
We also show that this assumption encompasses most of the personalized FL approaches previously proposed in the literature.

In our framework, a personalized model is a linear combination of $M$ shared component models. All clients jointly learn the $M$ components, while each client learns its personalized mixture weights.
We show that federated EM-like algorithms can be used for training. 
In particular, we propose \FedEM{} and  \DEM{} for the client-server and the fully decentralized settings, respectively, and we prove convergence guarantees.
Our approach also provides a principled and efficient way to infer personalized models for clients unseen at training time.
Our algorithms can easily be adapted to solve more general problems in a novel framework, which can be seen as a federated extension of the centralized surrogate optimization approach in \cite{mairal2013optimization}.
To the best of our knowledge, our paper is the first work to
propose federated surrogate optimization algorithms with convergence guarantees.

Through extensive experiments on FL benchmark datasets, we show that our approach generally yields models that 1) are on average more accurate, 2) are fairer across clients, and 3) generalize better to unseen clients than state-of-the-art personalized and non-personalized FL approaches.

\textbf{Paper outline.} The rest of the paper is organized as follows. 
In Section~\ref{sec:prblm_formulation} we provide our impossibility result, introduce our main assumptions, and show that several popular personalization approaches can be obtained as special cases of our framework. 
Section~\ref{sec:fedem} describes our algorithms, states their convergence results, and presents our general federated surrogate optimization framework.
Finally, we provide experimental results in Section~\ref{sec:experiments} before concluding in Section~\ref{sec:conclusion}.

\section{Problem Formulation}
\label{sec:prblm_formulation}
We consider a (countable) set $\mathcal T$ of classification (or regression) tasks which represent the set of possible clients.
We will use the terms task and client interchangeably.
Data at  client $t \in \mathcal T$ is generated according to a local distribution $\mathcal{D}_{t}$  over $\mathcal{X} \times \mathcal{Y}$.
Local data distributions $\{\mathcal{D}_{t}\}_{t \in \mathcal T}$ are in general  different, thus it is natural to fit a separate model (hypothesis) $h_{t} \in \mathcal{H}$ to each data distribution~$\mathcal{D}_{t}$.
The goal is then to solve (in parallel) the following optimization problems
\begin{equation}
    \label{eq:main_problem}
    \forall t \in \mathcal T,\quad \minimize_{h_{t} \in \mathcal{H}} \mathcal{L}_{\mathcal{D}_{t}}(h_{t}),
\end{equation}
where $h_{t}: \mathcal{X} \mapsto \Delta^{|\mathcal{Y}|}$ ($\Delta^{D}$ denoting the unitary simplex of dimension $D$), $l: \Delta^{|\mathcal{Y}|} \times \mathcal{Y} \mapsto \mathbb{R}^{+}$ is a loss function,\footnote{
    In the case of (multi-output) regression, we have $h_{t}: \mathcal{X} \mapsto \mathbb{R}^{d}$ for some $d\geq 1$ and $l: \mathbb{R}^{d}\times\mathbb{R}^{d} \mapsto \mathbb{R}^{+}$.
}
and $\mathcal{L}_{\mathcal{D}_{t}}(h_{t})=\E_{(\vec{x},y) \sim \mathcal{D}_{t}} \left[l(h_{t}(\vec{x}), y)\right]$ is the true risk of a model $h_{t}$ under data distribution $\mathcal{D}_{t}$. 
For $(\vec{x}, y) \in \mathcal{X}\times\mathcal{Y}$, we will denote the joint distribution density associated to $\mathcal{D}_{t}$ by $p_{t}(\vec{x}, y)$, and the marginal densities by $p_{t}(\vec{x})$ and $p_{t}(y)$. 

A set of $T$ clients $[T] \triangleq\{1, 2, \dots T\} \subseteq \mathcal T$ participate to the initial training phase; other clients may join the system in a later stage.
We denote by $\mathcal{S}_{t}=\{s_{t}^{(i)}=(\vec{x}_{t}^{(i)},~y_{t}^{(i)})\}_{i=1}^{n_t}$ the dataset at client $t \in [T]$ drawn i.i.d. from $\mathcal{D}_{t}$, and by $n=\sum_{t=1}^T n_t$ the total dataset size.  

The idea of federated learning is to enable each client to benefit from data samples available at other clients in order to get a better estimation of $\mathcal{L}_{\mathcal{D}_{t}}$, and therefore get a model with a better generalization ability to unseen examples. 

\subsection{An Impossibility Result}
\label{sec:impossibility_results}

We start by showing that some assumptions on the local distributions  $p_{t}(\vec{x}, y),~t\in \mathcal T$  are needed for federated learning to be possible, i.e., for each client to be able to take advantage of the data at other clients.
This holds even if all clients participate to the initial training phase (i.e., $\mathcal T = [T]$).

We consider the classic PAC learning framework where we fix a class of models $\mathcal{H}$ and seek a learning algorithm which is guaranteed, for all possible data distributions over $\mathcal{X} \times\mathcal{Y}$, to return with high probability a model with expected error $\epsilon$-close to the best possible error in the class $\mathcal{H}$.
The worst-case sample complexity then refers to the minimum amount of labeled data required by any algorithm to reach a given $\epsilon$-approximation. 

Our impossibility result for FL is based on a reduction to an impossibility result for Semi-Supervised Learning (SSL), which is the problem of learning from a training set with only a small amount of labeled data.
The authors of \cite{BenDavid2008DoesUD} conjectured that, when the quantity of unlabeled data goes to infinity, the worst-case sample complexity of SSL improves over supervised learning at most by a constant factor that only depends on the hypothesis class \cite[Conjecture~4]{BenDavid2008DoesUD}. 
This conjecture was later proved for the realizable case and hypothesis classes of finite VC dimension~\cite[Theorem~1]{Darnstdt2013UnlabeledDD}, even when the marginal distribution over the domain set $\mathcal X$ is known~\cite[Theorem~2]{gopfert2019can}.
\footnote{
    We note that whether the conjecture in \cite{BenDavid2008DoesUD} holds in the agnostic case is still an open problem.
}

In the context of FL, if the marginal distributions $p_{t}\left(\vec{x}\right)$ are identical, but the conditional distributions $p_{t}\left(y|\vec{x}\right)$ can be arbitrarily different, then each client $t$ can learn using: 1) its own local labeled dataset, and 2) the other clients' datasets, but only as unlabeled ones (because their labels have no relevance for~$t$).
The FL problem, with $T$ clients, then reduces to $T$  parallel SSL problems, or more precisely, it is at least as difficult as $T$  parallel SSL problems (because client $t$ has no direct access to the other local datasets but can only learn through the communication exchanges allowed by the FL algorithm).
The SSL impossibility result implies that, without any additional assumption on the local distributions $p_{t}\left(\vec{x}, y\right),~t\in[T]$, any FL algorithm can reduce the sample complexity of client-$t$'s problem in \eqref{eq:main_problem} only by a constant in comparison to local learning, independently of how many other clients participate to training and how large their datasets' sizes are.

\subsection{Learning under a Mixture Model}
Motivated by the above impossibility result, in this work we propose to consider that each local data distribution $\mathcal{D}_{t}$ is a mixture of $M$ underlying distributions $\tilde{\mathcal{D}}_{m},~1\leq m \leq M$, as formalized below.

\begin{assumption}
    \label{assum:mixture}
    There exist $M$ underlying (independent) distributions $\tilde{\mathcal{D}}_{m},~1\leq m \leq M$, such that for $t \in \mathcal{T}$, $\mathcal{D}_{t}$ is mixture of the distributions $\{\tilde{\mathcal{D}}_{m}\}_{m=1}^M$ with weights $\pi_{t}^* = \left[\pi_{t1}^*, \dots, \pi_{tM}^*\right] \in \Delta^{M}$, i.e.
    \begin{equation}
        \label{eq:mixture}
        z_{t} \sim \mathcal{M}(\pi_{t}^*) , \quad \left(\left(\vec{x}_{t}, y_{t}\right)|z_{t}=m \right)\sim \tilde{\mathcal{D}}_{m}, \quad\forall t\in \mathcal{T},
    \end{equation}
    where $\mathcal{M}(\pi)$ is a multinomial (categorical) distribution with parameters $\pi$. 
\end{assumption}

Similarly to what was done above, we use $p_{m}(\vec{x},y)$, $p_{m}(\vec{x})$, and $p_{m}(y)$ to denote the probability distribution densities associated to $\tilde{\mathcal{D}}_{m}$. We further assume that  marginals over $\mathcal X$ are identical.
\begin{assumption}
    \label{assum:uniform_marginal}
    For all $m \in[M]$, we have $p_{m}(\vec{x}) = p(\vec{x})$.
\end{assumption}
Assumption~\ref{assum:uniform_marginal} is not strictly required for our analysis to hold, but, in the most general case, solving Problem~\eqref{eq:main_problem} requires to learn generative models. Instead, under Assumption~\ref{assum:uniform_marginal} we can restrict our attention to discriminative models (e.g., neural
networks).
\footnote{
    A possible way to ensure that Assumption~\ref{assum:uniform_marginal} holds is to use the batch normalization technique from \cite{li2020fedbn} to account for feature shift.
} 
More specifically, we consider a parameterized set of models $\tilde {\mathcal{H}}$ with the following properties.
\begin{assumption}
    \label{assum:logloss}
    $\tilde{\mathcal{H}}=\{h_\theta\}_{\theta\in\mathbb{R}^{d}}$ is a set of hypotheses parameterized by $\theta\in\mathbb{R}^d$, whose convex hull is in $\mathcal H$. For each distribution $\tilde{\mathcal D}_m$ with $m\in[M]$, there exists a hypothesis ${h}_{\theta_{m}^*}$, 
    such that
    \begin{equation}
        \label{eq:log_loss}
        l\left({h}_{\theta_m^*}\!\left({\vec{x}}\right), {y}\right)  = - \log  p_{m}({y}|{\vec{x}}) + c,
    \end{equation}
    where $c\in\mathbb{R}$ is a normalization constant. The function $l(\cdot,\cdot)$ is then the log-loss associated to~$p_{m}({y}|{\vec{x}})$. 
\end{assumption}
We refer to the hypotheses in $\tilde{\mathcal H}$ as \emph{component models} or simply \emph{components}.
We denote by $\Theta^*\in\mathbb{R}^{M\times d}$ the matrix whose $m$-th row is $\theta_{m}^*$, and by  $\Pi^*\in \Delta^{T\times M}$ the matrix whose $t$-th row is $\pi_t^*\in\Delta^{M}$.
Similarly, we will use $\Theta$ and $\Pi$ to denote arbitrary parameters.

\begin{remark}
Assumptions~\ref{assum:uniform_marginal}--\ref{assum:logloss} are mainly technical and are not required for our approach to work in practice.
Experiments in Section~\ref{sec:experiments} show that our algorithms perform well on standard FL benchmark datasets, for which these assumptions do not hold in general. 
\end{remark}

Note that, under the above assumptions, $p_t(\vec x, y)$ depends on $\Theta^*$ and $\pi_t^*$.
Moreover, we can prove (see App.~\ref{proof:general_mtl_framework}) that the optimal local model $h_t^{*} \in \mathcal H$ for client $t$ is a weighted average of models in $\tilde{ \mathcal H}$.
\begin{prop}
    \label{prop:local_model_is_weighted_average_of_components}
    Let $l(\cdot,\cdot)$ be the mean squared error loss, the logistic loss or the cross-entropy loss,  and 
    $\breve\Theta$ and $\breve\Pi$ be a solution of the following optimization problem:
    \begin{equation}
        \label{eq:intermediate_problem}
        \minimize_{\Theta, \Pi} \E_{t\sim D_\mathcal{T} } \E_{(\vec{x},y) \sim \mathcal{D}_{t}} \left[ - \log p_t(\vec x,y | \Theta, \pi_t) \right],
    \end{equation}
    where $D_\mathcal{T}$ is any distribution with support $\mathcal T$. Under Assumptions~\ref{assum:mixture}, \ref{assum:uniform_marginal}, and~\ref{assum:logloss}, the predictors  
    \begin{equation}
        \label{eq:linear_combination}
        h_t^{*} = \sum_{m=1}^{M}\breve \pi_{t m}h_{\breve \theta_{m}}\left(\vec{x}\right), \quad\forall t \in \mathcal{T}
    \end{equation}
    minimize $\mathbb{E}_{(\vec{x}, y)\sim\mathcal{D}_{t}}\left[l(h_t(\vec{x}), y)\right]$ and thus solve Problem~\eqref{eq:main_problem}.
\end{prop}
Proposition~\ref{prop:local_model_is_weighted_average_of_components} suggests the following approach to solve Problem~\eqref{eq:main_problem}.
First, we estimate the parameters $\breve{\Theta}$ and $\breve{\pi}_{t},~1\leq t \leq T$, by minimizing the empirical version of Problem~\eqref{eq:intermediate_problem} on the training data, i.e., minimizing:
\begin{equation}
    \label{eq:empirical_loss}
    f(\Theta, \Pi) \triangleq - \frac{\log p(\mathcal{S}_{1:T}| \Theta, \Pi)}{n} \triangleq - \frac{1}{n}\sum_{t=1}^{T}\sum_{i=1}^{n_{t}} \log p(s_t^{(i)}|\Theta, \pi_{t}),
\end{equation}
which is the (negative) likelihood of the probabilistic model \eqref{eq:mixture}.
\footnote{
    As the distribution $\mathcal D_{\mathcal T}$ over tasks in Proposition~\ref{prop:local_model_is_weighted_average_of_components} is arbitrary, any positively weighted sum of clients' empirical losses could be considered.
}
Second, we use \eqref{eq:linear_combination} to get the client predictor for the $T$ clients present at training time.
Finally, to deal with a client $t_{\text{new}}\notin [T]$ not seen during training, we keep the mixture component models fixed and simply choose the weights $\pi_{t_{\text{new}}}$ that maximize the likelihood of the client data and get the client predictor via \eqref{eq:linear_combination}.
 
\subsection{Generalizing Existing Frameworks}
\label{sec:generalizing}
Before presenting our federated learning algorithms in Section~\ref{sec:fedem}, we show that the generative model in Assumption~\ref{assum:mixture} extends some popular  multi-task/personalized FL formulations in the literature.

\textbf{Clustered Federated Learning~\cite{sattler2020clustered,ghosh2020efficient}}~assumes that each client belongs to one among $C$ clusters and proposes that all clients in the same cluster learn the same model.
Our framework recovers this scenario considering  $M=C$ and $\pi_{t c}^{*}=1$ if task (client) $t$ is in cluster $c$ and $\pi_{t c}^{*}=0$ otherwise.

\textbf{Personalization via model interpolation~\cite{mansour2020three,deng2020adaptive}} relies on learning a global model $h_{\textrm{glob}}$ and $T$ local models $h_{\textrm{loc},t}$, and then using at each client the linear interpolation $h_t = \alpha_t h_{\textrm{loc},t} + (1-\alpha_t) h_{\textrm{glob}}$.
Each client model can thus be seen as a linear combination of $M=T+1$ models $h_m = h_{\textrm{loc},m}$ for $ m \in [T]$ and $h_0 = h_{\textrm{glob}}$ with specific weights $\pi^{*}_{tt} = \alpha_t$, $\pi^{*}_{t0} = 1-\alpha_t$, and $\pi^{*}_{tt'} = 0$ for $t'\in [T] \setminus \{t\}$.

\textbf{Federated MTL via task relationships.} The authors of~\cite{smith2017federated} proposed to learn personalized models by solving the following optimization problem inspired from classic MTL formulations:
\begin{equation}
    \label{eq:virginia}
    \min_{W,\Omega}~ \sum_{t=1}^{T}\sum_{i=1}^{n_{t}}l(h_{w_{t}}(\vec{x}_{t}^{(i)}), y_{t}^{(i)}) +  \lambda \tr\left(W\Omega W^{\intercal}\right),
\end{equation}
where $h_{w_t}$ are linear predictors parameterized by the rows of matrix $W$ and 
the matrix $\Omega$ captures task relationships (similarity). 
This formulation is motivated by the alternating structure optimization method (ASO)~\cite{JMLR:v6:ando05a,Zhou2011ClusteredMTLviaASO}.
In App.~\ref{sec:general_mtl_framework}, we show that, when predictors $h_{\theta^{*}_m}$ are linear and have bounded norm, our framework leads to the same ASO formulation that motivated Problem~\eqref{eq:virginia}. 
Problem~\eqref{eq:virginia} can also be justified by probabilistic priors~\cite{zhang2010convex} or graphical models~\cite{lauritzen1996graphical}~(see \cite[App. B.1]{smith2017federated}).
Similar considerations hold for our framework (see again App.~\ref{sec:general_mtl_framework}). 
Reference~\cite{pmlr-v108-zantedeschi20a} extends the approach in~\cite{smith2017federated} by letting each client learn a personalized model as a weighted combination of $M$ \emph{known} hypotheses.
Our approach is more general and flexible as clients learn both the weights and the hypotheses.
Finally, other personalized FL algorithms, like \pFedMe~\cite{dinh2020personalized}, \texttt{FedU}~\cite{dinh2021fedu}, and those studied in \cite{hanzely2020federated} and in \cite{hanzely2020lower}, can be framed as special cases of formulation \eqref{eq:virginia}.
Their assumptions can thus also be seen as a particular case of our framework.

\section{Federated Expectation-Maximization}
\label{sec:fedem}
\subsection{Centralized Expectation-Maximization}

Our goal is to estimate the optimal components' parameters $\Theta^{*} = \left(\theta^{*}_{m}\right)_{1\leq m \leq M}$ and mixture weights $\Pi^*=(\pi^*_t)_{1\leq t \leq T}$ by minimizing  the negative log-likelihood $f(\Theta, \Pi)$ in \eqref{eq:empirical_loss}.
A natural approach to solve such non-convex problems is the Expectation-Maximization algorithm (EM), which alternates between two steps.
Expectation steps update the distribution (denoted by $q_t$) over the latent variables~$z_t^{(i)}$ for every data point $s_t^{(i)}=(\vec{x}_{t}^{(i)},y_{t}^{(i)})$ given the current estimates of the parameters $\left\{\Theta, \Pi\right\}$.
Maximization steps update the parameters $\left\{\Theta, \Pi\right\}$ by maximizing the expected log-likelihood, where the expectation is computed according to the current  latent variables' distributions.

The following proposition provides the EM updates for our problem (proof in  App.~\ref{proof:centralized_em}). 
\begin{prop}
    \label{prop:em}
    Under Assumptions~\ref{assum:mixture} and~\ref{assum:uniform_marginal}, at the $k$-th iteration the EM algorithm updates parameter estimates through the following steps:
    \begin{flalign}
        \textbf{E-step:} && q^{k+1}_{t}(z_{t}^{(i)}=m) & \propto  \pi_{t m}^{k} \cdot \exp\left(-l(h_{\theta_{m}^{k}}(\vec x_t^{(i)}), y_{t}^{(i)})\right), & t \in [T],~ m \in [M],~ i \in [n_{t}]~~
        \label{eq:e_step}
    \end{flalign}
        \begin{flalign}
        \textbf{M-step:} &&\quad~\pi^{k+1}_{t m} & = \frac{\sum_{i=1}^{n_{t}}q^{k+1}_{t}(z_{t}^{(i)}=m)}{n_{t}}, &   t \in [T],~\ m \in [M]  \label{eq:update_pi}
        \\
        &&\theta^{k+1}_{m} & \in \argmin_{\theta\in\mathbb{R}^{d}} \sum_{t=1}^{T}\sum_{i=1}^{n_{t}} q^{k+1}_{t}(z_{t}^{(i)}=m)  l\big(h_{\theta}(\vec{x}_{t}^{(i)}), y_{t}^{(i)}\big), & m \in [M]   \label{eq:update_h}
    \end{flalign}
\end{prop}

The EM updates in Proposition~\ref{prop:em} have a natural interpretation. In the E-step, given current component models $\Theta^k$ and mixture weights $\Pi^k$, \eqref{eq:e_step} updates the a-posteriori probability $q_t^{k+1}(z_t^{(i)}=m)$  that point $s_t^{(i)}$ of client $t$ was drawn from the $m$-th distribution based on the current mixture weight $\pi_{tm}^k$ and on how well the corresponding component $\theta^k_m$ classifies $s_t^{(i)}$. 
The M-step consists of two updates under fixed probabilities $q_t^{k+1}$.
First, \eqref{eq:update_pi} updates the mixture weights $\pi_t^{k+1}$ to reflect the prominence of each distribution $\tilde{\mathcal D}_m$ in $\mathcal{S}_t$ as given by $q_t^{k+1}$.
Finally, \eqref{eq:update_h} updates the components' parameters $\Theta^{k+1}$ by solving $M$ independent, weighted empirical risk minimization problems with weights given by $q_t^{k+1}$. 
These weights aim to construct an unbiased estimate of the true risk over each underlying distribution $\tilde{\mathcal{D}}_{m}$ using only points sampled from the client mixtures, similarly to importance sampling strategies used to learn from data with sample selection bias \cite{Sugiyama_sample_bias,cortes_sample_bias1,cortes_sample_bias2,weightedERM}.

\subsection{Client-Server Algorithm}
\label{sec:centralized}

Federated learning aims to train machine learning models directly on the clients, without exchanging raw data, and thus we should run EM while assuming that only client $t$ has access to dataset $\mathcal S_t$.
The E-step~\eqref{eq:e_step} and the $\Pi$ update \eqref{eq:update_pi} in the M-step operate separately on each local dataset $\mathcal S_t$ and can thus be performed locally at each client $t$.
On the contrary, the $\Theta$ update \eqref{eq:update_h} requires interaction with other clients, since the computation spans all data samples $\mathcal{S}_{1:T}$. 

In this section, we consider a client-server setting, in which each client $t$ can communicate only with a centralized server (the orchestrator) and wants to learn components' parameters $\Theta^{*} = \left(\theta_{m}^{*}\right)_{1\leq m \leq M}$ and its own mixture weights $\pi_{t}^{*}$. 

We propose the algorithm \FedEM{} for \emph{Federated Expectation-Maximization} (Alg.~\ref{alg:fed_em_short}).
\FedEM{} proceeds through communication rounds similarly to most FL algorithms including  \FedAvg~\cite{mcmahan2017communication}, \FedProx~\cite{Sahu2018OnTC}, SCAFFOLD~\cite{karimireddy2020scaffold},  and \pFedMe~\cite{dinh2020personalized}.
At each round, 1) the central server broadcasts the (shared) component models to the clients, 2) each client locally updates components and its personalized mixture weights, and 3) sends the updated components back to the server, 4) the server aggregates the updates. 
The local update performed at client $t$ consists in performing the steps in~\eqref{eq:e_step} and~\eqref{eq:update_pi} and updating the local estimates of $\theta_m$ through a solver which approximates the exact minimization in~\eqref{eq:update_h} using only the local dataset $\mathcal S_t$ (see line~\ref{line:fedEM_local_solver_step}).
\FedEM{} can operate with different local solvers---even different across clients---as far as they satisfy some local improvement guarantees (see the discussion in App.~\ref{app:black_box_solver}).
In what follows, we restrict our focus on the practically important case where the local solver performs  multiple stochastic gradient descent updates (local SGD~\cite{stich2018local}). 
\begin{algorithm}[t]
    \SetKwFunction{LocalSolver}{LocalSolver}
    \SetKwInOut{Input}{Input}
    \SetKwInOut{Output}{Output}
    \SetKw{Iterations}{iterations}
    \SetKw{Sample}{sample}
    \SetKw{Tasks}{tasks}
    \SetKw{Component}{component}
    \SetKw{InParallel}{in parallel}
    \SetKw{Server}{server}
    \SetKw{Clients}{clients}
    \SetKw{Client}{client}
    \SetKw{Broadcasts}{broadcasts}
    \SetKw{Sends}{sends}
    \SetKw{Over}{over}
    \SetKw{Tothe}{to the}
    \SetKwProg{Fn}{Function}{:}{}
    \SetAlgoLined
    \Input{ Data $\mathcal{S}_{1:T}$; number of mixture distributions $M$; number of communication rounds $K$}
    \Output{ $\theta^{K}_{m},~~m\in[M]$}
    \For{\Iterations $k=1, \dots, K$}{
        \Server \Broadcasts $\theta^{k-1}_{m},~1\leq m \leq M$, \Tothe $T$ \Clients\;
        \For{\Tasks $t=1, \dots, T$ \InParallel \Over $T$ \Clients}{
            \For{\Component $m=1, \dots, M$}{
                update $q^{k}_{t}(z_{t}^{(i)}=m)$ as in \eqref{eq:e_step},~ $\forall i \in \{1, \dots, n_t\}$\;
                update $\pi_{t m}^{k}$ as in \eqref{eq:update_pi}\;
                $\theta_{m,t}^{k} \gets$ \LocalSolver{$m$, $\theta^{k-1}_{m}$, $q_{t}^{k}$, $\mathcal{S}_{t}$}\; \label{line:fedEM_local_solver_step}
            }
            \Client $t$ \Sends $\theta_{m,t}^{k},~1\leq m \leq M$, \Tothe \Server\; 
        }
        \For{\Component $m=1, \dots, M$}{
            $\theta_{m}^{k} \gets \sum_{t=1}^{T}\frac{n_{t}}{n} \times \theta_{m,t}^{k}$\; \label{line:fedEM_aggregation}
        }
    }
    \caption{\FedEM{} (see also the more detailed Alg.~\ref{alg:fed_em} in App.~\ref{app:alg_centralized})}
    \label{alg:fed_em_short}
\end{algorithm}
Under the following standard assumptions (see e.g., \cite{wang2020tackling}), \FedEM{} converges to a stationary point of $f$. 
Below, we use the more compact notation $l(\theta; s_{t}^{(i)}) \triangleq l(h_{\theta}(\vec{x}_{t}^{(i)}), y_{t}^{(i)})$.
\begin{assumption}
    \label{assum:bounded_f}
    The negative log-likelihood $f$ is bounded below by $f^{*} \in \mathbb{R}$.
\end{assumption}
\begin{assumption}
    \label{assum:smoothness}
    (Smoothness)
    For all $t\in[T]$ and $i\in[n_{t}]$, the function $\theta \mapsto l(\theta; s_{t}^{(i)})$ is $L$-smooth and twice continuously differentiable.  
\end{assumption}
\begin{assumption}
    \label{assum:finitie_variance}
    (Unbiased gradients and bounded variance)
    Each client $t\in [T]$ can sample a random batch  $\xi$ from $\mathcal S_t$ and compute an unbiased estimator $\textsl{g}_t(\theta, \xi)$ of the local gradient with bounded variance, i.e.,  $\mathbb{E}_{\xi}[\textsl{g}_t(\theta, \xi)] = \frac{1}{n_t}\sum_{i=1}^{n_t} \nabla_{\theta} l(\theta; s_{t}^{(i)})$ and 
    $\mathbb{E}_{\xi}\|\textsl{g}_t(\theta, \xi) - \frac{1}{n_{t}}\sum_{i=1}^{n_{t}}\nabla_{\theta} l(\theta; s_{t}^{(i)})\|^{2} \leq \sigma^{2}$.
\end{assumption}
\begin{assumption}
    \label{assum:bounded_gradient}
    (Bounded gradient)
    There exists a constant $B>0$, such that for all $t\in[T]$ and $i\in[n_{t}]$, the function $\left\|\theta \mapsto l(\theta; s_{t}^{(i)})\right\| \leq B$.  
\end{assumption}
\begin{assumption}
    \label{assum:bounded_dissimilarity}
    (Bounded dissimilarity)
    There exist $\beta$ and $G$ such that for any set of weights $\alpha\in\Delta^{M}$:
    \[
        \sum_{t=1}^{T}\frac{n_{t}}{n} \Big\|\frac{1}{n_{t}}\sum_{i=1}^{n_{t}}\sum_{m=1}^{M}\alpha_{m} \cdot \nabla l(\theta; s_{t}^{(i)})\Big\|^{2} \leq G^{2} + \beta^{2} \Big\|\frac{1}{n}\sum_{t=1}^{T} \sum_{i=1}^{n_{t}}\sum_{m=1}^{M}\alpha_{m} \cdot \nabla l(\theta; s_{t}^{(i)}) \Big\|^{2}.
    \]
\end{assumption}
Assumption \ref{assum:bounded_dissimilarity} limits the level of dissimilarity of the different tasks, similarly to what is done in \cite{wang2020tackling}.
\begin{thm}
    \label{thm:centralized_convergence}
    Under Assumptions~\ref{assum:mixture}--\ref{assum:bounded_dissimilarity}, when clients use SGD as local solver with learning rate $\eta =\frac{a_{0}}{\sqrt{K}}$, after a large enough number of communication rounds $K$, \FedEM's iterates  satisfy:
    \begin{equation}
        \label{eq:theta_pi_convergence}
        \frac{1}{K} \sum_{k=1}^K \mathbb{E}\left\|\nabla _{\Theta} f\left(\Theta^{k}, \Pi^{k}\right)\right\|^{2}_{F} \leq \mathcal{O}\!\left(\frac{1}{\sqrt{K}}\right),    \qquad
    \frac{1}{K} \sum_{k=1}^{K} \Delta_{\Pi} f(\Theta^k,\Pi^k) \leq \mathcal{O}\!\left(\frac{1}{K^{3/4}}\right),    
    \end{equation}
where the expectation is over the random batches samples, and $\Delta_{\Pi} f(\Theta^k,\Pi^k) \triangleq f\left(\Theta^{k}, \Pi^{k}\right) - f\left(\Theta^{k}, \Pi^{k+1}\right) \ge 0 $.
\end{thm}

Theorem~\ref{thm:centralized_convergence} (proof in App.~\ref{proof:centralized}) expresses the convergence of both sets of parameters ($\Theta$ and $\Pi$) to a stationary point of $f$.
Indeed, the gradient of $f$ with respect to $\Theta$ becomes arbitrarily small (left inequality in~\eqref{eq:theta_pi_convergence}) and the update in Eq.~\eqref{eq:update_pi} leads to arbitrarily small improvements of $f$ (right inequality in~\eqref{eq:theta_pi_convergence}). 

We conclude this section observing that \FedEM{} allows an \emph{unseen client}, i.e., a client $t_{\text{new}}\notin [T]$ arriving after the distributed training procedure, to learn its personalized model.
The client simply retrieves the learned components' parameters $\Theta^K$ and computes its personalized weights $\pi_{t_{\text{new}}}$ (starting for example from a uniform initialization) through one E-step~\eqref{eq:e_step} and the first update in the M-step~\eqref{eq:update_pi}.

\subsection{Fully Decentralized Algorithm}
\label{sec:decentralized}
In some cases, clients may want to communicate directly in a peer-to-peer fashion instead of relying on the central server mediation \cite[see][Section 2.1]{kairouz2019advances}. 
In fact, fully decentralized schemes  may provide stronger privacy guarantees~\cite{cyffers2021privacy} and speed-up training as they better use communication resources~\cite{lian17,marfoq20neurips} and  reduce the effect of stragglers~\cite{neglia19infocom}.
For these reasons, they have attracted significant interest recently in the machine learning community \cite{lian17,Vanhaesebrouck2017a,Lian2018,tang18a,Bellet2018a,neglia2020,marfoq20neurips,Koloskova2020AUT}.
We refer to~\cite{nedic2018network} for a comprehensive survey of fully decentralized optimization (also known as consensus-based optimization), and to \cite{Koloskova2020AUT} for a unified theoretical analysis of decentralized SGD. 

We propose \DEM{} (Alg.~\ref{alg:d_em} in App.~\ref{app:alg_decentralized}), a \emph{fully decentralized version} of our federated expectation maximization algorithm.
As in \FedEM, the M-step for $\Theta$ update is replaced by an approximate maximization step consisting of local updates.
The global aggregation step in \FedEM{} (Alg.~\ref{alg:fed_em_short}, line~\ref{line:fedEM_aggregation}) is replaced by a partial aggregation step, where each client computes a weighted average of its current components and those of a subset of clients (its \emph{neighborhood}), which may vary over time. 
The convergence of decentralized optimization schemes requires certain assumptions to guarantee that each client can influence the estimates of other clients over time.
In our paper, we consider the general assumption in \cite[Assumption~4]{Koloskova2020AUT} (restated as Assumption~\ref{assum:expected_consensus_rate} in App.~\ref{app:full_decentralized_assumptions} for completeness).
For instance, this assumption is satisfied if the graph of clients' communications is strongly connected every $\tau$ rounds.

\DEM{} converges to a stationary point of $f$ (formal statement  in App.~\ref{app:full_decentralized_assumptions} and proof in App.~\ref{proof:decentralized}).
\begin{thm}[Informal]
    \label{thm:decentralized_convergence}
    In the same setting of Theorem~\ref{thm:centralized_convergence} and under the additional Assumption~\ref{assum:expected_consensus_rate}, \DEM's individual estimates $(\Theta_{t}^{k})_{1\leq t \leq T}$ converge to a common value $\bar \Theta^{k}$.
    Moreover, $\bar \Theta^{k}$ and $\Pi^k$ converge to a stationary point of $f$.
\end{thm}

\subsection{Federated Surrogate Optimization}
\label{sec:fed_surrogate}
\FedEM{} and \DEM{} can be seen as particular instances of a more general framework---of potential interest for other applications---that we call \emph{federated surrogate optimization}.

The standard majorization-minimization principle \parencite{Lange2012Optimization} iteratively minimizes, at each iteration $k$, a surrogate function $g^k$ majorizing the objective function $f$.
The work \cite{mairal2013optimization} studied this approach when  each $g^k$ is a first-order surrogate of $f$ (the formal definition from \cite{mairal2013optimization} is given in App.~\ref{dfn:first_order_surrogate}).

Our novel federated surrogate optimization framework  considers that the objective function $f$ is a weighted sum $f = \sum_{t=1}^{T} \omega_{t} f_{t}$ of $T$ functions and iteratively minimizes $f$ in a distributed fashion using \emph{partial} first-order surrogates $g^{k}_{t}$ for each function $f_t$.
``Partial'' refers to the fact that $g_{t}^{k}$ is not required to be a first order surrogate wrt the whole set of  parameters, as defined formally below.
\begin{definition}[Partial first-order surrogate]
    \label{def:partial_first_order_surrogate}
    A function $g(\vec{u},\vec{v}) : \mathbb{R}^{d_u} \times \mathcal{V} \to \mathbb R$ is a partial-first-order surrogate of $f(\vec{u},\vec{v})$  wrt $\vec{u}$ near $(\vec{u}_0, \vec{v}_0) \in 
    \mathbb \mathbb{R}^{d_u} \times \mathcal{V}$ when the following conditions are satisfied:
    \begin{enumerate}
        \setlength\itemsep{-0.5em}
        \item$g(\vec{u},\vec{v})\ge f(\vec{u},\vec{v})$ for all $\vec{u} \in \mathbb{R}^{d_u}$ and $\vec v\in  \mathcal{V}$;
        \item $r(\vec{u},\vec{v}) \triangleq g(\vec{u},\vec{v})- f(\vec{u},\vec{v})$ is differentiable and $L$-smooth with respect to $\vec{u}$. Moreover, we have $r(\vec{u}_0,\vec{v}_0) = 0$ and $\nabla_{\vec u} r(\vec{u}_0,\vec{v}_0) = 0$.
        \item $g(\vec u, \vec v_{0}) - g(\vec u, \vec v) = d_{\mathcal{V}}\left(\vec v_{0}, \vec v\right)$ for all $\vec{u} \in \mathbb{R}^{d_u}$ and $\vec v\in \argmin_{\vec{v}' \in \mathcal{V}} g(\vec{u}, \vec{v}') $, where $d_{\mathcal{V}}$ is non-negative and $d_{\mathcal{V}}(\vec{v}, \vec{v}') = 0 \iff \vec{v} = \vec{v}'$.
    \end{enumerate}
\end{definition}
Under the assumption that each client $t$ can compute a partial first-order surrogate of $f_{t}$, we propose algorithms for federated surrogate optimization in both the client-server setting (Alg.~\ref{alg:fed_surrogate_opt}) and the fully decentralized one (Alg.~\ref{alg:decentralized_surrogate_opt}) and prove their convergence under mild conditions (App.~\ref{proof:centralized} and \ref{proof:decentralized}).
\FedEM{} and \DEM{} can be seen as particular instances of these algorithms and Theorem.~\ref{thm:centralized_convergence} and Theorem.~\ref{thm:decentralized_convergence} follow from the more general convergence results for federated surrogate optimization.
We can also use our framework to analyze the convergence of other FL algorithms such as \texttt{pFedMe} \cite{dinh2020personalized}, as we illustrate in App.~\ref{app:pfedme_fso}.

\section{Experiments}
\label{sec:experiments}
\textbf{Datasets and models.} We evaluated our method on five federated benchmark datasets spanning a wide range of machine learning tasks: image classification (CIFAR10 and CIFAR100 \cite{Krizhevsky09learningmultiple}), handwritten character recognition
(EMNIST \cite{cohen2017emnist} and FEMNIST \cite{caldas2018leaf}),\footnote{
    For training, we sub-sampled $10\%$ and $15\%$ from EMNIST and FEMNIST datasets respectively.
} and language modeling (Shakespeare \cite{caldas2018leaf, mcmahan2017communication}). 
Shakespeare dataset (resp.~FEMNIST) was naturally partitioned by assigning all lines from the same characters (resp.~all images from the same writer) to the same client. 
We created federated versions of CIFAR10 and EMNIST by distributing samples with the same label across the clients according to a symmetric Dirichlet distribution with parameter $0.4$, as in \cite{Wang2020Federated}. 
For CIFAR100, we exploited the availability  of ``coarse''  and ``fine'' labels, using a two-stage Pachinko allocation method~\cite{li2006pachinko} to assign $600$ sample to each of the $100$ clients, as in \cite{reddi2021adaptive}.
We also evaluated our method on a synthetic dataset verifying Assumptions~\ref{assum:mixture}--\ref{assum:logloss}.
For all tasks, we randomly split each local dataset into training ($60\%$), validation ($20\%$) and test  ($20\%$) sets.
Table~\ref{tab:datasets_models} summarizes datasets, models, and number of clients (more details can be found in App.~\ref{app:datasets_models}). Code is available at \url{https://github.com/omarfoq/FedEM}.
\begin{table}
    \caption{Datasets and models (details in App.~\ref{app:datasets_models}).}
    \centering
    \scriptsize
    \begin{tabular*}{\textwidth}{ l @{\extracolsep{\fill}} l @{\extracolsep{\fill}} r @{\extracolsep{\fill}} r  @{\extracolsep{\fill}} l}
    \toprule
        \textbf{Dataset} & \textbf{Task}  &  \textbf{Clients} & \textbf{Total samples} & \textbf{Model} \\
         \midrule
        FEMNIST \cite{caldas2018leaf} & Handwritten character recognition & $539$ & $120,772$ & 2-layer CNN + 2-layer FFN \\
         EMNIST \cite{cohen2017emnist} & Handwritten character recognition & $100$ & $81,425$ & 2-layer CNN + 2-layer FFN \\
         CIFAR10 \cite{Krizhevsky09learningmultiple} & Image classification  & $80$ & $60,000$ & MobileNet-v2 \cite{sandler2018mobilenetv2} \\
         CIFAR100 \cite{Krizhevsky09learningmultiple} & Image classification  & $100$ & $60,000$ & MobileNet-v2 \cite{sandler2018mobilenetv2}  \\
         Shakespeare \cite{caldas2018leaf, mcmahan2017communication} & Next-Character Prediction  & $778$  & $4,226,158$ & Stacked-LSTM \cite{HochSchm97LSTM}\\
         Synthetic  & Binary Classification  & $300$ & $1,570,507$ &Linear model \\
    \bottomrule
    \end{tabular*}%
    \label{tab:datasets_models} 
\end{table}

\textbf{Other FL approaches.} We compared our algorithms with global models trained with \FedAvg~\cite{mcmahan2017communication} and $\FedProx$~\cite{Sahu2018OnTC} as well as different personalization approaches: a personalized model trained only on the local dataset, \FedAvg{} with local tuning (\FedAvg+)~\cite{jiang2019improving}, Clustered FL~\cite{sattler2020clustered} and $\pFedMe$~\cite{dinh2020personalized}.
 For each method and each task, the learning rate and the other hyperparameters were tuned via grid search (details in App.~\ref{app:implementation_details}).
\FedAvg+ updated the local model through a single pass on the local dataset. 
Unless otherwise stated, the number of components considered by \FedEM{} was $M=3$, training occurred over $80$ communication rounds for Shakespeare and $200$ rounds for all other datasets. At each round, clients train for one  epoch.  Results for \DEM{} are in App.~\ref{app:d_fedem_results}.
A comparison with \texttt{MOCHA} \cite{smith2017federated}, which can only train linear models, is presented in App.~\ref{app:mocha}.
\begin{table}
    \caption{Test accuracy: average across clients\,/\,bottom decile.}
    \label{tab:results_summary}
    \centering
    \scriptsize
    \begin{tabular*}{\textwidth}{l  @{\extracolsep{\fill}} r  @{\extracolsep{\fill}} r @{\extracolsep{\fill}} r  @{\extracolsep{\fill}} r   @{\extracolsep{\fill}} r   @{\extracolsep{\fill}} r  @{\extracolsep{\fill}} r}
        \toprule
        Dataset & Local & \FedAvg~\cite{mcmahan2017communication}  & \FedProx~\cite{Sahu2018OnTC}  & \FedAvg+~\cite{jiang2019improving} & \texttt{Clustered FL}~\cite{sattler2020clustered} &  \texttt{pFedMe}~\cite{dinh2020personalized} & \FedEM~(Ours)   \\
        \midrule
            FEMNIST & $71.0 \,/\, 57.5$ & $78.6 \,/\, 63.9$ & $78.9 \,/\, 64.0$ & $75.3 \,/\, 53.0$ & $73.5 \,/\, 55.1$ & $74.9 \,/\, 57.6$ & $\mathbf{79.9} \,/\, \mathbf{64.8}$  \\
            EMNIST & $71.9 \,/\, 64.3$ & $82.6 \,/\, 75.0$ & $83.0 \,/\, 75.4$ & $83.1 \,/\, 75.8$ & $82.7 \,/\, 75.0$ &  $83.3 \,/\, 76.4$ & $\mathbf{83.5} \,/\, \mathbf{76.6}$  \\
            CIFAR10  & $70.2 \,/\, 48.7$ & $78.2 \,/\, 72.4$ & $78.0 \,/\, 70.8$ & $82.3 \,/\, 70.6$ & $78.6 \,/\, 71.2$  & $81.7 \,/\, 73.6$ & $\mathbf{84.3} \,/\, \mathbf{78.1}$  \\
            CIFAR100 & $31.5 \,/\, 19.9$ & $40.9 \,/\, 33.2$ & $41.0 \,/\, 33.2$ & $39.0 \,/\, 28.3$ & $41.5 \,/\, 34.1$ & $41.8 \,/\, 32.5$ & $\mathbf{44.1} \,/\, \mathbf{35.0}$  \\
            Shakespeare & $32.0 \,/\, 16.6$ & $\mathbf{46.7} \,/\, 42.8$  & $45.7 \,/\, 41.9$ & $40.0 \,/\, 25.5$ & $46.6 \,/\, 42.7$ & $41.2 \,/\, 36.8$ & $\mathbf{46.7} \,/\, \mathbf{43.0}$  \\
            Synthetic & $65.7 \,/\, 58.4$ & $68.2 \,/\, 58.9$ & $68.2 \,/\, 59.0$ & $68.9 \,/\, 60.2$ & $69.1 \,/\, 59.0$ & $69.2 \,/\, 61.2$ & $\mathbf{74.7} \,/\, \mathbf{66.7}$  \\
        \bottomrule
    \end{tabular*}
\end{table}

\textbf{Average performance of personalized models.} 
The performance of each personalized model (which is the same for all clients in the case of \FedAvg{} and \FedProx) is evaluated on the local test dataset (unseen at training).
Table~\ref{tab:results_summary} shows the average weighted accuracy with weights proportional to local dataset sizes. We observe that \FedEM{} obtains the best performance across all datasets.

\textbf{Fairness across clients.} \FedEM's improvement in terms of  average accuracy could be the result of learning particularly good models for some clients at the expense of bad models for other clients.
Table~\ref{tab:results_summary} shows the bottom decile of the accuracy of local models, i.e., the $(T/10)$-th worst accuracy (the minimum accuracy is particularly noisy, notably because some local test datasets are very small).
Even clients with the worst personalized models  are still better off when \FedEM{} is used for training.

\textbf{Clients sampling.} In cross-device federated learning, only a subset of clients may be available at each round. We ran CIFAR10 experiments with different levels of participation: at each round a given fraction of all clients were sampled uniformly without replacement. 
We restrict the comparison to $\FedEM$ and \FedAvg+, as 1) \FedAvg+ performed better than $\FedProx$ and $\FedAvg$ in the previous CIFAR10 experiments, 2) it is not clear how to extend $\pFedMe$ and Clustered FL to handle client sampling.
Results in Fig.~\ref{fig:sample_rate_and_M} (left) show that \FedEM{} is more robust to low clients' participation levels.
We provide additional results on client sampling, including a comparison with \texttt{APFL}~\cite{deng2020adaptive}, in App.~\ref{app:client_sampling}.

\textbf{Generalization to unseen clients.} As discussed in Section~\ref{sec:centralized}, \FedEM{} allows new clients arriving after the distributed training to easily learn their personalized models. With the exception of \FedAvg+, it is not clear how the other personalized FL algorithms should be extended to tackle the same goal (see discussion in  App.~\ref{app:generalization_to_unseen_clients}). 
In order to evaluate the quality of new clients' personalized models, we performed an experiment where only 80\% of the clients (``old'' clients) participate to the training. 
The remaining 20\% join the system in a second phase and use their local training datasets to learn their personalized weights. 
Table~\ref{tab:results_summary_zero_shot} shows that \FedEM{} allows new clients to learn a personalized model at least as good as \FedAvg's global one and always better than \FedAvg+'s one.
Unexpectedly, new clients achieve sometimes  a significantly higher test accuracy than old clients (e.g., 47.5\% against 44.1\% on CIFAR100). 
Our investigation in App.~\ref{app:generalization_to_unseen_clients} suggests that, by selecting their mixture weights on local datasets that were not used to train the components, new clients can compensate for potential overfitting in the initial training phase.
We also investigate in App.~\ref{app:generalization_to_unseen_clients} the effect of the local dataset size on the accuracy achieved by unseen clients, showing that personalization is effective even when unseen clients have small datasets.
\begin{table}
    \caption{Average test accuracy across clients unseen at training  (train accuracy in parenthesis). 
    }
    \label{tab:results_summary_zero_shot}
    \centering
    \scriptsize
    \begin{tabular}{l   r   r   r }
        \toprule
        Dataset  & \FedAvg~\cite{mcmahan2017communication}  & \FedAvg+~\cite{jiang2019improving} &   \FedEM~(Ours)   \\
        \midrule
            FEMNIST & $78.3$ ($80.9$) & $74.2$ ($84.2$) &  $\mathbf{79.1}$ ($81.5$)   \\
            EMNIST  & $83.4$ ($82.7$) & $83.7$ ($92.9$) & $\mathbf{84.0}$ ($83.3$)   \\
            CIFAR10 & $77.3$ ($77.5$) & $80.4$ ($80.5$) & $\mathbf{85.9}$ ($90.7$)  \\
            CIFAR100 & $41.1$ ($42.1$) & $36.5$ ($55.3$)  & $\mathbf{47.5}$ ($46.6$)   \\
            Shakespeare & $\mathbf{46.7}$ ($47.1$) & $40.2$ ($93.0$)  & $\mathbf{46.7}$ ($46.6$)  \\
            Synthetic & $68.6$ ($70.0$) & $69.1$ ($72.1$) & $\mathbf{73.0}$ ($74.1$)   \\
        \bottomrule
    \end{tabular}
\end{table}
\begin{figure}[t] 
    \centering
    \begin{minipage}[b]{0.4\linewidth}
        \includegraphics[width=.9\linewidth]{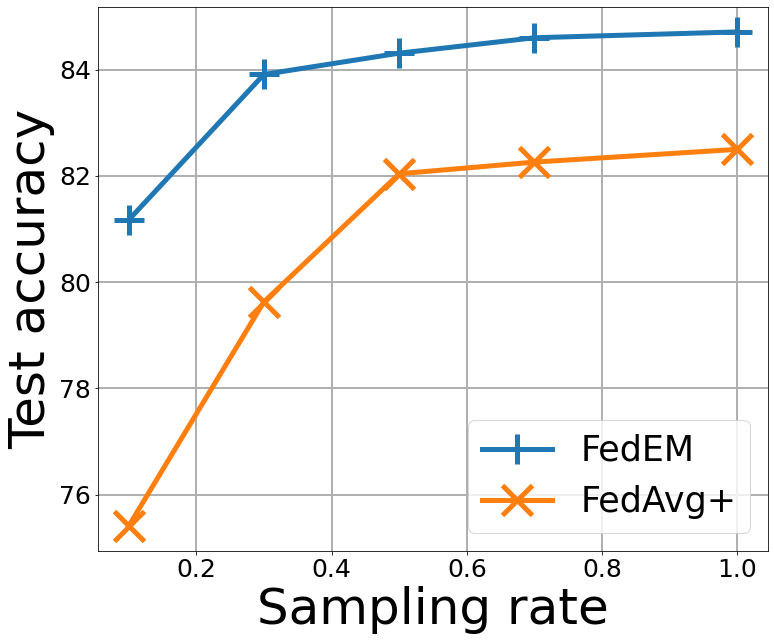}
    \end{minipage}
   \begin{minipage}[b]{0.4\linewidth}
        \centering
        \includegraphics[width=.9\linewidth]{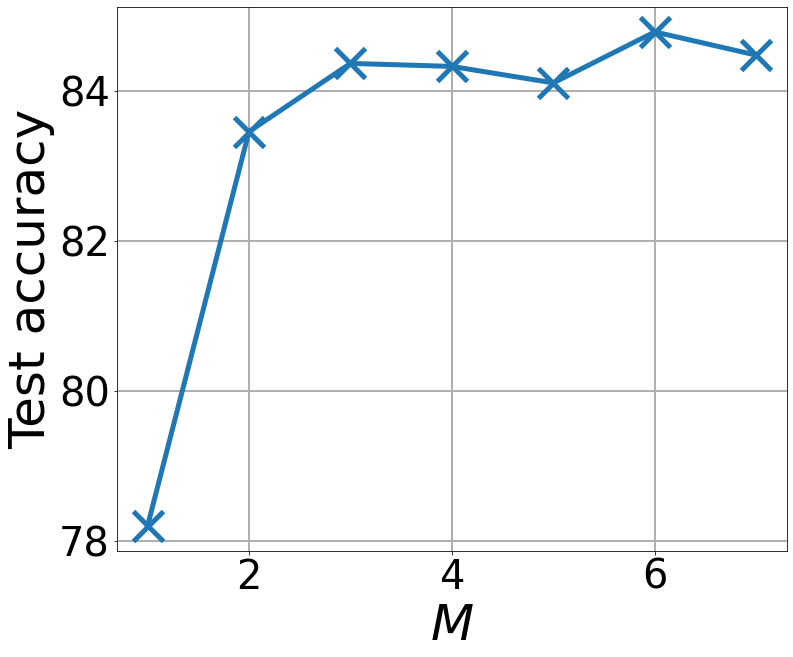}
  \end{minipage}
  \caption{Effect of client sampling rate (left) and \FedEM{}  number of mixture components $M$ (right) on the test accuracy for CIFAR10~\cite{Krizhevsky09learningmultiple}. 
  }
  \label{fig:sample_rate_and_M}
\end{figure}

\textbf{Effect of $M$.} A limitation of \FedEM{} is that each client needs to update and transmit $M$ components at each round, requiring roughly $M$ times more computation and $M$ times larger messages. 
Nevertheless, the number of components to consider in practice is quite limited.
We used $M=3$ in our previous experiments, and Fig.~\ref{fig:sample_rate_and_M} (right) shows that larger values do not yield much improvement and $M=2$ already provides a significant level of personalization.
In all experiments above, the number of communication rounds allowed all approaches to converge. As a consequence, even if other methods trained over $M=3$ times more rounds---in order to have as much computation and communication as \FedEM---the conclusions would not change.
As a final experiment, we considered a time-constrained setting, where \FedEM{} is limited to run one third ($=1/M$) of the rounds (Table~\ref{tab:results_summary_less_rounds} in App.~\ref{app:effect_of_M}). 
Even if \FedEM{} does not reach its maximum accuracy, it still outperforms the other methods on 3 datasets.

\section{Conclusion}
\label{sec:conclusion}
In this paper, we proposed a novel federated MTL approach based on the flexible assumption that local data distributions are mixtures of underlying distributions.
Our EM-like algorithms allow clients to jointly learn shared component models and personalized mixture weights in client-server and fully decentralized settings.
We proved convergence guarantees for our algorithms through a general federated surrogate optimization framework which can be used to analyze other FL formulations. 
Extensive empirical evaluation shows that our approach learns models with higher accuracy and fairness than state-of-the-art FL algorithms, even for clients not present at training time.

In future work, we aim to reduce the local computation and communication of our algorithms.
Aside from standard compression schemes \cite{pmlr-v130-haddadpour21a}, a promising direction is to limit the number of component models that a client updates/transmits at each step. 
This could be done in an adaptive manner based on the client's current mixture weights. 
A simultaneously published work~\cite{dieuleveut2021federated} proposes a federated EM algorithm (also called \FedEM{}), which does not address personalization but reduces communication requirements by compressing appropriately defined complete data sufficient statistics.

A second interesting research direction is to study personalized FL approaches under privacy constraints (quite unexplored until now with the notable exception of \cite{Bellet2018a}).
Some features of our algorithms may be beneficial for privacy (e.g., the fact that personalized weights are kept locally and that all users contribute to all shared models).
We hope to design differentially private versions of our algorithms and characterize their privacy-utility trade-offs.

\section{Acknowledgements}
\label{sec:acknowledgement}
This work has been supported by the French government, through the 3IA Côte d’Azur Investments in the Future project managed by the National Research Agency (ANR) with the reference number ANR-19-P3IA-0002, and through grants ANR-16-CE23-0016 (Project PAMELA) and ANR-20-CE23-0015 (Project PRIDE).
{The authors are grateful to the OPAL infrastructure from Universit\'e C\^ote d'Azur for providing computational resources and technical support.}

\printbibliography

@inproceedings{tang18a,
  title =    {{$D^2$: Decentralized Training over Decentralized Data}},
  author =   {Tang, Hanlin and Lian, Xiangru and Yan, Ming and Zhang, Ce and Liu, Ji},
  booktitle =    {ICML},
  year =   {2018}
}

@inproceedings{neglia2020,
  title={Decentralized gradient methods: does topology matter?},
  author={Giovanni Neglia and Chuan Xu and Don Towsley and Gianmarco Calbi},
  booktitle={AISTATS},
  year={2020}
}

@inproceedings{zhang2020personalized,
  title={Personalized Federated Learning with First Order Model Optimization},
  author={Zhang, Michael and Sapra, Karan and Fidler, Sanja and Yeung, Serena and Alvarez, Jose M},
  booktitle={International Conference on Learning Representations},
  year={2020}
}

@inproceedings{li2021ditto,
  title={Ditto: Fair and robust federated learning through personalization},
  author={Li, Tian and Hu, Shengyuan and Beirami, Ahmad and Smith, Virginia},
  booktitle={International Conference on Machine Learning},
  pages={6357--6368},
  year={2021},
  organization={PMLR}
}

@article{popoviciu1965certaines,
  title={Sur certaines in{\'e}galit{\'e}s qui caract{\'e}risent les fonctions convexes},
  author={Popoviciu, Tiberiu},
  journal={Analele Stiintifice Univ.“Al. I. Cuza”, Iasi, Sectia Mat},
  volume={11},
  pages={155--164},
  year={1965}
}

@InProceedings{acar2021debiasing,
  title = 	 {Debiasing Model Updates for Improving Personalized Federated Training},
  author =       {Acar, Durmus Alp Emre and Zhao, Yue and Zhu, Ruizhao and Matas, Ramon and Mattina, Matthew and Whatmough, Paul and Saligrama, Venkatesh},
  booktitle = 	 {Proceedings of the 38th International Conference on Machine Learning},
  pages = 	 {21--31},
  year = 	 {2021},
  editor = 	 {Meila, Marina and Zhang, Tong},
  volume = 	 {139},
  series = 	 {Proceedings of Machine Learning Research},
  month = 	 {7},
  publisher =    {PMLR},
  url = 	 {https://proceedings.mlr.press/v139/acar21a.html}
}

@inproceedings{marfoq20neurips,
	author = {Marfoq, Othmane and Xu, Chuan and Neglia, Giovanni and Vidal, Richard},
	booktitle = {Advances in Neural Information Processing Systems},
	editor = {H. Larochelle and M. Ranzato and R. Hadsell and M. F. Balcan and H. Lin},
	pages = {19478--19487},
	publisher = {Curran Associates, Inc.},
	title = {Throughput-Optimal Topology Design for Cross-Silo Federated Learning},
	url = {https://proceedings.neurips.cc/paper/2020/file/e29b722e35040b88678e25a1ec032a21-Paper.pdf},
	volume = {33},
	year = {2020}}

@inproceedings{dieuleveut2021federated,
  title={Federated Expectation Maximization with heterogeneity mitigation and variance reduction},
  author={Dieuleveut, Aymeric and Fort, Gersende and Moulines, Eric and Robin, Genevi{\`e}ve},
  booktitle = {Advances in Neural Information Processing Systems},
  year={2021},
  volume = {34}
}

@InProceedings{shamsian2021personalized,
  title = 	 {Personalized Federated Learning using Hypernetworks},
  author =       {Shamsian, Aviv and Navon, Aviv and Fetaya, Ethan and Chechik, Gal},
  booktitle = 	 {Proceedings of the 38th International Conference on Machine Learning},
  pages = 	 {9489--9502},
  year = 	 {2021},
  editor = 	 {Meila, Marina and Zhang, Tong},
  volume = 	 {139},
  series = 	 {Proceedings of Machine Learning Research},
  month = 	 {7},
  publisher =    {PMLR},
  url = 	 {https://proceedings.mlr.press/v139/shamsian21a.html}
}

@inproceedings{li2020fedbn,
  title={FedBN: Federated Learning on Non-IID Features via Local Batch Normalization},
  author={Li, Xiaoxiao and JIANG, Meirui and Zhang, Xiaofei and Kamp, Michael and Dou, Qi},
  booktitle={International Conference on Learning Representations},
  year={2020}
}

@inproceedings{huang2021personalized,
  title={Personalized cross-silo federated learning on non-iid data},
  author={Huang, Yutao and Chu, Lingyang and Zhou, Zirui and Wang, Lanjun and Liu, Jiangchuan and Pei, Jian and Zhang, Yong},
  booktitle={Proceedings of the AAAI Conference on Artificial Intelligence},
  volume={35},
  number={9},
  pages={7865--7873},
  year={2021}
}

@inproceedings{Lian2018,
  Author = {Xiangru Lian and Wei Zhang and Ce Zhang and Ji Liu},
  Booktitle = {ICML},
  Title = {{Asynchronous Decentralized Parallel Stochastic Gradient Descent}},
  Year = {2018}}

@InProceedings{pmlr-v130-haddadpour21a,
  title = 	 {Federated Learning with Compression: Unified Analysis and Sharp Guarantees},
  author =       {Haddadpour, Farzin and Kamani, Mohammad Mahdi and Mokhtari, Aryan and Mahdavi, Mehrdad},
  booktitle = 	 {ICML},
  year = 	 {2021}
}

@INPROCEEDINGS{Bellet2018a,
  author = {Aur\'{e}lien Bellet and Rachid Guerraoui and Mahsa Taziki and Marc Tommasi},
  title = {{P}ersonalized and {P}rivate {P}eer-to-{P}eer {M}achine {L}earning},
  booktitle = {{AISTATS}},
  year = {2018}
}

@INPROCEEDINGS{Vanhaesebrouck2017a,
  author = {Paul Vanhaesebrouck and Aur\'{e}lien Bellet and Marc Tommasi},
  title = {{D}ecentralized {C}ollaborative {L}earning of {P}ersonalized {M}odels over {N}etworks},
  booktitle = {{AISTATS}},
  year = {2017}
}

@INPROCEEDINGS{neglia19infocom,
  author={Neglia, Giovanni and Calbi, Gianmarco and Towsley, Don and Vardoyan, Gayane},
  booktitle={IEEE INFOCOM 2019 - IEEE Conference on Computer Communications}, 
  title={The Role of Network Topology for Distributed Machine Learning}, 
  year={2019},
  volume={},
  number={},
  pages={2350-2358},
  doi={10.1109/INFOCOM.2019.8737602}}

@misc{cyffers2021privacy,
      title={Privacy Amplification by Decentralization}, 
      author={Edwige Cyffers and Aur\'{e}lien Bellet},
      year={2021},
      eprint={2012.05326},
      archivePrefix={arXiv},
      primaryClass={cs.LG},
      note = {Presented at the Privacy Preserving Machine Learning workshop (in conjunction with NeurIPS 2020)}
}

@inproceedings{lian17,
author = {Lian, Xiangru and Zhang, Ce and Zhang, Huan and Hsieh, Cho-Jui and Zhang, Wei and Liu, Ji},
title = {Can Decentralized Algorithms Outperform Centralized Algorithms? A Case Study for Decentralized Parallel Stochastic Gradient Descent},
year = {2017},
isbn = {9781510860964},
publisher = {Curran Associates Inc.},
address = {Red Hook, NY, USA},
booktitle = {Proceedings of the 31st International Conference on Neural Information Processing Systems},
pages = {5336–5346},
numpages = {11},
location = {Long Beach, California, USA},
series = {NIPS'17}
}

@inproceedings{cortes_sample_bias2,
 author = {Cortes, Corinna and Mansour, Yishay and Mohri, Mehryar},
 booktitle = {Advances in Neural Information Processing Systems},
 editor = {J. Lafferty and C. Williams and J. Shawe-Taylor and R. Zemel and A. Culotta},
 pages = {},
 publisher = {Curran Associates, Inc.},
 title = {Learning Bounds for Importance Weighting},
 url = {https://proceedings.neurips.cc/paper/2010/file/59c33016884a62116be975a9bb8257e3-Paper.pdf},
 volume = {23},
 year = {2010}
}

@inproceedings{cortes_sample_bias1,
author = {Cortes, Corinna and Mohri, Mehryar and Riley, Michael and Rostamizadeh, Afshin},
title = {Sample Selection Bias Correction Theory},
year = {2008},
booktitle = {ALT}
}

@inproceedings{Sugiyama_sample_bias,
 author = {Sugiyama, Masashi and Nakajima, Shinichi and Kashima, Hisashi and Buenau, Paul and Kawanabe, Motoaki},
 booktitle = {NIPS},
 title = {Direct Importance Estimation with Model Selection and Its Application to Covariate Shift Adaptation},
 year = {2008}
}

@inproceedings{weightedERM,
 author = {Robin Vogel and Mastane Achab and St\'{e}phan Cl\'{e}men\c{c}on and Charles Tillier},
 booktitle = {ESANN},
 title = {Weighted Emprirical Risk Minimization: Transfer Learning based on Importance Sampling},
 year = {2020}
}

@inproceedings{ghosh2020efficient,
author = {Avishek Ghosh and Jichan Chung and Dong Yin and Kannan Ramchandran},
title = {An Efficient Framework for Clustered Federated Learning},
year = {2020},
booktitle = {NeurIPS}
}

@article{JMLR:v6:ando05a,
  author  = {Rie Kubota Ando and Tong Zhang},
  title   = {A Framework for Learning Predictive Structures from Multiple Tasks and Unlabeled Data},
  journal = {Journal of Machine Learning Research},
  year    = {2005},
  volume  = {6},
  number  = {61},
  pages   = {1817-1853}
}

@inproceedings{smith2017federated,
    author = {Smith, Virginia and Chiang, Chao-Kai and Sanjabi, Maziar and Talwalkar, Ameet},
    title = {Federated Multi-Task Learning},
    year = {2017},
    isbn = {9781510860964},
    publisher = {Curran Associates Inc.},
    address = {Red Hook, NY, USA},
    booktitle = {Proceedings of the 31st International Conference on Neural Information Processing Systems},
    pages = {4427–4437},
    numpages = {11},
    location = {Long Beach, California, USA},
    series = {NIPS'17}
}

@inproceedings{zhang2010convex,
  title={A Convex Formulation for Learning Task Relationships in Multi-task Learning},
  author={Zhang, Yu and Yeung, Dit Yan},
  booktitle={Proceedings of the 26th Conference on Uncertainty in Artificial Intelligence, UAI 2010},
  pages={733},
  year={2010}
}

@InProceedings{pmlr-v108-zantedeschi20a,
  title = {Fully Decentralized Joint Learning of Personalized Models and Collaboration Graphs},
  author = {Zantedeschi, Valentina and Bellet, Aur\'{e}lien and Tommasi, Marc},
  pages = {864--874},
  year = {2020}, 
  editor = {Silvia Chiappa and Roberto Calandra},
  volume = {108},
  series = {Proceedings of Machine Learning Research},
  address = {Online},
  month = {8}, 
  publisher = {PMLR},
  url = {http://proceedings.mlr.press/v108/zantedeschi20a.html}
}

@article{hanzely2020federated,
    title={Federated Learning of a Mixture of Global and Local Models},
    author={Filip Hanzely and Peter Richt\'{a}rik},
    year={2020},
    eprint={2002.05516},
    archivePrefix={arXiv},
    primaryClass={cs.LG}
}

@article{sattler2020clustered,
  title={Clustered Federated Learning: Model-Agnostic Distributed Multitask Optimization Under Privacy Constraints},
  author={Sattler, Felix and M{\"u}ller, Klaus-Robert and Samek, Wojciech},
  journal={IEEE Transactions on Neural Networks and Learning Systems},
  year={2020},
  publisher={IEEE}
}

@inproceedings{mcmahan2017communication,
  title={Communication-efficient learning of deep networks from decentralized data},
  author={McMahan, Brendan and Moore, Eider and Ramage, Daniel and Hampson, Seth and y Arcas, Blaise Aguera},
  booktitle={Artificial Intelligence and Statistics},
  pages={1273--1282},
  year={2017},
  organization={PMLR}
}

@article{konevcny2016federated,
  title={Federated learning: Strategies for improving communication efficiency},
  author={Kone{\v{c}}n{\'y}, Jakub and McMahan, H Brendan and Yu, Felix X and Richt{\'a}rik, Peter and Suresh, Ananda Theertha and Bacon, Dave},
  journal={arXiv preprint arXiv:1610.05492},
  year={2016},
  note = {Presented at NIPS 2016 Workshop on Private Multi-Party Machine Learning}
}

@article{li2020federated,
  title={Federated learning: Challenges, methods, and future directions},
  author={Li, Tian and Sahu, Anit Kumar and Talwalkar, Ameet and Smith, Virginia},
  journal={IEEE Signal Processing Magazine},
  volume={37},
  number={3},
  pages={50--60},
  year={2020},
  publisher={IEEE}
}

@article{kairouz2019advances,
url = {http://dx.doi.org/10.1561/2200000083},
year = {2021},
volume = {14},
journal = {Foundations and Trends® in Machine Learning},
title = {Advances and Open Problems in Federated Learning},
doi = {10.1561/2200000083},
issn = {1935-8237},
number = {1–2},
pages = {1-210},
author = {Peter Kairouz and H. Brendan McMahan and Brendan Avent and Aurélien Bellet and Mehdi Bennis and Arjun Nitin Bhagoji and Kallista Bonawitz and Zachary Charles and Graham Cormode and Rachel Cummings and Rafael G. L. D’Oliveira and Hubert Eichner and Salim El Rouayheb and David Evans and Josh Gardner and Zachary Garrett and Adrià Gascón and Badih Ghazi and Phillip B. Gibbons and Marco Gruteser and Zaid Harchaoui and Chaoyang He and Lie He and Zhouyuan Huo and Ben Hutchinson and Justin Hsu and Martin Jaggi and Tara Javidi and Gauri Joshi and Mikhail Khodak and Jakub Konecný and Aleksandra Korolova and Farinaz Koushanfar and Sanmi Koyejo and Tancrède Lepoint and Yang Liu and Prateek Mittal and Mehryar Mohri and Richard Nock and Ayfer Özgür and Rasmus Pagh and Hang Qi and Daniel Ramage and Ramesh Raskar and Mariana Raykova and Dawn Song and Weikang Song and Sebastian U. Stich and Ziteng Sun and Ananda Theertha Suresh and Florian Tramèr and Praneeth Vepakomma and Jianyu Wang and Li Xiong and Zheng Xu and Qiang Yang and Felix X. Yu and Han Yu and Sen Zhao}
}

@inproceedings{stich2018local,
  title={Local SGD Converges Fast and Communicates Little},
  author={Stich, Sebastian U},
  booktitle={International Conference on Learning Representations},
  year={2018}
}

@inproceedings{sim2019investigation,
  title={An Investigation Into On-device Personalization of End-to-end Automatic Speech Recognition Models},
  author={Sim, Khe Chai and Zadrazil, Petr and Beaufays, Fran{\c{c}}oise},
  booktitle={INTERSPEECH},
  year={2019}
}

@article{jiang2019improving,
  title={Improving federated learning personalization via model agnostic meta learning},
  author={Jiang, Yihan and Kone{\v{c}}n{\'y}, Jakub and Rush, Keith and Kannan, Sreeram},
  journal={arXiv preprint arXiv:1909.12488},
  year={2019},
  note = {Presented at NeurIPS FL workshop 2019.}
}

@inproceedings{fallah2020personalized,
  title={Personalized federated learning: A meta-learning approach},
  author={Fallah, Alireza and Mokhtari, Aryan and Ozdaglar, Asuman},
  booktitle={34th Conference on Neural Information Processing Systems (NeurIPS 2020)},
  year={2020}
}

@inproceedings{khodak2019adaptive,
  title={Adaptive gradient-based meta-learning methods},
  author={Khodak, Mikhail and Balcan, Maria-Florina F and Talwalkar, Ameet S},
  booktitle={Advances in Neural Information Processing Systems},
  pages={5917--5928},
  year={2019}
}

@inproceedings{hanzely2020lower,
  title={Lower bounds and optimal algorithms for personalized federated learning},
  author={Hanzely, Filip and Hanzely, Slavom{\'\i}r and Horv{\'a}th, Samuel and Richt{\'a}rik, Peter},
  booktitle={34th Conference on Neural Information Processing Systems (NeurIPS 2020)},
  year={2020}
}

@inproceedings{karimireddy2020scaffold,
  title={SCAFFOLD: Stochastic controlled averaging for federated learning},
  author={Karimireddy, Sai Praneeth and Kale, Satyen and Mohri, Mehryar and Reddi, Sashank and Stich, Sebastian and Suresh, Ananda Theertha},
  booktitle={International Conference on Machine Learning},
  pages={5132--5143},
  year={2020},
  organization={PMLR}
}

@inproceedings{dinh2020personalized,
  title={Personalized Federated Learning with Moreau Envelopes},
  author={Dinh, Canh T and Tran, Nguyen H and Nguyen, Tuan Dung},
  booktitle={34th Conference on Neural Information Processing Systems (NeurIPS 2020)},
  year={2020}
}

@article{mansour2020three,
  title={Three approaches for personalization with applications to federated learning},
  author={Mansour, Yishay and Mohri, Mehryar and Ro, Jae and Suresh, Ananda Theertha},
  journal={arXiv preprint arXiv:2002.10619},
  year={2020}
}

@article{deng2020adaptive,
  title={Adaptive Personalized Federated Learning},
  author={Deng, Yuyang and Kamani, Mohammad Mahdi and Mahdavi, Mehrdad},
  journal={arXiv preprint arXiv:2003.13461},
  year={2020}
}

@misc{corinzia2019variational,
      title={Variational Federated Multi-Task Learning}, 
      author={Luca Corinzia and Joachim M. Buhmann},
      year={2019},
      eprint={1906.06268},
      archivePrefix={arXiv},
      primaryClass={cs.LG}
}

@ARTICLE{Boyd03fastestmixing,
    author = {Stephen Boyd and Persi Diaconis and Lin Xiao},
    title = {Fastest Mixing Markov Chain on A Graph},
    journal = {SIAM REVIEW},
    year = {2003},
    volume = {46},
    pages = {667--689}
}

@book{nestrov2003introduction,
   title =     {Introductory Lectures on Convex Optimization: A Basic Course},
   author =    {Y. Nesterov},
   publisher = {Springer},
   year =      {2003},
   series =    {Applied Optimization},
   edition =   {1},
   volume =    {},
   url =       {http://gen.lib.rus.ec/book/index.php?md5=488d3c36f629a6e021fc011675df02ef}
}

@misc{bubeck2015convex,
      title={Convex Optimization: Algorithms and Complexity}, 
      author={S\'{e}bastien Bubeck},
      year={2015},
      eprint={1405.4980},
      archivePrefix={arXiv},
      primaryClass={math.OC}
}

@article{erdos59a,
  author = {Erd\"{o}s, P. and R\'{e}nyi, A.},
  journal = {Publicationes Mathematicae Debrecen},
  pages = 290,
  title = {On Random Graphs I},
  volume = 6,
  year = 1959
}

@inproceedings{wang2020tackling,
  title={Tackling the objective inconsistency problem in heterogeneous federated optimization},
  author={Wang, Jianyu and Liu, Qinghua and Liang, Hao and Joshi, Gauri and Poor, H Vincent},
  booktitle={34th Conference on Neural Information Processing Systems (NeurIPS 2020)},
  year={2020}
}

@inproceedings{mohri2019agnostic,
  title={Agnostic Federated Learning},
  author={Mohri, Mehryar and Sivek, Gary and Suresh, Ananda Theertha},
  booktitle={International Conference on Machine Learning},
  pages={4615--4625},
  year={2019}
}

@inproceedings{gopfert2019can,
  title={When can unlabeled data improve the learning rate?},
  author={G{\"o}pfert, Christina and Ben-David, Shai and Bousquet, Olivier and Gelly, Sylvain and Tolstikhin, Ilya and Urner, Ruth},
  booktitle={Conference on Learning Theory},
  pages={1500--1518},
  year={2019},
  organization={PMLR}
}

@inproceedings{marcel2010torchvision,
author = {Marcel, S\'{e}bastien and Rodriguez, Yann},
title = {Torchvision the Machine-Vision Package of Torch},
year = {2010},
isbn = {9781605589336},
publisher = {Association for Computing Machinery},
address = {New York, NY, USA},
url = {https://doi.org/10.1145/1873951.1874254},
doi = {10.1145/1873951.1874254},
booktitle = {Proceedings of the 18th ACM International Conference on Multimedia},
pages = {1485–1488},
numpages = {4},
location = {Firenze, Italy},
series = {MM '10}
}

@incollection{paszke2019pytorch,
  title = {PyTorch: An Imperative Style, High-Performance Deep Learning Library},
  author = {Paszke, Adam and Gross, Sam and Massa, Francisco and Lerer, Adam and Bradbury, James and Chanan, Gregory and Killeen, Trevor and Lin, Zeming and Gimelshein, Natalia and Antiga, Luca and Desmaison, Alban and Kopf, Andreas and Yang, Edward and DeVito, Zachary and Raison, Martin and Tejani, Alykhan and Chilamkurthy, Sasank and Steiner, Benoit and Fang, Lu and Bai, Junjie and Chintala, Soumith},
  booktitle = {Advances in Neural Information Processing Systems 32},
  editor = {H. Wallach and H. Larochelle and A. Beygelzimer and F. d\textquotesingle Alch\'{e}-Buc and E. Fox and R. Garnett},
  pages = {8024--8035},
  year = {2019},
  publisher = {Curran Associates, Inc.},
  url = {http://papers.neurips.cc/paper/9015-pytorch-an-imperative-style-high-performance-deep-learning-library.pdf}
}

@inproceedings{BenDavid2008DoesUD,
  title={Does Unlabeled Data Provably Help? Worst-case Analysis of the Sample Complexity of Semi-Supervised Learning},
  author={Shai Ben-David and Tyler Lu and D. P{\'a}l},
  booktitle={COLT},
  year={2008}
}

@inproceedings{Darnstdt2013UnlabeledDD,
  title={Unlabeled Data Does Provably Help},
  author={Malte Darnst{\"a}dt and H. U. Simon and Bal{\'a}zs Sz{\"o}r{\'e}nyi},
  booktitle={STACS},
  year={2013}
}

@inproceedings{Zhou2011ClusteredMTLviaASO,
  author = {Zhou, Jiayu and Chen, Jianhui and Ye, Jieping},
  booktitle = {Advances in Neural Information Processing Systems},
  editor = {J. Shawe-Taylor and R. Zemel and P. Bartlett and F. Pereira and K. Q. Weinberger},
  pages = {},
  publisher = {Curran Associates, Inc.},
  title = {Clustered Multi-Task Learning Via Alternating Structure Optimization},
  url = {https://proceedings.neurips.cc/paper/2011/file/a516a87cfcaef229b342c437fe2b95f7-Paper.pdf},
  volume = {24},
  year = {2011}
}

@mastersthesis{Krizhevsky09learningmultiple,
    author = {Alex Krizhevsky},
    title = {Learning multiple layers of features from tiny images},
    institution = {},
    type ={MSc thesis},
    year = {2009}
}

@inproceedings{mairal2013optimization,
  title={Optimization with first-order surrogate functions},
  author={Mairal, Julien},
  booktitle={International Conference on Machine Learning},
  pages={783--791},
  year={2013}
}

@article{caldas2018leaf,
  title={Leaf: A benchmark for federated settings},
  author={Caldas, Sebastian and Duddu, Sai Meher Karthik and Wu, Peter and Li, Tian and Kone{\v{c}}n{\'y}, Jakub and McMahan, H Brendan and Smith, Virginia and Talwalkar, Ameet},
  journal={arXiv preprint arXiv:1812.01097},
  year={2018},
  note = {Presented at the 2nd International Workshop on Federated Learning for Data Privacy and Confidentiality (in conjunction with NeurIPS 2019)}
}

@article{nedic2018network,
  author={A. {Nedi\'{c}} and A. {Olshevsky} and M. G. {Rabbat}},
  journal={Proceedings of the IEEE}, 
  title={Network Topology and Communication-Computation Tradeoffs in Decentralized Optimization}, 
  year={2018},
  volume={106},
  number={5},
  pages={953-976},
  doi={10.1109/JPROC.2018.2817461}
 }

@inproceedings{Koloskova2020AUT,
  title={A Unified Theory of Decentralized SGD with Changing Topology and Local Updates},
  author={Anastasia Koloskova and N. Loizou and Sadra Boreiri and M. Jaggi and S. Stich},
  booktitle={ICML},
  year={2020}
}

@Article{HochSchm97LSTM,
  author      = {Sepp Hochreiter and J{\"u}rgen Schmidhuber},
  journal     = {Neural Computation},
  title       = {Long Short-Term Memory},
  year        = {1997},
  number      = {8},
  pages       = {1735--1780},
  volume      = {9},
  optabstract = {Learning to store information over extended time intervals by recurrent backpropagation takes a very long time, mostly because of insufficient, decaying error backflow. We briefly review Hochreiter's (1991) analysis of this problem, then address it by introducing a novel, efficient, gradient based method called long short-term memory (LSTM). Truncating the gradient where this does not do harm, LSTM can learn to bridge minimal time lags in excess of 1000 discrete-time steps by enforcing constant error flow through constant error carousels within special units. Multiplicative gate units learn to open and close access to the constant error flow. LSTM is local in space and time; its computational complexity per time step and weight is O. 1. Our experiments with artificial data involve local, distributed, real-valued, and noisy pattern representations. In comparisons with real-time recurrent learning, back propagation through time, recurrent cascade correlation, Elman nets, and neural sequence chunking, LSTM leads to many more successful runs, and learns much faster. LSTM also solves complex, artificial long-time-lag tasks that have never been solved by previous recurrent network algorithms.},
  optdoi      = {10.1162/neco.1997.9.8.1735},
  opteprint   = {http://dx.doi.org/10.1162/neco.1997.9.8.1735},
  opturl      = {http://dx.doi.org/10.1162/neco.1997.9.8.1735},
}

@inproceedings{sandler2018mobilenetv2,
  title={Mobilenetv2: Inverted residuals and linear bottlenecks},
  author={Sandler, Mark and Howard, Andrew and Zhu, Menglong and Zhmoginov, Andrey and Chen, Liang-Chieh},
  booktitle={Proceedings of the IEEE conference on computer vision and pattern recognition},
  pages={4510--4520},
  year={2018}
}

@inproceedings{cohen2017emnist,
  title={EMNIST: Extending MNIST to handwritten letters},
  author={Cohen, Gregory and Afshar, Saeed and Tapson, Jonathan and Van Schaik, Andre},
  booktitle={2017 International Joint Conference on Neural Networks (IJCNN)},
  pages={2921--2926},
  year={2017},
  organization={IEEE}
}

@inproceedings{Sahu2018OnTC,
      title={Federated Optimization in Heterogeneous Networks}, 
      author={Tian Li and Anit Kumar Sahu and Manzil Zaheer and Maziar Sanjabi and Ameet Talwalkar and Virginia Smith},
      year={2020},
      booktitle ={Third MLSys Conference}
}

@article{Lange2012Optimization,
  ISSN = {10618600},
  URL = {http://www.jstor.org/stable/1390605},
  author = {Kenneth Lange and David R. Hunter and Ilsoon Yang},
  journal = {Journal of Computational and Graphical Statistics},
  number = {1},
  pages = {1--20},
  publisher = {[American Statistical Association, Taylor & Francis, Ltd., Institute of Mathematical Statistics, Interface Foundation of America]},
  title = {Optimization Transfer Using Surrogate Objective Functions},
  volume = {9},
  year = {2000}
}

@inproceedings{Wang2020Federated,
  title={Federated Learning with Matched Averaging},
  author={Hongyi Wang and Mikhail Yurochkin and Yuekai Sun and Dimitris Papailiopoulos and Yasaman Khazaeni},
  booktitle={International Conference on Learning Representations},
  year={2020},
  url={https://openreview.net/forum?id=BkluqlSFDS}
}

@inproceedings{reddi2021adaptive,
  title={Adaptive Federated Optimization},
  author={Sashank J. Reddi and Zachary Charles and Manzil Zaheer and Zachary Garrett and Keith Rush and Jakub Kone{\v{c}}n{\'y} and Sanjiv Kumar and Hugh Brendan McMahan},
  booktitle={International Conference on Learning Representations},
  year={2021},
  url={https://openreview.net/forum?id=LkFG3lB13U5}
}

@inproceedings{li2006pachinko,
  author = {Li, Wei and McCallum, Andrew},
  title = {Pachinko Allocation: DAG-Structured Mixture Models of Topic Correlations},
  year = {2006},
  isbn = {1595933832},
  publisher = {Association for Computing Machinery},
  address = {New York, NY, USA},
  url = {https://doi.org/10.1145/1143844.1143917},
  doi = {10.1145/1143844.1143917},
  booktitle = {Proceedings of the 23rd International Conference on Machine Learning},
  pages = {577–584},
  numpages = {8},
  location = {Pittsburgh, Pennsylvania, USA},
  series = {ICML '06}
}

@article{dinh2021fedu,
  title={FedU: A Unified Framework for Federated Multi-Task Learning with Laplacian Regularization},
  author={Dinh, Canh T and Vu, Tung T and Tran, Nguyen H and Dao, Minh N and Zhang, Hongyu},
  journal={arXiv preprint arXiv:2102.07148},
  year={2021}
}

@book{lauritzen1996graphical,
  title = "Graphical models",
  author = "Lauritzen, {Steffen L.}",
  year = "1996",
  language = "English",
  isbn = "0198522193",
  series = "Oxford Statistical Science Series",
  publisher = "Clarendon Press",
  number = "17",
}

\iftrue
\else
    \section*{Checklist}
    \label{sec:checklist}

\begin{enumerate}

\item For all authors...
\begin{enumerate}
  \item Do the main claims made in the abstract and introduction accurately reflect the paper's contributions and scope?
    \answerYes{}
  \item Did you describe the limitations of your work?
    \answerYes{See paragraph \textbf{Effect of $M$} in Section~\ref{sec:experiments}.}
  \item Did you discuss any potential negative societal impacts of your work?
    \answerYes{Broader impact is discussed in Appendix~\ref{app:borader_imapct}}.
  \item Have you read the ethics review guidelines and ensured that your paper conforms to them?
    \answerYes{}
\end{enumerate}

\item If you are including theoretical results...
\begin{enumerate}
  \item Did you state the full set of assumptions of all theoretical results?
    \answerYes{See Assumptions~\ref{assum:mixture}--\ref{assum:bounded_dissimilarity}.}
	\item Did you include complete proofs of all theoretical results?
    \answerYes{}
\end{enumerate}

\item If you ran experiments...
\begin{enumerate}
  \item Did you include the code, data, and instructions needed to reproduce the main experimental results (either in the supplemental material or as a URL)?
    \answerYes{The code is provided in supplementary materials. \path{README.md} file gives instructions on how to reproduce all the results in the paper.}
  \item Did you specify all the training details (e.g., data splits, hyperparameters, how they were chosen)?
    \answerYes{See paragraphs \textbf{Datasets and models} and \textbf{Other FL Approaches} in Section~\ref{sec:experiments}, and Appendix~\ref{app:experiments_details}.}
	\item Did you report error bars (e.g., with respect to the random seed after running experiments multiple times)?
    \answerNo{}
	\item Did you include the total amount of compute and the type of resources used (e.g., type of GPUs, internal cluster, or cloud provider)?
    \answerYes{See Appendix~\ref{app:implementation_details}}.
\end{enumerate}

\item If you are using existing assets (e.g., code, data, models) or curating/releasing new assets...
\begin{enumerate}
  \item If your work uses existing assets, did you cite the creators?
    \answerYes{We used datasets from LEAF~\cite{caldas2018leaf}, CIFAR10/CIFAR100~\cite{Krizhevsky09learningmultiple}, and the federated split of Shakespeare dataset proposed first in~\cite{mcmahan2017communication}. All datasets are cited in the paper (see Section~\ref{sec:experiments}). Our implementation is using Pytorch~\cite{paszke2019pytorch}, it is cited in Appendix~\ref{app:implementation_details}.}
  \item Did you mention the license of the assets?
    \answerYes{For LEAF license, see files \path{data\femnist\get_file_dirs.py}, \path{data\femnist\get_hashes.py}, \path{data\femnist\group_by_writer.py}, and \path{data\femnist\match_hashes.py}. For the license of the federated split of Shakespeare dataset see file \path{data\shakespeare\preprocess_shakespeare.py}.}
  \item Did you include any new assets either in the supplemental material or as a URL?
    \answerYes{The code is provided in supplementary materials.}
  \item Did you discuss whether and how consent was obtained from people whose data you're using/curating?
    \answerNA{}
  \item Did you discuss whether the data you are using/curating contains personally identifiable information or offensive content?
    \answerNA{}
\end{enumerate}

\item If you used crowdsourcing or conducted research with human subjects...
\begin{enumerate}
  \item Did you include the full text of instructions given to participants and screenshots, if applicable?
    \answerNA{}
  \item Did you describe any potential participant risks, with links to Institutional Review Board (IRB) approvals, if applicable?
    \answerNA{}
  \item Did you include the estimated hourly wage paid to participants and the total amount spent on participant compensation?
    \answerNA{}
\end{enumerate}

\end{enumerate}

\fi

\newpage

\appendix

\addcontentsline{toc}{section}{Appendix} 
\part{Appendix} 
\parttoc 

\newpage

\section{Proof of Proposition \ref{prop:local_model_is_weighted_average_of_components}}
\label{proof:general_mtl_framework}

For $h\in\mathcal{H}$ and $\left(\vec{x}, y\right) \in \mathcal{X} \times \mathcal{Y}$,  let $p_{h}\left(y|\vec{x}\right)$ denote the conditional probability distribution of $y$ given $\vec{x}$ under model $h$, i.e.,
\begin{equation}
    p_{h}\left(y|\vec{x}\right) \triangleq e^{c_{h}\left(\vec{x}\right)} \times \exp\Big\{-l\left(h\left(\vec{x}\right), y\right)\Big\},
\end{equation}
where
\begin{equation}
    c_{h}\left(\vec{x}\right) \triangleq -\log\left[\int_{y\in\mathcal{Y}} \exp\Big\{-l\left(h\left(\vec{x}\right), y\right)\Big\} \dd y \right].
\end{equation}

We also remind that the entropy of a probability distribution $q$ over $\mathcal{Y}$ is given by
\begin{equation}
    H\left(q\right) \triangleq -\int_{y\in\mathcal{Y}} q\left(y\right) \cdot \log q\left(y\right) \dd y,
\end{equation}
and that the Kullback-Leibler divergence between two probability distributions $q_{1}$ and $q_{2}$ over $\mathcal{Y}$ is given by
\begin{equation}
    \mathcal{KL}\left(q_{1}||q_{2}\right) \triangleq \int_{y\in\mathcal{Y}} q_{1}\left(y\right) \cdot \log \frac{q_{1}\left(y\right)}{q_{2}\left(y\right)}\dd y.
\end{equation}

\begin{repprop}{prop:local_model_is_weighted_average_of_components}
    Let $l(\cdot,\cdot)$ be the mean squared error loss, the logistic loss or the cross-entropy loss,  and 
    $\breve\Theta$ and $\breve\Pi$ be a solution of the following optimization problem:
    \begin{equation}
        \tag{\ref{eq:intermediate_problem}}
        \minimize_{\Theta, \Pi} \E_{t\sim D_\mathcal{T} } \E_{(\vec{x},y) \sim \mathcal{D}_{t}} \left[ - \log p_t(\vec x,y | \Theta, \pi_t) \right],
    \end{equation}
    where $D_\mathcal{T}$ is any distribution with support $\mathcal T$. Under Assumptions~\ref{assum:mixture}, \ref{assum:uniform_marginal}, and~\ref{assum:logloss}, the predictors  
    \begin{equation}
        \tag{\ref{eq:linear_combination}}
        h_t^{*} = \sum_{m=1}^{M}\breve \pi_{t m}h_{\breve \theta_{m}}, \quad\forall t \in \mathcal{T}
    \end{equation}
    minimize $\mathbb{E}_{(\vec{x}, y)\sim\mathcal{D}_{t}}\left[l(h_{t}(\vec{x}), y)\right]$ and thus solve Problem~\eqref{eq:main_problem}.
\end{repprop}

\begin{proof}
    We prove the result for each of the three possible cases of the loss function. We verify  that $c_{h}$ does not depend on $h$ in each of the three cases, then we use Lemma~\ref{lem:convex_combination_dsitributions} to conclude.
    
    \paragraph{Mean Squared Error Loss} This is the case of a regression problem where $\mathcal{Y} = \mathbb{R}^{d}$ for some $d>0$. For $\vec{x}, y \in \mathcal{X} \times \mathcal{Y}$ and $h\in\mathcal{H}$, we have
    \begin{equation}
        p_{h}\left(y|\vec{x}\right) = \frac{1}{\sqrt{\left(2\pi\right)^{d}}} \cdot \exp\left\{- \frac{\left\|h\left(\vec{x}\right) - y\right\|^{2}}{2}\right\},
    \end{equation}
    and 
    \begin{equation}
        c_{h}\left(\vec{x}\right) = -\log\left(\sqrt{\left(2\pi\right)^{d}}\right) 
    \end{equation}
    
    \paragraph{Logistic Loss} This is the case of a binary classification problem where $\mathcal{Y} = \left\{0, 1\right\}$. For $\vec{x}, y \in \mathcal{X} \times \mathcal{Y}$ and $h\in\mathcal{H}$, we have
    \begin{equation}
        p_{h}\left(y|\vec{x}\right) = \left(h\left(\vec{x}\right)\right)^{y} \cdot \left(1-h\left(\vec{x}\right)\right)^{1-y},
    \end{equation}
    and 
    \begin{equation}
        c_{h}\left(\vec{x}\right) = 0 
    \end{equation}
    
    \paragraph{Cross-entropy loss} This is the case of a classification problem where $\mathcal{Y} = [L]$ for some $L>1$. For $\vec{x}, y \in \mathcal{X} \times \mathcal{Y}$ and $h\in\mathcal{H}$, we have
    \begin{equation}
        p_{h}\left(y|\vec{x}\right) = \prod_{l=1}^{L} \left(h\left(\vec{x}\right)\right)^{\mathds{1}_{\left\{y=l\right\}}},
    \end{equation}
    and 
    \begin{equation}
        c_{h}\left(\vec{x}\right) = 0
    \end{equation}
    
    \paragraph{Conclusion} For $t\in\mathcal{T}$, consider a predictor $h_{t}^{*}$ minimizing  $\mathbb{E}_{(\vec{x}, y)\sim\mathcal{D}_{t}}\left[l(h_{t}(\vec{x}), y)\right]$. Using Lemma~\ref{lem:convex_combination_dsitributions}, for $\left(\vec{x}, y\right) \in \mathcal{X} \times \mathcal{Y}$, we have 
    \begin{equation}
         p_{h_{t}^{*}}\left(y|\vec{x}\right) = \sum_{m=1}^{M}\breve{\pi}_{tm} \cdot p_{m}\left(y|\vec{x}, \breve{\theta}_{m}\right).
    \end{equation}
    We multiply both sides of this equality by $y$ and we integrate over $y \in \mathcal{Y}$. Note that in all three cases we have 
    \begin{equation}
        \forall \vec{x} \in \mathcal{X},\quad\int_{y\in\mathcal{Y}} y \cdot p_{h}\left(\cdot |\vec{x}\right) \dd y = h(\vec{x}).
    \end{equation}
    It follows that
    \begin{equation}
        h_t^{*} = \sum_{m=1}^{M}\breve \pi_{t m}h_{\breve \theta_{m}}, \quad\forall t \in \mathcal{T}.
    \end{equation}
\end{proof}

\subsection*{Supporting Lemmas}

\begin{lem}
    \label{lem:likelihood_entropy_kl}
    Suppose that Assumptions~\ref{assum:mixture} and~\ref{assum:logloss} hold, and consider $\breve{\Theta}$ and $\breve{\Pi}$ to be a solution of Problem~\eqref{eq:intermediate_problem}. Then
    \begin{equation}
        p_t(\vec x,y | \breve{\Theta}, \breve{\pi}_t) =  p_t(\vec x,y | \Theta^{*}, \pi_t^{*}), ~\forall t\in \mathcal{T}.
    \end{equation}
\end{lem}
\begin{proof}
    For $t\in\mathcal{T}$, 
    \begin{flalign}
        \E_{(\vec{x},y) \sim \mathcal{D}_{t}} &\left[ - \log p_t(\vec x,y | \breve{\Theta}, \breve{\pi}_t) \right] &
        \\
        &= - \int_{\left(\vec{x}, y\right) \in \mathcal{X}\times \mathcal{Y}}   p_t(\vec x,y | \Theta^{*}, \pi_t^{*}) \cdot \log p_t(\vec x,y | \breve{\Theta}, \breve{\pi}_t) \dd\vec{x} \dd y &
        \\
        &= - \int_{\left(\vec{x}, y\right) \in \mathcal{X}\times \mathcal{Y}}   p_t(\vec x,y | \Theta^{*}, \pi_t^{*}) \cdot \log\frac{ p_t(\vec x,y | \breve{\Theta}, \breve{\pi}_t)}{p_t(\vec x,y | \Theta^{*}, \pi_t^{*})} \dd\vec{x} \dd y \nonumber
        \\
        & \quad - \int_{\left(\vec{x}, y\right) \in \mathcal{X}\times \mathcal{Y}}    p_t(\vec x,y | \Theta^{*}, \pi_t^{*}) \cdot \log p_t(\vec x,y | \Theta^{*}, \pi^{*}_t) \dd\vec{x} \dd y
        \\
        & = \mathcal{KL} \left(p_{t}\left(\cdot|\Theta^{*}, \pi_{t}^{*}\right) \| p_{t}\big(\cdot|\breve{\Theta}, \breve{\pi}_{t}\big) \right) + H\left[p_{t}\left(\cdot|\Theta^{*}, \pi_{t}^{*}\right)\right],\label{eq:ll_decomposition_entropy_puls_kl}
    \end{flalign}
    Since the $\mathcal{KL}$ divergence is non-negative, we have
    \begin{equation}
            \E_{(\vec{x},y) \sim \mathcal{D}_{t}} \left[ - \log p_t(\vec x,y | \breve{\Theta}, \breve{\pi}_t)\right] \geq  H\left[p_{t}\left(\cdot|\Theta^{*}, \pi_{t}^{*}\right)\right]
            =  \E_{(\vec{x},y) \sim \mathcal{D}_{t}} \left[ - \log p_t(\vec x,y | \Theta^{*}, \pi_t^{*}) \right].
    \end{equation}
    Taking the expectation over $t\sim \mathcal{D}_{\mathcal{T}}$, we write
    \begin{equation}
        \label{eq:ll_entropy_geq}
        \E_{t\sim \mathcal{D}_{\mathcal{T}}}\E_{(\vec{x},y) \sim \mathcal{D}_{t}} \left[ - \log p_t(\vec x,y | \breve{\Theta}, \breve{\pi}_t)\right] \geq  \E_{t\sim \mathcal{D}_{\mathcal{T}}}\E_{(\vec{x},y) \sim \mathcal{D}_{t}} \left[ - \log p_t(\vec x,y | \Theta^{*}, \pi_t^{*}) \right].
    \end{equation}
    Since $\breve{\Theta}$ and $\breve{\Pi}$ is a solution of Problem~\eqref{eq:intermediate_problem}, we also have 
    \begin{equation}
        \label{eq:ll_entropy_leq}
        \E_{t\sim \mathcal{D}_{\mathcal{T}}}\E_{(\vec{x},y) \sim \mathcal{D}_{t}} \left[ - \log p_t(\vec x,y | \breve{\Theta}, \breve{\pi}_t)\right] \leq  \E_{t\sim \mathcal{D}_{\mathcal{T}}}\E_{(\vec{x},y) \sim \mathcal{D}_{t}} \left[ - \log p_t(\vec x,y | \Theta^{*}, \pi_t^{*}) \right].
    \end{equation}
    Combining \eqref{eq:ll_entropy_geq}, \eqref{eq:ll_entropy_leq}, and \eqref{eq:ll_decomposition_entropy_puls_kl}, we have
    \begin{equation}
        \E_{t\sim \mathcal{D}_{\mathcal{T}}} \mathcal{KL} \left(p_{t}\left(\cdot|\Theta^{*}, \pi_{t}^{*}\right) \| p_{t}\big(\cdot|\breve{\Theta}, \breve{\pi}_{t}\big) \right) = 0.
    \end{equation}
    Since $\mathcal{KL}$ divergence is non-negative, and the support of $\mathcal{D}_{\mathcal{T}}$ is the countable set $\mathcal{T}$, it follows that
    \begin{equation}
        \forall t\in \mathcal{T},\quad  \mathcal{KL} \left(p_{t}\left(\cdot|\Theta^{*}, \pi_{t}^{*}\right) \| p_{t}\big(\cdot|\breve{\Theta}, \breve{\pi}_{t}\big) \right) = 0.
    \end{equation}
   Thus,
    \begin{equation}
        p_t(\vec x,y | \breve{\Theta}, \breve{\pi}_t) =  p_t(\vec x,y | \Theta^{*}, \pi_t^{*}),\quad\forall t\in \mathcal{T}.
    \end{equation}
\end{proof}


\begin{lem}
    \label{lem:minimize_convex_combin_kl}
    Consider $M$ probability distributions on $\mathcal{Y}$, that we denote $q_{m},~m\in[M]$, and $\alpha=\left(\alpha_{1}, \dots, \alpha_{m}\right) \in \Delta^{M}$. For any probability distribution $q$ over $\mathcal{Y}$, we have
    \begin{equation}
        \sum_{m=1}^{M}\alpha_{m} \cdot \mathcal{KL}\left(q_{m}\|\sum_{m'=1}^{M}\alpha_{m'}\cdot q_{m'}\right) \leq \sum_{m=1}^{M}\alpha_{m}\cdot \mathcal{KL}\left(q_{m}\|q\right),
    \end{equation}
    with equality if and only if,
    \begin{equation}
        q = \sum_{m=1}^{M}\alpha_{m}\cdot q_{m}.
    \end{equation}
\end{lem}

\begin{proof}
    \begin{flalign}
        \sum_{m=1}^{M}\alpha_{m} & \cdot \mathcal{KL}\left(q_{m}\|q\right) - \sum_{m=1}^{M}\alpha_{m} \cdot \mathcal{KL}\left(q_{m}\|\sum_{m'=1}^{M}\alpha_{m'}\cdot q_{m'}\right) \nonumber &
        \\
        & = \sum_{m=1}^{M}\alpha_{m} \cdot \left[\mathcal{KL}\left(q_{m}\|q\right) - \mathcal{KL}\left(q_{m}\|\sum_{m'=1}^{M}\alpha_{m'}\cdot q_{m'}\right)\right]
        \\
        &= - \sum_{m=1}^{M}\alpha_{m} \int_{y\in\mathcal{Y}} q_{m}\left(y\right) \cdot \log\left(\frac{q\left(y\right)}{\sum_{m'=1}^{M}\alpha_{m'}\cdot q_{m'}\left(y\right)}\right)
        \\
        &= -  \int_{y\in\mathcal{Y}} \left\{\sum_{m=1}^{M}\alpha_{m} \cdot q_{m}\left(y\right)\right\} \cdot \log\left(\frac{q\left(y\right)}{\sum_{m'=1}^{M}\alpha_{m'}\cdot q_{m'}\left(y\right)}\right) \dd y
        \\
        &= \mathcal{KL}\left( \sum_{m=1}^{M}\alpha_{m} \cdot q_{m} \| q\right) \geq 0.
    \end{flalign}    
    The equality holds, if and only if, 
    \begin{equation}
        q = \sum_{m=1}^{M}\alpha_{m}\cdot q_{m}.
    \end{equation}
\end{proof}

\begin{lem}
    \label{lem:convex_combination_dsitributions}
    Consider $\breve{\Theta}$ and $\breve{\Pi}$ to be a solution of Problem~\eqref{eq:intermediate_problem}.
    Under Assumptions~\ref{assum:mixture}, \ref{assum:uniform_marginal}, and~\ref{assum:logloss}, if $c_{h}$ does not depend on $h \in\mathcal{H}$, then the predictors $h_{t}^{*},~t\in\mathcal{T}$, minimizing  $\mathbb{E}_{(\vec{x}, y)\sim\mathcal{D}_{t}}\left[l(h_{t}(\vec{x}), y)\right]$, verify for  $\left(\vec{x}, y\right) \in \mathcal{X} \times \mathcal{Y}$
    \begin{equation}
        \label{eq:convex_combination_dsitributions}
        p_{h_{t}^{*}}\left(y|\vec{x}\right) = \sum_{m=1}^{M}\breve{\pi}_{tm} \cdot p_{m}\left(y|\vec{x}, \breve{\theta}_{m}\right).
    \end{equation}
\end{lem}

\begin{proof}
        For $t \in \mathcal{T}$ and $h_{t} \in\mathcal{H}$, under Assumptions~\ref{assum:mixture}, \ref{assum:uniform_marginal}, and~\ref{assum:logloss}, we have
    \begin{equation}
        \mathbb{E}_{(\vec{x}, y)\sim\mathcal{D}_{t}}\left[l(h_t(\vec{x}), y)\right] = \int_{\vec{x}, y \in \mathcal{X} \times \mathcal{Y}} l(h_t(\vec{x}), y) \cdot p_{t}\left(\vec{x}, y|\Theta^{*}, \pi_{t}^{*}\right)\dd\vec{x} \dd y. 
    \end{equation}
    Using Lemma~\ref{lem:likelihood_entropy_kl}, it follows that 
    \begin{equation}
        \mathbb{E}_{(\vec{x}, y)\sim\mathcal{D}_{t}}\left[l(h_t(\vec{x}), y)\right] = \int_{\vec{x}, y \in \mathcal{X} \times \mathcal{Y}} l(h_t(\vec{x}), y) \cdot p_{t}\left(\vec{x}, y|\breve{\Theta}, \breve{\pi}_{t}\right)\dd\vec{x} \dd y.
    \end{equation}
    Thus, using Assumptions~\ref{assum:mixture} and~\ref{assum:uniform_marginal} we have,
    \begin{flalign}
        \mathbb{E}&_{(\vec{x}, y)\sim\mathcal{D}_{t}} \left[l(h_t(\vec{x}), y)\right]  &
        \\ 
        &= \int_{\vec{x}, y \in \mathcal{X} \times \mathcal{Y}} l(h_t(\vec{x}), y) \cdot p_{t}\left(\vec{x}, y|\breve{\Theta}, \breve{\pi}_{t}\right)\dd\vec{x} \dd y
        \\
        &= \int_{\vec{x}, y \in \mathcal{X} \times \mathcal{Y}} l(h_t(\vec{x}), y) \cdot \left(\sum_{m=1}^{M}\breve{\pi}_{tm} \cdot p_{m}\left(y|\vec{x}, \breve{\theta}_{m}\right) \right) p\left(\vec{x}\right)\dd\vec{x} \dd y
        \\
        &= \int_{\vec{x}\in\mathcal{X}}\left[\sum_{m=1}^{M}\breve{\pi}_{tm} \int_{y\in\mathcal{Y}}l(h_t(\vec{x}), y) \cdot  p_{m}\left(y|\vec{x}, \breve{\theta}_{m}\right) \dd y \right] p\left(\vec{x}\right)\dd\vec{x} 
        \\
        &= \int_{\vec{x}\in\mathcal{X}}\left[\sum_{m=1}^{M}\breve{\pi}_{tm} \left\{c_{h_{t}}\left(\vec{x}\right) - \int_{y\in\mathcal{Y}} p_{m}\left(y|\vec{x}, \breve{\theta}_{m}\right)   \log p_{h_{t}}\left(y|\vec{x}\right) \dd y\right\}  \right] p\left(\vec{x}\right)\dd \vec{x}
        \\
        &= \int_{\vec{x}\in\mathcal{X}} \left[c_{h_{t}}\left(\vec{x}\right) - \sum_{m=1}^{M}\breve{\pi}_{tm}  \int_{y\in\mathcal{Y}} p_{m}\left(y|\vec{x}, \breve{\theta}_{m}\right)   \log p_{h_{t}}\left(y|\vec{x}\right) \dd y  \right] p\left(\vec{x}\right)\dd \vec{x}
        \\
        &= \int_{\vec{x}\in\mathcal{X}} \left[c_{h_{t}}\left(\vec{x}\right) + \sum_{m=1}^{M}\breve{\pi}_{tm} \cdot  H\left(p_{m}\left(\cdot|\vec{x}, \breve{\theta}_{m}\right)\right)   \right] p\left(\vec{x}\right)\dd \vec{x} \nonumber 
        \\
        & \qquad + \int_{\vec{x}\in\mathcal{X}} \left[ \sum_{m=1}^{M}\breve{\pi}_{tm} \cdot  \mathcal{KL}\left(p_{m}\big(\cdot|\vec{x}, \breve{\theta}_{m}\big)\| p_{h_{t}}\left(\cdot|\vec{x}\right) \right)     \right] p\left(\vec{x}\right)\dd \vec{x}.
        \label{eq:convex_combination_dsitributions_step_1}
    \end{flalign}
    
   Let $h_{t}^\circ$ be a predictor satisfying the following equality:
   \[p_{h_{t}^{\circ}}\left(y|\vec{x}\right) = \sum_{m=1}^{M}\breve{\pi}_{tm} \cdot p_{m}\left(y|\vec{x}, \breve{\theta}_{m}\right).\]
   Using Lemma~\ref{lem:minimize_convex_combin_kl}, we have
   \begin{equation}
        \label{eq:convex_combination_dsitributions_step_2}
        \sum_{m=1}^{M} \breve{\pi}_{tm} \cdot \mathcal{KL}\left(p_{m}\big(\cdot|\vec{x}, \breve{\theta}_{m}\big)\| p_{h_{t}}\left(\cdot|\vec{x}\right) \right) \geq 
        \sum_{m=1}^{M} \breve{\pi}_{tm} \cdot \mathcal{KL}\left(p_{m}\big(\cdot|\vec{x}, \breve{\theta}_{m}\big)\| p_{h_{t}^{\circ}}\left(\cdot|\vec{x}\right) \right) 
   \end{equation}
   with equality if and only if
   \begin{equation}
       p_{h_{t}}\left(\cdot|\vec{x}\right) = p_{h_{t}^{\circ}}\left(\cdot|\vec{x}\right).
   \end{equation}
   Since $c_{h}$ does not depend on $h$, replacing \eqref{eq:convex_combination_dsitributions_step_2} in \eqref{eq:convex_combination_dsitributions_step_1},  it follows that
   \begin{equation}
    \mathbb{E}_{(\vec{x}, y)\sim\mathcal{D}_{t}} \left[l(h_{t}(\vec{x}), y)\right] \geq
       \mathbb{E}_{(\vec{x}, y)\sim\mathcal{D}_{t}} \left[l(h_{t}^{\circ}(\vec{x}), y)\right].
   \end{equation} 
   This inequality holds for any predictor $h_t$ and in particular for $h_{t}^{*} \in \argmin \mathbb{E}_{(\vec{x}, y)\sim\mathcal{D}_{t}}\left[l(h_{t}(\vec{x}), y)\right]$, for which it also holds the opposite inequality, then:
   \begin{equation}
    \mathbb{E}_{(\vec{x}, y)\sim\mathcal{D}_{t}} \left[l(h_{t}^{*}(\vec{x}), y)\right] =
       \mathbb{E}_{(\vec{x}, y)\sim\mathcal{D}_{t}} \left[l(h_{t}^{\circ}(\vec{x}), y)\right],
   \end{equation} 
   and the equality implies that    
   \begin{equation}
          p_{h_{t}^{*}}\left(\cdot|\vec{x}\right) = p_{h_{t}^{\circ}}\left(\cdot|\vec{x}\right) = \sum_{m=1}^{M}\breve{\pi}_{tm} \cdot p_{m}\left(\cdot|\vec{x}, \breve{\theta}_{m}\right).
   \end{equation}
   
\end{proof}

\newpage
\section{Relation with Other Multi-Task Learning Frameworks}
\label{sec:general_mtl_framework}
In this appendix, we give more details about the relation of our formulation with existing frameworks for (federated) MTL sketched in Section~\ref{sec:generalizing}. 
We suppose that Assumptions~\ref{assum:mixture}--\ref{assum:logloss} hold and that each client learns a predictor of the form~\eqref{eq:linear_combination}.
Note that this is more general than  \cite{pmlr-v108-zantedeschi20a}, where each client learns a personal hypothesis as a weighted combination of a set of $M$ base \emph{known} hypothesis, since the base hypothesis and \emph{not only the weights} are learned in our case. 

\paragraph{Alternating Structure Optimization \cite{Zhou2011ClusteredMTLviaASO}.} Alternating structure optimization (ASO) is a popular MTL approach that learns a shared low-dimensional predictive structure on hypothesis spaces from multiple related tasks, i.e., all tasks are assumed to share a common feature space $P \in \mathbb{R}^{d' \times d}$, where $d' \leq \min(T, d)$ is the dimensionality of the shared feature space and $P$ has orthonormal columns ($PP^{\intercal}=I_{d'}$), i.e., $P$ is \emph{semi-orthogonal matrix}.
ASO leads to the following formulation:
 \begin{equation}
        \label{eq:aso_formulation}
        \begin{aligned}
            \minimize_{W, P:PP^{\intercal}=I_{d'} } \quad & \sum_{t=1}^{T}\sum_{i=1}^{n_{t}}l\left(h_{w_{t}}\left(\vec{x}_{t}^{(i)}\right), y_{t}^{(i)}\right) + \alpha \left(\tr\left(W W ^{\intercal}\right) - \tr\left(WP^{\intercal}PW^{\intercal}\right)\right) + \beta \tr\left(W W ^{\intercal}\right),\\
        \end{aligned}
\end{equation}
where $\alpha \geq 0$ is the regularization parameter for task relatedness and $\beta \geq 0$ is an additional L2 regularization parameter. 

When the hypothesis $\left(h_{\theta}\right)_{\theta}$ are assumed to be linear, Eq.~\eqref{eq:linear_combination} can be written as $W = \Pi \Theta$.
Writing the LQ decomposition\footnote{Note that when $\Theta$ is a full rank matrix, this decomposition is unique.} of matrix $\Theta$, i.e., $\Theta = LQ$, where $L \in \mathbb{R}^{M\times M}$ is a lower triangular matrix and $Q \in \mathbb{R}^{M  \times d}$ is a semi-orthogonal matrix ($Q Q^{\intercal} = I_{M}$), \eqref{eq:linear_combination} becomes $W = \Pi L Q \in \mathbb{R}^{T \times d}$, thus, $W = W Q^{\intercal} Q$, leading to the constraint $\left\|W - W Q^{\intercal} Q\right\|_{F}^{2} = \tr\left(W W ^{\intercal}\right) - \tr\left(WQ^{\intercal}QW^{\intercal}\right) = 0$. If we assume $\left\|\theta_{m}\right\|^{2}_{2}$ to be bounded by a constant $B>0$  for all $m \in[M]$, we get the constraint $\tr\left(W W^{\intercal} \right) \leq TB$.
It means that minimizing $\sum_{t=1}^{T}\sum_{i=1}^{n_{t}}l\left(h_{w_{t}}\left(\vec{x}_{t}^{(i)}\right), y_{t}^{(i)}\right)$ under our Assumption~\ref{assum:mixture} can be formulated as the following constrained optimization problem
\begin{equation}
        \label{eq:constrained_pblm_aso}
        \begin{aligned}
            \minimize_{W, Q: QQ^{\intercal}=I_{M}} \quad & \sum_{t=1}^{T}\sum_{i=1}^{n_{t}}l\left(h_{w_{t}}\left(\vec{x}_{t}^{(i)}\right), y_{t}^{(i)}\right),\\
            \mathrm{subject \;\; to} \quad & \tr\left\{W W ^{\intercal}\right\} - \tr\left\{WQ^{\intercal}QW^{\intercal}\right\} = 0, \\
            & \tr\left(W W^{\intercal} \right) \leq TB.
        \end{aligned}
\end{equation}

Thus, there exists Lagrange multipliers $\alpha \in \mathbb{R}$ and $\beta > 0$, for which Problem~\eqref{eq:constrained_pblm_aso} is equivalent to the following regularized optimization problem
\begin{equation}
        \minimize_{W, Q:QQ^{\intercal}=I_{M} }  \sum_{t=1}^{T}\sum_{i=1}^{n_{t}}l\left(h_{w_{t}}\left(\vec{x}_{t}^{(i)}\right), y_{t}^{(i)}\right) + \alpha \left(\tr\left\{W W ^{\intercal}\right\} - \tr\left\{WQ^{\intercal}QW^{\intercal}\right\}\right) + \beta \tr\left\{W W ^{\intercal}\right\},
\end{equation}
which is exactly Problem~\eqref{eq:aso_formulation}. 

\paragraph{Federated MTL via task relationships.}
The ASO formulation above motivated the authors of~\cite{smith2017federated} to learn personalized models by solving the following problem
\begin{equation}
    \label{eq:virginia_app}
    \min_{W,\Omega} \sum_{t=1}^{T}\sum_{i=1}^{n_{t}}l\left(h_{w_{t}}\left(\vec{x}_{t}^{(i)}\right), y_{t}^{(i)}\right) +  \lambda \tr\left(W\Omega W^{\intercal}\right),
\end{equation}
Two alternative MTL formulations are presented in~\cite{smith2017federated} to justify Problem~\eqref{eq:virginia_app}: MTL with probabilistic priors~\cite{zhang2010convex} and  MTL with graphical models~\cite{lauritzen1996graphical}.
Both of them can be covered using our Assumption~\ref{assum:mixture} as follows:
    \begin{itemize}
        \item Considering $T=M$ and $\Pi = I_{M}$ in Assumption~\ref{assum:mixture} and introducing a prior on $\Theta$ of the form 
            \begin{equation}
                \label{eq:prior_on_theta_app}
                \Theta \sim \left(\prod\mathcal{N}\left(0, \sigma^{2}I_{d}\right)\right) \mathcal{MN}\left(I_{d} \otimes \Omega \right)
            \end{equation}
        lead to a formulation similar to MTL with probabilistic priors \cite{zhang2010convex}.
        \item Two tasks $t$ and $t'$ are independent if $\langle \pi_{t}, \pi_{t'} \rangle  = 0$, thus using $\Omega_{t,t'} = \langle \pi_{t}, \pi_{t'} \rangle$ leads to the same graphical model as in \cite{lauritzen1996graphical}.
    \end{itemize}
Several personalized FL formulations, e.g., \pFedMe \cite{dinh2020personalized}, \texttt{FedU} \cite{dinh2021fedu} and the formulation studied in \cite{hanzely2020federated} and in \cite{hanzely2020lower}, are special cases of formulation~\eqref{eq:prior_on_theta_app}.

\newpage
\section{Centralized Expectation Maximization}
\label{proof:centralized_em}
\begin{repprop}{prop:em}
    Under Assumptions~\ref{assum:mixture} and~\ref{assum:uniform_marginal},
    at the $k$-th iteration the EM algorithm updates parameter estimates through the following steps:
    \begin{flalign}
        \textbf{E-step:} && q^{k+1}_{t}(z_{t}^{(i)}=m) & \propto  \pi_{t m}^{k} \cdot \exp\left(-l(h_{\theta_{m}^{k}}(\vec x_t^{(i)}), y_{t}^{(i)})\right), &    t \in [T],~ m \in [M],~ i \in [n_{t}]~~
        \tag{\ref{eq:e_step}}
    \end{flalign}
        \begin{flalign}
        \textbf{M-step:} &&\quad~\pi^{k+1}_{t m} & = \frac{\sum_{i=1}^{n_{t}}q^{k+1}_{t}(z_{t}^{(i)}=m)}{n_{t}}, &   t \in [T],~\ m \in [M]  \tag{\ref{eq:update_pi}}\\
        &&\theta^{k+1}_{m} & \in \argmin_{\theta\in\mathbb{R}^{d}} \sum_{t=1}^{T}\sum_{i=1}^{n_{t}} q^{k+1}_{t}(z_{t}^{(i)}=m)  l\big(h_{\theta}(\vec{x}_{t}^{(i)}), y_{t}^{(i)}\big), & m \in [M]   \tag{\ref{eq:update_h}}
    \end{flalign}
\end{repprop}
\begin{proof}
    The objective is to learn parameters $\{\breve{\Theta}, \breve{\Pi}\}$ from the data $\mathcal{S}_{1:T}$ by maximizing the likelihood $p\left(S_{1:T}|\Theta, \Pi\right)$. We introduce functions $q_{t}(z),~t\in[T]$ such that $q_{t} \geq 0$ and $\sum_{z=1}^{M}q_{t}(z) = 1$ in the expression of the likelihood. For $\Theta \in \mathbb{R}^{M\times d}$ and $\Pi \in \Delta^{T\times M}$, we have 
    \begin{flalign}
        \log p(\mathcal{S}_{1:T}|\Theta, \Pi) &= \sum_{t=1}^{T}\sum_{i=1}^{n_{t}}\log p_t\left(s_{t}^{(i)}|\Theta, \pi_{t}\right) &
        \\
        &= \sum_{t=1}^{T}\sum_{i=1}^{n_{t}}\log \left[\sum_{m=1}^{M}\left(\frac{p_t\left(s_{t}^{(i)}, z_{t}^{(i)}=m|\Theta, \pi_{t}\right)}{q_t\left(z_{t}^{(i)}=m\right)}\right) q_t\left(z_{t}^{(i)}=m\right)\right]
        \\
        &\geq \sum_{t=1}^{T}\sum_{i=1}^{n_{t}}\sum_{m=1}^{M}q_t\left(z_{t}^{(i)}=m\right)\log\frac{p_t\left(s_{t}^{(i)}, z_{t}^{(i)}=m|\Theta, \pi_{t}\right)}{q_t\left(z_{t}^{(i)}=m\right)}
        \\
        &=  \sum_{t=1}^{T}\sum_{i=1}^{n_{t}}\sum_{m=1}^{M}q_t\left(z_{t}^{(i)}=m\right) \log p_t\left(s_{t}^{(i)},z_{t}^{(i)}=m|\Theta, \pi_{t}\right) \nonumber
        \\
        & \qquad - \sum_{t=1}^{T}\sum_{i=1}^{n_{t}}\sum_{m=1}^{M}q_t\left(z_{t}^{(i)}=m\right)\log q_t\left(z_{t}^{(i)}=m\right)
        \\
        \label{eq:variational_lower_bound}
        &\triangleq \mathfrak{L}(\Theta, \Pi, Q_{1:T}),
    \end{flalign}
    where we used Jensen's inequality because $\log$ is concave. $\mathfrak{L}(\Theta, \Pi, Q_{1:T})$ is an \emph{evidence lower bound}. The centralized EM-algorithm corresponds to iteratively maximizing this bound with respect to $Q_{1:T}$ (E-step) and with respect to $\{\Theta, \Pi\}$ (M-step).
    
    \paragraph{E-step.} The difference between the log-likelihood and the evidence lower bound $\mathfrak{L}(\Theta, \Pi, Q_{1:T})$ can be expressed in terms of a sum of $\mathcal{KL}$ divergences:
    \begin{flalign}
        \log  & p(\mathcal{S}_{1:T}|\Theta, \Pi) - \mathfrak{L}(\Theta, \Pi, Q_{1:T})= \nonumber &
        \\ &=  \sum_{t=1}^{T}\sum_{i=1}^{n_{t}}\left\{\log p_t\left(s_{t}^{(i)}|\Theta, \pi_{t}\right) - \sum_{m=1}^{M}q_t\left(z_{t}^{(i)}=m\right)\log \frac{p_t\left(s_{t}^{(i)},z_{t}^{(i)}=m|\Theta, \pi_{t}\right) }{q_t\left(z_{t}^{(i)}=m\right)}\right\} 
        \\
        &=\sum_{t=1}^{T}\sum_{t=1}^{n_{t}}\sum_{m=1}^{M} q_{t}\left(z_{t}^{(i)}=m\right)\left( \log p_{t}\left(s_{t}^{(i)}| \Theta, \pi_{t}\right)  - \log \frac{p_{t}\left(s_{t}^{(i)}, z_{t}^{(i)}=m|  \Theta, \pi_{t}\right)}{q_{t}\left(z_{t}^{(i)}=m\right)} \right)
        \\
        &= \sum_{t=1}^{T}\sum_{t=1}^{n_{t}}\sum_{m=1}^{M} q_{t}\left(z_{t}^{(i)}=m\right) \log \frac{p_{t}\left(s_{t}^{(i)}| \Theta, \pi_{t}\right) \cdot q_{t}\left(z_{t}^{(i)}=m\right)}{p_{t}\left(s_{t}^{(i)}, z_{t}^{(i)}=m|  \Theta, \pi_{t}\right)}
        \\
        & = \sum_{t=1}^{T}\sum_{t=1}^{n_{t}}\sum_{m=1}^{M} q_{t}\left(z_{t}^{(i)}=m\right) \log \frac{ q_{t}\left(z_{t}^{(i)}=m\right)}{p_{t}\left(z_{t}^{(i)}=m|s_{t}^{(i)}, \Theta, \pi_{t}\right)}
        \\
        &= \sum_{t=1}^{T}\sum_{i=1}^{n_{t}}\mathcal{KL}\left(q_{t}\left(z_{t}^{(i)}\right)||p_{t}\left(z_{t}^{(i)} | s_{t}^{(i)}, \Theta, \pi_{t}\right)\right) \geq 0. \label{eq:variational_gap_is_kl}
    \end{flalign}
     For fixed parameters $\{\Theta, \Pi\}$,  the maximum of $\mathfrak{L}(\Theta, \Pi, Q_{1:T})$ is reached when   \[\sum_{t=1}^{T}\sum_{i=1}^{n_{t}}\mathcal{KL}\left(q_{t}\left(z_{t}^{(i)}\right)||p_{t}\left(z_{t}^{(i)} | s_{t}^{(i)}, \Theta, \pi_{t}\right)\right) = 0.\] Thus for $t\in[T]$ and $i\in[n_{t}]$, we have:
    \begin{flalign}
        q_{t}(z_{t}^{(i)}=m) &= p_{t}(z_{t}^{(i)}=m|s_{t}^{(i)}, \Theta, \pi_{t}) &
        \\
        &= \frac{p_{t}(s_{t}^{(i)} | z_{t}^{(i)}=m,\Theta, \pi_{t}) \times p_{t}(z_{t}^{(i)}=m|\Theta, \pi_{t})}{p_{t}\left(s_{t}^{(i)}|\Theta, \pi_{t}\right)}
        \\
        &= \frac{p_{m}(s_{t}^{(i)}| \theta_m) \times \pi_{t m}}{\sum_{m'=1}^{M} p_{m'}(s_{t}^{(i)}) \times \pi_{t m'}}
        \\
        &=  \frac{p_{{m}}\left(y_{t}^{(i)}|\vec{x}_{t}^{(i)}, \theta_m\right)\times p_{m}\left(\vec{x}_{t}^{(i)}\right) \times \pi_{t m}}{\sum_{m'=1}^{M} p_{{m'}}\left(y_{t}^{(i)}|\vec{x}_{t}^{(i)}, \theta_{m'}\right)\times p_{m'}\left(\vec{x}_{t}^{(i)}\right) \times \pi_{t m'}}
        \\
        &=  \frac{p_{{m}}\left(y_{t}^{(i)}|\vec{x}_{t}^{(i)}, \theta_{m} \right)\times p\left(\vec{x}_{t}^{(i)}\right) \times \pi_{t m}}{\sum_{m'=1}^{M} p_{{m'}}\left(y_{t}^{(i)}|\vec{x}_{t}^{(i)}, , \theta_{m'}\right)\times p\left(\vec{x}_{t}^{(i)}\right) \times \pi_{t m'}},\label{eq:ass_marginals_use}
    \end{flalign}
    where \eqref{eq:ass_marginals_use} relies on Assumption \ref{assum:uniform_marginal}.
    It follows that
    \begin{equation}
        \label{eq:app_e_step}
        q_{t}(z_{t}^{(i)}=m) = p_{t}(z_{t}^{(i)}=m|s_{t}^{(i)}, \Theta, \pi_{t})  = \frac{p_{{m}}\left(y_{t}^{(i)}|\vec{x}_{t}^{(i)}, \theta_{m}\right)\times \pi_{t m}}{\sum_{m'=1}^{M} p_{{m'}}\left(y_{t}^{(i)}|\vec{x}_{t}^{(i)}, \theta_{m'}\right) \times \pi_{t m'}}.
    \end{equation}

    \paragraph{M-step.} We maximize now $\mathfrak{L}(\Theta, \Pi, Q_{1:T})$ with respect to $\{\Theta, \Pi\}$. By dropping the terms  not depending on $\{\Theta, \Pi\}$ in the expression of  $\mathfrak{L}(\Theta, \Pi, Q_{1:T})$ we write:
    \begin{flalign}
        \mathfrak{L}& (\Theta, \Pi, Q_{1:T})  & \nonumber \\
        &= \sum_{t=1}^{T}\sum_{i=1}^{n_{t}}\sum_{m=1}^{M}q_t\left(z_{t}^{(i)}=m\right)\log p_t\left(s_{t}^{(i)}, z_{t}^{(i)}=m|\Theta, \pi_{t}\right) + c 
        \\
        &= \sum_{t=1}^{T}\sum_{i=1}^{n_{t}}\sum_{m=1}^{M}q_t\left(z_{t}^{(i)}=m\right) \Big[\log p_t\left(s_{t}^{(i)}| z_{t}^{(i)}=m, \Theta, \pi_{t}\right) + \log p_t\left(z_{t}^{(i)}=m| \Theta, \pi_{t}\right) \Big] +  c
        \\
        &= \sum_{t=1}^{T}\sum_{i=1}^{n_{t}}\sum_{m=1}^{M}q_t\left(z_{t}^{(i)}=m\right) \left[\log p_{\theta_{m}}\left(s_{t}^{(i)}\right) + \log \pi_{t m} \right]+  c
        \\
        &= \sum_{t=1}^{T}\sum_{i=1}^{n_{t}}\sum_{m=1}^{M}q_t\left(z_{t}^{(i)}=m\right) \left[\log p_{\theta_{m}}\left(y_{t}^{(i)}| \vec{x}_{t}^{(i)} \right) + \log p_m\left( \vec{x}_{t}^{(i)} \right) + \log \pi_{t m} \right] +c 
        \\
        &= \sum_{t=1}^{T}\sum_{i=1}^{n_{t}}\sum_{m=1}^{M}q_t\left(z_{t}^{(i)}=m\right) \left[\log p_{\theta_{m}}\left(y_{t}^{(i)}| \vec{x}_{t}^{(i)} \right)  + \log \pi_{t m} \right]+  c',\\
    \end{flalign}
    where $c$ and $c'$ are constant not depending on $\left\{\Theta, \Pi\right\}$.

    Thus, for $t\in[T]$ and $m\in[M]$, by solving a simple optimization problem we update $\pi_{t m}$ as follows:
    \begin{equation}
        \label{eq:app_update_pi}
        \pi_{t m} =\frac{\sum_{i=1}^{n_{t}}q_{t}(z_{t}^{(i)}=m)}{n_{t}}.
    \end{equation}
    On the other hand, for $m\in[M]$, we update $\theta_{m}$ by solving:
    \begin{equation}
        \label{eq:app_update_h}
        \theta_{m} \in \argmin_{\theta\in\mathbb{R}^{d}} \sum_{t=1}^{T}\sum_{i=1}^{n_{t}} q_{t}(z_{t}^{(i)}=m) \times  l\left(h_{\theta}(\vec{x}_{t}^{(i)}), y_{t}^{(i)}\right).
    \end{equation}
\end{proof}

\newpage
\section{Detailed Algorithms}
\subsection{Client-Server Algorithm}
\label{app:alg_centralized}
Alg.~\ref{alg:fed_em} is a detailed version of Alg.~\ref{alg:fed_em_short} (\FedEM), with local SGD used as local solver.

Alg.~\ref{alg:fed_surrogate_opt} gives our general algorithm for federated surrogate optimization, from which Alg.~\ref{alg:fed_em} is derived.

\begin{algorithm}
    \SetKwFunction{LocalSolver}{LocalSolver}
    \SetKwInOut{Input}{Input}
    \SetKwInOut{Output}{Output}
    \SetKw{Iterations}{iterations}
    \SetKw{Sample}{sample}
    \SetKw{Tasks}{tasks}
    \SetKw{Component}{component}
    \SetKw{InParallel}{in parallel}
    \SetKw{Server}{server}
    \SetKw{Clients}{clients}
    \SetKw{Client}{client}
    \SetKw{Broadcasts}{broadcasts}
    \SetKw{Sends}{sends}
    \SetKw{Over}{over}
    \SetKw{Tothe}{to the}
    \SetKwProg{Fn}{Function}{:}{}
    \SetAlgoLined
    \Input{ Data $\mathcal{S}_{1:T}$; number of mixture components $M$;  number of communication rounds $K$; number of local steps $J$}
    \Output{$\theta^{K}_{m}$ for $1\in[M]$; $\pi^{K}_{t}$ for $t\in[T]$}
    \tcp{Initialization}
    \Server randomly initialize $\theta^{0}_{m}\in\mathbb{R}^{d}$ for $1\leq m\leq M$\;
    \For{\Tasks $t=1, \dots, T$ \InParallel \Over $T$ \Clients}{
        Randomly initialize $\pi^{0}_{t} \in\Delta^{M}$\;
    }
    \tcp{Main loop}
    \For{\Iterations $k=1, \dots, K$}{
        \Server \Broadcasts $\theta^{k-1}_{m},~1\leq m \leq M$ \Tothe $T$ \Clients\;
        \For{\Tasks $t=1, \dots, T$ \InParallel \Over $T$ \Clients}{
            \For{\Component $m=1, \dots, M$}{
                \tcp{E-step}
                \For{\Sample $i=1, \dots, n_{t}$}{
                    $q^{k}_{t}\left(z_{t}^{(i)}=m\right) \gets \frac{\pi_{t m}^{k} \cdot \exp\left(-l(h_{\theta_{m}^{k}}(\vec x_t^{(i)}), y_{t}^{(i)})\right)}{\sum_{m'=1}^{M}\pi_{t m'}^{k} \cdot \exp\left(-l(h_{\theta_{m'}^{k}}(\vec x_t^{(i)}), y_{t}^{(i)})\right)}$ \;
                }
                \tcp{M-step}
                $\pi_{t m}^{k} \gets \frac{\sum_{i=1}^{n_{t}}q^{k}_{t}(z_{t}^{(i)}=m)}{n_{t}}$ \;
                $\theta_{m,t}^{k} \gets$ \LocalSolver{$J$, $m$, $\theta^{k-1}_{m}$, $q_{t}^{k}$, $\mathcal{S}_{t}$}
                \;
            }
            \Client $t$ \Sends $\theta_{m,t}^{k},~1\leq m \leq M$ \Tothe \Server\;
        }
        \For{\Component $m=1, \dots, M$}{
            $\theta_{m}^{k} \gets \sum_{t=1}^{T}\frac{n_{t}}{n} \cdot \theta_{m,t}^{k}$\;
        }
    }
    \BlankLine
    \BlankLine
    \BlankLine
    \Fn{\LocalSolver{$J$, $m$, $\theta$, $q$, $\mathcal{S}$}}{
        \For{$j=0,\dots, J-1$}{
        Sample indexes $\mathcal{I}$ uniformly from $1, \dots, |\mathcal{S}|$\;
            $\theta \gets \theta - \eta_{k-1,j}\sum_{i\in\mathcal{I}}q(z^{(i)}=m)\cdot \nabla_{\theta} l\left(h_{\theta}\left(\vec{x}^{(i)}\right), y^{(i)}\right)$\;
        }
        \Return $\theta$\;
    }
    \caption{\FedEM: Federated Expectation-Maximization}
    \label{alg:fed_em}
\end{algorithm}

\newpage

\begin{algorithm}
    \SetKwInOut{Input}{Input}
    \SetKwInOut{Output}{Output}
    \SetKw{Iterations}{iterations}
    \SetKw{Tasks}{tasks}
    \SetKw{InParallel}{in parallel}
    \SetKw{Server}{server}
    \SetKw{Clients}{clients}
    \SetKw{Client}{client}
    \SetKw{Broadcasts}{broadcasts}
    \SetKw{Sends}{sends}
    \SetKw{Over}{over}
    \SetKw{Tothe}{to the}
    \SetKwFunction{LocalSolver}{LocalSolver}
    \SetKwProg{Fn}{Function}{:}{}
    \SetAlgoLined
    \Input{ $\vec{u}^{0} \in \mathbb{R}^{d_{u}}$; $\vec{V}^{0}=\left(\vec{v}_{t}^{0}\right)_{1\leq t \leq T} \in \mathcal{V}^{T}$; number of iterations $K$; number of local steps $J$}
   \Output{ $\vec{u}^{K}$;~$\vec{v}_{t}^{K}$}
    \For{\Iterations $k=1, \dots, K$}{
        \Server \Broadcasts $\vec{u}^{k-1}$ \Tothe $T$ \Clients\;
        \For{\Tasks $t=1, \dots, T$ \InParallel \Over $T$ \Clients}{
            Compute  partial first-order surrogate function $g_{t}^{k}$ of $f_{t}$ near $\left\{\vec{u}^{k-1}, \vec{v}_{t}^{k-1}\right\}$\;
            $\vec{v}_{t}^{k} \gets \argmin\limits_{\vec{v}\in\mathcal{V}} g_{t}^{k}\left(\vec{u}^{k-1}, \vec{v}\right)$\;
            $u_{t}^{k} \gets$ \LocalSolver{$J$, $\vec{u}_t^{k-1}$, $\vec{v}_{t}^{k-1}$, $g_{t}^{k}$, $\mathcal{S}_t$}\; \label{line:surogate_optimization_local_slover}
            \Client $t$ \Sends $\vec{u}_{t}^{k}$ \Tothe \Server\;
        }
        $\vec{u}^{k} \gets \sum_{t=1}^{T}\omega_{t}\cdot \vec{u}_{t}^{k}$\;
    }
    \BlankLine
    \BlankLine
    \BlankLine
    \Fn{\LocalSolver{$J$, $\vec{u}$, $\vec{v}$, $g$, $\mathcal{S}$}}{
        \For{$j=0,\dots, J-1$}{
            sample $\xi^{k-1,j}$ from $\mathcal{S}$\;
            $\vec{u} \gets \vec{u} - \eta_{k-1,j} \cdot \nabla_{\vec{u}} g(\vec{u}, \vec{v}; \xi^{k-1,j})$\;
        }
        \Return $\Theta$\;
    }
    \caption{Federated Surrogate Optimization}
    \label{alg:fed_surrogate_opt}
\end{algorithm}

\newpage
\subsection{Fully Decentralized Algorithm}
\label{app:alg_decentralized}
Alg.~\ref{alg:d_em} shows \DEM{}, the fully decentralization version of our federated expectation maximization algorithm.

Alg.~\ref{alg:decentralized_surrogate_opt} gives our general fully decentralized algorithm for federated surrogate optimization, from which Alg.~\ref{alg:d_em} is derived.

\begin{algorithm}
    \SetKwFunction{LocalSolver}{LocalSolver}
    \SetKwInOut{Input}{Input}
    \SetKwInOut{Output}{Output}
    \SetKw{Iterations}{iterations}
    \SetKw{Sample}{sample}
    \SetKw{Tasks}{tasks}
    \SetKw{Component}{component}
    \SetKw{InParallel}{in parallel}
    \SetKw{Server}{server}
    \SetKw{Clients}{clients}
    \SetKw{Client}{client}
    \SetKw{Broadcasts}{broadcasts}
    \SetKw{Sends}{sends}
    \SetKw{Over}{over}
    \SetKw{Tothe}{to the}
    \SetKw{Receives}{receives}
    \SetKwProg{Fn}{Function}{:}{}
    \SetAlgoLined
    \Input{ Data $\mathcal{S}_{1:T}$; number of mixture components $M$; number of iterations $K$; number of local steps $J$; mixing matrix distributions $\mathcal{W}^{k}$ for $k\in [K]$}
    \Output{ $\theta^{K}_{m,t}$ for $m\in[M]$ and $t\in[T]$; $\pi_{t}$ for $t\in[T]$}
    \tcp{Initialization}
    \For{\Tasks $t=1, \dots, T$ \InParallel \Over $T$ \Clients}{
        Randomly initialize $\Theta_{t} = (\theta_{m, t})_{1\leq m\leq M} \in\mathbb{R}^{M \times d}$ \;
        Randomly initialize $\pi^{0}_{t} \in\Delta^{M}$\;
    }
    \tcp{Main loop}
    \For{\Iterations $k=1, \dots, K$}{
        \tcp{Select the communication topology and the aggregation weights}
        Sample $W^{k-1} \sim \mathcal{W}^{k-1}$\;
        \For{\Tasks $t=1, \dots, T$ \InParallel \Over $T$ \Clients}{
            \For{\Component $m=1, \dots, M$}{
                \tcp{E-step}
                \For{\Sample $i=1, \dots, n_{t}$}{
                    $q^{k}_{t}\left(z_{t}^{(i)}=m\right) \gets \frac{\pi_{t m}^{k} \cdot \exp\left(-l(h_{\theta_{m}^{k}}(\vec x_t^{(i)}), y_{t}^{(i)})\right)}{\sum_{m'=1}^{M}\pi_{t m'}^{k} \cdot \exp\left(-l(h_{\theta_{m'}^{k}}(\vec x_t^{(i)}), y_{t}^{(i)})\right)}$\;
                }
                \tcp{M-step}
                $\pi_{t m}^{k} \gets \frac{\sum_{i=1}^{n_{t}}q^{k}_{t}(z_{t}^{(i)}=m)}{n_{t}}$ \;
                $\theta_{m,t}^{k-\frac{1}{2}} \gets$  \LocalSolver{$J$, $m$, $\theta_{m,t}^{k-1}$, $q_{t}^{k}$, $\mathcal{S}_{t}$, $t$}\label{line:d_EM_local_solver_step}\;
            }
            Send $\theta_{m,t}^{k-\frac{1}{2}},~1\leq m \leq M$ to neighbors\;
            Receive $\theta_{m,s}^{k-\frac{1}{2}},~1\leq m \leq M$ from neighbors\;
            \For{\Component $m=1, \dots, M$}{
                $\theta_{m,t}^{k} \gets \sum_{s=1}^{T}w^{k-1}_{s,t} \cdot \theta_{m,s}^{k-\frac{1}{2}}$\; \label{line:d_EM_aggregation}
            }
        }
    }
    \BlankLine
    \BlankLine
    \BlankLine
    \Fn{\LocalSolver{$J$, $m$, $\theta$, $q$, $\mathcal{S}$, $t$}}{
        \For{$j=0,\dots, J-1$}{
        Sample indexes $\mathcal{I}$ uniformly from $1, \dots, |\mathcal{S}|$\;
            $\theta \gets \theta -  \frac{n_{t}}{n} \cdot \eta_{k-1,j}\sum_{i\in\mathcal{I}}q(z^{(i)}=m)\cdot \nabla_{\theta} l\left(h_{\theta}\left(\vec{x}^{(i)}\right), y^{(i)}\right)$\;
        }
        \Return $\theta$\;
    }
    \caption{\DEM: Fully Decentralized Federated Expectation-Maximization}
    \label{alg:d_em}
\end{algorithm}

\newpage
\begin{algorithm}
    \SetKwInOut{Input}{Input}
    \SetKwInOut{Output}{Output}
    \SetKw{Iterations}{iterations}
    \SetKw{Tasks}{tasks}
    \SetKw{InParallel}{in parallel}
    \SetKw{Server}{server}
    \SetKw{Clients}{clients}
    \SetKw{Client}{client}
    \SetKw{Broadcasts}{broadcasts}
    \SetKw{Sends}{sends}
    \SetKw{Receives}{receives}
    \SetKw{Over}{over}
    \SetKw{Tothe}{to the}
    \SetKwFunction{LocalSolver}{LocalSolver}
    \SetKwProg{Fn}{Function}{:}{}
    \SetAlgoLined
    \Input{ $\vec{u}^{0} \in \mathbb{R}^{d_{u}}$; $\vec{V}^{0}=\left(\vec{v}_{t}^{0}\right)_{1\leq t \leq T} \in \mathcal{V}^{T}$; number of iterations $K$; number of local step $J$; mixing matrix distributions $\mathcal{W}^{k}$ for $k\in [K]$}
    \Output{ $\vec{u}^{K}_{t}$ for $t\in[T]$; $\vec{v}_{t}^{K}$ for $t\in[T]$}
    \For{\Iterations $k=1, \dots, K$}{
        \tcp{Select the communication topology and the aggregation weights}
        Sample $W^{k-1} \sim \mathcal{W}^{k-1}$\;
        \For{\Tasks $t=1, \dots, T$ \InParallel \Over $T$ \Clients}{
            compute partial first-order surrogate function $g_{t}^{k}$ of $f_{t}$ near $\left\{\vec{u}_{t}^{k-1}, \vec{v}_{t}^{k-1}\right\}$\;
            $\vec{v}_{t}^{k} \gets \argmin\limits_{v\in\mathcal{V}} g_{t}^{k}\left(\vec{u}_{t}^{k-1}, \vec{v}\right)$\;
            $\vec{u}_{t}^{k-\frac{1}{2}} \gets$  \LocalSolver{$J$, $\vec{u}_{t}^{k-1}$, $\vec{v}_{t}^{k-1}$, $g_{t}^{k}$, $t$}\; \label{line:decentralized_surogate_optimization_local_slover}
            Send $\vec{u}_{t}^{k-\frac{1}{2}}$ to neighbors\;
            Receive $\vec{u}_{s}^{k-\frac{1}{2}}$ from neighbors\;
            $\vec{u}^{k}_{t} \gets \sum_{s=1}^{T}w^{k-1}_{t s}\times \vec{u}_{s}^{k-\frac{1}{2}}$\;
        }
    }
    \BlankLine
    \BlankLine
    \BlankLine
    \Fn{\LocalSolver{$J$, $\vec{u}$, $\vec{v}$, $g$, $\mathcal{S}$, $t$}}{
        \For{$j=0,\dots, J-1$}{
            sample $\xi^{k-1,j}$ from $\mathcal{S}$ \;
            $\vec{u} \gets \vec{u} - \omega_{t} \cdot \eta_{k-1,j} \nabla_{\vec{u}} g(\vec{u}, \vec{v}, \xi^{k-1,j})$\;
        }
        \Return $\vec{u}$\;
    }
    \caption{Fully-Decentralized Federated Surrogate Optimization}
    \label{alg:decentralized_surrogate_opt}
\end{algorithm}

\newpage
\section{Details on the Fully Decentralized Setting}
\label{app:full_decentralized_assumptions}
As mentioned in Section~\ref{sec:decentralized},
the convergence of decentralized optimization schemes requires certain assumptions on the sequence of mixing matrices $(W^{k})_{k>0}$, to guarantee that each client can influence the estimates of other clients over time. In our paper, we consider the following general assumption.
\begin{assumption}[{\cite[Assumption~4]{Koloskova2020AUT}}]
    \label{assum:expected_consensus_rate}
    Symmetric doubly stochastic mixing matrices are drawn at each round $k$ from (potentially different) distributions $W^{k}\sim\mathcal{W}^{k}$ and 
    there exists two constants $p\in(0,1]$, and integer $\tau\geq 1$ such that for all $\Xi \in\mathbb{R}^{M \times d\times T}$ and all integers $l\in\left\{0, \dots, K/\tau\right\}$:
    \begin{equation}
        \mathbb{E}\left\|\Xi W_{l,\tau} - \bar{\Xi}\right\|_{\mathcal{F}}^{2} \leq (1-p)\left\|\Xi - \bar{\Xi}\right\|_{\mathcal{F}}^{2},
    \end{equation}
    where $W_{l,\tau} \triangleq W^{(l+1)\tau-1}\dots W^{l \tau}$, $\bar{\Xi}\triangleq \Xi\frac{\mathbf{1}\mathbf{1}^{\intercal}}{T}$, and the expectation  is taken over the random distributions $W^{k}\sim\mathcal{W}^{k}$.
\end{assumption}
Assumption~\ref{assum:expected_consensus_rate} expresses the fact that the sequence of mixing matrices, on average and every $\tau$ communication rounds, brings the values in the columns of $\Xi$ closer to their row-wise average (thereby mixing the clients' updates over time).
For instance, the assumption is satisfied if the communication graph is strongly connected every $\tau$ rounds, i.e., the graph $([T],\mathcal E)$, where the edge $(i,j)$ belongs to the graph if $w_{i,j}^h>0$ for some $h \in \{k+1, \dots, k+\tau \}$ is connected.

We provide below the rigorous statement of Theorem~\ref{thm:decentralized_convergence}, which was informally presented in Section~\ref{sec:decentralized}. It shows that \DEM{} converges to a consensus stationary point of $f$ (proof in App.~\ref{proof:decentralized}).
\begin{repthm}{thm:decentralized_convergence}
    Under Assumptions~\ref{assum:mixture}--\ref{assum:expected_consensus_rate}, when clients use SGD as local solver with learning rate $\eta =\frac{a_{0}}{\sqrt{K}}$, \DEM's iterates  satisfy the following inequalities after a large enough number of communication rounds $K$:
    \begin{equation}
        \frac{1}{K}\sum_{k=1}^K \mathbb{E}\left\|\nabla_{\Theta}f\left(\bar{\Theta}^{k}, \Pi^{k}\right)\right\|_{F}^{2} \leq \mathcal{O} \left(\frac{1}{\sqrt{K}}\right), \quad
        \frac{1}{K}\sum_{k=1}^K\sum_{t=1}^{T}\frac{n_{t}}{n}\mathcal{KL}\left(\pi^{k}_{t}, \pi_{t}^{k-1}\right)\leq \mathcal{O}\left(\frac{1}{K}\right),
    \end{equation}
    where $\bar \Theta^k = \left[\Theta_{1}^k, \dots \Theta_{T}^{k}\right] \cdot  \frac{\mathbf{1}\mathbf{1}^{\intercal}}{T}$.  
    Moreover, individual estimates $\left(\Theta_{t}^{k}\right)_{1\leq t \leq T}$ converge to consensus, i.e., to $\bar \Theta^k$:
    \begin{equation*}
        \min_{k\in[K]}\mathbb{E}\sum_{t=1}^{T}\left\|\Theta_{t}^{k} - \bar{\Theta}^{k}\right\|_{F}^{2} \leq \mathcal{O} \left(\frac{1}{\sqrt{K}}\right).
    \end{equation*}
\end{repthm}

\newpage
\section{Federated Surrogate Optimization}
\label{app:fed_surrogate_optimization}
In this appendix, we give more details on the federated surrogate optimization framework introduced in Section~\ref{sec:fed_surrogate}. In particular, we provide the assumptions under which Alg.~\ref{alg:fed_surrogate_opt} and Alg.~\ref{alg:decentralized_surrogate_opt} converge. We also illustrate how our framework can be used to study existing algorithms.

\subsection{Reminder on Basic (Centralized) Surrogate Optimization}
\label{app:surrogate_optimization}
In this appendix, we recall the (centralized) \emph{first-order surrogate optimization} framework introduced in \cite{mairal2013optimization}. In this framework, given a continuous function $f:\mathbb{R}^{d} \mapsto \mathbb{R}$, we are interested in solving
\begin{equation*}
    \min_{\theta\in\mathbb{R}^{d}}f(\theta)
\end{equation*}
using the majoration-minimization scheme presented in Alg.~\ref{alg:app_basic_surrogate_opt}. 
\begin{algorithm}
    \SetKwInOut{Input}{Input}
    \SetKwInOut{Output}{Output}
    \SetKw{Iterations}{iterations}
    \SetAlgoLined
    \Input{$\;\theta^{0} \in \mathbb{R}^{d}$; number of iterations $K$;}
    \Output{$\;\theta^{K}$}
    \For{\Iterations $k=1, \dots, K$}{
        Compute $g^{k}$, a surrogate function of $f$ near $\theta^{k-1}$\;
        Update solution: $\theta^{k} \in \argmin_{\theta}g^{k}(\theta)$\;
    }
    \caption{Basic Surrogate Optimization}
    \label{alg:app_basic_surrogate_opt}
\end{algorithm}

This procedure relies on surrogate functions, that approximate well the objective function in a neighborhood of a point. Reference~\cite{mairal2013optimization} focuses on \emph{first-order surrogate functions} defined below.
\begin{dfn}[First-Order Surrogate \cite{mairal2013optimization}]
    \label{dfn:first_order_surrogate}
    A function $g:\mathbb{R}^{d}\mapsto \mathbb{R}$ is a first order surrogate of $f$ near $\theta^{k}\in\mathbb{R}^{d}$ when the following is satisfied:
    \begin{itemize}
        \item \textbf{Majorization:}  we have $g(\theta') \geq f(\theta')$ for all $\theta'\in\argmin_{\theta\in\mathbb{R}^{d}}g(\theta)$. When the more general condition $g\geq f$ holds, we say that $g$ is a \textbf{majorant} function.  
        \item \textbf{Smoothness:} the approximation error $r\triangleq g -f$ is differentiable, and its gradient is $L$-Lipschitz. Moreover, we have $r(\theta^{k})=0$ and $\nabla r(\theta^{k})=0$.
    \end{itemize}
\end{dfn}

\subsection{Novel Federated Version}

As discussed in Section~\ref{sec:fed_surrogate}, our novel federated surrogate optimization framework  minimizes an objective function $\left(\vec u, \vec v_{1:T}\right) \mapsto f\left(\vec u, \vec v_{1:T}\right)$ that can be written as a weighted sum $f\left(\vec u, \vec v_{1:T}\right)=\sum_{t=1}^{T}\omega_{t} f_{t}\left(\vec u, \vec v_{t}\right)$ of $T$ functions. We suppose that each client $t\in[T]$ can compute a partial first order surrogate of $f_{t}$, defined as follows.

\begin{repdefinition}{def:partial_first_order_surrogate}[Partial first-order surrogate]
    A function $g(\vec{u},\vec{v}) : \mathbb{R}^{d_u} \times \mathcal{V} \to \mathbb R$ is a partial first-order surrogate of $f(\vec{u},\vec{v})$  wrt $\vec{u}$ near $(\vec{u}_0, \vec{v}_0) \in 
    \mathbb \mathbb{R}^{d_u} \times \mathcal{V}$ when the following conditions are satisfied:
    \begin{enumerate}
        \setlength\itemsep{-0.5em}
        \item$g(\vec{u},\vec{v})\ge f(\vec{u},\vec{v})$ for all $\vec{u} \in \mathbb{R}^{d_u}$ and $\vec v\in  \mathcal{V}$;
        \item $r(\vec{u},\vec{v}) \triangleq g(\vec{u},\vec{v})- f(\vec{u},\vec{v})$ is differentiable and $L$-smooth with respect to $\vec{u}$. Moreover, we have $r(\vec{u}_0,\vec{v}_0) = 0$ and $\nabla_{\vec u} r(\vec{u}_0,\vec{v}_0) = 0$.
        \item $g(\vec u, \vec v_{0}) - g(\vec u, \vec v) = d_{\mathcal{V}}\left(\vec v_{0}, \vec v\right)$ for all $\vec{u} \in \mathbb{R}^{d_u}$ and $\vec v\in \argmin_{\vec{v}' \in \mathcal{V}} g(\vec{u}, \vec{v}') $, where $d_{\mathcal{V}}$ is non-negative and $d_{\mathcal{V}}(v, v') = 0 \iff v = v'$.
    \end{enumerate}
\end{repdefinition}

Under the assumption that each client $t$ can compute a partial first order surrogate of $f_{t}$, we propose algorithms for federated surrogate optimization in both the client-server setting (Alg.~\ref{alg:fed_surrogate_opt}) and the fully decentralized one (Alg.~\ref{alg:decentralized_surrogate_opt}). Both algorithms are iterative and distributed: at each iteration $k>0$, client $t\in[T]$ computes a partial first-order surrogate $g^{k}_{t}$ of $f_{t}$ near $\left\{u^{k-1}, v_{t}^{k-1}\right\}$ (resp.~$\left\{u_{t}^{k-1}, v_{t}^{k-1}\right\}$) for federated surrogate optimization in Alg.~\ref{alg:fed_surrogate_opt} (resp. for fully decentralized surrogate optimization in Alg~\ref{alg:decentralized_surrogate_opt}). 

The convergence of those two algorithms requires the following standard assumptions. Each of them generalizes one of the Assumptions~\ref{assum:bounded_f}--\ref{assum:bounded_dissimilarity} for our EM algorithms.

\begin{assumptionbis}{assum:bounded_f}
    \label{assum:bounded_f_bis}
    The objective function $f$ is bounded below by $f^{*} \in \mathbb{R}$.
\end{assumptionbis}
\begin{assumptionbis}{assum:smoothness}
    \label{assum:smoothness_bis}
    (Smoothness)
    For all $t\in[T]$ and $k>0$, $g_{t}^{k}$ is $L$-smooth wrt to $\vec u$. 
\end{assumptionbis}
\begin{assumptionbis}{assum:finitie_variance}
    \label{assum:finite_variance_bis}
    (Unbiased gradients and bounded variance)
    Each client $t\in [T]$ can sample a random batch  $\xi$ from $\mathcal S_t$ and compute an unbiased estimator $\nabla_{\vec u}g^{k}_{t}(\vec u, \vec v; \xi)$ of the local gradient with bounded variance, i.e.,  $\mathbb{E}_{\xi}[\nabla_{\vec u}g^{k}_{t}(\vec u, \vec v; \xi)] = \nabla_{\vec u}g^{k}_{t}(\vec u, \vec v)$ and 
    $\mathbb{E}_{\xi}\|\nabla_{\vec u}g^{k}_{t}(\vec u, \vec v; \xi) - \nabla_{\vec u}g^{k}_{t}(\vec u, \vec v)\|^{2} \leq \sigma^{2}$.
\end{assumptionbis}
\begin{assumptionbis}{assum:bounded_dissimilarity}
    \label{assum:bounded_dissimilarity_bis}
    (Bounded dissimilarity)
    There exist $\beta$ and $G$ such that
    \[
        \sum_{t=1}^{T}\omega_{t} \cdot \Big\|\nabla_{\vec u}g^{k}_{t}(\vec u, \vec v)\Big\|^{2} \leq G^{2} + \beta^{2} \Big\| \sum_{t=1}^{T}\omega_{t} \cdot \nabla_{\vec u}g^{k}_{t}(\vec u, \vec v) \Big\|^{2}.
    \]
\end{assumptionbis}

Under these assumptions a parallel result to Theorem.~\ref{thm:centralized_convergence} holds for the client-server setting.

\begin{thmbis}{thm:centralized_convergence}
    \label{thm:centralized_convergence_bis}
    Under Assumptions~\ref{assum:bounded_f_bis}--\ref{assum:bounded_dissimilarity_bis}, when clients use SGD as local solver with learning rate $\eta =\frac{a_{0}}{\sqrt{K}}$, after a large enough number of communication rounds $K$, the iterates of federated surrogate optimization (Alg.~\ref{alg:fed_surrogate_opt}) satisfy:
    \begin{equation}
        \label{eq:theta_pi_convergence_bis}
        \frac{1}{K} \sum_{k=1}^K \mathbb{E}\left\|\nabla _{\vec u} f\left(\vec u^{k}, \vec v_{1:T}^{k}\right)\right\|^{2}_{F} \leq \mathcal{O}\!\left(\frac{1}{\sqrt{K}}\right),    \qquad
    \frac{1}{K} \sum_{k=1}^K  \Delta_{\vec v} f(\vec u^k,\vec v^{k}_{1:T})
    \leq \mathcal{O}\!\left(\frac{1}{K^{3/4}}\right),    
    \end{equation}
    where the expectation is over the random batches samples, and $\Delta_{v} f(\vec u^k,\vec v_{1:T}^{k}) \triangleq f\left(\vec u^{k}, \vec v_{1:T}^{k}\right) - f\left(\vec u^{k}, \vec v_{1:T}^{k+1}\right) \ge 0 $.
\end{thmbis}

In the fully decentralized setting, if in addition to Assumptions~\ref{assum:bounded_f_bis}-\ref{assum:bounded_dissimilarity_bis}, we suppose that Assumption~\ref{assum:expected_consensus_rate} holds, a parallel result to Theorem.~\ref{thm:decentralized_convergence} holds.
\begin{thmbis}{thm:decentralized_convergence}
    \label{thm:decentralized_convergence_bis}
    Under Assumptions~\ref{assum:bounded_f_bis}--\ref{assum:bounded_dissimilarity_bis} and Assumption~\ref{assum:expected_consensus_rate}, when clients use SGD as local solver with learning rate $\eta =\frac{a_{0}}{\sqrt{K}}$, after a large enough number of communication rounds $K$, the iterates of fully decentralized federated surrogate optimization (Alg.~\ref{alg:decentralized_surrogate_opt}) satisfy:
    \begin{equation}
        \frac{1}{K}\sum_{k=1}^{K}\mathbb{E}\left\|\nabla_{\vec u}f\left(\bar{\vec u}^{k}, v_{1:T}^{k}\right)\right\|^{2} \leq \mathcal{O} \left(\frac{1}{\sqrt{K}}\right), \qquad
        \frac{1}{K}\sum_{k=1}^{K}\sum_{t=1}^{T}\omega_{t} \cdot d_{\mathcal{V}}\left(\vec v^{k}_{t}, \vec  \vec{v}^{k+1}_{t}\right)\leq \mathcal{O}\left(\frac{1}{K}\right),
    \end{equation}
    where $\bar{\vec u}^k = \frac{1}{T}\sum_{t=1}^{T}\vec u^{k}_{t}$.  
    Moreover, local estimates $\left(\vec u_{t}^{k}\right)_{1\leq t \leq T}$ converge to consensus, i.e., to $\bar{\vec u}^k$:
    \begin{equation*}
        \frac{1}{K}\sum_{k=1}^{K}\sum_{t=1}^{T}\left\|\vec u_{t}^{k} - \bar{\vec u}^{k}\right\|^{2} \leq \mathcal{O} \left(\frac{1}{\sqrt{K}}\right).
    \end{equation*}
\end{thmbis}

The proofs of Theorem~\ref{thm:centralized_convergence_bis} and Theorem~\ref{thm:decentralized_convergence_bis} are in Section~\ref{proof:centralized} and Section~\ref{proof:decentralized}, respectively.

\subsection{Illustration: Analyzing \texttt{pFedMe} with Federated Surrogate Optimization}
\label{app:pfedme_fso}

In this section, we show that \texttt{pFedMe} \cite{dinh2020personalized} can be studied through our federated surrogate optimization framework. With reference to the general formulation of \texttt{pFedMe} in \cite[Eq.~(2)~and~(3)]{dinh2020personalized}, consider 
\begin{equation}
    g_{t}^{k}\left(\vec{w}\right) =  f_{t}\left(\theta^{k-1}\right) + \frac{\lambda}{2} \cdot \left\|\theta^{k-1} - \omega\right\|^{2},
\end{equation}
where
$\theta^{k-1} = \prox_{\frac{f_{t}}{\lambda}}\left(\omega^{k-1}\right)\triangleq \argmin_{\theta} \left\{f_{t}\left(\theta\right) + \frac{\lambda}{2} \cdot \left\|\theta - \omega^{k-1}\right\|^{2}\right\}$. We can verify that $g_{t}^{k}$ is a first-order surrogate of $f_{t}$ near $\theta^{k-1}$: 
\begin{enumerate}
    \item It is clear that $g_{t}^{k}\left(\theta^{k-1}\right) = f_{t}\left(\theta^{k-1}\right)$.
    \item Since $\theta^{k-1} = \prox_{\frac{f_{t}}{\lambda}}\left(\omega^{k-1}\right)$, using the envelope theorem (assuming that $f_{t}$ is proper, convex and lower semi-continuous), it follows that $\nabla f_{t}\left(\omega^{k-1}\right) = \lambda \left(\theta^{k-1} - \omega^{k-1}\right) = \nabla g_{k}^{k}\left(\omega^{k-1}\right)$.
\end{enumerate} 

Therefore, \texttt{pFedMe} can be seen as a particular case of the federated surrogate optimization algorithm (Alg.~\ref{alg:fed_surrogate_opt}), to which our convergence results apply.

\newpage
\section{Convergence Proofs}
\label{app:proofs}
We study the client-server setting and the fully decentralized setting in Section~\ref{proof:centralized} and Section~\ref{proof:decentralized}, respectively. In both cases, we first prove the more general result for the federated surrogate optimization introduced in App.~\ref{app:fed_surrogate_optimization}, and then derive the specific result for \FedEM{} and \DEM{}.

\subsection{Client-Server Setting}
\label{proof:centralized}
\subsubsection{Additional Notations}

\begin{remark}
    For convenience and without loss of generality, we suppose in this section that $\omega \in \Delta^{T}$, i.e., $\forall t \in [T],~\omega_{t} \geq 0 $ and $\sum_{t'=1}^{T}\omega_{t'} = 1$.  
\end{remark}

At iteration $k>0$, we use $\vec{u}_{t}^{k-1, j}$ to denote the $j$-th iterate of the local solver at client $t\in[T]$, thus
\begin{equation}
    \vec{u}_{t}^{k-1,0} = \vec{u}^{k-1},
\end{equation}
and 
\begin{equation}
    \vec{u}^{k} = \sum_{t=1}^{T} \omega_{t} \cdot \vec{u}_{t}^{k-1, J}.
\end{equation}
At iteration $k>0$, the local solver's updates at client $t\in[T]$ can be written as (for $0\leq j\leq J-1$):
\begin{equation}
    \vec{u}_{t}^{k-1, j+1} = \vec{u}_{t}^{k-1, j} - \eta_{k-1, j} \nabla_{\vec{u}}g_{t}^{k}\left(\vec{u}_{t}^{k-1, j}, \vec{v}_{t}^{k-1}; \xi_{t}^{k-1,j}\right),
\end{equation}
where $\xi_{t}^{k-1,j}$ is the batch drawn at the $j$-th local update of $\vec{u}_{t}^{k-1}$.

We introduce $\eta_{k-1}=\sum_{j=0}^{J-1}\eta_{k-1, j}$, and we define the normalized update of the local solver at client $t\in[T]$ as,
\begin{equation}
    \hat{\delta}^{k-1}_{t} \triangleq-\frac{\vec{u}^{k-1,J}_{t} - \vec{u}^{k-1,0}_{t}}{\eta_{k-1}} = \frac{\sum_{j=0}^{J-1}\eta_{k-1,j} \cdot \nabla_{\vec{u}} g^{k}_{t}\left(\vec{u}_{t}^{k-1,j}, \vec{v}_{t}^{k-1}; \xi_{t}^{k-1,j}\right)}{\sum_{j=0}^{J-1}\eta_{k-1,j}}, 
\end{equation}
and also define 
\begin{equation}
    \delta^{k-1}_{t} \triangleq\frac{\sum_{j=0}^{J-1}\eta_{k-1,j} \cdot \nabla_{\vec{u}} g^{k}_{t}\left(\vec{u}_{t}^{k-1,j}, \vec{v}_{t}^{k-1}\right)}{\eta_{k-1}}.
\end{equation}
With this notation, 
\begin{equation}
    \label{eq:define_delta_u}
    \vec{u}^{k} - \vec{u}^{k-1} = -\eta_{k-1} \cdot \sum_{t=1}^{T} \omega_{t} \cdot \hat{\delta}_{t}^{k-1}.
\end{equation}
Finally, we define $g^{k},~k>0$ as
\begin{equation}
    \label{eq:app_gk_def}
    g^{k}\left(\vec{u},\vec{v}_{1:T}\right) \triangleq \sum_{t=1}^{T}\omega_{t} \cdot g^{k}_{t}\left(\vec{u}, \vec{v}_{t}\right).
\end{equation}
Note that $g^{k}$ is a convex combination of functions $g_{t}^{k},~t\in[T]$.

\subsubsection{Proof of Theorem~\ref{thm:centralized_convergence_bis}}

\begin{lem}
    \label{lem:centralized_case_bound_gradient_g_theta}
    Suppose that Assumptions~\ref{assum:smoothness_bis}--\ref{assum:bounded_dissimilarity_bis} hold. Then, for $k>0$, and $\left(\eta_{k, j}\right)_{0\leq j \leq J-1}$ 
    such that $\eta_k\triangleq \sum_{j=0}^{J-1}\eta_{k, j} \leq\min \left\{ \frac{1}{2\sqrt{2}L}, \frac{1}{4L\beta}\right\}$, the updates of federated surrogate optimization (Alg~\ref{alg:fed_surrogate_opt}) verify
    \begin{flalign}
        \label{eq:centralized_case_bound_gradient_g_theta}
        \mathbb{E} & \Bigg[\frac{f(\vec{u}^{k}, \vec{v}^{k}_{1:T}) - f (\vec{u}^{k-1}, \vec{v}^{k-1}_{1:T})}{\eta_{k-1}}\Bigg] \leq   & \nonumber
        \\
        & \qquad \qquad -\frac{1}{4}\E\left\|\nabla_{\vec{u}} f\left(\vec{u}^{k-1}, \vec{v}_{1:T}^{k-1}\right)\right\|^{2} -\frac{1}{\eta_{k-1}} \sum_{t=1}^{T}\omega_{t}\cdot d_{\mathcal{V}}\left(\vec{v}_{t}^{k-1},  \vec{v}_{t}^{k}\right) \nonumber
        \\
        & \qquad \qquad + 2\eta_{k-1} L\left(  \sum_{j=0}^{J-1} \frac{\eta_{k-1,j}^{2}}{\eta_{k-1}} L +  1\right) \sigma^{2} + 4\eta_{k-1}^{2}L^{2} G^{2}.
    \end{flalign}
\end{lem}

\begin{proof}
    This proof uses standard techniques from distributed stochastic optimization. It is inspired by \cite[Theorem~1]{wang2020tackling}. 
    
    For $k>0$, $g^{k}$ is $L$-smooth wrt $\vec{u}$, because it is a convex combination of $L$-smooth functions $g_{t}^{k},~t\in[T]$. Thus, we write
    \begin{equation}
        g^{k}\left(\vec{u}^{k}, \vec{v}_{1:T}^{k-1}\right) - g^{k}\left(\vec{u}^{k-1}, \vec{v}_{1:T}^{k-1}\right) \leq  \biggl< \vec{u}^{k} - \vec{u}^{k-1}, \nabla_{\vec{u}} g^{k}(\vec{u}^{k-1}, \vec{v}_{1:T}^{k-1})\biggr> + \frac{L}{2}\left\|\vec{u}^{k} - \vec{u}^{k-1}\right\|^{2},
        \label{eq:write_g_is_l_smooth}
    \end{equation}
    where $<\vec{u},\vec{u}'>$ denotes the scalar product of vectors $\vec{u}$ and $\vec{u}'$.
    Using Eq.~\eqref{eq:define_delta_u}, and taking the expectation over random batches $\left(\xi^{k-1,j}_t\right)_{\substack{0\leq j \leq J-1\\1\leq t \leq T}}$, we have
    \begin{align}
        \E\Big[g^{k}\big(\vec{u}^{k} &, \vec{v}_{1:T}^{k-1}\big) - g^{k}\left(\vec{u}^{k-1}, \vec{v}_{1:T}^{k-1}\right)\Big] \leq \nonumber
        \\
        & -\eta_{k-1}  \underbrace{\E \biggl< \sum_{t=1}^{T}\omega_{t} \cdot\hat{\delta}_{t}^{k-1} , \nabla_{\vec{u}} g^{k}(\vec{u}^{k-1}, \vec{v}_{1:T}^{k-1})\biggr>}_{\triangleq T_{1}} 
        + \frac{L\eta_{k-1}^{2}}{2} \cdot  \underbrace{\E\left\|\sum_{t=1}^{T}\omega_{t}\cdot \hat{\delta}_{t}^{k-1}\right\|^{2}}_{\triangleq T_{2}}.
        \label{eq:centralized_case_bound_gradient_g_theta_part_1}
    \end{align}
    We bound each of those terms separately. For $T_{1}$ we have
    \begin{flalign}
        T_{1} &= \mathbb{E}\biggl<\sum_{t=1}^{T}\omega_{t} \cdot \hat{\delta}_{t}^{k-1}, \nabla_{\vec{u}} g^{k}\left(\vec{u}^{k-1}, \vec{v}_{1:T}^{k-1}\right)\biggr> &
        \\
        &= \mathbb{E}\biggl<\sum_{t=1}^{T}\omega_{t}\cdot \left(\hat{\delta}_{t}^{k-1} -\delta_{t}^{k-1} \right), \nabla_{\vec{u}} g^{k}\left(\vec{u}^{k-1}, \vec{v}_{1:T}^{k-1}\right)\biggr> \nonumber
        \\ 
        & \qquad + \mathbb{E}\biggl<\sum_{t=1}^{T}\omega_{t}\cdot \delta_{t}^{k-1}, \nabla_{\vec{u}} g^{k}\left(\vec{u}^{k-1}, \vec{v}_{1:T}^{k-1}\right)\biggr>.
    \end{flalign}
    Because stochastic gradients are unbiased (Assumption~\ref{assum:finite_variance_bis}), we have
    \begin{equation}
        \E\left[\hat{\delta}_{t}^{k-1} - \delta_{t}^{k-1}\right] = 0,
    \end{equation}
    thus,
    \begin{flalign}
        T_1 &= \mathbb{E}\biggl<\sum_{t=1}^{T}\omega_{t}\cdot \delta_{t}^{k-1}, \nabla_{\vec{u}} g^{k}\left(\vec{u}^{k-1}, \vec{v}_{1:T}^{k-1}\right)\biggr> &
        \\
        &= \frac{1}{2}\left(\left\|\nabla_{\vec{u}} g^{k}\left(\vec{u}^{k-1}, \vec{v}_{1:T}^{k-1}\right)\right\|^{2} +  \mathbb{E}\left\|\sum_{t=1}^{T}\omega_{t} \cdot \delta_{t}^{k-1}\right\|^{2} \right) \nonumber
        \\
        & \qquad - \frac{1}{2}\mathbb{E}\left\|\nabla_{\vec{u}} g^{k}\left(\vec{u}^{k-1}, \vec{v}_{1:T}^{k-1}\right)  - \sum_{t=1}^{T}\omega_{t}\cdot \delta_{t}^{k-1}\right\|^{2}.
        \label{eq:bound_t1}
    \end{flalign}
    For $T_{2}$ we have for $k>0$,
    \begin{flalign}
        T_{2} &= \E\left\|\sum_{t=1}^{T}\omega_{t} \cdot \hat{\delta}_{t}^{k-1}\right\|^{2} & 
        \\
        &= \E\left\|\sum_{t=1}^{T}\omega_{t} \cdot \left( \hat{\delta}_{t}^{k-1} - \delta_{t}^{k-1} \right) + \sum_{t=1}^{T}\omega_{t} \cdot \delta_{t}^{k-1}\right\|^{2}
        \\
        &\leq 2  \E\left\|\sum_{t=1}^{T}\omega_{t} \cdot \left( \hat{\delta}_{t}^{k-1} - \delta_{t}^{k-1} \right)\right\|^{2} + 2  \E\left\| \sum_{t=1}^{T}\omega_{t} \cdot \delta_{t}^{k-1}\right\|^{2}
        \\
        &= 2 \sum_{t=1}^{T} \omega_{t}^{2} \cdot \E\left\|\hat{\delta}_{t}^{k-1} - \delta_{t}^{k-1}\right\|^{2} + 2 \sum_{1\leq s \neq t \leq T} \omega_{t} \omega_{s}  \E \biggl< \hat{\delta}_{t}^{k-1} - \delta_{t}^{k-1}, \hat{\delta}_{s}^{k-1} - \delta_{s}^{k-1} \biggr> \nonumber
        \\
        & \qquad + 2\E \left\|\sum_{t=1}^{T} \omega_{t} \delta_{t}^{k-1}\right\|^{2}.
    \end{flalign}
    Since clients sample batches independently, and stochastic gradients are unbiased (Assumption~\ref{assum:finite_variance_bis}), we have
    \begin{equation}
        \E \biggl< \hat{\delta}_{t}^{k-1} - \delta_{t}^{k-1}, \hat{\delta}_{s}^{k-1} - \delta_{s}^{k-1} \biggr> = 0,
    \end{equation}
    thus,
    \begin{flalign}
        T_{2} &\leq 2 \sum_{t=1}^{T} \omega_{t}^{2} \cdot \E\left\|\hat{\delta}_{t}^{k-1} - \delta_{t}^{k-1}\right\|^{2} + 2 \E \left\|\sum_{t=1}^{T} \omega_{t} \delta_{t}^{k-1}\right\|^{2} &
        \\
        &= 2\sum_{t=1}^{T}\omega_{t}^{2} \E\left \| \sum_{j=0}^{J-1} \frac{\eta_{k-1,j}}{\eta_{k-1}} \left[ \nabla_{\vec{u}} g_{t}^{k}\left(\vec{u}_{t}^{k-1, j}, \vec{v}_{t}^{k-1}\right) - \nabla_{\vec{u}} g_{t}^{k}\left(\vec{u}_{t}^{k-1, j}, \vec{v}_{t}^{k-1}; \xi_{t}^{k-1,j}\right) \right] \right\|^{2} \nonumber
        \\
        & \qquad + 2\E\left\|\sum_{t=1}^{T}\omega_{t} \delta_{t}^{k-1}\right\|^{2}.
    \end{flalign}
    Using Jensen inequality, we have
    \begin{flalign}
        \left \| \sum_{j=0}^{J-1} \frac{\eta_{k-1,j}}{\eta_{k-1}} \left[ \nabla_{\vec{u}} g_{t}^{k}\left(\vec{u}_{t}^{k-1, j}, \vec{v}_{t}^{k-1}\right) - \nabla_{\vec{u}} g_{t}^{k}\left(\vec{u}_{t}^{k-1, j}, \vec{v}_{t}^{k-1}; \xi_{t}^{k-1,j}\right) \right] \right\|^{2}  & \leq    & \nonumber
        \\
        \sum_{j=0}^{J-1} \frac{\eta_{k-1,j}}{\eta_{k-1}} \Big\| \nabla_{\vec{u}} g_{t}^{k}\left(\vec{u}_{t}^{k-1, j}, \vec{v}_{t}^{k-1}\right) - \nabla_{\vec{u}} g_{t}^{k}\Big(\vec{u}_{t}^{k-1, j}, \vec{v}_{t}^{k-1} &; \xi_{t}^{k-1,j} \Big) \Big\|^{2},
    \end{flalign}
    and since the variance of stochastic gradients is bounded by $\sigma^{2}$ (Assumption~\ref{assum:finite_variance_bis}), it follows that
    \begin{flalign}
          \E&\left \| \sum_{j=0}^{J-1} \frac{\eta_{k-1,j}}{\eta_{k-1}} \left[ \nabla_{\vec{u}} g_{t}^{k}\left(\vec{u}_{t}^{k-1, j}, \vec{v}_{t}^{k-1}\right)  - \nabla_{\vec{u}} g_{t}^{k}\left(\vec{u}_{t}^{k-1, j}, \vec{v}_{t}^{k-1}; \xi_{t}^{k-1,j}\right) \right] \right\|^{2}   \nonumber &
          \\
          &\quad \leq \sum_{j=0}^{J-1}\frac{\eta_{k-1,j}}{\eta_{k-1}}  \sigma^{2}= \sigma^{2}.
    \end{flalign}
    Replacing back in the expression of $T_{2}$, we have
    \begin{equation}
        T_{2} \leq 2\sum_{t=1}^{T} \omega_{t}^{2} \sigma^{2} + 2 \E\left\|\sum_{t=1}^{T}\omega_{t}\cdot \delta_{t}^{k-1}\right\|^{2}. 
    \end{equation}
    Finally, since $0 \leq \omega_{t} \leq 1,~t\in[T]$ and $\sum_{t=1}^{T}\omega_{t} = 1$, we have
    \begin{equation}
        \label{eq:bound_t2}
        T_{2} \leq 2 \sigma^{2} + 2 \E\left\|\sum_{t=1}^{T}\omega_{t}\cdot \delta_{t}^{k-1}\right\|^{2}. 
    \end{equation}
    Having bounded $T_1$ and $T_2$, we can replace Eq.~\eqref{eq:bound_t1} and Eq.~\eqref{eq:bound_t2} in Eq.~\eqref{eq:centralized_case_bound_gradient_g_theta_part_1}, and we get
    \begin{flalign}
        \mathbb{E}\Big[g^{k}(\vec{u}^{k}, \vec{v}^{k-1}_{1:T}) - g^{k}& (\vec{u}^{k-1}, \vec{v}^{k-1}_{1:T})\Big] \leq  -\frac{\eta_{k-1}}{2}\left\|\nabla_{\vec{u}} g^{k}\left(\vec{u}^{k-1}, \vec{v}_{1:T}^{k-1}\right)\right\|^{2} + \eta_{k-1}^{2} L\sigma^{2} & \nonumber
        \\
        & -\frac{\eta_{k-1}}{2}\left(1-2L\eta_{k-1}\right) \cdot \mathbb{E} \left\|\sum_{t=1}^{T}\omega_t \cdot \delta_{t}^{k-1}\right\|^{2} \nonumber
        \\
        & + \frac{\eta_{k-1}}{2}\mathbb{E}\Big\|\nabla_{\vec{u}} g^{k}\left(\vec{u}^{k-1}, \vec{v}_{1:T}^{k-1}\right)  -  \sum_{t=1}^{T}\omega_{t} \cdot 
        \delta_{t}^{k-1}\Big\|^{2}.
    \end{flalign}
    As $\eta_{k-1} \leq \frac{1}{2\sqrt{2}L} \leq \frac{1}{2L}$, we have
    \begin{flalign}
        \mathbb{E}\Big[g^{k}(\vec{u}^{k}, \vec{v}^{k-1}_{1:T}) - g^{k}& (\vec{u}^{k-1}, \vec{v}^{k-1}_{1:T})\Big] \leq  -\frac{\eta_{k-1}}{2}\left\|\nabla_{\vec{u}} g^{k}\left(\vec{u}^{k-1}, \vec{v}_{1:T}^{k-1}\right)\right\|^{2} + \eta_{k-1}^{2} L\sigma^{2} & \nonumber
        \\
        & + \frac{\eta_{k-1}}{2}\mathbb{E}\Big\|\nabla_{\vec{u}} g^{k}\left(\vec{u}^{k-1}, \vec{v}_{1:T}^{k-1}\right)  -  \sum_{t=1}^{T}\omega_{t}\delta_{t}^{k-1}\Big\|^{2}.
        \label{eq:centralized_case_bound_gradient_g_theta_part_2}
    \end{flalign}
    Replacing  $\nabla_{\vec{u}}g^{k}\left(\vec{u}^{k-1}, \vec{v}_{1:T}^{k-1}\right) = \sum_{t=1}^{T}\omega_{t} \cdot \nabla_{\vec{u}} g_{t}^{k}\left(\vec{u}^{k-1}, \vec{v}_{t}^{k-1}\right)$, and using Jensen inequality to bound the last term in the RHS of Eq.~\eqref{eq:centralized_case_bound_gradient_g_theta_part_2}, we have
    \begin{flalign}
        \mathbb{E}\Big[g^{k}(\vec{u}^{k}, \vec{v}^{k-1}_{1:T}) - g^{k}& (\vec{u}^{k-1}, \vec{v}^{k-1}_{1:T})\Big] \leq  -\frac{\eta_{k-1}}{2}\left\|\nabla_{\vec{u}} g^{k}\left(\vec{u}^{k-1}, \vec{v}_{1:T}^{k-1}\right)\right\|^{2} + \eta_{k-1}^{2} L\sigma^{2}  & \nonumber
         \\
        & + \frac{\eta_{k-1}}{2} \sum_{t=1}^{T} \omega_{t} \cdot \underbrace{\mathbb{E}\Big\|\nabla_{\vec{u}} g_{t}^{k}\left(\vec{u}^{k-1}, \vec{v}_{t}^{k-1}\right)  -  \delta_{t}^{k-1}\Big\|^{2}}_{\triangleq T_{3}}.
        \label{eq:centralized_case_bound_gradient_g_theta_part_3}
    \end{flalign}
    We now bound the term $T_{3}$:
    \begin{flalign}
        T_{3} &=  \mathbb{E}\Big\|\nabla_{\vec{u}} g_{t}^{k}\left(\vec{u}^{k-1}, \vec{v}_{t}^{k-1}\right)  -  \delta_{t}^{k-1}\Big\|^{2} &
        \\
        &= \mathbb{E}\left\|\nabla_{\vec{u}} g_{t}^{k}\left(\vec{u}^{k-1}, \vec{v}_{t}^{k-1}\right) - \sum_{j=0}^{J-1}\frac{\eta_{k-1,j}}{\eta_{k-1}} \nabla_{\vec{u}} g^{k}_{t}\left(\vec{u}^{k-1,j}_{t}, \vec{v}_{t}^{k-1}\right)\right\|^{2}\\
        &= \mathbb{E}\left\|\sum_{j=0}^{J-1}\frac{\eta_{k-1,j}}{\eta_{k-1}} \left[\nabla_{\vec{u}} g_{t}^{k}\left(\vec{u}^{k-1}, \vec{v}_{t}^{k-1}\right)  -  \nabla_{\vec{u}} g^{k}_{t}\left(\vec{u}^{k-1, j}_{t}, \vec{v}_{t}^{k-1}\right)\right]\right\|^{2}\\
        &\leq \sum_{j=0}^{J-1} \frac{\eta_{k-1,j}}{\eta_{k-1}} \mathbb{E} \left\|\nabla_{\vec{u}} g_{t}^{k}\left(\vec{u}^{k-1}, \vec{v}_{t}^{k-1}\right)  -  \nabla_{\vec{u}} g^{k}_{t}\left(\vec{u}^{k-1, j}_{t}, \vec{v}_{t}^{k-1}\right)\right\|^{2}\\
        &\leq \sum_{j=0}^{J-1} \frac{\eta_{k-1,j}}{\eta_{k-1}} L^{2} \mathbb{E} \left\|\vec{u}^{k-1}  -  \vec{u}^{k-1, j}_{t}\right\|^{2}, \label{eq:bound_t3_part1}
    \end{flalign}
    where the first inequality follows from Jensen inequality and the second one follow from the $L$-smoothness of $g_{t}^k$ (Assumption~\ref{assum:smoothness_bis}). We bound now the term $\E\left\|\vec{u}^{k-1}  -  \vec{u}^{k-1, j}_{t}\right\|$ for $j\in\left\{0,\dots, J-1\right\}$ and $t\in[T]$,
    \begin{flalign}
        \E\Big\|\vec{u}^{k-1} & - \vec{u}_{t}^{k-1,j}\Big\|^{2} = \mathbb{E}\left\| \vec{u}_{t}^{k-1,j} - \vec{u}_{t}^{k-1,0} \right\|^{2} & 
        \\
        &= \mathbb{E}\left\|\sum_{l=0}^{j-1}\left( \vec{u}_{t}^{k-1,l+1} - \vec{u}_{t}^{k-1,l}\right)\right\|^{2} 
        \\
        &= \mathbb{E}\left\|\sum_{l=0}^{j-1} \eta_{k-1,l}  \nabla_{\vec{u}} g^{k}_{t}\left(\vec{u}_{t}^{k-1,j}, \vec{v}_{t}^{k-1}; \xi_{t}^{k-1,l}\right)\right\|^{2} 
        \\
        &\leq 2\mathbb{E}\Bigg\|\sum_{l=0}^{j-1} \eta_{k-1,l}   \left[\nabla_{\vec{u}} g^{k}_{t}\left(\vec{u}_{t}^{k-1,l}, \vec{v}_{t}^{k-1}; \xi_{t}^{k-1,l}\right) - \nabla_{\vec{u}} g^{k}_{t}\left(\vec{u}_{t}^{k-1,l}, \vec{v}_{t}^{k-1}\right)\right]\Bigg\|^{2}  \nonumber
        \\
        &\qquad \qquad + 2\mathbb{E}\left\|\sum_{l=0}^{j-1} \eta_{k-1,l} \nabla_{\vec{u}} g^{k}_{t}\left(\vec{u}_{t}^{k-1,l}, \vec{v}_{t}^{k-1}\right)\right\|^{2}
        \\
        &= 2\sum_{l=0}^{j-1} \eta_{k-1, l}^{2}\mathbb{E} \Bigg\|\nabla_{\vec{u}} g^{k}_{t}\left(\vec{u}_{t}^{k-1,l}, \vec{v}_{t}^{k-1}; \xi_{t}^{k-1,l}\right)  
        -\nabla_{\vec{u}} g^{k}_{t}\left(\vec{u}_{t}^{k-1,l}, \vec{v}_{t}^{k-1}\right)\Bigg\|^{2} \nonumber
        \\
        &\qquad \qquad + 2\mathbb{E}\Bigg\|\sum_{l=0}^{j-1} \eta_{k-1,l} \nabla_{\vec{u}} g^{k}_{t}\left(\vec{u}_{t}^{k-1,l}, \vec{v}_{t}^{k-1}\right)\Bigg\|^{2}
        \\
        &\leq 2\sigma^{2}\sum_{l=0}^{j-1} \eta_{k-1, l}^{2} + 2\mathbb{E}\left\|\sum_{l=0}^{j-1} \eta_{k-1,l} \nabla_{\vec{u}} g^{k}_{t}\left(\vec{u}_{t}^{k-1,l}, \vec{v}_{t}^{k-1}\right)\right\|^{2}, \label{eq:bound_theta_j_minus_theta_0_part1}
    \end{flalign}
    where, in the last two steps, we used the fact that stochastic gradients are unbiased and have bounded variance (Assumption~\ref{assum:finite_variance_bis}).
    We bound now the last term in the RHS of Eq.~\eqref{eq:bound_theta_j_minus_theta_0_part1},
    \begin{flalign}
        \mathbb{E}\Bigg\| & \sum_{l=0}^{j-1}  \eta_{k-1,l}   \nabla_{\vec{u}} g^{k}_{t}  \left(\vec{u}_{t}^{k-1,l}, \vec{v}_{t}^{k-1}\right)\Bigg\|^{2}  = & \nonumber
        \\
        & \mathbb{E}\Bigg\|\left(\sum_{l'=0}^{j-1}\eta_{k-1,l'}\right) \cdot \sum_{l=0}^{j-1} \frac{\eta_{k-1,l}}{\sum_{l'=0}^{j-1}\eta_{k-1,l'}}   \nabla_{\vec{u}} g^{k}_{t}\left(\vec{u}_{t}^{k-1,l}, \vec{v}_{t}^{k-1}\right)\Bigg\|^{2} 
        \\
         \leq & \left(\sum_{l'=0}^{j-1}\eta_{k-1,l'}\right)^2 \cdot \sum_{l=0}^{j-1}\frac{\eta_{k-1,l}}{\sum_{l'=0}^{j-1}\eta_{k-1,l'}} \mathbb{E} \left\| \nabla_{\vec{u}} g^{k}_{t}\left(\vec{u}_{t}^{k-1,l}, \vec{v}_{t}^{k-1}\right)\right\|^{2}
         \\
         = & \left(\sum_{l=0}^{j-1}\eta_{k-1,l}\right)\cdot \sum_{l=0}^{j-1}\eta_{k-1,l}\mathbb{E}\left\|\nabla_{\vec{u}} g^{k}_{t}\left(\vec{u}_{t}^{k-1,l}, \vec{v}_{t}^{k-1}\right)\right\|^{2} 
         \\
        =& \left(\sum_{l=0}^{j-1}\eta_{k-1,l}\right)\cdot \sum_{l=0}^{j-1}\eta_{k-1,l}\mathbb{E}\Big\|\nabla_{\vec{u}} g^{k}_{t}\left(\vec{u}_{t}^{k-1,0}, \vec{v}_{t}^{k-1}\right) \nonumber
        \\
        &\qquad    - \nabla_{\vec{u}} g^{k}_{t}\left(\vec{u}_{t}^{k-1,0}, \vec{v}_{t}^{k-1}\right) + \nabla_{\vec{u}} g^{k}_{t}\left(\vec{u}_{t}^{k-1,l}, \vec{v}_{t}^{k-1}\right)\Big\|^{2} 
        \\
        \le & 2 \left(\sum_{l=0}^{j-1}\eta_{k-1,l}\right)\cdot \sum_{l=0}^{j-1}\eta_{k-1,l}\cdot \Bigg[ \mathbb{E}\left\|\nabla_{\vec{u}} g^{k}_{t}\left(\vec{u}_{t}^{k-1,0}, \vec{v}_{t}^{k-1}\right)\right\|^{2} \nonumber
        \\
        &\qquad   + \mathbb{E} \left\|\nabla_{\vec{u}} g^{k}_{t}\left(\vec{u}_{t}^{k-1,l}, \vec{v}_{t}^{k-1}\right) - \nabla_{\vec{u}} g^{k}_{t}\left(\vec{u}_{t}^{k-1,0}, \vec{v}_{t}^{k-1}\right)\right\|^{2} \Bigg ]
        \\
        =& 2 \left(\sum_{l=0}^{j-1}\eta_{k-1,l}\right)\cdot \sum_{l=0}^{j-1}\eta_{k-1,l}\cdot \Bigg[ \mathbb{E}\left\|\nabla_{\vec{u}} g^{k}_{t}\left(\vec{u}^{k-1}, \vec{v}_{t}^{k-1}\right)\right\|^{2} \nonumber
        \\
        &\qquad   + \mathbb{E} \left\|\nabla_{\vec{u}} g^{k}_{t}\left(\vec{u}_{t}^{k-1,l}, \vec{v}_{t}^{k-1}\right) - \nabla_{\vec{u}} g^{k}_{t}\left(\vec{u}^{k-1}, \vec{v}_{t}^{k-1}\right)\right\|^{2} \Bigg ]
        \\
        \leq & 2 \left(\sum_{l=0}^{j-1}\eta_{k-1,l}\right) \sum_{l=0}^{j-1}\eta_{k-1,l} \Bigg[ \mathbb{E}\left\|\nabla_{\vec{u}} g^{k}_{t}\left(\vec{u}^{k-1}, \vec{v}_{t}^{k-1}\right)\right\|^{2} + L^{2}   \mathbb{E} \left\| \vec{u}_{t}^{k-1,l} -  \vec{u}^{k-1} \right\|^{2} \Bigg ] 
        \\
        = & 2 L^{2} \left(\sum_{l=0}^{j-1}\eta_{k-1,l}\right) \sum_{l=0}^{j-1}\eta_{k-1,l} \cdot \mathbb{E} \left\| \vec{u}_{t}^{k-1,l} -  \vec{u}^{k-1} \right\|^{2} \nonumber
        \\
        & \qquad +  2 \left(\sum_{l=0}^{j-1}\eta_{k-1,l}\right)^{2} \mathbb{E}\left\|\nabla_{\vec{u}} g^{k}_{t}\left(\vec{u}^{k-1}, \vec{v}_{t}^{k-1}\right)\right\|^{2},
        \label{eq:bound_theta_j_minus_theta_0_part2}
    \end{flalign}
    where the first inequality is obtained using Jensen inequality, and the last one is a result of the $L$-smoothness of $g_{t}$ (Assumption~\ref{assum:smoothness_bis}). Replacing Eq.~\eqref{eq:bound_theta_j_minus_theta_0_part2} in  Eq.~\eqref{eq:bound_theta_j_minus_theta_0_part1}, we have
    \begin{flalign}
        \sum_{j=0}^{J-1}\frac{\eta_{k-1,j}}{\eta_{k-1}} & \cdot \mathbb{E}\Big\|\vec{u}^{k-1} - \vec{u}_{t}^{k-1,j}\Big\|^{2} \leq  2\sigma^{2}  \left(\sum_{j=0}^{J-1} \frac{\eta_{k-1,j}}{\eta_{k-1}} \cdot \sum_{l=0}^{j-1}\eta_{k-1,l}^{2}\right) \nonumber & 
        \\
        &\qquad + 4L^{2} \sum_{j=0}^{J-1}\left(\frac{\eta_{k-1,j}}{\eta_{k-1}} \sum_{l=0}^{j-1}\eta_{k-1,l}  \right) \cdot \left(\sum_{l=0}^{j-1}\eta_{k-1,l} \cdot \mathbb{E} \left\| \vec{u}_{t}^{k-1,l} -  \vec{u}_{t}^{k-1} \right\|^{2} \right)  \nonumber
        \\
        & \qquad + 4 \left(\sum_{j=0}^{J-1}\frac{\eta_{k-1,j}}{\eta_{k-1}}\left(\sum_{l=0}^{j-1}\eta_{k-1,l}\right)^{2} \right) \cdot  \mathbb{E}\left\|\nabla_{\vec{u}} g^{k}_{t}\left(\vec{u}_{t}^{k-1}, \vec{v}_{t}^{k-1}\right)\right\|^{2}.
    \end{flalign}
    Since $\sum_{l=0}^{j-1}\eta_{k-1,l} \cdot \mathbb{E} \left\| \vec{u}_{t}^{k-1,l} -  \vec{u}_{t}^{k-1} \right\|^{2} \leq \sum_{j=0}^{J-1}\eta_{k-1,j} \cdot \mathbb{E} \left\| \vec{u}_{t}^{k-1,j} -  \vec{u}_{t}^{k-1} \right\|^{2}$, we have
    \begin{flalign}
        \sum_{j=0}^{J-1}\frac{\eta_{k-1,j}}{\eta_{k-1}} & \cdot \mathbb{E}\Big\|\vec{u}^{k-1} - \vec{u}_{t}^{k-1,j}\Big\|^{2} \leq  2\sigma^{2}  \left(\sum_{j=0}^{J-1} \frac{\eta_{k-1,j}}{\eta_{k-1}} \cdot \sum_{l=0}^{j-1}\eta_{k-1,l}^{2}\right) \nonumber & 
        \\
        &\qquad + 4L^{2} \left( \sum_{j=0}^{J-1}\frac{\eta_{k-1,j}}{\eta_{k-1}} \sum_{l=0}^{j-1}\eta_{k-1,l}  \right) \cdot \left(\sum_{j=0}^{J-1}\eta_{k-1,j} \cdot \mathbb{E} \left\| \vec{u}_{t}^{k-1,j} -  \vec{u}^{k-1} \right\|^{2} \right) \nonumber
        \\
        & \qquad + 4 \left(\sum_{j=0}^{J-1}\frac{\eta_{k-1,j}}{\eta_{k-1}}\left(\sum_{l=0}^{j-1}\eta_{k-1,l}\right)^{2} \right) \cdot  \mathbb{E}\left\|\nabla_{\vec{u}} g^{k}_{t}\left(\vec{u}^{k-1}, \vec{v}_{t}^{k-1}\right)\right\|^{2}  .
    \end{flalign}
    We use Lemma~\ref{lem:bound_sums_of_rates} to simplify the last expression, obtaining
    \begin{flalign}
        \sum_{j=0}^{J-1} & \frac{\eta_{k-1,j}}{\eta_{k-1}}  \cdot \mathbb{E}\Big\|\vec{u}^{k-1} - \vec{u}_{t}^{k-1,j}\Big\|^{2} \leq    2\sigma^{2} \cdot \left\{ \sum_{j=0}^{J-1} \eta_{k-1,j}^{2}\right\} & \nonumber \\ 
        & + 4\eta_{k-1}^{2}  \mathbb{E}\left\|\nabla_{\vec{u}} g^{k}_{t}\left(\vec{u}^{k-1}, \vec{v}_{t}^{k-1}\right)\right\|^{2}  + 4\eta_{k-1} L^{2} \cdot \sum_{j=0}^{J-1}\eta_{k-1,j} \mathbb{E} \left\| \vec{u}_{t}^{k-1,j} -  \vec{u}^{k-1} \right\|^{2}.
    \end{flalign}
    Rearranging the terms, we have
    \begin{flalign}
        \left(1 - 4\eta_{k-1}^{2}L^{2}\right) & \cdot \sum_{j=0}^{J-1}  \frac{\eta_{k-1,j}}{\eta_{k-1}}  \cdot \mathbb{E}\Big\|\vec{u}^{k-1} - \vec{u}_{t}^{k-1,j}\Big\|^{2} \leq   2\sigma^{2} \cdot \left\{ \sum_{j=0}^{J-1} \eta_{k-1,j}^{2}\right\} \nonumber &
        \\
        & \qquad + 4\eta_{k-1}^{2} \cdot \mathbb{E}\left\|\nabla_{\vec{u}} g^{k}_{t}\left(\vec{u}^{k-1}, \vec{v}_{t}^{k-1}\right)\right\|^{2}. 
        \label{eq:bound_theta_j_minus_theta_0_part3}
    \end{flalign}
    Finally, replacing Eq.~\eqref{eq:bound_theta_j_minus_theta_0_part3} into Eq.~\eqref{eq:bound_t3_part1}, we have
    \begin{equation}
        \left(1 - 4\eta_{k-1}^{2}L^{2}\right) \cdot T_{3} \leq 2\sigma^{2}L^{2} \cdot \left( \sum_{j=0}^{J-1} \eta_{k-1,j}^{2}\right) +  4\eta_{k-1}^{2} L^{2} \cdot \mathbb{E}\left\|\nabla_{\vec{u}} g^{k}_{t}\left(\vec{u}^{k-1}, \vec{v}_{t}^{k-1}\right)\right\|^{2}. 
    \end{equation}
    For $\eta_{k-1}$ small enough, in particular if $\eta_{k-1} \leq \frac{1}{2\sqrt{2}L}$, then $\frac{1}{2} \leq 1 - 4\eta_{k-1}^{2}L^{2}$, thus
    \begin{equation}
        \label{eq:bound_t3_part2}
        \frac{T_{3}}{2} \leq 2\sigma^{2}L^{2} \cdot \left( \sum_{j=0}^{J-1} \eta_{k-1,j}^{2}\right) +  4\eta_{k-1}^{2} L^{2} \cdot \mathbb{E}\left\|\nabla_{\vec{u}} g^{k}_{t}\left(\vec{u}^{k-1}, \vec{v}_{t}^{k-1}\right)\right\|^{2}. 
    \end{equation}
    Replacing the bound of $T_{3}$ from Eq.~\eqref{eq:bound_t3_part2} into Eq.~\eqref{eq:centralized_case_bound_gradient_g_theta_part_3}, we have obtained
    \begin{flalign}
        \mathbb{E}\Big[g^{k}(\vec{u}^{k}, \vec{v}^{k-1}_{1:T}) - g^{k}& (\vec{u}^{k-1}, \vec{v}^{k-1}_{1:T})\Big] \leq  -\frac{\eta_{k-1}}{2}\E\left\|\nabla_{\vec{u}} g^{k}\left(\vec{u}^{k-1}, \vec{v}_{1:T}^{k-1}\right)\right\|^{2}  & \nonumber
         \\
        & + 4\eta_{k-1}^{3} L^{2} \sum_{t=1}^{T} \omega_{t} \cdot  \mathbb{E}\left\|\nabla_{\vec{u}} g^{k}_{t}\left(\vec{u}^{k-1}, \vec{v}_{t}^{k-1}\right)\right\|^{2} \nonumber \\
        & + 2\eta_{k-1} L\left( \sum_{j=0}^{J-1} \eta_{k-1,j}^{2} L +  \eta_{k-1}\right)\cdot \sigma^{2}.
    \end{flalign}
    Using Assumption~\ref{assum:bounded_dissimilarity_bis}, we have
    \begin{flalign}
        \mathbb{E}\Big[g^{k}(\vec{u}^{k}, \vec{v}^{k-1}_{1:T}) - g^{k}& (\vec{u}^{k-1}, \vec{v}^{k-1}_{1:T})\Big] \leq  -\frac{\eta_{k-1}}{2}\E\left\|\nabla_{\vec{u}} g^{k}\left(\vec{u}^{k-1}, \vec{v}_{1:T}^{k-1}\right)\right\|^{2}  & \nonumber
         \\
        & + 4\eta_{k-1}^{3} L^{2}\beta^{2}  \cdot \mathbb{E}\left\| \sum_{t=1}^{T} \omega_{t} \cdot \nabla_{\vec{u}} g^{k}_{t}\left(\vec{u}^{k-1}, \vec{v}_{t}^{k-1}\right)\right\|^{2} \nonumber
        \\
        & + 2\eta_{k-1} L\left(  \sum_{j=0}^{J-1} \eta_{k-1,j}^{2} L +  \eta_{k-1}\right)\cdot \sigma^{2} + 4\eta_{k-1}^{3}L^{2} G^{2}.
    \end{flalign}
    Dividing by $\eta_{k-1}$, we get
    \begin{flalign}
        \mathbb{E} & \Big[\frac{g^{k}(\vec{u}^{k}, \vec{v}^{k-1}_{1:T}) - g^{k} (\vec{u}^{k-1}, \vec{v}^{k-1}_{1:T})}{\eta_{k-1}}\Big] \leq  \frac{8\eta_{k-1}^{2}L^{2} \beta^{2} - 1}{2}\E\left\|\nabla_{\vec{u}} g^{k}\left(\vec{u}^{k-1}, \vec{v}_{1:T}^{k-1}\right)\right\|^{2}  & \nonumber
        \\
        & \qquad \qquad + 2\eta_{k-1} L\left(  \sum_{j=0}^{J-1} \frac{\eta_{k-1,j}^{2}}{\eta_{k-1}} L +  1\right)\cdot \sigma^{2} + 4\eta_{k-1}^{2}L^{2} G^{2}. 
    \end{flalign}
    For $\eta_{k-1}$ small enough, if $\eta_{k-1} \leq \frac{1}{4L\beta}$, then $8\eta_{k-1}^{2}L^{2}\beta^{2} - 1 \leq \frac{1}{2}$. Thus,
    \begin{flalign}
        \mathbb{E} & \Big[\frac{g^{k}(\vec{u}^{k}, \vec{v}^{k-1}_{1:T}) - g^{k} (\vec{u}^{k-1}, \vec{v}^{k-1}_{1:T})}{\eta_{k-1}}\Big] \leq  -\frac{1}{4}\E\left\|\nabla_{\vec{u}} g^{k}\left(\vec{u}^{k-1}, \vec{v}_{1:T}^{k-1}\right)\right\|^{2}  & \nonumber
        \\
        & \qquad \qquad + 2\eta_{k-1} L\left(  \sum_{j=0}^{J-1} \frac{\eta_{k-1,j}^{2}}{\eta_{k-1}} L +  1\right)\cdot \sigma^{2} + 4\eta_{k-1}^{2}L^{2} G^{2}.
        \label{eq:centralized_case_bound_gradient_g_theta_part_4}
    \end{flalign}
    Since for $t\in[T]$, $g_{t}^{k}$ is a partial first-order surrogate of $f_{t}$ near $\left\{\vec{u}^{k-1}, \vec{v}_{t}^{k-1}\right\}$, we have (see Def.~\ref{def:partial_first_order_surrogate})
    \begin{flalign}
        g_{t}^{k}\left(\vec{u}^{k-1}, \vec{v}_{t}^{k-1}\right) &= f_{t}\left(\vec{u}^{k-1}, \vec{v}_{t}^{k-1}\right),\\
        \nabla_{\vec{u}} g_{t}^{k}\left(\vec{u}^{k-1}, \vec{v}_{t}^{k-1}\right) &= \nabla_{\vec{u}} f_{t}\left(\vec{u}^{k-1}, \vec{v}_{t}^{k-1}\right), \\
        g_{t}^{k}\left(\vec{u}^{k}, \vec{v}_{t}^{k-1}\right) &= g_{t}^{k}\left(\vec{u}^{k}, \vec{v}_{t}^{k}\right) + d_{\mathcal{V}}\left( \vec{v}_{t}^{k-1},  \vec{v}_{t}^{k}\right).
    \end{flalign}
    Multiplying by $\omega_{t}$ and summing over $t\in[T]$, we have
    \begin{flalign}
        g^{k}\left(\vec{u}^{k-1}, \vec{v}_{1:T}^{k-1}\right) &= f\left(\vec{u}^{k-1}, \vec{v}_{1:T}^{k-1}\right), \label{eq:surrogate_1}\\
        \nabla_{\vec{u}} g^{k}\left(\vec{u}^{k-1}, \vec{v}_{1:T}^{k-1}\right) &= \nabla_{\vec{u}} f\left(\vec{u}^{k-1}, \vec{v}_{1:T}^{k-1}\right), \label{eq:surrogate_2}\\
        g^{k}\left(\vec{u}^{k}, \vec{v}_{1:T}^{k-1}\right) &= g^{k}\left(\vec{u}^{k}, \vec{v}_{1:T}^{k}\right) +\sum_{t=1}^{T}\omega_{t}\cdot d_{\mathcal{V}}\left(\vec{v}_{t}^{k-1},  \vec{v}_{t}^{k}\right) \label{eq:surrogate_3}.
    \end{flalign} 
    Replacing Eq.~\eqref{eq:surrogate_1}, Eq.~\eqref{eq:surrogate_2} and Eq.~\eqref{eq:surrogate_3} in Eq.~\eqref{eq:centralized_case_bound_gradient_g_theta_part_4}, we have
    \begin{flalign}
        \mathbb{E} & \Bigg[\frac{g^{k}(\vec{u}^{k}, \vec{v}^{k}_{1:T}) - f (\vec{u}^{k-1}, \vec{v}^{k-1}_{1:T})}{\eta_{k-1}}\Bigg] \leq   & \nonumber
        \\
        & \qquad \qquad -\frac{1}{4}\E\left\|\nabla_{\vec{u}} f\left(\vec{u}^{k-1}, \vec{v}_{1:T}^{k-1}\right)\right\|^{2} -\frac{1}{\eta_{k-1}} \sum_{t=1}^{T}\omega_{t}\cdot d_{\mathcal{V}}\left(\vec{v}_{t}^{k-1},  \vec{v}_{t}^{k}\right) \nonumber
        \\
        & \qquad \qquad + 2\eta_{k-1} L\left( \left\{ \sum_{j=0}^{J-1} \frac{\eta_{k-1,j}^{2}}{\eta_{k-1}}\right\} L +  1\right)\cdot \sigma^{2} + 4\eta_{k-1}^{2}L^{2} G^{2}.
    \end{flalign}
    Using again Definition~\ref{def:partial_first_order_surrogate}, we have
    \begin{equation}
        g^{k}(\vec{u}^{k}, \vec{v}^{k}_{1:T}) \geq f(\vec{u}^{k}, \vec{v}^{k}_{1:T}),
    \end{equation}
    thus,
    \begin{flalign}
        \mathbb{E} & \Bigg[\frac{f(\vec{u}^{k}, \vec{v}^{k}_{1:T}) - f (\vec{u}^{k-1}, \vec{v}^{k-1}_{1:T})}{\eta_{k-1}}\Bigg] \leq   & \nonumber
        \\
        & \qquad \qquad -\frac{1}{4}\E\left\|\nabla_{\vec{u}} f\left(\vec{u}^{k-1}, \vec{v}_{1:T}^{k-1}\right)\right\|^{2} -\frac{1}{\eta_{k-1}} \sum_{t=1}^{T}\omega_{t}\cdot d_{\mathcal{V}}\left(\vec{v}_{t}^{k-1},  \vec{v}_{t}^{k}\right) \nonumber
        \\
        & \qquad \qquad + 2\eta_{k-1} L\left(  \sum_{j=0}^{J-1} \frac{\eta_{k-1,j}^{2}}{\eta_{k-1}} L +  1\right)\cdot \sigma^{2} + 4\eta_{k-1}^{2}L^{2} G^{2}.
    \end{flalign}
\end{proof}

\begin{lem}
    \label{lem:improvement_f_pi}
    For $k\geq 0$ and $t\in[T]$, the iterates of Alg.~\ref{alg:fed_surrogate_opt} verify
    \begin{equation}
        0 \leq d_{\mathcal{V}}\left(\vec{v}_{t}^{k+1}, \vec{v}^{k}_{t}\right) \leq f_{t}\left(\vec{u}^{k}, \vec{v}^{k}_{t}\right) - f_{t}(\vec{u}^{k}, \vec{v}^{k+1}_{t}) 
    \end{equation}
\end{lem}
\begin{proof}
    Since $\vec{v}_{t}^{k+1} \in \argmin_{v\in V} g^{k}_{t}\left(\vec{u}^{k-1}, v \right)$, and $g^k_{t}$ is a partial first-order surrogate of $f_{t}$ near $\{\vec{u}^{k-1}, \vec{v}^{k-1}_{t}\}$, we have
    \begin{equation}
        g_{t}^{k}\left(\vec{u}^{k-1}, \vec{v}_{t}^{k-1}\right) - g_{t}^{k}\left(\vec{u}^{k-1}, \vec{v}_{t}^{k}\right) = d_{\mathcal{V}}\left(\vec{v}_{t}^{k-1}, \vec{v}_{t}^{k}\right),
    \end{equation}
    thus, 
    \begin{equation}
        f_{t}\left(\vec{u}^{k-1}, \vec{v}_{t}^{k-1}\right) - f_{t}\left(\vec{u}^{k-1}, \vec{v}_{t}^{k}\right) \geq  d_{\mathcal{V}}\left(\vec{v}_{t}^{k-1}, \vec{v}_{t}^{k}\right),
    \end{equation}
    where we used the fact that
    \begin{equation}
        g_{t}^{k}\left(\vec{u}^{k-1}, \vec{v}_{t}^{k-1}\right) = f_{t}\left(\vec{u}^{k-1}, \vec{v}_{t}^{k-1}\right),
    \end{equation}
    and,
    \begin{equation}
        g_{t}^{k}\left(\vec{u}^{k-1}, \vec{v}_{t}^{k}\right) \geq f_{t}\left(\vec{u}^{k-1}, \vec{v}_{t}^{k}\right).
    \end{equation}
\end{proof}

\begin{repthm}{thm:centralized_convergence_bis}
    Under Assumptions~\ref{assum:bounded_f_bis}--\ref{assum:bounded_dissimilarity_bis}, when clients use SGD as local solver with learning rate $\eta =\frac{a_{0}}{\sqrt{K}}$, after a large enough number of communication rounds $K$, the iterates of federated surrogate optimization (Alg.~\ref{alg:fed_surrogate_opt}) satisfy:
    \begin{equation}
        \frac{1}{K} \sum_{k=1}^K \mathbb{E}\left\|\nabla_{\vec{u}} f\left(\vec{u}^{k}, \vec{v}_{1:T}^{k}\right)\right\|^{2}_{F} \leq \mathcal{O}\!\left(\frac{1}{\sqrt{K}}\right),    \qquad
    \frac{1}{K} \sum_{k=1}^K \E\left[\Delta_{\vec{v}} f(\vec{u}^{k},\vec{v}^{k}_{1:T})\right]
    \leq \mathcal{O}\!\left(\frac{1}{K^{3/4}}\right),  \tag{\ref{eq:theta_pi_convergence_bis}}  
    \end{equation}
    where the expectation is over the random batches samples, and $\Delta_{\vec{v}} f(\vec{u}^{k},\vec{v}_{1:T}^{k}) \triangleq f\left(\vec{u}^{k}, \vec{v}_{1:T}^{k}\right) - f\left(\vec{u}^{k}, \vec{v}_{1:T}^{k+1}\right) \ge 0 $.
\end{repthm}

\begin{proof}
    For $K$ large enough, $\eta = \frac{a_{0}}{\sqrt{K}} \leq \frac{1}{J} \min\left\{\frac{1}{2\sqrt{2}L}, \frac{1}{4L\beta}\right\}$, thus the assumptions of  Lemma~\ref{lem:centralized_case_bound_gradient_g_theta} are satisfied. Lemma~\ref{lem:centralized_case_bound_gradient_g_theta} and non-negativity of 
    $d_{\mathcal{V}}$ lead to 
    \begin{flalign}
        \mathbb{E} & \Big[\frac{f(\vec{u}^{k}, \vec{v}^{k}_{1:T}) - f (\vec{u}^{k-1}, \vec{v}^{k-1}_{1:T})}{J\eta}\Big] \leq  -\frac{1}{4}\E\left\|\nabla_{\vec{u}} f\left(\vec{u}^{k-1}, \vec{v}_{1:T}^{k-1}\right)\right\|^{2}  & \nonumber
        \\
        & \qquad \qquad + 2\eta L\left( \eta L +  1\right)\cdot \sigma^{2} + 4J^{2}\eta^{2}L^{2} G^{2}.
    \end{flalign}
    Rearranging the terms and summing for $k\in[K]$, we have
    \begin{flalign}
        \frac{1}{K} \sum_{k=1}^{K} \E &\left\|\nabla_{\vec{u}} f\left(\vec{u}^{k-1}, \vec{v}_{1:T}^{k-1}\right)\right\|^{2} \nonumber &
        \\
        & \leq 4\mathbb{E}  \Big[\frac{f(\vec{u}^{0}, \vec{v}^{0}_{1:T}) - f (\vec{u}^{K}, \vec{v}^{K}_{1:T})}{J\eta K}\Big]  + 8 \frac{\eta L\left( \eta L +  1\right)\cdot \sigma^{2} + 2J^{2}\eta^{2}L^{2} G^{2} }{K}
        \\
        & \leq 4\mathbb{E}  \Big[\frac{f(\vec{u}^{0}, \vec{v}^{0}_{1:T}) - f^* }{J\eta K}\Big]  + 8 \frac{\eta L\left( \eta L +  1\right)\cdot \sigma^{2} + 2J^{2}\eta^{2}L^{2} G^{2} }{K},\label{eq:convergence_of_nabla_u_centralized_part1}
    \end{flalign}
    where we use Assumption~\ref{assum:bounded_f_bis} to obtain \eqref{eq:convergence_of_nabla_u_centralized_part1}. Thus, 
    \begin{equation}
        \label{eq:convergence_of_nabla_u_centralized}
        \frac{1}{K} \sum_{k=1}^{K} \E \left\|\nabla_{\vec{u}} f\left(\vec{u}^{k-1}, \vec{v}_{1:T}^{k-1}\right)\right\|^{2} = \mathcal{O}\left(\frac{1}{\sqrt{K}}\right).
    \end{equation}
    To prove the second part of Eq.~\eqref{eq:theta_pi_convergence_bis},
    we first decompose $\Delta_{\vec{v}} \triangleq  f\left(\vec{u}^{k}, \vec{v}_{1:T}^{k}\right) - f\left(\vec{u}^{k}, \vec{v}_{1:T}^{k+1}\right) \ge 0$ as follow,
    \begin{equation}
        \Delta_{\vec{v}} =  \underbrace{f\left(\vec{u}^{k}, \vec{v}_{1:T}^{k}\right) - f\left(\vec{u}^{k+1}, \vec{v}_{1:T}^{k+1}\right)}_{\triangleq T^{k}_{1}} + \underbrace{f\left(\vec{u}^{k+1}, \vec{v}_{1:T}^{k+1}\right) - f\left(\vec{u}^{k}, \vec{v}_{1:T}^{k+1}\right)}_{\triangleq T^{k}_{2}}.
    \end{equation}
    Using again Lemma~\ref{lem:centralized_case_bound_gradient_g_theta} and Eq.~\eqref{eq:convergence_of_nabla_u_centralized}, it follows that 
    \begin{equation}
        \label{eq:t_1_k_v_update_centralized}
        \frac{1}{K}\sum_{k=1}^{K} \E\left[T^{k}_{1}\right] 
        \leq \mathcal{O}\left(\frac{1}{K}\right).
    \end{equation}
    For $T^{k}_{2}$, we use the fact that $f$ is $2L$-smooth (Lemma~\ref{lem:f_2l_smooth}) w.r.t. ${u}$ and  Cauchy-Schwartz inequality. Thus, for $k>0$, we write
    \begin{flalign}
        \label{eq:t_2_k_v_update_centralized_part1}
        T^{k}_{2} &= f\left(\vec{u}^{k+1}, \vec{v}_{1:T}^{k+1}\right) - f\left(\vec{u}^{k}, \vec{v}_{1:T}^{k+1}\right) & 
        \\
        & \leq \left\|\nabla_{\vec{u}}f\left(\vec{u}^{k+1}, \vec{v}_{1:T}^{k+1}\right)\right\|\cdot  \left\|\vec{u}^{k+1} - \vec{u}^{k}\right\| + 2L^{2} \left\|\vec{u}^{k+1} - \vec{u}^{k}\right\|^{2}.
    \end{flalign}
    Summing over $k$ and taking expectation:
    \begin{flalign}
        \frac{1}{K}\sum_{k=1}^{K}\E\left[T^{k}_{2}\right] & \leq \frac{1}{K} \sum_{k=1}^{K} \E\left[\left\|\nabla_{\vec{u}}f\left(\vec{u}^{k+1}, \vec{v}_{1:T}^{k+1}\right)\right\|\cdot  \left\|\vec{u}^{k+1} - \vec{u}^{k}\right\| \right]\nonumber\\
        & \quad + \frac{1}{K} \sum_{k=1}^{K} 2L^{2} \E\left[  \left\|\vec{u}^{k+1} - \vec{u}^{k}\right\|^{2}\right]\\
        & \leq \frac{1}{K} \sqrt{\sum_{k=1}^{K} \E\left[\left\|\nabla_{\vec{u}}f\left(\vec{u}^{k+1}, \vec{v}_{1:T}^{k+1}\right)\right\|^2\right]}  
        \sqrt{\sum_{k=1}^{K} \E\left[\left\|\vec{u}^{k+1} - \vec{u}^{k}\right\|^2 \right]}\nonumber\\
        & \quad + \frac{1}{K} \sum_{k=1}^{K} 2L^{2} \E\left[  \left\|\vec{u}^{k+1} - \vec{u}^{k}\right\|^{2}\right],
        \label{eq:t_2_k_v_update_centralized_part1_bis}
    \end{flalign}
    where the second inequality follows from Cauchy-Schwarz inequality.
    From Eq.~\eqref{eq:bound_theta_j_minus_theta_0_part3}, with $\eta_{k-1}= J\eta$, we have for $t\in[T]$
    \begin{flalign}
        \mathbb{E}\Big\|\vec{u}^{k} - \vec{u}_{t}^{k-1, J}\Big\|^{2} \leq   4\sigma^{2}J\eta^{2}  + 8 J^{3}\eta^{2} \cdot \mathbb{E}\left\|\nabla_{\vec{u}} g^{k}_{t}\left(\vec{u}^{k-1}, \vec{v}_{t}^{k-1}\right)\right\|^{2}. &
    \end{flalign}
    Multiplying the previous by $\omega_{t}$ and summing for $t\in[T]$, we have
    \begin{flalign}
        \sum_{t=1}^{T} \omega_{t} \cdot \mathbb{E}\Big\|\vec{u}^{k-1} - \vec{u}_{t}^{k-1,J}\Big\|^{2} \leq   4J^{2}\sigma^{2}\eta^{2} + 8J^{3} \eta^{2} \cdot \sum_{t=1}^{T}\omega_{t}\mathbb{E}\left\|\nabla_{\vec{u}} g^{k}_{t}\left(\vec{u}^{k-1}, \vec{v}_{t}^{k-1}\right)\right\|^{2}. & 
    \end{flalign}
    Using Assumption~\ref{assum:bounded_dissimilarity_bis}, it follows that
    \begin{equation}
        \sum_{t=1}^{T} \omega_{t}  \mathbb{E}\Big\|\vec{u}^{k-1} - \vec{u}_{t}^{k-1,J}\Big\|^{2} \leq   4J^{2}\eta^{2}  \left( 2JG^{2} + \sigma^{2}\right) + 8J^{3} \eta^{2}\beta^{2}  \mathbb{E}\left\|\sum_{t=1}^{T}\omega_{t} \nabla_{\vec{u}}  g^{k}_{t}\left(\vec{u}^{k-1}, \vec{v}_{t}^{k-1}\right)\right\|^{2}. 
    \end{equation}
    Finally using Jensen inequality and the fact that $g^{k}_{t}$ is a partial first-order of $f_{t}$ near $\left\{{u}^{k-1}, {v}^{k-1}_{t}\right\}$, we have
    \begin{equation}
          \mathbb{E}\Big\|\vec{u}^{k-1} - \vec{u}^{k}\Big\|^{2} \leq   4J^{2}\eta^{2}  \left( 2JG^{2} + \sigma^{2}\right) + 8J^{3} \eta^{2}\beta^{2}  \mathbb{E}\left\| \nabla_{\vec{u}}  f\left(\vec{u}^{k-1}, \vec{v}_{1:T}^{k-1}\right)\right\|^{2}.
    \end{equation}
    From Eq.~\eqref{eq:convergence_of_nabla_u_centralized} and $\eta \le \mathcal{O} (1/\sqrt{K})$, we obtain
    \begin{equation}
        \frac{1}{K} \sum_{k=1}^{K} \E \left\|\vec{u}^{k-1} - u^k\right\|^{2} \le \mathcal{O}\left(1\right),
    \end{equation}
    Replacing the last inequality in Eq.~\eqref{eq:t_2_k_v_update_centralized_part1_bis} and using again Eq.~\eqref{eq:convergence_of_nabla_u_centralized},
    we obtain
    \begin{equation}
        \label{eq:t_2_k_v_update_centralized}
        \frac{1}{K}\sum_{k=1}^{K} \E \left[T^{k}_{2}\right]  \leq \mathcal{O}\left(\frac{1}{K^{3/4}}\right).
    \end{equation}
    Combining Eq.~\eqref{eq:t_1_k_v_update_centralized} and Eq.~\eqref{eq:t_2_k_v_update_centralized}, it follows that 
    \begin{equation}
        \frac{1}{K} \sum_{k=1}^K \E\left[\Delta_{\vec{v}} f(u^k,\vec{v}^{k}_{1:T})\right] \leq \mathcal{O}\!\left(\frac{1}{K^{3/4}}\right).
    \end{equation}
\end{proof}

\subsubsection{Proof of Theorem~\ref{thm:centralized_convergence}}
\label{app:centralized_convergence}
In this section, $f$ denotes the negative log-likelihood function defined in Eq.~\eqref{eq:empirical_loss}. Moreover, we introduce the negative log-likelihood at client $t$ as follows
\begin{equation}
    \label{eq:empirical_loss_client}
    f_t(\Theta, \Pi) \triangleq - \frac{\log p(\mathcal{S}_{t}| \Theta, \Pi)}{n} \triangleq - \frac{1}{n_t}\sum_{i=1}^{n_{t}} \log p(s_t^{(i)}|\Theta, \pi_{t}).
\end{equation}

\begin{repthm}{thm:centralized_convergence}
    Under Assumptions~\ref{assum:mixture}--\ref{assum:bounded_dissimilarity}, when clients use SGD as local solver with learning rate $\eta =\frac{a_{0}}{\sqrt{K}}$, after a large enough number of communication rounds $K$, \FedEM's iterates  satisfy:
    \begin{equation}
        \frac{1}{K} \sum_{k=1}^K \mathbb{E}\left\|\nabla _{\Theta} f\left(\Theta^{k}, \Pi^{k}\right)\right\|^{2}_{F} \leq \mathcal{O}\!\left(\frac{1}{\sqrt{K}}\right),    \qquad
    \frac{1}{K} \sum_{k=1}^K \Delta_{\Pi} f(\Theta^k,\Pi^k)
    \leq \mathcal{O}\!\left(\frac{1}{K^{3/4}}\right), \tag{\ref{eq:theta_pi_convergence}}   
    \end{equation}
where the expectation is over the random batches samples, and $\Delta_{\Pi} f(\Theta^k,\Pi^k) \triangleq f\left(\Theta^{k}, \Pi^{k}\right) - f\left(\Theta^{k}, \Pi^{k+1}\right) \ge 0 $.
\end{repthm}

\begin{proof}
    We prove this result as a particular case of Theorem~\ref{thm:centralized_convergence_bis}. To this purpose, in this section, we consider that $\mathcal{V} \triangleq \Delta^{M}$, $\vec{u}=\Theta \in \mathbb{R}^{d M}$, $\vec{v}_{t} = \pi_{t}$, and $\omega_t = n_t/n$ for $t\in[T]$. For $k>0$, we define $g^{k}_{t}$ as follows:
    \begin{flalign}
        g^{k}_{t}\Big (\Theta, \pi_{t}\Big) =   \frac{1}{n_{t}}\sum_{i=1}^{n_{t}}\sum_{m=1}^{M}q_{t}^{k}\left(z^{(i)}_{t}=m\right) \cdot &\bigg( l\left(h_{\theta_m}(\vec{x}_{t}^{(i)}), y_{t}^{(i)}\right)  -  \log p_{m}(\vec{x}_{t}^{(i)}) - \log \pi_{t} \nonumber &
        \\
        & \qquad \qquad +  \log q_{t}^{k}\left(z_{t}^{(i)}=m\right) - c \bigg), 
    \end{flalign}
    where $c$ is the same constant appearing in Assumption~\ref{assum:logloss}, Eq.~\eqref{eq:log_loss}. With this definition, it is easy to check that the federated surrogate optimization algorithm (Alg.~\ref{alg:fed_surrogate_opt}) reduces to \FedEM{} (Alg.~\ref{alg:fed_em}). Theorem~\ref{thm:centralized_convergence} then follows immediately  from Theorem~\ref{thm:centralized_convergence_bis}, once we verify that $\left(g_{t}^{k}\right)_{1\leq t\leq T}$ satisfy the assumptions of Theorem~\ref{thm:centralized_convergence_bis}.
    
    Assumption~\ref{assum:bounded_f_bis}, Assumption~\ref{assum:finite_variance_bis}, and Assumption~\ref{assum:bounded_dissimilarity_bis} follow directly from  Assumption~\ref{assum:bounded_f}, Assumption~\ref{assum:finitie_variance}, and Assumption~\ref{assum:bounded_dissimilarity}, respectively.
    Lemma~\ref{lem:g_is_smooth} shows that for $k>0$, $g^{k}$ is smooth w.r.t. $\Theta$ and then Assumption~\ref{assum:smoothness_bis} is satisfied.
    Finally, Lemmas~\ref{lem:em_gap_is_kl_divergence}--\ref{lem:compute_kl_pi_decentralized} show that for $t\in[T]$ $g_{t}^{k}$ is a partial first-order surrogate of $f_{t}$ w.r.t.~$\Theta$ near $\left\{\Theta^{k-1}, \pi_{t}\right\}$ with $d_{\mathcal{V}}(\cdot, \cdot) = \mathcal{KL}(\cdot\|\cdot)$.
\end{proof}

\begin{lem}
    \label{lem:g_is_smooth}
    Under Assumption~\ref{assum:smoothness}, for $t\in[T]$ and $k>0$,  $g^k_{t}$ is $L$-smooth w.r.t $\Theta$.
\end{lem}
\begin{proof}
    $g^{k}_{t}$ is a convex combination of $L$-smooth function $\theta \mapsto l(\theta; s_{t}^{(i)}),~i\in[n_{t}]$. Thus it is also $L$-smooth.
\end{proof}
\begin{lem}
    \label{lem:em_gap_is_kl_divergence}
    Suppose that Assumptions~\ref{assum:mixture}--\ref{assum:logloss},  hold. Then, for $t\in[T]$, $\Theta \in\mathbb{R}^{M\times d}$ and $\pi_{t} \in \Delta^{M}$
    \[r_{t}^{k}\left( \Theta, \pi_{t}\right)  \triangleq g_{t}^{k}\left(\Theta, \pi_{t}\right) - f_{t}\left(\Theta, \pi_{t}\right) = \frac{1}{n_{t}}\sum_{i=1}^{n_{t}}\mathcal{KL}\left(q_{t}^{k}\left(z^{(t)}_{i}\right)\|p_{t}\left(z^{(t)}_{i}|s^{(t)}_{i}, \Theta, \pi_{t}\right)\right),\]
    where $\mathcal{KL}$ is Kullback–Leibler divergence.
\end{lem}
\begin{proof}
    Let $k>0$ and $t\in[T]$, and consider $\Theta\in\mathbb{R}^{M\times d}$ and $\pi_{t} \in \Delta^{M}$, then
    \begin{flalign}
        g^{k}_{t}\Big (\Theta, \pi_{t}\Big) &=   \frac{1}{n_{t}}\sum_{i=1}^{n_{t}}\sum_{m=1}^{M}q_{t}^{k}\left(z^{(i)}_{t}=m\right) \cdot \bigg( l\left(h_{\theta_{m}}(\vec{x}_{t}^{(i)}), y_{t}^{(i)}\right)  -  \log p_{m}(\vec{x}_{t}^{(i)}) - \log \pi_{t} \nonumber 
        \\
        & \qquad \qquad \qquad \qquad \qquad \qquad \qquad   +  \log q_{t}^{k}\left(z_{t}^{(i)}=m\right) - c \bigg), 
        \\
        & = \frac{1}{n_{t}}\sum_{i=1}^{n_{t}}\sum_{m=1}^{M}q_{t}^{k}\left(z^{(i)}_{t}=m\right) \cdot \bigg( -\log p_{m}\left(y_{t}^{(i)}| \vec{x}_{t}^{(i)}, \theta_{m}\right) -  \log p_{m}(\vec{x}_{t}^{(i)}) - \log \pi_{t} \nonumber 
        \\
        & \qquad \qquad \qquad \qquad \qquad \qquad \qquad+  \log q_{t}^{k}\left(z_{t}^{(i)}=m\right)  \bigg)
        \\
        & = \frac{1}{n_{t}}\sum_{i=1}^{n_{t}}\sum_{m=1}^{M}q_{t}^{k}\left(z^{(i)}_{t}=m\right) \cdot \bigg( -\log p_{m}\left(y_{t}^{(i)}| \vec{x}_{t}^{(i)}, \theta_{m}\right) \cdot p_{m}(\vec{x}_{t}^{(i)}) \cdot p_{t}\left(z_{t}^{(i)}=m\right) \nonumber 
        \\
        & \qquad \qquad \qquad \qquad \qquad \qquad \qquad+
        \log q_{t}^{k}\left(z_{t}^{(i)}=m\right)  \bigg)
        \\
        &= \frac{1}{n_{t}}\sum_{i=1}^{n_{t}}\sum_{m=1}^{M}q^{k}_{t}\left(z_{t}^{(i)}=m\right)\cdot \left( \log q^{k}_{t}\left(z_{t}^{(i)}=m\right) -  \log p_{t}\left(s_{t}^{(i)}, z_{t}^{(i)}=m\right| \Theta, \pi_{t})\right)
        \\
        &   = \frac{1}{n_{t}}\sum_{t=1}^{n_{t}}\sum_{m=1}^{M} q^{k}_{t}\left(z_{t}^{(i)}=m\right) \log \frac{ q^{k}_{t}\left(z_{t}^{(i)}=m\right)}{p_{t}\left(s_{t}^{(i)}, z_{t}^{(i)}=m| \Theta, \pi_{t}\right)}.
    \end{flalign}
    Thus, 
    \begin{flalign}
        r^k_{t}& \Big( \Theta, \pi_{t}\Big)  \triangleq g^{k}_{t}\left(\Theta, \pi_{t}\right) - f_{t}\left(\Theta,  \pi_{t}\right) & 
        \\
        &=  -\frac{1}{n_{t}}\sum_{t=1}^{n_{t}}\sum_{m=1}^{M}\left(q^{k}_{t}\left(z_{t}^{(i)}=m\right)\cdot \log \frac{p_{t}\left(s_{t}^{(i)}, z_{t}^{(i)}=m| \Theta, \pi_{t}\right)}{q^{k}_{t}\left(z_{t}^{(i)}=m\right)} \right) \nonumber
        \\
        &  \qquad \qquad \qquad \qquad + \frac{1}{n_{t}}\sum_{i=1}^{n_{t}}\log p_{t}\left(s_{t}^{(i)}| \Theta,  \pi_{t}\right) 
        \\
        &= \frac{1}{n_{t}}\sum_{t=1}^{n_{t}}\sum_{m=1}^{M} q^{k}_{t}\left(z_{t}^{(i)}=m\right)\Bigg( \log p_{t}\left(s_{t}^{(i)}| \Theta, \pi_{t}\right) \nonumber
        \\
        & \qquad \qquad \qquad \qquad \qquad \quad \qquad \qquad \qquad - \log \frac{p_{t}\left(s_{t}^{(i)}, z_{t}^{(i)}=m|  \Theta, \pi_{t}\right)}{q^{k}_{t}\left(z_{t}^{(i)}=m\right)} \Bigg)
        \\
        &= \frac{1}{n_{t}}\sum_{t=1}^{n_{t}}\sum_{m=1}^{M} q^{k}_{t}\left(z_{t}^{(i)}=m\right) \log \frac{p_{t}\left(s_{t}^{(i)}| \Theta, \pi_{t}\right) \cdot q^{k}_{t}\left(z_{t}^{(i)}=m\right)}{p_{t}\left(s_{t}^{(i)}, z_{t}^{(i)}=m|  \Theta, \pi_{t}\right)}\\
        & = \frac{1}{n_{t}}\sum_{t=1}^{n_{t}}\sum_{m=1}^{M} q^{k}_{t}\left(z_{t}^{(i)}=m\right) \cdot \log \frac{ q^{k}_{t}\left(z_{t}^{(i)}=m\right)}{p_{t}\left(z_{t}^{(i)}=m|s_{t}^{(i)}, \Theta, \pi_{t}\right)}.
    \end{flalign}
    Thus,
    \begin{equation}
        r_{t}^{k}\left( \Theta, \pi_{t}\right) = \frac{1}{n_{t}}\sum_{i=1}^{n_{t}} \mathcal{KL}\left(q^{k}_{t}(\cdot)\| p_{t}(\cdot|s_{i}^{(t)},  \Theta, \pi_{t})\right)  \geq 0 .    
    \end{equation}
\end{proof}

The following lemma shows that $g^k_t$ and $g^k$ (as defined in Eq.~\ref{eq:app_gk_def}) satisfy the first two properties in Definition~\ref{def:partial_first_order_surrogate}.
\begin{lem}
    \label{lem:g_is_surrogate}
    Suppose that  Assumptions~\ref{assum:mixture}--\ref{assum:logloss}  and Assumptions~\ref{assum:smoothness},~\ref{assum:bounded_gradient} hold and define $\tilde{L} \triangleq L + B^{2}$. For all $k\geq0$ and $t\in[T]$, $g_{t}^{k}$ is a majorant of $f_{t}$ and $r_{t}^{k} \triangleq g_{t}^{k} - f_{t}$ is $\tilde{L}$-smooth in $\Theta$. Moreover $r_{t}^{k}\left(\Theta^{k-1}, \pi_{t}^{k-1}\right) = 0$ and $\nabla_{\Theta} r_{t}^{k}\left(\Theta^{k-1}, \pi_{t}^{k-1}\right) = 0$.
    
    The same holds for $g^{k}$, i.e., $g^{k}$ is a majorant of $f$, $r^{k}\triangleq g^{k} - f$ is $\tilde{L}$-smooth in $\Theta$, $r^{k}\left(\Theta^{k-1}, \Pi^{k-1}\right) = 0$ and $\nabla_{\Theta} r^{k}\left(\Theta^{k-1}, \Pi^{k-1}\right) = 0$
\end{lem}

\begin{proof}
    For $t\in[T]$, consider $\Theta\in \mathbb{R}^{M\times d}$ and $\pi_{t} \in \Delta^{M}$, we have (Lemma~\ref{lem:em_gap_is_kl_divergence})
    \begin{equation}
        r^{k}_{t}\left( \Theta, \pi_{t}\right)  \triangleq g^{k}_{t}\left(\Theta, \pi_{t}\right) - f_{t}\left(\Theta, \pi_{t}\right) = \frac{1}{n_{t}}\sum_{i=1}^{n_{t}}\mathcal{KL}\left(q_{t}^{k}\left(z^{(t)}_{i}\right)\|p_{t}\left(z^{(i)}_{t}|s_{t}^{(i)}, \Theta, \pi_{t}\right)\right).
    \end{equation}
    Since $\mathcal{KL}$ divergence is non-negative, it follows that $g^{k}_{t}$ is a majorant of $f_{t}$, i.e.,
    \begin{equation}
        \forall~ \Theta \in \mathbb{R}^{M\times d},~\pi_{t} \in \Delta^{M}:~~g^{k}_{t}\left(\Theta, \pi\right) \geq f_{t}\left(\Theta, \pi_{t}\right).
    \end{equation}
    Moreover since, $q_{t}^{k}\left(z^{(i)}_t\right) = p_{t}\left(z^{(i)}_t|s^{(i)}_t, \Theta^{k-1}, \pi^{k-1}_{t}\right)$ for $k>0$, it follows that
    \begin{equation}
        r^{k}_{t}\left( \Theta^{k-1}, \pi^{k-1}_{t}\right) = 0.
    \end{equation}
    
    For $i\in[n_{t}]$ and $m\in[M]$, from Eq.~\ref{eq:app_e_step}, we have
    \begin{flalign}
        p_{t}\left(z_{t}^{(i)}=m|s_{t}^{(i)},  \Theta, \pi_{t}\right) &=  \frac{{p}_{m}\left(y_{t}^{(i)}|\vec{x}_{t}^{(i)}, \theta_{m}\right) \times \pi_{t m}}{\sum_{m'=1}^{M} {p}_{m'}\left(y_{t}^{(i)}|\vec{x}_{t}^{(i)}, \theta_{m'}\right) \times \pi_{t m'}} &
        \\
        &= \frac{\exp\left[-l\left(h_{\theta_{m}}(\vec{x}_{t}^{(i)}), y_{t}^{(i)}\right)\right] \times \pi_{t m}}{\sum_{m'=1}^{M} \exp\left[-l\left(h_{\theta_{m'}}(\vec{x}_{t}^{(i)}), y_{t}^{(i)}\right)\right] \times \pi_{t m'}}
        \\
        &= \frac{\exp\left[-l\left(h_{\theta_{m}}(\vec{x}_{t}^{(i)}), y_{t}^{(i)}\right) + \log\pi_{t m}\right]}{\sum_{m'=1}^{M} \exp\left[-l\left(h_{\theta_{m'}}(\vec{x}_{t}^{(i)}), y_{t}^{(i)}\right) + \log\pi_{t m'}\right]}.
    \end{flalign}
    For ease of notation, we introduce
    \begin{equation}
        l_{i}(\theta) \triangleq l\left(h_{\theta}(\vec{x}_{t}^{(i)}), y_{t}^{(i)}\right),\qquad \theta \in \mathbb{R}^{d},~m\in[M],~i \in [n_{t}],
    \end{equation}
    \begin{equation}
        \gamma_{m}\left(\Theta\right)\triangleq p_{t}\left(z_{t}^{(i)}=m|s_{t}^{(i)},  \Theta, \pi_{t}\right), \qquad m\in[M],
    \end{equation}
    and,
    \begin{equation}
        \varphi_{i}\left(\Theta\right) \triangleq   \mathcal{KL}\left(q_{t}^{k}\left(z^{(t)}_{i}\right)\|p_{t}\left(z^{(i)}_{t}|s_{t}^{(i)}, \Theta, \pi_{t}\right)\right).
    \end{equation}
    For $i\in[n_{t}]$, function $l_{i}$ is differentiable because smooth (Assum~\ref{assum:smoothness}), thus $\gamma_{m},~m\in[M]$ is differentiable as the composition of the softmax function and the function $\left\{\Theta\mapsto -l_{i}\left(\Theta\right) + \log\pi_{t m}\right\}$. Its gradient is given by
    \begin{equation}
        \begin{cases}
            \begin{aligned}
                \nabla_{\theta_{m}}\gamma_{m}\left(\Theta\right) &= - \gamma_{m}\left(\Theta\right) \cdot \left(1-\gamma_{m}\left(\Theta\right)\right) \cdot \nabla l_{i}\left(\theta_{m}\right),& 
                \\
                \nabla_{\theta_{m'}}\gamma_{m}\left(\Theta\right) &=  \gamma_{m}\left(\Theta\right)\cdot \gamma_{m'}\left(\Theta\right)  \cdot \nabla l_{i}\left(\theta_{m}\right), & m'\neq m.\\
            \end{aligned}
        \end{cases}
    \end{equation}
    Thus for $m\in[M]$, we have
    \begin{flalign}
        \nabla_{\theta_{m}} \varphi_{i}\left(\Theta\right) &= \sum_{m'=1}^{M} q_{t}^{k}\left(z^{(t)}_{i}=m'\right) \cdot 
        \frac{\nabla_{\theta_{m}}  \gamma_{m'}\left(\Theta\right)}{\gamma_{m'}\left(\Theta\right)} &
        \\
        &= \sum_{\substack{m'= 1\\ m'\neq m}} \left[q_{t}^{k}\left(z^{(t)}_{i}=m'\right) \cdot    \frac{\gamma_{m}\left(\Theta\right)\cdot \gamma_{m'}\left(\Theta\right)}{\gamma_{m'}\left(\Theta\right) \cdot } \cdot \nabla l_{i}\left(\theta_{m}\right)\right] \nonumber
        \\
        & \qquad - q_{t}^{k}\left(z^{(t)}_{i}=m\right) \cdot \frac{\gamma_{m}\left(\Theta\right)\cdot \left(1-\gamma_{m}\left(\Theta\right)\right)}{\gamma_{m}\left(\Theta\right)} \cdot \nabla l_{i}\left(\theta_{m}\right).
    \end{flalign}
     Using the fact that $\sum_{m'=1}^{M}q_{t}^{k}\left(z^{(t)}_{i}=m\right) = 1$, it follows that
    \begin{equation}
        \nabla_{\theta_{m}} \varphi_{i}\left(\Theta\right) = \left(\gamma_{m}\left(\Theta\right) - q_{t}^{k}\left(z^{(t)}_{i}=m\right)\right)\cdot \nabla l_{i}\left(\theta_{m}\right).
    \end{equation}
    Since $l_{i},~i\in[n_{t}]$ is twice continuously differentiable (Assumption~\ref{assum:smoothness}), and $\gamma_{m},~m\in[M]$ is differentiable, then $\phi_{i},~i\in[n_{t}]$ is twice continuously differentiable. We use $\mathbf{H}\left(\varphi_{i}\left(\Theta\right)\right) \in \mathbb{R}^{dM\times dM}$ (resp.~$\mathbf{H}\left(l_{i}\left(\theta\right)\right) \in \mathbb{R}^{d\times d}$) to denote the Hessian of $\varphi$ (resp.~$l_{i}$) at $\Theta$ (resp.~$\theta$). The Hessian of $\varphi_{i}$ is a block matrix given by
    \begin{equation}
        \label{eq:hessian_of_kl}
        \begin{cases}
            \begin{aligned}
                \Big ( \mathbf{H}\left(\varphi_{i}\left(\Theta\right)\right) \Big)_{m,m} &=  -\gamma_{m}\left(\Theta\right) \cdot \left(1 - \gamma_{m}\left(\Theta\right)\right) \cdot \Big(\nabla l_{i}(\theta_{m})\Big) \cdot \Big(\nabla l_{i}(\theta_{m})\Big)^{\intercal} 
                \\
                & \qquad  + \left(\gamma_{m}(\Theta) - q_{t}^{k}\left(z^{(t)}_{i}=m\right)\right) \cdot \mathbf{H}\left(l_{i}\left(\theta_{m}\right)\right)
                \\
               \Big ( \mathbf{H}\left(\varphi_{i}\left(\Theta\right)\right) \Big)_{m,m'} &=   \gamma_{m}\left(\Theta\right) \cdot \gamma_{m'}\left(\Theta\right) \cdot \Big(\nabla l_{i}(\theta_{m'})\Big) \cdot \Big(\nabla l_{i}(\theta_{m})\Big)^{\intercal},& m'\neq m.
                \\
            \end{aligned}
        \end{cases}
    \end{equation}
    We introduce the block matrix $\tilde{\mathbf{H}} \in \mathbb{R}^{dM \times dM}$, defined by 
    \begin{equation}
        \begin{cases}
            \begin{aligned}
                \tilde{\mathbf{H}}_{m, m} &= -\gamma_{m}\left(\Theta\right) \cdot \Big(1 - \gamma_{m}\left(\Theta\right)\Big) \cdot \Big(\nabla l_{i}(\theta_{m})\Big) \cdot \left(\nabla l_{i}(\theta_{m})\right)^{\intercal} &
                \\
                \tilde{\mathbf{H}}_{m, m'} &=  \gamma_{m}\left(\Theta\right) \cdot \gamma_{m'}\left(\Theta\right) \cdot \Big(\nabla l_{i}(\theta_{m})\Big) \cdot \Big(\nabla l_{i}(\theta_{m'})\Big)^{\intercal}, & m'\neq m,
            \end{aligned}
        \end{cases}
    \end{equation}
    Eq.~\eqref{eq:hessian_of_kl} can be written as
    \begin{equation}
        \label{eq:hessian_minus_h_tilde}
        \begin{cases}
            \begin{aligned}
                 \Big ( \mathbf{H}\left(\varphi_{i}\left(\Theta\right)\right) \Big)_{m,m} - \tilde{\mathbf{H}}_{m, m} &=    \left( \gamma_{m}(\Theta) - q_{t}^{k}\left(z^{(t)}_{i}=m\right)\right) \cdot \mathbf{H}\left(l_{i}\left(\theta_{m}\right)\right) &
                \\
                 \Big ( \mathbf{H}\left(\varphi_{i}\left(\Theta\right)\right) \Big)_{m,m'} - \tilde{\mathbf{H}}_{m, m'} &= 0, & m'\neq m.
                \\
            \end{aligned}
        \end{cases}
    \end{equation}
     We recall that a twice differentiable function is $L$ smooth if and only if the eigenvalues of its Hessian are smaller then $L$ in absolute value, see e.g., \cite[Lemma~1.2.2]{nestrov2003introduction} or \cite[Section 3.2]{bubeck2015convex}.  We have for $\theta \in \mathbb{R}^{d}$,
    \begin{equation}
            - L \cdot I_{d} \preccurlyeq \mathbf{H}\left(l_{i}\left(\theta\right)\right) \preccurlyeq L \cdot I_{d}.
    \end{equation}
    Using Lemma~\ref{lem:hessian_is_negative} and the fact that $\left\|\nabla l_{i}\left(\theta_{m}\right)\right\| \leq B$ (Assumption~\ref{assum:bounded_gradient}), we can conclude that matrix $\tilde{\mathbf{H}}$ is semi-definite negative and that $\tilde{\mathbf{H}} \succcurlyeq - B^{2}\cdot I_{d M}$. Since 
    \begin{equation}
        -1 \leq \gamma_{m}(\Theta) - q_{t}^{k}\left(z^{(t)}_{i}=m\right) \leq 1,
    \end{equation}
    it follows that
    \begin{equation}
        - \left(L + B^{2}\right)\cdot I_{dM} \preccurlyeq \mathbf{H}\left(\varphi_{i}\left(\Theta\right)\right)   \preccurlyeq L \cdot I_{dM} \preccurlyeq \underbrace{\left(L + B^{2}\right)}_{\triangleq \tilde{L}}\cdot I_{dM}.
   \end{equation}
   The last equation proves that $\varphi_{i}$ is $\tilde{L}$-smooth. Thus  $r_{t}^{k}$ is $\tilde{L}$-smooth with respect to $\Theta$ as the average of $\tilde{L}$-smooth function.
   
    Moreover, since $r_{t}^{k}(\Theta^{k-1}, \pi_{t}^{k-1}) = 0$ and $\forall \Theta, \Pi;~~r_{t}^{k}( \Theta, \pi_{t})\geq 0$, it follows that  $\Theta^{k-1}$ is a minimizer of $\left\{\Theta \mapsto r_{t}^{k}\left(\Theta, \pi_{t}^{k-1}\right)\right\}$. Thus, $\nabla_{\Theta} r_{t}^{k}(\Theta^{k-1}, \pi_{t}^{k-1}) = 0$.

    For $\Theta\in\mathbb{R}^{M\times d}$ and $\Pi \in \Delta^{T \times M}$, we have
    \begin{flalign}
        r^{k}\left( \Theta, \Pi\right) &\triangleq  g^{k}\left( \Theta, \Pi\right) - f\left( \Theta, \Pi\right) &
        \\
        &\triangleq \sum_{t=1}^{T}\frac{n_{t}}{n}\cdot \left[g_{t}^{k}\left( \Theta, \pi_{t}\right) - f_{t}\left( \Theta, \pi_{t}\right)\right]
        \\
        &= \sum_{t=1}^{T}\frac{n_{t}}{n} r^{k}_{t}\left( \Theta, \pi_{t}\right).
    \end{flalign}
    We see that $r^{k}$ is a weighted average of $\left(r_{t}^{k}\right)_{1 \leq t \leq T}$. Thus, $r^{k}$ is $\tilde{L}$-smooth in $\Theta$, $ r^{k}\left( \Theta, \Pi\right) \geq 0$, moreover  $r_{t}^{k}\left(\Theta^{k-1}, \Pi^{k-1}\right) = 0$ and $\nabla_{\Theta}r_{t}^{k}\left(\Theta^{k-1}, \Pi^{k-1}\right) = 0$. 
\end{proof}

The following lemma shows that $g^k_t$ and $g^k$ satisfy the third property in Definition~\ref{def:partial_first_order_surrogate}. 
\begin{lem}
    \label{lem:compute_kl_pi_decentralized}
    Suppose that Assumption~\ref{assum:mixture} holds and consider $\Theta \in \mathbb{R}^{M\times d}$ and $\Pi \in \Delta^{T\times M}$, for $k>0$, the iterates of Alg.~\ref{alg:fed_surrogate_opt} verify
    \[g^{k}\left(\Theta, \Pi\right) = g^{k}\left(\Theta, \Pi^{k}\right) + \sum_{t=1}^{T} \frac{n_{t}}{n}\mathcal{KL}\left(\pi^{k}_{t}, \pi_{t}\right).\]
\end{lem}
\begin{proof}
    For $t\in[T]$ and $k>0$, consider $\Theta\in\mathbb{R}^{M\times d}$ and $\pi_{t} \in \Delta^{M}$ such that $\forall m \in[M]; \pi_{tm} \neq 0$, we have
    \begin{align}
        g_{t}^{k}\left(\Theta, \pi_{t}\right) - g_{t}^{k}\left(\Theta,\pi_{t}^{k}\right) &= \sum_{m=1}^{M} \underbrace{\left\{\frac{1}{n_{t}}\sum_{i=1}^{n_{t}} q_{t}^{k}\left(z_{t}^{(i)}=m\right)\right\}}_{=\pi_{tm}^{k}~(\text{Proposition~\ref{prop:em})} }\times\left( \log \pi^{k}_{tm} -\log\pi_{tm} \right)\\
        &= \sum_{m=1}^{M} \pi_{tm}^{k}\log \frac{\pi^{k}_{tm}}{\pi_{tm}}\\
        &= \mathcal{KL}\left(\pi^{k}_{t}, \pi_{t}\right).
    \end{align}
    We multiply by $\frac{n_{t}}{n}$ and some for $t\in[T]$. It follows that
    \begin{equation}
        g^{k}\left(\Theta, \Pi^{k}\right) + \sum_{t=1}^{T}\frac{n_{t}}{n}\mathcal{KL}\left(\pi^{k}_{t}, \pi_{t}\right) = g^{k}\left(\Theta, \Pi\right).
    \end{equation}
\end{proof}

\subsection{Fully Decentralized Setting}
\label{proof:decentralized}
\subsubsection{Additional Notations}
\begin{remark}
    For convenience and without loss of generality, we suppose in this section that $\omega_{t} = 1,~t\in[T]$.  
\end{remark}

We introduce the following matrix notation:
\begin{align}
    \vec{U}^{k} &\triangleq \left[\vec{u}_{1}^{k}, \dots, \vec{u}_{T}^{k}\right] \in\mathbb{R}^{d_{u} \times T}
    \\
    \bar{\vec{U}}^{k} &\triangleq \left[\bar{\vec{u}}^{k}, \dots, \bar{\vec{u}}^{k}\right] \in \mathbb{R}^{d_{u} \times T}
    \\
    \partial g^{k}\left(\vec{U}^{k}, \vec{v}^{k}_{1:T}; \xi^{k}\right) &\triangleq\left[\nabla_{\vec{u}} g^{k}_{1}\left(\vec{u}^{k}_{1}, \vec{v}_{1}^{k}; \xi^{k}_{1}\right), \dots, \nabla_{\vec{u}} g^{k}_{T}\left(\vec{u}^{k}_{T}, \vec{v}_{T}^{k}; \xi^{k}_{T}\right)\right] \in \mathbb{R}^{d_{u}\times T}
\end{align}
where $\bar{\vec{u}}^{k} = \frac{1}{T}\sum_{t=1}^{T}\vec{u}_{t}^{k}$ and $\vec{v}_{1:T}^{k} = \left(\vec{v}^{k}_{t}\right)_{1\leq t \leq T}\in\mathcal{V}^{T}$.

We denote by $\vec{u}_t^{k-1,j}$ the $j$-th iterate of the local solver at global iteration $k$ at client $t\in[T]$, and by $\vec{U}^{k-1,j}$ the matrix whose column $t$ is $\vec{u}_t^{k-1,j}$,
thus,
\begin{equation}
    \vec{u}_{t}^{k-1, 0} = \vec{u}^{k-1}_{t};\qquad  \vec{U}^{k-1, 0} = \vec{U}^{k-1},
\end{equation}
and,
\begin{equation}
    \vec{u}_{t}^{k} = \sum_{s=1}^{T}w^{k-1}_{st}\vec{u}^{k-1, J}_{s}; \qquad  \vec{U}^{k} = \vec{U}^{k-1,J} W^{k-1}.
\end{equation}
Using this notation, the updates of Alg.~\ref{alg:decentralized_surrogate_opt} can be summarized as
\begin{equation}
    \label{eq:decentralized_surrogate_sgd_matricial}
    \vec{U}^{k} = \left[\vec{U}^{k-1} - \sum_{j=0}^{J-1}\eta_{k-1,j}\partial g^{k}\left(\vec{U}^{k-1,j}, \vec{v}_{1:T}; \xi^{k-1, j}\right)\right]W^{k-1}.
\end{equation}
Similarly to the client-server setting, we define the normalized update of local solver at client $t\in[T]$:
\begin{equation}
    \hat{\delta}^{k-1}_{t} \triangleq -\frac{\vec{u}^{k-1,J}_{t} - \vec{u}^{k-1,0}_{t}}{\eta_{k-1}} =  \frac{\sum_{j=0}^{J-1}\eta_{k-1,j} \nabla_{\vec{u}} g^{k}_{t}\left(\vec{u}_{t}^{k-1,j}, \vec{v}_{t}^{k}; \xi_{t}^{k-1,j}\right)}{\sum_{j=0}^{J-1}\eta_{k-1,j}},
\end{equation}
and
\begin{equation}
    \delta^{k-1}_{t} \triangleq    \frac{\sum_{j=0}^{J-1}\eta_{k-1,j} \nabla_{\vec{u}} g^{k}_{t}\left(\vec{u}_{t}^{k-1,j}, \vec{v}_{t}^{k}\right)}{\eta_{k-1}}.
\end{equation}
Because clients updates are independent, and stochastic gradient are unbiased, it is clear that 
\begin{equation}
    \label{eq:unbiased_gradients}
     \mathbb{E}\left[\delta^{k-1}_{t} - \hat{\delta}^{k-1}_{t}\right] = 0,
\end{equation}
and that 
\begin{equation}
    \label{eq:covarianceof_local_updates_is_zero_decentralized}
     \forall~t,s \in[T]~\text{s.t.}~s\neq t,~~\mathbb{E}\langle \delta^{k-1}_{t} - \hat{\delta}^{k-1}_{t}, \delta^{k-1}_{s} - \hat{\delta}^{k-1}_{s}\rangle = 0.
\end{equation}
We introduce the matrix notation, 
\begin{align}
    \hat{\Upsilon}^{k-1} \triangleq \left[\hat{\delta}^{k-1}_{1}, \dots, \hat{\delta}^{k-1}_{T}\right] \in \mathbb{R}^{d_{u} \times T};\qquad \Upsilon^{k-1} \triangleq \left[\delta^{k-1}_{1}, \dots, \delta^{k-1}_{T}\right] \in \mathbb{R}^{d_{u} \times T}.
\end{align}
Using this notation, Eq.~\eqref{eq:decentralized_surrogate_sgd_matricial} becomes
\begin{equation}
    \label{eq:decentralized_surrogate_sgd_full_matricial}
    \vec{U}^{k} = \left[\vec{U}^{k-1} -\eta_{k-1} \hat{\Upsilon}^{k-1}\right]W^{k-1}.
\end{equation}

\subsubsection{Proof of Theorem~\ref{thm:decentralized_convergence_bis}}

In fully decentralized optimization, proving the convergence usually consists in deriving a recurrence on a term measuring the optimality of the average iterate (in our case this term is $\mathbb{E}\left\|\nabla_{\vec{u}} f\left(\bar{\vec{u}}^{k}, \vec{v}_{1:T}^{k}\right)\right\|^{2}$) and a term measuring the distance to consensus, i.e., $\mathbb{E}\sum_{t=1}^{T}\left\|\vec{u}_{t}^{k} - \bar{\vec{u}}^{k}\right\|^{2}$. In what follows we obtain those two recurrences, and then prove the convergence.

\begin{lem}[Average iterate term recursion]
    \label{lem:descent_lem_decentralized}
     Suppose that Assumptions~\ref{assum:smoothness_bis}--\ref{assum:bounded_dissimilarity_bis} and Assumption~\ref{assum:expected_consensus_rate} hold. Then, for $k>0$, and $\left(\eta_{k, j}\right)_{1\leq j \leq J-1}$ 
    such that $\eta_k\triangleq \sum_{j=0}^{J-1}\eta_{k, j} \leq\min \left\{ \frac{1}{2\sqrt{2}L}, \frac{1}{8L\beta}\right\}$, the updates of fully decentralized federated surrogate optimization (Alg.~\ref{alg:decentralized_surrogate_opt}) verify
    \begin{flalign}
         \mathbb{E}\Bigg[f & (\bar{\vec{u}}^{k}  , \vec{v}_{1:T}^{k}) - f(\bar{\vec{u}}^{k-1}, \vec{v}_{1:T}^{k-1})\Bigg] \leq  -\frac{1}{T}\sum_{t=1}^{T}\E d_{\mathcal{V}}\left(\vec{v}_{t}^{k}, \vec{v}_{t}^{k-1}\right) & \nonumber
        \\
        &  -\frac{\eta_{k-1}}{8}\E\left\|\nabla_{\vec{u}} f\left(\bar{\vec{u}}^{k-1}, \vec{v}_{1:T}^{k-1}\right)\right\|^{2} +  \frac{\left(12 + T\right)\eta_{k-1}L^{2}}{4T}  \cdot \sum_{t=1}^{T} \E\left\|\vec{u}_{t}^{k-1} - \bar{\vec{u}}^{k-1}\right\|^{2}  \nonumber
        \\
        &  + \frac{\eta_{k-1}^{2}L}{T} \left(4 \sum_{j=0}^{J-1}\frac{L \cdot \eta_{k-1,j}^{2}}{\eta_{k-1}} + 1\right) \sigma^{2} + \frac{16\eta_{k-1}^{3}L^{2}}{T} G^{2}.
    \end{flalign}
\end{lem}
\begin{proof}
    We multiply both sides of Eq.~\eqref{eq:decentralized_surrogate_sgd_full_matricial} by $\frac{\mathbf{1}\mathbf{1}^{\intercal}}{T}$, thus for $k>0$ we have,
    \begin{equation}
        \vec{U}^{k} \cdot \frac{\mathbf{1}\mathbf{1}^{\intercal}}{T} = \left[\vec{U}^{k-1} -\eta_{k-1} \hat{\Upsilon}^{k-1}\right]W^{k-1} \frac{\mathbf{1}\mathbf{1}^{\intercal}}{T},
    \end{equation}
    since $W^{k-1}$ is doubly stochastic (Assumption~\ref{assum:expected_consensus_rate}), i.e., $W^{k-1} \frac{\mathbf{1}\mathbf{1}^{\intercal}}{T} = \frac{\mathbf{1}\mathbf{1}^{\intercal}}{T}$, is follows that,
    \begin{equation}
        \bar{\vec{U}}^{k} = \bar{\vec{U}}^{k-1} -\eta_{k-1} \hat{\Upsilon}^{k-1}\cdot \frac{\mathbf{1}\mathbf{1}^{\intercal}}{T},
    \end{equation}
    thus,
    \begin{equation}
        \bar{\vec{u}}^{k} = \bar{\vec{u}}^{k-1} -\frac{\eta_{k-1}}{T} \cdot \sum_{t=1}^{T} \hat{\delta}_{t}^{k-1}.
    \end{equation}
    Using the fact that $g^{k}$ is $L$-smooth with respect to $\vec{u}$ (Assumption~\ref{assum:smoothness_bis}), we write
    \begin{flalign}
        \mathbb{E} \Bigg[g^{k}\left(\bar{\vec{u}}^{k}, \vec{v}^{k-1}_{1:T}\right) & \Bigg] = \mathbb{E}\Bigg[g^{k}\left( \bar{\vec{u}}^{k-1}  - \frac{\eta_{k-1}}{T}\sum_{t=1}^{T}\hat{\delta}_{t}^{k-1}, \vec{v}^{k-1}_{1:T}\right)\Bigg] &
        \\
        &\leq g^{k}(\bar{\vec{u}}^{k-1}, \vec{v}^{k-1}_{1:T}) - \mathbb{E}\biggl< \nabla_{\vec{u}} g^{k}(\bar{\vec{u}}^{k-1}, \vec{v}^{k-1}_{1:T}), \frac{\eta_{k-1}}{T}\sum_{t=1}^{T}\hat{\delta}^{k-1}_{t}\biggr> \nonumber \\
        & \qquad + \frac{L}{2}\mathbb{E}\left\| \frac{\eta_{k-1}}{T}\sum_{t=1}^{T}\hat{\delta}_{t}^{k-1}\right\|^{2}
        \\
        &=  g^{k}(\bar{\vec{u}}^{k-1}, \vec{v}^{k-1}_{1:T}) - \eta_{k-1} \underbrace{\mathbb{E}\biggl< \nabla_{\vec{u}} g^{k}(\bar{\vec{u}}^{k-1}, \vec{v}^{k-1}_{1:T}), \frac{1}{T}\sum_{t=1}^{T}\hat{\delta}^{k-1}_{t}\biggr>}_{\triangleq T_{1}} \nonumber
        \\
        & \qquad + \frac{\eta_{k-1}^{2} \cdot L}{2T^{2}}\underbrace{\mathbb{E}\left\| \sum_{t=1}^{T}\hat{\delta}_{t}^{k-1}\right\|^{2}}_{\triangleq T_{2}}, \label{eq:descent_reccursion_step_1}
    \end{flalign}
    where the expectation is taken over local random batches. As in the client-server case, we bound the terms $T_{1}$ and $T_{2}$. First, we bound  $T_{1}$, for $k>0$, we have
    \begin{flalign}
        T_{1} &= \mathbb{E}\biggl< \nabla_{\vec{u}} g^{k}(\bar{\vec{u}}^{k-1}, \vec{v}_{1:T}^{k-1}), \frac{1}{T}\sum_{t=1}^{T}\hat{\delta}_{t}^{k-1}\biggr>
        & \\
        &= \underbrace{\mathbb{E}\biggl<\nabla_{\vec{u}} g^{k}\left(\bar{\vec{u}}^{k-1}, \vec{v}_{1:T}^{k-1}\right), \frac{1}{T}\sum_{t=1}^{T}\left(\hat{\delta}_{t}^{k-1} -\delta_{t}^{k-1} \right)\biggr>}_{=0,~\text{because}~ \mathbb{E}\left[\delta^{k-1}_{t} - \hat{\delta}^{k-1}_{t}\right] = 0} \nonumber \\
        & \qquad \qquad \qquad + \mathbb{E}\biggl<\nabla_{\vec{u}} g^{k}\left(\bar{\vec{u}}^{k-1}, \vec{v}_{1:T}^{k-1}\right), \frac{1}{T}\sum_{t=1}^{T}\delta_{t}^{k-1}\biggr>
        \\
        &= \mathbb{E}\biggl<\nabla_{\vec{u}} g^{k}\left(\bar{\vec{u}}^{k-1}, \vec{v}_{1:T}^{k-1}\right), \frac{1}{T} \sum_{t=1}^{T}\delta_{t}^{k-1}\biggr>
        \\
        &= \frac{1}{2}\E\left\|\nabla_{\vec{u}} g^{k}\left(\bar{\vec{u}}^{k-1}, \vec{v}_{1:T}^{k-1}\right)\right\|^{2} + \frac{1}{2} \mathbb{E}\left\|\frac{1}{T}\sum_{t=1}^{T}\delta_{t}^{k-1}\right\|^{2} \nonumber
        \\
        & \qquad \qquad \qquad -\frac{1}{2}\mathbb{E}\left\|\nabla_{\vec{u}} g^{k}\left(\bar{\vec{u}}^{k-1}, \vec{v}_{1:T}^{k-1}\right)  - \frac{1}{T}\sum_{t=1}^{T}\delta_{t}^{k-1}\right\|^{2}. \label{eq:bound_t1_decentralized}
    \end{flalign}
    We bound now $T_{2}$. For $k >0$, we have,
    \begin{flalign}
        T_{2} &= \mathbb{E}\left\|\sum_{t=1}^{T}\hat{\delta}_{t}^{k-1}\right\|^{2} & 
        \\
        &= \mathbb{E}\left\|\sum_{t=1}^{T}\left(\hat{\delta}_{t}^{k-1} - \delta_{t}^{k-1}\right) + \sum_{t=1}^{T}\delta_{t}^{k-1}\right\|^{2}
        \\
        &\leq  2\mathbb{E}\left\|\sum_{t=1}^{T}\left(\hat{\delta}_{t}^{k-1} - \delta_{t}^{k-1}\right)\right\|^{2} + 2\cdot\mathbb{E}\left\| \sum_{t=1}^{T}\delta_{t}^{k-1}\right\|^{2}
        \\
        &= 2\cdot \sum_{t=1}^{T}\E\left\|\hat{\delta}_{t}^{k-1} - \delta_{t}^{k-1}\right\|^{2} +2\sum_{1\leq t\neq s\leq T}\mathbb{E} \underbrace{\biggl< \hat{\delta}_{t}^{k-1} - \delta_{t}^{k-1}, \hat{\delta}_{s}^{k-1} - \delta_{s}^{k-1}\biggr>}_{=0;~\text{because of Eq.~\eqref{eq:covarianceof_local_updates_is_zero_decentralized}}}   \nonumber
        \\
        & \qquad \qquad \qquad +2\mathbb{E}\left\| \sum_{t=1}^{T}\delta_{t}^{k-1}\right\|^{2}
        \\
        &=  2\cdot \sum_{t=1}^{T}\E\left\|\hat{\delta}_{t}^{k-1} - \delta_{t}^{k-1}\right\|^{2} +2\cdot \mathbb{E}\left\| \sum_{t=1}^{T}\delta_{t}^{k-1}\right\|^{2}
        \\
        &= 2\cdot \mathbb{E}\left\| \sum_{t=1}^{T}\delta_{t}^{k-1}\right\|^{2} + 2 \cdot \sum_{t=1}^{T}\Bigg(\frac{1}{ \eta_{k-1}^{2}}\E\Bigg\|\sum_{j=0}^{J-1}\eta_{k-1,j} \cdot \Big[\nabla_{\vec{u}} g^{k}_{t}\left(\vec{u}^{k-1,j}_{t}, \vec{v}_{t}^{k-1}\right) \nonumber
        \\
        & \qquad \qquad \qquad \qquad \qquad \qquad \qquad \qquad \qquad  - \nabla_{\vec{u}} g^{k}_{t}\left(\vec{u}^{k-1,j}_{t}, \vec{v}_{t}^{k-1}; \xi^{k-1,j}_{t}\right)\Big]\Bigg\|^{2} \Bigg).
    \end{flalign}
    Since batches are sampled independently, and stochastic gradients are unbiased with finite variance (Assumption~\ref{assum:finite_variance_bis}), the last term in the RHS of the previous equation can be bounded using $\sigma^{2}$, leading to 
    \begin{flalign}
        T_{2 }& \leq  2\cdot \sum_{t=1}^{T}\left[ \frac{\sum_{j=0}^{J-1}\eta_{k-1,j}^{2}}{\eta_{k-1}^{2}}\sigma^{2}\right] +2\cdot \mathbb{E}\left\| \sum_{t=1}^{T}\delta_{t}^{k-1}\right\|^{2} &
        \\
        &= 2T\cdot  \sigma^{2} \cdot \left(\sum_{t=1}^{T}\cdot \frac{\sum_{j=0}^{J-1}\eta_{k-1,j}^{2}}{ \eta_{k-1}^{2}}\right) +2\mathbb{E}\left\| \sum_{t=1}^{T}\delta_{t}^{k-1}\right\|^{2}
        \\
        &\leq 2T \cdot \sigma^{2}  +2 \cdot \mathbb{E}\left\| \sum_{t=1}^{T}\delta_{t}^{k-1}\right\|^{2}.\label{eq:bound_t2_decentralized}
    \end{flalign}
    Replacing Eq.~\eqref{eq:bound_t1_decentralized} and Eq.~\eqref{eq:bound_t2_decentralized} in Eq.~\eqref{eq:descent_reccursion_step_1}, we have
    \begin{flalign}
        \mathbb{E}\Bigg[g^{k}&(\bar{\vec{u}}^{k}, \vec{v}^{k-1}_{1:T}) - g^{k}(\bar{\vec{u}}^{k-1}, \vec{v}^{k-1}_{1:T})\bigg] \leq & \nonumber 
        \\
        & -\frac{\eta_{k-1}}{2}\E\left\|\nabla_{\vec{u}} g^{k}\left(\bar{\vec{u}}^{k-1}, \vec{v}^{k-1}_{1:T}\right)\right\|^{2} -\frac{\eta_{k-1}}{2}\left(1-2L\eta_{k-1}\right)\mathbb{E}\left\|\frac{1}{T}\sum_{t=1}^{T}\delta_{t}^{k-1}\right\|^{2} \nonumber
        \\
        & + \frac{L}{T} \eta_{k-1}^{2}  \sigma^{2} + \frac{\eta_{k-1}}{2}\mathbb{E}\left\|\nabla_{\vec{u}} g^{k}\left(\bar{\vec{u}}^{k-1}, \vec{v}^{k-1}_{1:T}\right)  - \frac{1}{T}\sum_{t=1}^{T}\delta_{t}^{k-1}\right\|^{2}.
    \end{flalign}
    For $\eta_{k-1}$  small enough, in particular for $\eta_{k-1} \leq \frac{1}{2L}$, we have
    \begin{flalign}
        \mathbb{E}\Bigg[g^{k}(\bar{\vec{u}}^{k}, \vec{v}_{1:T}^{k-1}) &- g^{k}(\bar{\vec{u}}^{k-1}, \vec{v}_{1:T}^{k-1})\Bigg] \leq & \nonumber
        \\
        &  \qquad -\frac{\eta_{k-1}}{2}\E\left\|\nabla_{\vec{u}} g^{k}\left(\bar{\vec{u}}^{k-1}, \vec{v}_{1:T}^{k-1}\right)\right\|^{2} + \frac{L}{T} \eta_{k-1}^{2}  \sigma^{2} \nonumber
        \\
        & \qquad + \frac{\eta_{k-1}}{2}\mathbb{E}\left\|\frac{1}{T}\sum_{t=1}^{T}\left(\nabla_{\vec{u}} g_{t}^{k}\left(\bar{\vec{u}}^{k-1}, \vec{v}_{t}^{k-1}\right)  - \delta_{t}^{k-1}\right)\right\|^{2}.
    \end{flalign}
    We use Jensen inequality to bound the last term in the RHS of the previous equation, leading to
    \begin{flalign}
        \mathbb{E}\Bigg[g^{k}(\bar{\vec{u}}^{k}, \vec{v}_{1:T}^{k-1}) &- g^{k}(\bar{\vec{u}}^{k-1}, \vec{v}_{1:T}^{k-1})\Bigg] \leq & \nonumber
        \\
        &  \qquad -\frac{\eta_{k-1}}{2}\E\left\|\nabla_{\vec{u}} g^{k}\left(\bar{\vec{u}}^{k-1}, \vec{v}_{1:T}^{k-1}\right)\right\|^{2} + \frac{L}{T} \eta_{k-1}^{2}  \sigma^{2} \nonumber
        \\
        & \qquad + \frac{\eta_{k-1}}{2T} \cdot \sum_{t=1}^{T}\underbrace{\E\left\|\nabla_{\vec{u}} g_{t}^{k}\left(\bar{\vec{u}}^{k-1}, \vec{v}_{t}^{k-1}\right)  - \delta_{t}^{k-1}\right\|^{2}}_{T_{3}}.
        \label{eq:bound_g_imporvement_u_bar_part_1}
    \end{flalign}
    We bound now the term $T_{3}$:
    \begin{flalign}
        T_{3} &= \E\left\|\nabla_{\vec{u}} g_{t}^{k}\left(\bar{\vec{u}}^{k-1}, \vec{v}_{t}^{k-1}\right)  - \delta_{t}^{k-1}\right\|^{2} &
        \\ 
        &= \E\left\|\nabla_{\vec{u}} g_{t}^{k}\left(\bar{\vec{u}}^{k-1}, \vec{v}_{t}^{k-1}\right) - \frac{\sum_{j=0}^{J-1} \eta_{k-1,j} \cdot \nabla_{\vec{u}} g_{t}^{k}\left(\vec{u}_{t}^{k-1,j}, \vec{v}_{t}^{k-1}\right)}{\eta_{k-1}}\right\|^{2}
        \\
        &= \E \left\| \sum_{j=0}^{J-1} \frac{\eta_{k-1,j}}{\eta_{k-1}} \cdot\left[\nabla_{\vec{u}} g_{t}^{k}\left(\bar{\vec{u}}^{k-1}, \vec{v}_{t}^{k-1}\right) -  \nabla_{\vec{u}} g_{t}^{k}\left(\vec{u}_{t}^{k-1,j}, \vec{v}_{t}^{k-1}\right)\right]\right\|^{2}.
    \end{flalign}
    Using Jensen inequality, it follows that
    \begin{flalign}
        T_{3} &\leq \sum_{j=0}^{J-1} \frac{\eta_{k-1,j}}{\eta_{k-1}} \cdot  \E \left\| \nabla_{\vec{u}} g_{t}^{k}\left(\bar{\vec{u}}^{k-1}, \vec{v}_{t}^{k-1}\right) -  \nabla_{\vec{u}} g_{t}^{k}\left(\vec{u}_{t}^{k-1,j}, \vec{v}_{t}^{k-1}\right)\right\|^{2} & 
        \\
        &= \sum_{j=0}^{J-1} \frac{\eta_{k-1,j}}{\eta_{k-1}} \cdot  \E \Bigg\| \nabla_{\vec{u}} g_{t}^{k}\left(\bar{\vec{u}}^{k-1}, \vec{v}_{t}^{k-1}\right) -  \nabla_{\vec{u}} g_{t}^{k}\left(\vec{u}_{t}^{k-1}, \vec{v}_{t}^{k-1}\right) \nonumber
        \\
        & \qquad \qquad \qquad \qquad \qquad + \nabla_{\vec{u}} g_{t}^{k}\left(\vec{u}_{t}^{k-1}, \vec{v}_{t}^{k-1}\right) -  \nabla_{\vec{u}} g_{t}^{k}\left(\vec{u}_{t}^{k-1,j}, \vec{v}_{t}^{k-1}\right)\Bigg\|^{2}
        \\
        &\leq 2\cdot  \E \Bigg\| \nabla_{\vec{u}} g_{t}^{k}\left(\bar{\vec{u}}^{k-1}, \vec{v}_{t}^{k-1}\right) -  \nabla_{\vec{u}} g_{t}^{k}\left(\vec{u}_{t}^{k-1}, \vec{v}_{t}^{k-1}\right)  \Bigg\|^{2} \nonumber
        \\
        &\qquad \quad + 2 \cdot  \sum_{j=0}^{J-1} \frac{\eta_{k-1,j}}{\eta_{k-1}} \cdot \E \Bigg\|\nabla_{\vec{u}} g_{t}^{k}\left(\vec{u}_{t}^{k-1}, \vec{v}_{t}^{k-1}\right) -  \nabla_{\vec{u}} g_{t}^{k}\left(\vec{u}_{t}^{k-1,j}, \vec{v}_{t}^{k-1}\right)\Bigg\|^{2}
        \\ 
        &\leq 2L^{2} \cdot \E\left\|\bar{\vec{u}}^{k-1} - \vec{u}_{t}^{k-1}\right\|^{2} + 2L^{2} \cdot  \sum_{j=0}^{J-1} \frac{\eta_{k-1,j}}{\eta_{k-1}} \cdot \E\left\|\vec{u}_{t}^{k-1, j} - \vec{u}_{t}^{k-1, 0}\right\|^{2},
        \label{eq:bound_t3_decentralized_part1}
    \end{flalign}
    where we used the $L$-smoothness of $g^{k}_{t}$ (Assumption~\ref{assum:smoothness_bis}) to obtain the last inequality. As in the centralized case (Lemma~\ref{lem:centralized_case_bound_gradient_g_theta}), we bound terms $\left\|\vec{u}_{t}^{k-1, j} - \vec{u}_{t}^{k-1, 0}\right\|^{2},~j\in\left\{0,\dots, J-1\right\}$. Using exactly the same steps as in the proof of Lemma~\ref{lem:centralized_case_bound_gradient_g_theta}, Eq.~\eqref{eq:bound_theta_j_minus_theta_0_part3} holds with $\vec{u}_{t}^{k-1,0}$ instead of  $\vec{u}^{k-1}_{t}$, i.e.,
    \begin{flalign}
        \left(1 - 4\eta_{k-1}^{2}L^{2}\right) & \cdot \sum_{j=0}^{J-1}  \frac{\eta_{k-1,j}}{\eta_{k-1}}  \cdot \mathbb{E}\Big\|\vec{u}^{k-1, 0}_{t} - \vec{u}_{t}^{k-1,j}\Big\|^{2} \leq   2\sigma^{2} \cdot \left\{ \sum_{j=0}^{J-1} \eta_{k-1,j}^{2}\right\} \nonumber &
        \\
        & \qquad + 4\eta_{k-1}^{2} \cdot \mathbb{E}\left\|\nabla_{u} g^{k}_{t}\left(\vec{u}_{t}^{k-1, 0}, \vec{v}_{t}^{k-1}\right)\right\|^{2}. 
    \end{flalign}
    For $\eta_{k-1}$ small enough, in particular for $\eta_{k-1} \leq \frac{1}{2\sqrt{2} L}$, we have
    \begin{flalign}
        \sum_{j=0}^{J-1} & \frac{\eta_{k-1,j}}{\eta_{k-1}}  \cdot \mathbb{E}\Big\|\vec{u}^{k-1, 0}_{t}  - \vec{u}_{t}^{k-1,j}\Big\|^{2} & \nonumber \\
        & \leq 8\eta_{k-1}^{2} \cdot \mathbb{E}\left\|\nabla_{u} g^{k}_{t}\left(\vec{u}_{t}^{k-1, 0}, \vec{v}_{t}^{k-1}\right)\right\|^{2} + 4\sigma^{2} \cdot \left\{ \sum_{j=0}^{J-1} \eta_{k-1,j}^{2}\right\}  
        \\
        &\leq 8\eta_{k-1}^{2} \cdot \mathbb{E}\left\|\nabla_{u} g^{k}_{t}\left(\vec{u}_{t}^{k-1, 0}, \vec{v}_{t}^{k-1}\right) -\nabla_{u} g^{k}_{t}\left(\bar{\vec{u}}^{k-1}, \vec{v}_{t}^{k-1}\right) + \nabla_{u} g^{k}_{t}\left(\bar{\vec{u}}^{k-1}, \vec{v}_{t}^{k-1}\right) \right\|^{2} \nonumber
        \\
        & \qquad + 4\sigma^{2} \cdot \left\{ \sum_{j=0}^{J-1} \eta_{k-1,j}^{2}\right\}
        \\
        &\leq 16\eta_{k-1}^{2} \cdot \mathbb{E}\left\|\nabla_{u} g^{k}_{t}\left(\vec{u}_{t}^{k-1, 0}, \vec{v}_{t}^{k-1}\right) -\nabla_{u} g^{k}_{t}\left(\bar{\vec{u}}^{k-1}, \vec{v}_{t}^{k-1}\right)\right\|^{2} \nonumber
        \\
        & \qquad + 16 \eta_{k-1}^{2} \cdot \left\| \nabla_{u} g^{k}_{t}\left(\bar{\vec{u}}^{k-1}, \vec{v}_{t}^{k-1}\right) \right\|^{2}  + 4\sigma^{2} \cdot \left\{ \sum_{j=0}^{J-1} \eta_{k-1,j}^{2}\right\}
        \\
        &\leq 16\eta_{k-1}^{2} L^{2} \cdot \mathbb{E}\left\|\vec{u}_{t}^{k-1} -\bar{\vec{u}}^{k-1}\right\|^{2} + 16 \eta_{k-1}^{2} \cdot \left\| \nabla_{u} g^{k}_{t}\left(\bar{\vec{u}}^{k-1}, \vec{v}_{t}^{k-1}\right) \right\|^{2} \nonumber
        \\
        & \qquad   + 4\sigma^{2} \cdot \left\{ \sum_{j=0}^{J-1} \eta_{k-1,j}^{2}\right\},
        \label{eq:u_j_minus_u_0_decentralized}
    \end{flalign}
    where the last inequality follows from the $L$-smoothness of $g_{t}^{k}$. Replacing Eq.~\eqref{eq:u_j_minus_u_0_decentralized} in Eq.~\eqref{eq:bound_t3_decentralized_part1}, we have
    \begin{flalign}
        T_{3} & \leq 32\eta_{k-1}^{2} L^{4}\cdot \mathbb{E}\left\| \vec{u}_{t}^{k-1} -\bar{\vec{u}}^{k-1}\right\|^{2}  + 8L^{2}\sigma^{2} \cdot \left\{ \sum_{j=0}^{J-1} \eta_{k-1,j}^{2}\right\}  \nonumber
        \\
        & \qquad +  32\eta_{k-1}^{2}  L^{2}\cdot \E \left\|\nabla_{u} g^{k}_{t}\left(\bar{\vec{u}}^{k-1}, \vec{v}_{t}^{k-1}\right) \right\|^{2} + 2L^{2} \cdot \mathbb{E}\left\|\bar{\vec{u}}^{k-1} - \vec{u}_{t}^{k-1}\right\|^{2}.
    \end{flalign}
    For $\eta_{k}$ small enough, in particular if $\eta_{k} \leq \frac{1}{2\sqrt{2}L}$ we have,
    \begin{flalign}
        \label{eq:bound_t3_decentralized}
        T_{3} & \leq 6 L^{2}  \mathbb{E}\left\| \vec{u}_{t}^{k-1} -\bar{\vec{u}}^{k-1}\right\|^{2}  + 8L^{2}\sigma^{2}  \sum_{j=0}^{J-1} \eta_{k-1,j}^{2}   +  32\eta_{k-1}^{2}  L^{2}  \left\|\nabla_{u} g^{k}_{t}\left(\bar{\vec{u}}^{k-1}, \vec{v}_{t}^{k-1}\right) \right\|^{2}.& 
    \end{flalign}
    Replacing Eq.~\eqref{eq:bound_t3_decentralized} in Eq.~\eqref{eq:bound_g_imporvement_u_bar_part_1}, we have
    \begin{flalign}
         \mathbb{E}\Bigg[g^{k} & (\bar{\vec{u}}^{k}  , \vec{v}_{1:T}^{k-1}) - g^{k}(\bar{\vec{u}}^{k-1}, \vec{v}_{1:T}^{k-1})\Bigg] \leq & \nonumber
        \\
        &   \frac{3\eta_{k-1}L^{2}}{T} \cdot \sum_{t=1}^{T} \E\left\|\vec{u}_{t}^{k-1} - \bar{\vec{u}}^{k-1}\right\|^{2}  + \frac{\eta_{k-1}^{2}L}{T} \left(4 \sum_{j=0}^{J-1}\frac{T L\cdot \eta_{k-1,j}^{2}}{\eta_{k-1}} + 1\right) \sigma^{2} \nonumber
        \\
        &  -\frac{\eta_{k-1}}{2}\E\left\|\nabla_{\vec{u}} g^{k}\left(\bar{\vec{u}}^{k-1}, \vec{v}_{1:T}^{k-1}\right)\right\|^{2} + \frac{16\eta_{k-1}^{3}L^{2}}{T} \sum_{t=1}^{T} \left\|\nabla_{u} g^{k}_{t}\left(\bar{\vec{u}}^{k-1}, \vec{v}_{t}^{k-1}\right) \right\|^{2}.
    \end{flalign}
    We use now Assumption~\ref{assum:bounded_dissimilarity_bis} to bound the last term in the RHS of the previous equation, leading to
     \begin{flalign}
         \mathbb{E}\Bigg[g^{k} & (\bar{\vec{u}}^{k}  , \vec{v}_{1:T}^{k-1}) - g^{k}(\bar{\vec{u}}^{k-1}, \vec{v}_{1:T}^{k-1})\Bigg] \leq & \nonumber
        \\
        &   \frac{3\eta_{k-1}L^{2}}{T} \cdot \sum_{t=1}^{T} \E\left\|\vec{u}_{t}^{k-1} - \bar{\vec{u}}^{k-1}\right\|^{2}  + \frac{\eta_{k-1}^{2}L}{T} \left(4 \sum_{j=0}^{J-1}\frac{T L \cdot \eta_{k-1,j}^{2}}{\eta_{k-1}} + 1\right) \sigma^{2} \nonumber
        \\
        &  -\frac{\eta_{k-1}\cdot\left( 1 - 32\eta^{2}_{k-1}L^{2}\beta^{2}\right)}{2}\E\left\|\nabla_{\vec{u}} g^{k}\left(\bar{\vec{u}}^{k-1}, \vec{v}_{1:T}^{k-1}\right)\right\|^{2} + \frac{16\eta_{k-1}^{3}L^{2}}{T} G^{2}.
    \end{flalign}
    For $\eta_{k-1}$ small enough, in particular, if $\eta_{k-1} \leq  \frac{1}{8L\beta}$, we have
    \begin{flalign}
         \mathbb{E}\Bigg[g^{k} & (\bar{\vec{u}}^{k}  , \vec{v}_{1:T}^{k-1}) - g^{k}(\bar{\vec{u}}^{k-1}, \vec{v}_{1:T}^{k-1})\Bigg] \leq   & \nonumber
        \\
        &  -\frac{\eta_{k-1}}{4}\E\left\|\nabla_{\vec{u}} g^{k}\left(\bar{\vec{u}}^{k-1}, \vec{v}_{1:T}^{k-1}\right)\right\|^{2} + \frac{3\eta_{k-1}L^{2}}{T} \cdot \sum_{t=1}^{T} \E\left\|\vec{u}_{t}^{k-1} - \bar{\vec{u}}^{k-1}\right\|^{2}  \nonumber
        \\
        &  + \frac{\eta_{k-1}^{2}L}{T} \left(4 \sum_{j=0}^{J-1}\frac{T L \cdot \eta_{k-1,j}^{2}}{\eta_{k-1}} + 1\right) \sigma^{2} + \frac{16\eta_{k-1}^{3}L^{2}}{T} G^{2}.
    \end{flalign}
    We use Lemma~\ref{lem:bound_g_theta_bar} to get
    \begin{flalign}
         \mathbb{E}\Bigg[g^{k} & (\bar{\vec{u}}^{k}  , \vec{v}_{1:T}^{k-1}) - f(\bar{\vec{u}}^{k-1}, \vec{v}_{1:T}^{k-1})\Bigg] \leq   & \nonumber
        \\
        &  -\frac{\eta_{k-1}}{8}\E\left\|\nabla_{\vec{u}} f\left(\bar{\vec{u}}^{k-1}, \vec{v}_{1:T}^{k-1}\right)\right\|^{2} +  \frac{\left(12 + T\right)\eta_{k-1}L^{2}}{4T} \cdot \sum_{t=1}^{T} \E\left\|\vec{u}_{t}^{k-1} - \bar{\vec{u}}^{k-1}\right\|^{2}  \nonumber
        \\
        &  + \frac{\eta_{k-1}^{2}L}{T} \left(4 \sum_{j=0}^{J-1}\frac{L \cdot \eta_{k-1,j}^{2}}{\eta_{k-1}} + 1\right) \sigma^{2} + \frac{16\eta_{k-1}^{3}L^{2}}{T} G^{2}.
    \end{flalign}
    Finally, since $g_{t}^{k}$ is a partial first-order surrogate of $f_{t}$ near $\left\{\vec{u}^{k-1}, \vec{v}_{t}^{k-1}\right\}$, we have
    \begin{flalign}
         \mathbb{E}\Bigg[f & (\bar{\vec{u}}^{k}  , \vec{v}_{1:T}^{k}) - f(\bar{\vec{u}}^{k-1}, \vec{v}_{1:T}^{k-1})\Bigg] \leq  -\frac{1}{T}\sum_{t=1}^{T}\E d_{\mathcal{V}}\left(\vec{v}_{t}^{k}, \vec{v}_{t}^{k-1}\right) & \nonumber
        \\
        &  -\frac{\eta_{k-1}}{8}\E\left\|\nabla_{\vec{u}} f\left(\bar{\vec{u}}^{k-1}, \vec{v}_{1:T}^{k-1}\right)\right\|^{2} +  \frac{\left(12 + T\right)\eta_{k-1}L^{2}}{4T} \cdot \sum_{t=1}^{T} \E\left\|\vec{u}_{t}^{k-1} - \bar{\vec{u}}^{k-1}\right\|^{2}  \nonumber
        \\
        &  + \frac{\eta_{k-1}^{2}L}{T} \left(4 \sum_{j=0}^{J-1}\frac{L \cdot \eta_{k-1,j}^{2}}{\eta_{k-1}} + 1\right) \sigma^{2} + \frac{16\eta_{k-1}^{3}L^{2}}{T} G^{2}.
    \end{flalign}
\end{proof}

\begin{lem}[Recursion for consensus distance, part 1]
    \label{lem:consensus_recursion_part1}
    Suppose that Assumptions~\ref{assum:smoothness_bis}--\ref{assum:bounded_dissimilarity_bis} and Assumption~\ref{assum:expected_consensus_rate} hold. For $k\geq \tau$, consider $m = \floor*{\frac{k}{\tau}}-1$ and $\left(\eta_{k, j}\right)_{1\leq j \leq J-1}$ 
    such that $\eta_k\triangleq \sum_{j=0}^{J-1}\eta_{k, j} \leq\min \left\{ \frac{1}{4L}, \frac{1}{4L\beta}\right\}$  then, the updates of fully decentralized federated surrogate optimization (Alg~\ref{alg:decentralized_surrogate_opt}) verify
    \begin{flalign*}
        \mathbb{E}\sum_{t=1}^{T} &\left\|\vec{u}_{t}^{k} - \bar{\vec{u}}^{k}\right\|_{F}^{2} \leq & \nonumber
        \\
        & (1-\frac{p}{2})\mathbb{E}\left\|\vec{U}^{m\tau} - \bar{\vec{U}}^{m\tau}\right\|_{F}^{2} + 44\tau\left(1+\frac{2}{p}\right)L^{2}\sum_{l=m\tau}^{k-1}\eta_{l}^{2}\E\left\|\vec{U}^{l} - \bar{\vec{U}}^{l}\right\|_{F}^{2} \nonumber
        \\
        & ~ + T\cdot \sigma^{2} \cdot \sum_{l=m\tau}^{k-1}\left\{\eta_{l}^{2} + 16\tau L^{2}\left(1+\frac{2}{p}\right)\cdot \left\{\sum_{j=0}^{J-1}\eta_{l,j}^{2}\right\} \right\} + 16\tau\left(1+\frac{2}{p}\right)G^{2}\sum_{l=m\tau}^{k-1}\eta_{l}^{2} \nonumber
        \\
        & ~ +  16\tau\left(1+\frac{2}{p}\right)\beta^{2}\sum_{l=m\tau}^{k-1}\eta_{l}^{2}  \E\left\| \nabla_{\vec{u}}f \left(\bar{\vec{u}}^{l,j}, \vec{v}_{1:T}^{l}\right)\right\|^{2}.
    \end{flalign*}
\end{lem}
\begin{proof}
    For $k\geq \tau$, and $m = \floor*{\frac{k}{\tau}}-1$, we have
    \begin{flalign}
        \mathbb{E}\sum_{t=1}^{T}\left\|\vec{u}_{t}^{k} - \bar{\vec{u}}^{k}\right\|_{F}^{2} &= \mathbb{E}\left\|\vec{U}^{k} - \bar{\vec{U}}^{k}\right\|_{F}^{2} &
        \\
        & =  \mathbb{E}\left\|\vec{U}^{k} - \bar{\vec{U}}^{m\tau} -\left(\bar{\vec{U}}^{k} -  \bar{\vec{U}}^{m\tau} \right)\right\|_{F}^{2}
        \\
        & \leq \mathbb{E}\left\|\vec{U}^{k} - \bar{\vec{U}}^{m\tau}\right\|_{F}^{2},
    \end{flalign}
    where we used the fact that $\left\|A - \bar{A}\right\|_{F}^{2} = \left\|A \cdot \left(I - \frac{\mathbf{1}\mathbf{1}^{\intercal}}{T}\right) \right\|_{F}\leq \left\|I - \frac{\mathbf{1}\mathbf{1}^{\intercal}}{T}\right\|_{2}\cdot \left\|A\right\|_{F}^{2} = \left\|A\right\|_{F}^{2}$ to obtain the last inequality.
    Using Eq.~\eqref{eq:decentralized_surrogate_sgd_full_matricial} recursively, we have
    \begin{flalign}
        \vec{U}^{k} = \vec{U}^{m\tau}\left\{\prod_{l'=m\tau}^{k-1}W^{l'}\right\} - \sum_{l=m\tau}^{k-1}\eta_{l}\hat{\Upsilon}^{l}\left\{\prod_{l'=l}^{k-1}W^{l'}\right\}.
    \end{flalign}
    Thus,
    \begin{flalign}
         \mathbb{E}\sum_{t=1}^{T} & \Big\|\vec{u}_{t}^{k} -  \bar{\vec{u}}^{k}\Big\|_{F}^{2}  \leq  \E\left\|\vec{U}^{m\tau}\left\{\prod_{l'=m\tau}^{k-1}W^{l'}\right\} - \bar{\vec{U}}^{m\tau} - \sum_{l=m\tau}^{k-1}\eta_{l}\hat{\Upsilon}^{l}\left\{\prod_{l'=l}^{k-1}W^{l'}\right\}\right\|^{2}_{F} &
         \\
        & = \mathbb{E}\Bigg\|\vec{U}^{m\tau}\left\{\prod_{l'=m\tau}^{k-1}W^{l'}\right\} - \bar{\vec{U}}^{m\tau} - \sum_{l=m\tau}^{k-1}\eta_{l}\Upsilon^{l}\left\{\prod_{l'=l}^{k-1}W^{l'}\right\} \nonumber
        \\
        & \qquad \qquad +  \sum_{l=m\tau}^{k-1}\eta_{l}\left(\Upsilon^{l} - \hat{\Upsilon}^{l}\right)\left\{\prod_{l'=l}^{k-1}W^{l'}\right\} \Bigg\|^{2}_{F} 
        \\
        & = \mathbb{E}\left\|\vec{U}^{m\tau}\left\{\prod_{l'=m\tau}^{k-1}W^{l'}\right\} - \bar{\vec{U}}^{m\tau} - \sum_{l=m\tau}^{k-1}\eta_{l}\Upsilon^{l}\left\{\prod_{l'=l}^{k-1}W^{l'}\right\}\right\|^{2}_{F} \nonumber
        \\
        & \qquad \qquad +  \mathbb{E}\left\|\sum_{l=m\tau}^{k-1}\eta_{l}\left(\Upsilon^{l} - \hat{\Upsilon}^{l}\right)\left\{\prod_{l'=l}^{k-1}W^{l'}\right\}\right\|_{F}^{2} \nonumber
        \\
        &  \qquad + 2\mathbb{E}\bigg<\vec{U}^{m\tau}\left\{\prod_{l'=m\tau}^{k-1}W^{l'}\right\} - \bar{\vec{U}}^{m\tau} - \sum_{l=m\tau}^{k-1}\eta_{l}\Upsilon^{l}\left\{\prod_{l'=l}^{k-1}W^{l'}\right\}, \nonumber
        \\
        & \qquad \qquad \qquad \qquad \qquad \qquad \qquad \qquad \qquad  \sum_{l=m\tau}^{k-1}\eta_{l}\left(\Upsilon^{l} - \hat{\Upsilon}^{l}\right)\left\{\prod_{l'=l}^{k-1}W^{l'}\right\} \bigg>_{F}.
    \end{flalign} 
    Since stochastic gradients are unbiased, the last term in the RHS of the previous equation is equal to zero. Using the following standard inequality  for Euclidean norm with $\alpha > 0$, 
    \begin{equation}
        \left\|\vec{a} + \vec{b}\right\|^{2} \leq \left(1+\alpha\right) \left\|\vec{a}\right\|^{2} + \left(1+\alpha^{-1}\right)\left\|\vec{b}\right\|^{2},
    \end{equation}
    we have
    \begin{flalign}
        \mathbb{E}\sum_{t=1}^{T} & \Big\|\vec{u}_{t}^{k} -  \bar{\vec{u}}^{k}\Big\|_{F}^{2} \leq &
        \\
        & \left(1+\alpha\right)\mathbb{E}\left\|\vec{U}^{m\tau}\left\{\prod_{l'=m\tau}^{k-1}W^{l'}\right\} - \bar{\vec{U}}^{m\tau} \right\|^{2}_{F} + \left(1+\alpha^{-1}\right)\mathbb{E}\left\|\sum_{l=m\tau}^{k-1}\eta_{l}\Upsilon^{l}\left\{\prod_{l'=l}^{k-1}W^{l'}\right\}\right\|_{F}^{2}\nonumber 
        \\
        & \qquad \qquad  +  \sum_{l=m\tau}^{k-1}\eta^{2}_{l}\mathbb{E}\left\|\left(\Upsilon^{l} - \hat{\Upsilon}^{l}\right)\left\{\prod_{l'=l}^{k-1}W^{l'}\right\}\right\|_{F}^{2}.
    \end{flalign}
    Since $k\geq (m+1)\tau$ and matrices $\left(W^{l}\right)_{ l\geq 0}$ are doubly stochastic, we have
    \begin{flalign}
        \mathbb{E}& \sum_{t=1}^{T} \left\|\vec{u}_{t}^{k} - \bar{\vec{u}}^{k}\right\|_{F}^{2} \leq &\nonumber \label{eq:consensus_recursion_step0}
        \\
        & \left(1+\alpha\right)\mathbb{E}\left\|\vec{U}^{m\tau}\left\{\prod_{l'=m\tau}^{(m+1)\tau - 1}W^{l'}\right\} - \bar{\vec{U}}^{m\tau} \right\|^{2}_{F} + \left(1+\alpha^{-1}\right)\mathbb{E}\left\|\sum_{l=m\tau}^{k-1}\eta_{l}\Upsilon^{l}\right\|_{F}^{2} \nonumber
        \\
        & \qquad +  \sum_{l=m\tau}^{k-1}\eta^{2}_{l}\mathbb{E}\left\|\Upsilon^{l} - \hat{\Upsilon}^{l}\right\|_{F}^{2} 
       \\
        & \leq \left(1+\alpha\right)\mathbb{E}\left\|\vec{U}^{m\tau}\left\{\prod_{l'=m\tau}^{(m+1)\tau - 1}W^{l'}\right\} - \bar{\vec{U}}^{m\tau} \right\|^{2}_{F} + \left(1+\alpha^{-1}\right)\cdot\left(k - m\tau\right)\sum_{l=m\tau}^{k-1}\eta_{l}^{2}\mathbb{E}\left\|\Upsilon^{l}\right\|_{F}^{2} \nonumber
        \\
        & \qquad +  \sum_{l=m\tau}^{k-1}\eta^{2}_{l}\mathbb{E}\left\|\Upsilon^{l} - \hat{\Upsilon}^{l}\right\|_{F}^{2},
    \end{flalign}
    where we use the fact that $\left\|AB\right\|_{F} \leq \left\|A\right\|_{2}\left\|B\right\|_{F}$ and that $\left\|A\right\| =1$ when $A$ is a doubly stochastic matrix to obtain the first inequality, and Cauchy-Schwarz inequality to obtain the second one.
    Using Assumption~\ref{assum:expected_consensus_rate} to bound the first term of the RHS of the previous equation and the fact that that $k\leq (m+2)\tau$, it follows that
    \begin{flalign}
        \mathbb{E}\sum_{t=1}^{T} &\left\|\vec{u}_{t}^{k} - \bar{\vec{u}}^{k}\right\|_{F}^{2} \leq & \nonumber
        \\
        & (1+\alpha)(1-p)\mathbb{E}\left\|\vec{U}^{m\tau} - \bar{\vec{U}}^{m\tau}\right\|_{F}^{2} +  2\tau\left(1+\alpha^{-1}\right)\sum_{l=m\tau}^{k-1}\eta_{l}^{2}\mathbb{E}\left\|\Upsilon^{l}\right\|_{F}^{2} \nonumber
        \\
        & \qquad+  \sum_{l=m\tau}^{k-1}\eta^{2}_{l}\mathbb{E}\left\|\Upsilon^{l} - \hat{\Upsilon}^{l}\right\|_{F}^{2}. \label{eq:consensus_recursion_step1}
    \end{flalign}
    We use the fact that stochastic gradients have bounded variance (Assumption~\ref{assum:finite_variance_bis}) to bound $\mathbb{E}\left\|\Upsilon^{l} - \hat{\Upsilon}^{l}\right\|_{F}^{2}$ as follows,
    \begin{flalign}
        \mathbb{E}\left\|\Upsilon^{l} - \hat{\Upsilon}^{l}\right\|_{F}^{2} & = \sum_{t=1}^{T}\E\left\|\delta_{t}^{l} - \hat{\delta}_{t}^{l}\right\|^{2} \\
        &= \sum_{t=1}^{T} \E\Bigg\|\sum_{j=0}^{J-1} \frac{\eta_{l,j}}{\eta_{l}} \cdot \Bigg( \nabla_{\vec{u}} g^{l+1}_{t}\left(\vec{u}^{l,j}_{t}, \vec{v}_{t}^{k-1}\right) - \nabla_{\vec{u}} g^{l+1}_{t}\left(\vec{u}^{l,j}_{t}, \vec{v}_{t}^{l}; \xi^{l,j}_{t}\right) \Bigg)\Bigg\|^{2}
        \\
        & \leq \sum_{t=1}^{T} \sum_{j=0}^{J-1} \frac{\eta_{l,j}}{\eta_{l}} \cdot \E\Bigg\| \Bigg( \nabla_{\vec{u}} g^{l+1}_{t}\left(\vec{u}^{l,j}_{t}, \vec{v}_{t}^{k-1}\right) - \nabla_{\vec{u}} g^{l+1}_{t}\left(\vec{u}^{l,j}_{t}, \vec{v}_{t}^{l}; \xi^{l,j}_{t}\right) \Bigg)\Bigg\|^{2}
        \\
        &\leq  \sum_{t=1}^{T} \sum_{j=0}^{J-1} \frac{\eta_{l,j}}{\eta_{l}} \sigma^{2}
        \\
        &= T \cdot \sigma^{2},
    \end{flalign}
    where we used Jensen inequality to obtain the first inequality and Assumption~\ref{assum:finite_variance_bis} to obtain the second inequality. Replacing back in Eq.~\eqref{eq:consensus_recursion_step1}, we have
    \begin{flalign}
        \mathbb{E}&\sum_{t=1}^{T} \left\|\vec{u}_{t}^{k} - \bar{\vec{u}}^{k}\right\|_{F}^{2} \leq & \nonumber
        \\
        & (1+\alpha)(1-p)\mathbb{E}\left\|\vec{U}^{m\tau} - \bar{\vec{U}}^{m\tau}\right\|_{F}^{2} +  2\tau\left(1+\alpha^{-1}\right)\sum_{l=m\tau}^{k-1}\eta_{l}^{2}\mathbb{E}\left\|\Upsilon^{l}\right\|_{F}^{2} + T\cdot \sigma^{2}\cdot \left\{\sum_{l=m\tau}^{k-1}\eta_{l}^{2}\right\}. \label{eq:consensus_recursion_step2}
    \end{flalign}
    The last step of the proof consists in bounding $\mathbb{E}\left\|\Upsilon^{l}\right\|^{2}_{F}$ for $l\in\left\{m\tau, \dots, k-1\right\}$,
    \begin{flalign}
        \mathbb{E}\left\|\Upsilon^{l}\right\|_{F}^{2} &= \sum_{t=1}^{T} \E\left\|\delta^{l}_{t}\right\|^{2} &
        \\
        & = \sum_{t=1}^{T}\E\left\|\sum_{j=0}^{J-1}\frac{\eta_{l,j}}{\eta_{l}} \cdot \nabla_{\vec{u}} g_{t}^{l+1}\left(\vec{u}_{t}^{l,j}, \vec{v}_{t}^{l}\right)\right\|^{2}
        \\
        & \leq \sum_{t=1}^{T} \sum_{j=0}^{J-1}\frac{\eta_{l,j}}{\eta_{l}} \cdot \E\left\| \nabla_{\vec{u}} g_{t}^{l+1}\left(\vec{u}_{t}^{l,j}, \vec{v}_{t}^{l}\right)\right\|^{2}
        \\
        & \leq \sum_{t=1}^{T} \sum_{j=0}^{J-1}\frac{\eta_{l,j}}{\eta_{l}} \cdot \E\left\| \nabla_{\vec{u}} g_{t}^{l+1}\left(\vec{u}_{t}^{l,j}, \vec{v}_{t}^{l}\right) - \nabla_{\vec{u}} f_{t}\left(\vec{u}_{t}^{l}, \vec{v}_{t}^{l}\right) + \nabla_{\vec{u}} f_{t}\left(\vec{u}_{t}^{l}, \vec{v}_{t}^{l}\right)\right\|^{2}
        \\
        &\leq 2  \sum_{t=1}^{T} \sum_{j=0}^{J-1}\frac{\eta_{l,j}}{\eta_{l}} \cdot \E\left\| \nabla_{\vec{u}} g_{t}^{l+1}\left(\vec{u}_{t}^{l,j}, \vec{v}_{t}^{l}\right) - \nabla_{\vec{u}} f_{t}\left(\vec{u}_{t}^{l}, \vec{v}_{t}^{l}\right)\right\|^{2} \nonumber
        \\
        &\qquad + 2  \sum_{t=1}^{T} \E\left\|\nabla_{\vec{u}} f_{t}\left(\vec{u}_{t}^{l}, \vec{v}_{t}^{l}\right)\right\|^{2}.
    \end{flalign}
    Since $g_{t}^{l+1}$ is a first order surrogate of $f$ near $\left\{\vec{u}_{t}^{l}, \vec{v}^{l}_{t}\right\}$, we have
    \begin{flalign}
        \mathbb{E}\left\|\Upsilon^{l}\right\|_{F}^{2} &\leq  2  \sum_{t=1}^{T} \sum_{j=0}^{J-1}\frac{\eta_{l,j}}{\eta_{l}} \cdot \E\left\| \nabla_{\vec{u}} g_{t}^{l+1}\left(\vec{u}_{t}^{l,j}, \vec{v}_{t}^{l}\right) - \nabla_{\vec{u}} g^{l+1}_{t}\left(\vec{u}_{t}^{l,0}, \vec{v}_{t}^{l}\right)\right\|^{2} & \nonumber
        \\
        &\quad + 2  \sum_{t=1}^{T} \E\left\|\nabla_{\vec{u}} f_{t}\left(\vec{u}_{t}^{l}, \vec{v}_{t}^{l}\right) - \nabla_{\vec{u}} f_{t}\left(\bar{\vec{u}}^{l}, \vec{v}_{t}^{l}\right) + \nabla_{\vec{u}} f_{t}\left(\bar{\vec{u}}^{l}, \vec{v}_{t}^{l}\right)\right\|^{2}
        \\
        &\leq  2  \sum_{t=1}^{T} \sum_{j=0}^{J-1}\frac{\eta_{l,j}}{\eta_{l}} \cdot \E\left\| \nabla_{\vec{u}} g_{t}^{l+1}\left(\vec{u}_{t}^{l,j}, \vec{v}_{t}^{l}\right) - \nabla_{\vec{u}} g^{l+1}_{t}\left(\vec{u}_{t}^{l,0}, \vec{v}_{t}^{l}\right)\right\|^{2} & \nonumber
        \\
       &\quad + 4 \sum_{t=1}^{T} \E\left\|\nabla_{\vec{u}} f_{t}\left(\vec{u}_{t}^{l}, \vec{v}_{t}^{l}\right) - \nabla_{\vec{u}} f_{t}\left(\bar{\vec{u}}^{l}, \vec{v}_{t}^{l}\right)\right\|^{2} + 4 \sum_{t=1}^{T} \E\left\|\nabla_{\vec{u}} f_{t}\left(\bar{\vec{u}}^{l}, \vec{v}_{t}^{l}\right)\right\|^{2}.
    \end{flalign}
    Since $f$ is $2L$-smooth w.r.t $\vec{u}$ (Lemma~\ref{lem:f_2l_smooth}) and $g$ is $L$-smooth w.r.t $\vec{u}$ (Assumption~\ref{assum:smoothness_bis}), we have
    \begin{flalign}
        \mathbb{E}\left\|\Upsilon^{l}\right\|_{F}^{2} &\leq  2  \sum_{t=1}^{T} \sum_{j=0}^{J-1}\frac{\eta_{l,j}}{\eta_{l}} \cdot L^{2} \E\left\| \vec{u}_{t}^{l,j} -  \vec{u}_{t}^{l,0}\right\|^{2} + 16L^{2}\cdot \sum_{t=1}^{T} \E\left\| \vec{u}_{t}^{l} - \bar{\vec{u}}^{l}\right\|^{2} & \nonumber
        \\
       &\qquad + 4 \sum_{t=1}^{T} \E\left\|\nabla_{\vec{u}} f_{t}\left(\bar{\vec{u}}^{l}, \vec{v}_{t}^{l}\right)\right\|^{2}.
    \end{flalign}
    We use Eq.~\eqref{eq:u_j_minus_u_0_decentralized} to bound the first term in the RHS of the previous equation, leading to 
    \begin{flalign}
        \mathbb{E}\left\|\Upsilon^{l}\right\|_{F}^{2} &\leq  32\eta_{l}^{2}L^{2}  \sum_{t=1}^{T}  \E\left\| \nabla_{\vec{u}}g_{t}^{l+1}\left(\bar{\vec{u}}^{l,j}, \vec{v}_{t}^{l}\right)\right\|^{2} +  16L^{2}\left(1 + 2  \eta_{l}^{2}L^{2}\right)\cdot \sum_{t=1}^{T} \E\left\| \vec{u}_{t}^{l} - \bar{\vec{u}}^{l}\right\|^{2} & \nonumber
        \\
       &\qquad + 4 \sum_{t=1}^{T} \E\left\|\nabla_{\vec{u}} f_{t}\left(\bar{\vec{u}}^{l}, \vec{v}_{t}^{l}\right)\right\|^{2} + 8TL^{2}\sigma^{2}\cdot \left\{\sum_{j=0}^{J-1}\eta_{l,j}^{2}\right\}.
    \end{flalign}
    Using Lemma~\ref{lem:bound_g_theta_bar}, we have
    \begin{flalign}
        \mathbb{E}\left\|\Upsilon^{l}\right\|_{F}^{2} &\leq  4 \left(1 + 16\eta_{l}^{2}L^{2} \right)\cdot \sum_{t=1}^{T}  \E\left\| \nabla_{\vec{u}}f_{t}\left(\bar{\vec{u}}^{l,j}, \vec{v}_{t}^{l}\right)\right\|^{2}  & \nonumber
        \\
       &\qquad +  16L^{2}\left(1 + 6  \eta_{l}^{2}L^{2}\right)\cdot \sum_{t=1}^{T} \E\left\| \vec{u}_{t}^{l} - \bar{\vec{u}}^{l}\right\|^{2} + 8L^{2}\sigma^{2}T \cdot \left\{\sum_{j=0}^{J-1}\eta_{l,j}^{2}\right\}.
    \end{flalign}
    For $\eta_{l}$ small enough, in particular, for $\eta_{l} \leq \frac{1}{4L}$, we have
    \begin{flalign}
        \mathbb{E}\left\|\Upsilon^{l}\right\|_{F}^{2} &\leq  8 \sum_{t=1}^{T}  \E\left\| \nabla_{\vec{u}}f_{t}\left(\bar{\vec{u}}^{l,j}, \vec{v}_{t}^{l}\right)\right\|^{2} +  22L^{2}\E\left\| \vec{U}^{l} - \bar{\vec{U}}^{l}\right\|_{F}^{2}  +  8L^{2}\sigma^{2}T  \left\{\sum_{j=0}^{J-1}\eta_{l,j}^{2}\right\}.
       \label{eq:bound_upsilon_l}
    \end{flalign}
    Replacing Eq.~\eqref{eq:bound_upsilon_l} in Eq.~\eqref{eq:consensus_recursion_step2}, we have
    \begin{flalign}
        \mathbb{E}\sum_{t=1}^{T} &\left\|\vec{u}_{t}^{k} - \bar{\vec{u}}^{k}\right\|_{F}^{2} \leq & \nonumber
        \\
        & (1+\alpha)(1-p)\mathbb{E}\left\|\vec{U}^{m\tau} - \bar{\vec{U}}^{m\tau}\right\|_{F}^{2} + 44\tau\left(1+\alpha^{-1}\right)L^{2}\sum_{l=m\tau}^{k-1}\eta_{l}^{2}\E\left\|\vec{U}^{l} - \bar{\vec{U}}^{l}\right\|_{F}^{2} \nonumber
        \\
        & \quad +  16\tau\left(1+\alpha^{-1}\right)\sum_{l=m\tau}^{k-1}\eta_{l}^{2} \sum_{t=1}^{T}  \E\left\| \nabla_{\vec{u}}f_{t}\left(\bar{\vec{u}}^{l,j}, \vec{v}_{t}^{l}\right)\right\|^{2} \nonumber
        \\
        & \quad + T\cdot \sigma^{2} \cdot \sum_{l=m\tau}^{k-1}\left\{\eta_{l}^{2} + 16\tau L^{2}\left(1+\alpha^{-1}\right)\cdot \left\{\sum_{j=0}^{J-1}\eta_{l,j}^{2}\right\} \right\}. \label{eq:consensus_recursion_step3} 
    \end{flalign}
    Using Lemma~\ref{lem:bounded_dissimilarity_for_f} and considering $\alpha=\frac{p}{2}$, we have
    \begin{flalign}
        \mathbb{E}\sum_{t=1}^{T} &\left\|\vec{u}_{t}^{k} - \bar{\vec{u}}^{k}\right\|_{F}^{2} \leq & \nonumber
        \\
        & (1-\frac{p}{2})\mathbb{E}\left\|\vec{U}^{m\tau} - \bar{\vec{U}}^{m\tau}\right\|_{F}^{2} + 44\tau\left(1+\frac{2}{p}\right)L^{2}\sum_{l=m\tau}^{k-1}\eta_{l}^{2}\E\left\|\vec{U}^{l} - \bar{\vec{U}}^{l}\right\|_{F}^{2} \nonumber
        \\
        & ~ + T\cdot \sigma^{2} \cdot \sum_{l=m\tau}^{k-1}\left\{\eta_{l}^{2} + 16\tau L^{2}\left(1+\frac{2}{p}\right)\cdot \left\{\sum_{j=0}^{J-1}\eta_{l,j}^{2}\right\} \right\} + 16\tau\left(1+\frac{2}{p}\right)G^{2}\sum_{l=m\tau}^{k-1}\eta_{l}^{2} \nonumber
        \\
        & ~ +  16\tau\left(1+\frac{2}{p}\right)\beta^{2}\sum_{l=m\tau}^{k-1}\eta_{l}^{2}  \E\left\| \nabla_{\vec{u}}f \left(\bar{\vec{u}}^{l,j}, \vec{v}_{1:T}^{l}\right)\right\|^{2}.
    \end{flalign}
\end{proof}

\begin{lem}[Recursion for consensus distance, part 2]
    \label{lem:consensus_recursion_part2}
    Suppose that Assumptions~\ref{assum:smoothness_bis}--\ref{assum:bounded_dissimilarity_bis} and Assumption~\ref{assum:expected_consensus_rate} hold. Consider $m = \floor*{\frac{k}{\tau}}$, then, for $\left(\eta_{k, j}\right)_{1\leq j \leq J-1}$ 
    such that $\eta_k\triangleq \sum_{j=0}^{J-1}\eta_{k, j} \leq\min \left\{ \frac{1}{4L}, \frac{1}{4L\beta}\right\}$, the updates of fully decentralized federated surrogate optimization (Alg~\ref{alg:decentralized_surrogate_opt}) verify
    \begin{flalign}
        \mathbb{E}\sum_{t=1}^{T} &\left\|\vec{u}_{t}^{k} - \bar{\vec{u}}^{k}\right\|_{F}^{2} \leq & \nonumber
        \\
        & (1+\frac{p}{2})\mathbb{E}\left\|\vec{U}^{m\tau} - \bar{\vec{U}}^{m\tau}\right\|_{F}^{2} + 44\tau\left(1+\frac{2}{p}\right)L^{2}\sum_{l=m\tau}^{k-1}\eta_{l}^{2}\E\left\|\vec{U}^{l} - \bar{\vec{U}}^{l}\right\|_{F}^{2} \nonumber
        \\
        & ~ + T\cdot \sigma^{2} \cdot \sum_{l=m\tau}^{k-1}\left\{\eta_{l}^{2} + 16 \tau L^{2}\left(1+\frac{2}{p}\right)\cdot \left\{\sum_{j=0}^{J-1}\eta_{l,j}^{2}\right\} \right\} + 16\tau\left(1+\frac{2}{p}\right)G^{2}\sum_{l=m\tau}^{k-1}\eta_{l}^{2} \nonumber
        \\
        & ~ +  16\tau\left(1+\frac{2}{p}\right)\beta^{2}\sum_{l=m\tau}^{k-1}\eta_{l}^{2}  \E\left\| \nabla_{\vec{u}}f \left(\bar{\vec{u}}^{l,j}, \vec{v}_{1:T}^{l}\right)\right\|^{2}.
    \end{flalign}
\end{lem}
\begin{proof}
    We use exactly the same proof as in Lemma~\ref{lem:consensus_recursion_part1}, with the only difference that Eq.~\eqref{eq:consensus_recursion_step0}--Eq.~\eqref{eq:consensus_recursion_step1} is replaced by
    \begin{flalign}
        \mathbb{E}\sum_{t=1}^{T} &\left\|\vec{u}_{t}^{k} - \bar{\vec{u}}^{k}\right\|_{F}^{2} \leq & \nonumber
        \\
        & (1+\alpha)\mathbb{E}\left\|\vec{U}^{m\tau} - \bar{\vec{U}}^{m\tau}\right\|_{F}^{2} +  2\tau\left(1+\alpha^{-1}\right)\sum_{l=m\tau}^{k-1}\eta_{l}^{2}\mathbb{E}\left\|\Upsilon^{l}\right\|_{F}^{2} \nonumber
        \\
        & \qquad+  \sum_{l=m\tau}^{k-1}\eta^{2}_{l}\mathbb{E}\left\|\Upsilon^{l} - \hat{\Upsilon}^{l}\right\|_{F}^{2},
    \end{flalign}
    resulting from the fact that $\left\{\prod_{l'=m\tau}^{(m+1)\tau-1}W^{l'}\right\}$ is a doubly stochastic matrix.
\end{proof}

\begin{lem}
    \label{lem:consensus_lemma_non_convex_final}
    Under Assum.~\ref{assum:smoothness_bis}-\ref{assum:bounded_dissimilarity_bis} and Assum~\ref{assum:expected_consensus_rate}. For $\eta_{k,j} = \frac{\eta}{J}$ with
    \begin{equation*}
        \eta \leq \min\left\{\frac{1}{4L}, \frac{p}{92\tau L}, \frac{1}{4\beta L}, \frac{1}{32\sqrt{2}}\cdot \frac{p}{\tau \beta}\right\},
    \end{equation*}
    the iterates of Alg.~\ref{alg:decentralized_surrogate_opt} verifies
    \begin{equation}
        \label{eq:consensus_lemma_non_convex_final}
        \frac{(12+T)L^{2}}{4T}\sum_{k=0}^{K} \mathbb{E}\left\|\vec{U}^{k} - \bar{\vec{U}}^{k}\right\|_{F}^{2} \leq \frac{1}{16} \sum_{k=0}^{K}\E \left\|\nabla_{\vec{u}} f \left(\bar{\vec{u}}^{k}, \vec{v}_{1:T}^{k}\right)\right\|^{2} + 16A \cdot\frac{12+T}{T}\cdot \frac{\tau L^{2}}{p} (K+1)\eta^{2},
    \end{equation}
    for some constant $A>0$ and $K>0$.
\end{lem}
\begin{proof}
    Note that for $k>0$, $\eta_{k}= \sum_{j=0}^{J-1}\eta_{k j} = \eta$, and that $\sum_{l=m\tau}^{k-1}\eta_{l}^{2}=\sum_{l=m\tau}^{k-1}\eta^{2} \leq 2\tau\cdot \eta^{2}$
    
    Using Lemma~\ref{lem:consensus_recursion_part1} and Lemma~\ref{lem:consensus_recursion_part2}, and the fact that $p\leq 1$, we have for $m =\floor*{\frac{k}{\tau}}-1$
    \begin{flalign}
        \mathbb{E}\left\|\vec{U}^{k} - \bar{\vec{U}}^{k}\right\|_{F}^{2} &\leq (1-\frac{p}{2})\mathbb{E}\left\|\vec{U}^{m\tau} - \bar{\vec{U}}^{m\tau}\right\|_{F}^{2} + \frac{132\tau}{p} L^{2}\eta^{2}\sum_{l=m\tau}^{k-1}\mathbb{E}\left\|\vec{U}^{l} - \bar{\vec{U}}^{l}\right\|_{F}^{2} \nonumber
        \\
        + & \eta^{2}\underbrace{2\tau\left\{T\sigma^{2}\left(1+\frac{16\tau  L^{2}}{J}\left(1+\frac{2}{p}\right)\right) + 16 \tau \left(1+\frac{2}{p}\right)G^{2}\right\}}_{\triangleq A} \nonumber
        \\
        + &  \frac{16\tau}{p}\beta^{2}\eta^{2}\sum_{l=m\tau}^{k-1}\mathbb{E}\left\|\nabla_{\vec{u}} f \left(\bar{\vec{u}}^{l}, \vec{v}_{1:T}^{l}\right)\right\|^{2}.
     \end{flalign}
    and for $m =\floor*{\frac{k}{\tau}}$,
    \begin{flalign}
        \mathbb{E}\left\|\vec{U}^{k} - \bar{\vec{U}}^{k}\right\|_{F}^{2} &\leq (1+\frac{p}{2})\mathbb{E}\left\|\vec{U}^{m\tau} - \bar{\vec{U}}^{m\tau}\right\|_{F}^{2} + \frac{132\tau}{p} L^{2}\eta^{2}\sum_{l=m\tau}^{k-1}\mathbb{E}\left\|\vec{U}^{l} - \bar{\vec{U}}^{l}\right\|_{F}^{2} \nonumber
        \\
        + & \eta^{2}\underbrace{2\tau\left\{T\sigma^{2}\left(1+\frac{16\tau  L^{2}}{J}\left(1+\frac{2}{p}\right)\right) + 16 \tau \left(1+\frac{2}{p}\right)G^{2}\right\}}_{\triangleq A} \nonumber
        \\
        + &  \underbrace{\frac{16\tau}{p}\beta^{2}}_{\triangleq D}\eta^{2}\sum_{l=m\tau}^{k-1}\mathbb{E}\left\|\nabla_{\vec{u}} f \left(\bar{\vec{u}}^{l}, \vec{v}_{1:T}^{l}\right)\right\|^{2}.
    \end{flalign}
    Using the fact that $\eta \leq \frac{p}{92\tau L}$, it follows that 
    for $m =\floor*{\frac{k}{\tau}}-1$
    \begin{flalign}
        \mathbb{E}\left\|\vec{U}^{k} - \bar{\vec{U}}^{k}\right\|_{F}^{2} &\leq (1-\frac{p}{2})\mathbb{E}\left\|\vec{U}^{m\tau} - \bar{\vec{U}}^{m\tau}\right\|_{F}^{2} + \frac{p}{64\tau}\sum_{l=m\tau}^{k-1}\mathbb{E}\left\|\vec{U}^{l} - \bar{\vec{U}}^{l}\right\|^{2} \nonumber
        \\
        + & \eta^{2} A +  D\eta^{2}\sum_{l=m\tau}^{k-1}\mathbb{E}\left\|\nabla_{\vec{u}} f \left(\bar{\vec{u}}^{l}, \vec{v}_{1:T}^{l}\right)\right\|^{2},
    \end{flalign}
    and for $m =\floor*{\frac{k}{\tau}}$,
    \begin{flalign}
        \mathbb{E}\left\|\vec{U}^{k} - \bar{\vec{U}}^{k}\right\|_{F}^{2} &\leq (1+\frac{p}{2})\mathbb{E}\left\|\vec{U}^{m\tau} - \bar{\vec{U}}^{m\tau}\right\|_{F}^{2} + \frac{p}{64\tau} \sum_{l=m\tau}^{k-1}\mathbb{E}\left\|\vec{U}^{l} - \bar{\vec{U}}^{l}\right\|_{F}^{2} \nonumber
        \\
        + & \eta^{2} A +   D\eta^{2}\sum_{l=m\tau}^{k-1}\mathbb{E}\left\|\nabla_{\vec{u}} f \left(\bar{\vec{u}}^{l}, \vec{v}_{1:T}^{l}\right)\right\|^{2}.
    \end{flalign}
    The rest of the proof follows using \cite[Lemma 14]{Koloskova2020AUT} with $B=\frac{(12+T)L^{2}}{4T}$, $b=\frac{1}{8}$, constant (thus $\frac{8\tau}{p}$-slow\footnote{The notion of
       $\tau$-slow decreasing sequence is defined in \cite[Defintion 2]{Koloskova2020AUT}.
    }) steps-size
    $\eta \leq \frac{1}{32\sqrt{2}} \frac{p}{\tau \beta} = \frac{1}{16} \sqrt{\frac{p/8}{D\tau}}$ and constant weights $\omega_{k}=1$.
\end{proof}

\begin{repthm}{thm:decentralized_convergence_bis}
    Under Assumptions~\ref{assum:bounded_f_bis}--\ref{assum:bounded_dissimilarity_bis} and Assumption~\ref{assum:expected_consensus_rate}, when clients use SGD as local solver with learning rate $\eta =\frac{a_{0}}{\sqrt{K}}$, after a large enough number of communication rounds $K$, the iterates of fully decentralized federated surrogate optimization (Alg.~\ref{alg:decentralized_surrogate_opt}) satisfy:
    \begin{equation}
        \label{eq:convergence_u_decentralized}
        \frac{1}{K}\sum_{k=1}^{K}\mathbb{E}\left\|\nabla_{\vec u}f\left(\bar{\vec u}^{k}, \vec{v}_{1:T}^{k}\right)\right\|^{2} \leq \mathcal{O} \left(\frac{1}{\sqrt{K}}\right),
    \end{equation}
    and,
    \begin{equation}
        \label{eq:convergence_v_decentralized}
        \frac{1}{K}\sum_{k=1}^{K}\sum_{t=1}^{T}\omega_{t} \cdot \E d_{\mathcal{V}}\left(\vec \vec{v}^{k}_{t}, \vec  \vec{v}^{k+1}_{t}\right)\leq \mathcal{O}\left(\frac{1}{K}\right),
    \end{equation}
    where $\bar{\vec u}^k = \frac{1}{T}\sum_{t=1}^{T}\vec u^{k}_{t}$.  
    Moreover, local estimates $\left(\vec u_{t}^{k}\right)_{1\leq t \leq T}$ converge to consensus, i.e., to $\bar{\vec u}^k$:
    \begin{equation}
        \label{eq:consensus_convergence}
        \frac{1}{K}\sum_{k=1}^{K}\sum_{t=1}^{T}\E\left\|\vec u_{t}^{k} - \bar{\vec u}^{k}\right\|^{2} \leq \mathcal{O} \left(\frac{1}{\sqrt{K}}\right).
    \end{equation}
\end{repthm}
\begin{proof}
     We prove first the convergence to a stationary point in $\vec{u}$, i.e. Eq.~\eqref{eq:convergence_u_decentralized}, using \cite[Lemma 17]{Koloskova2020AUT}, then we prove Eq.~\eqref{eq:convergence_v_decentralized} and Eq.~\eqref{eq:consensus_convergence}.
     
     Note that for $K$ large enough, $\eta \leq \min\left\{\frac{1}{4L}, \frac{p}{92\tau L}, \frac{1}{4\beta L}, \frac{1}{32\sqrt{2}}\cdot \frac{p}{\tau \beta}\right\}$.
     
    \paragraph{Proof of Eq.~\ref{eq:convergence_u_decentralized}.}
    Rearranging the terms in the result of Lemma~\ref{lem:descent_lem_decentralized} and dividing it by $\eta$  we have 
    \begin{flalign}
         \frac{1}{\eta} \cdot \mathbb{E}\Bigg[f & (\bar{\vec{u}}^{k}  , \vec{v}_{1:T}^{k}) - f(\bar{\vec{u}}^{k-1}, \vec{v}_{1:T}^{k-1})\Bigg] \leq   -\frac{1}{8}\E\left\|\nabla_{\vec{u}} f\left(\bar{\vec{u}}^{k-1}, \vec{v}_{1:T}^{k-1}\right)\right\|^{2} & \nonumber
        \\
        &  +  \frac{\left(12 + T\right)L^{2}}{4T}  \cdot  \E\left\|\vec{U}^{k-1} - \bar{\vec{U}}^{k-1}\right\|^{2} + \frac{\eta L}{T} \left(\frac{4L}{J} + 1\right) \sigma^{2} + \frac{16\eta^{2}L^{2}}{T} G^{2}.
    \end{flalign}
    Summing over $k\in [K+1]$, we have
    \begin{flalign}
         \frac{1}{\eta} \cdot \mathbb{E}\Bigg[f & (\bar{\vec{u}}^{K+1}  , \vec{v}_{1:T}^{K+1}) - f(\bar{\vec{u}}^{0}, \vec{v}_{1:T}^{0})\Bigg] \leq   -\frac{1}{8}\sum_{k=0}^{K}\E\left\|\nabla_{\vec{u}} f\left(\bar{\vec{u}}^{k}, \vec{v}_{1:T}^{k}\right)\right\|^{2} & \nonumber
        \\
        &  +  \frac{\left(12 + T\right)L^{2}}{4T}  \cdot \sum_{k=0}^{K} \E\left\|\vec{U}^{k} - \bar{\vec{U}}^{k}\right\|^{2}   + \frac{(K+1) \eta L}{T} \left(\frac{4L}{J} + 1\right) \sigma^{2} \nonumber
        \\
        &+ \frac{16(K+1)\cdot \eta^{2}L^{2}}{T} G^{2}.
    \end{flalign}
    Using Lemma~\ref{lem:consensus_lemma_non_convex_final}, we have
     \begin{flalign}
         \frac{1}{\eta} \cdot \mathbb{E}\Bigg[f & (\bar{\vec{u}}^{K+1}  , \vec{v}_{1:T}^{K+1}) - f(\bar{\vec{u}}^{0}, \vec{v}_{1:T}^{0})\Bigg] \leq   -\frac{1}{16}\sum_{k=0}^{K}\E\left\|\nabla_{\vec{u}} f\left(\bar{\vec{u}}^{k}, \vec{v}_{1:T}^{k}\right)\right\|^{2} & \nonumber
        \\
        &  +  16A \cdot\frac{12+T}{T}\cdot \frac{\tau L^{2}}{p} (K+1)\eta^{2}   + \frac{(K+1) \eta L}{T} \left(\frac{4L}{J} + 1\right) \sigma^{2} \nonumber
        \\
        & + \frac{16(K+1) \eta^{2}L^{2}}{T} G^{2}.
    \end{flalign}
    Using Assumption~\ref{assum:bounded_f_bis}, it follows that
    \begin{flalign}
         \frac{1}{16}\sum_{k=0}^{K}&\E\left\|\nabla_{\vec{u}} f\left(\bar{\vec{u}}^{k}, \vec{v}_{1:T}^{k}\right)\right\|^{2} \leq  \frac{f(\bar{\vec{u}}^{0}, \vec{v}_{1:T}^{0}) -f^{*}}{\eta} & \nonumber
        \\
        &  + 16A \cdot\frac{12+T}{T}\cdot \frac{\tau L^{2}}{p} (K+1)\eta^{2}   + \frac{(K+1) \eta L}{T} \left(\frac{4L}{J} + 1\right) \sigma^{2} + \frac{16(K+1) \eta^{2}L^{2}}{T} G^{2}.
    \end{flalign}
    We divide by $K+1$ and we have
     \begin{flalign}
         \frac{1}{16(K+1)}\sum_{k=0}^{K}&\E\left\|\nabla_{\vec{u}} f\left(\bar{\vec{u}}^{k}, \vec{v}_{1:T}^{k}\right)\right\|^{2} \leq  \frac{f(\bar{\vec{u}}^{0}, \vec{v}_{1:T}^{0}) -f^{*}}{\eta(K+1)} & \nonumber
        \\
        &  +  16A \cdot\frac{12+T}{T}\cdot \frac{\tau L^{2}}{p} \eta^{2}   + \frac{ \eta L}{T} \left(\frac{4L}{J} + 1\right) \sigma^{2} + \frac{16 \eta^{2}L^{2}}{T} G^{2}.
    \end{flalign}
    The final result follows from \cite[Lemma~17]{Koloskova2020AUT}.
    \paragraph{Proof of Eq.~\ref{eq:consensus_convergence}.} We multiply Eq.~\eqref{eq:consensus_lemma_non_convex_final}~(Lemma~\ref{lem:consensus_lemma_non_convex_final}) by $\frac{1}{K+1}$, and we have 
    \begin{equation}
        \frac{1}{K+1} \sum_{k=0}^{K} \mathbb{E}\left\|\vec{U}^{k} - \bar{\vec{U}}^{k}\right\|_{F}^{2} \leq \frac{1}{16(K+1)} \sum_{k=0}^{K} \E\left\|\nabla_{\vec{u}} f \left(\bar{\vec{u}}^{k}, \vec{v}_{1:T}^{k}\right)\right\|_{F}^{2} + \frac{64A \tau}{p(K+1)} K\eta^{2},
    \end{equation}
    since $\eta \leq \mathcal{O}\left(\frac{1}{\sqrt{K}}\right)$, using Eq.~\eqref{eq:convergence_u_decentralized}, it follows that
    \begin{equation}
        \frac{1}{K} \sum_{k=1}^{K} \mathbb{E}\left\|\vec{U}^{k} - \bar{\vec{U}}^{k}\right\|_{F}^{2} \leq \mathcal{O}\left(\frac{1}{\sqrt{K}}\right).
    \end{equation}
    Thus,
    \begin{equation}
        \frac{1}{K} \sum_{k=1}^{K}\sum_{t=1}^{T} \E\left\|\vec{u}_{t}^{k} - \bar{\vec{u}}^{k}\right\|_{F}^{2} \leq \mathcal{O} \left(\frac{1}{\sqrt{K}}\right).
    \end{equation}
    \paragraph{Proof of Eq.~\ref{eq:convergence_v_decentralized}.} Using the result of Lemma~\ref{lem:descent_lem_decentralized} we have
    \begin{flalign}
         \frac{1}{T}\sum_{t=1}^{T} \E \left[d_{\mathcal{V}}\left(\vec{v}_{t}^{k}, \vec{v}_{t}^{k-1}\right)\right] & \leq  \mathbb{E}\Bigg[f (\bar{\vec{u}}^{k-1}  , \vec{v}_{1:T}^{k-1}) - f(\bar{\vec{u}}^{k}, \vec{v}_{1:T}^{k})\Bigg]   & \nonumber
        \\
        &  \quad + \frac{\left(12 + T\right)\eta_{k-1}L^{2}}{4T} \cdot \sum_{t=1}^{T} \E\left\|\vec{u}_{t}^{k-1} - \bar{\vec{u}}^{k-1}\right\|^{2}  \nonumber
        \\
        &  \quad + \frac{\eta_{k-1}^{2}L}{T} \left(4 \sum_{j=0}^{J-1}\frac{L \cdot \eta_{k-1,j}^{2}}{\eta_{k-1}} + 1\right) \sigma^{2} + \frac{16\eta_{k-1}^{3}L^{2}}{T} G^{2}.
    \end{flalign}
    The final result follows from the fact that $\eta = \mathcal{O}\left(\frac{1}{\sqrt{K}}\right)$ and Eq.~\eqref{eq:consensus_convergence}.
\end{proof}

\subsubsection{Proof of Theorem~\ref{thm:decentralized_convergence}}

We state the formal version of Theorem~\ref{thm:decentralized_convergence}, for which only an informal version was given in the main text.

\begin{repthm}{thm:decentralized_convergence}
    Under Assumptions~\ref{assum:mixture}--\ref{assum:expected_consensus_rate}, when clients use SGD as local solver with learning rate $\eta =\frac{a_{0}}{\sqrt{K}}$, \DEM's iterates  satisfy the following inequalities after a large enough number of communication rounds $K$:
    \begin{equation}
        \frac{1}{K}\sum_{k=1}^K \mathbb{E}\left\|\nabla_{\Theta}f\left(\bar{\Theta}^{k}, \Pi^{k}\right)\right\|_{F}^{2} \leq \mathcal{O} \left(\frac{1}{\sqrt{K}}\right), \quad
        \frac{1}{K}\sum_{k=1}^K\sum_{t=1}^{T}\frac{n_{t}}{n}\mathcal{KL}\left(\pi^{k}_{t}, \pi_{t}^{k-1}\right)\leq \mathcal{O}\left(\frac{1}{K}\right),
    \end{equation}
    where $\bar \Theta^k = \left[\Theta_{1}^k, \dots \Theta_{T}^{k}\right] \cdot  \frac{\mathbf{1}\mathbf{1}^{\intercal}}{T}$.  
    Moreover, individual estimates $\left(\Theta_{t}^{k}\right)_{1\leq t \leq T}$ converge to consensus, i.e., to $\bar \Theta^k$:
    \begin{equation*}
        \min_{k\in[K]}\mathbb{E}\sum_{t=1}^{T}\left\|\Theta_{t}^{k} - \bar{\Theta}^{k}\right\|_{F}^{2} \leq \mathcal{O} \left(\frac{1}{\sqrt{K}}\right).
    \end{equation*}
\end{repthm}

\begin{proof}
    We prove this result as a particular case of Theorem~\ref{thm:decentralized_convergence_bis}. To this purpose, we consider that $\mathcal{V} \triangleq \Delta^{M}$, $\vec{u}=\Theta \in \mathbb{R}^{d M}$, $\vec{v}_{t} = \pi_{t}$, and $\omega_t = n_t/n$ for $t\in[T]$. For $k>0$, we define $g^{k}_{t}$ as follow,
    \begin{flalign}
        g^{k}_{t}\Big (\Theta, \pi_{t}\Big) =   \frac{1}{n_{t}}\sum_{i=1}^{n_{t}}\sum_{m=1}^{M}q_{t}^{k}\left(z^{(i)}_{t}=m\right) \cdot &\bigg( l\left(h_{\theta_m}(\vec{x}_{t}^{(i)}), y_{t}^{(i)}\right)  -  \log p_{m}(\vec{x}_{t}^{(i)}) - \log \pi_{t} \nonumber &
        \\
        & \qquad \qquad +  \log q_{t}^{k}\left(z_{t}^{(i)}=m\right) - c \bigg), 
    \end{flalign}
    where $c$ is the same constant appearing in Assumption~\ref{assum:logloss}, Eq.~\eqref{eq:log_loss}. With this definition, it is easy to check that the federated surrogate optimization algorithm (Alg.~\ref{alg:decentralized_surrogate_opt}) reduces to \DEM{} (Alg.~\ref{alg:d_em}). Theorem~\ref{thm:decentralized_convergence} then follows immediately  from Theorem~\ref{thm:decentralized_convergence_bis}, once we verify that $\left(g_{t}^{k}\right)_{1\leq t\leq T}$ satisfy the assumptions of Theorem~\ref{thm:decentralized_convergence_bis}.
    
    Assumption~\ref{assum:bounded_f_bis}, Assumption~\ref{assum:finite_variance_bis}, and Assumption~\ref{assum:bounded_dissimilarity_bis} follow directly from  Assumption~\ref{assum:bounded_f}, Assumption~\ref{assum:finitie_variance}, and Assumption~\ref{assum:bounded_dissimilarity}, respectively.
    Lemma~\ref{lem:g_is_smooth} shows that for $k>0$, $g^{k}$ is smooth w.r.t. $\Theta$ and then Assumption~\ref{assum:smoothness_bis} is satisfied.
    Finally, Lemmas~\ref{lem:em_gap_is_kl_divergence}--\ref{lem:compute_kl_pi_decentralized} show that for $t\in[T]$ $g_{t}^{k}$ is a partial first-order surrogate of $f_{t}$ near $\left\{\Theta_{t}^{k-1}, \pi_{t}\right\}$ with $d_{\mathcal{V}}(\cdot, \cdot) = \mathcal{KL}(\cdot\|\cdot)$.
\end{proof}

\subsection{Supporting Lemmas}
\label{proof:support}
\begin{lem}
    \label{lem:bound_sums_of_rates}
    Consider $J\geq2$ and  positive real numbers $\eta_{j},~j=0, \dots, J-1$, then:
    \begin{align*}
        \frac{1}{\sum_{j=0}^{J-1}{\eta_{j}}} \cdot \sum_{j=0}^{J-1} \left\{ \eta_{j} \cdot \sum_{l=0}^{j-1}\eta_{l} \right\}&\leq \sum_{j=0}^{J-2}{\eta_{j}},
        \\
        \frac{1}{\sum_{j=0}^{J-1}{\eta_{j}}} \cdot \sum_{j=0}^{J-1} \left\{ \eta_{j} \cdot \sum_{l=0}^{j-1}\eta^{2}_{l} \right\} &\leq \sum_{j=0}^{J-2}{\eta_{j}}^{2},
        \\
        \frac{1}{\sum_{j=0}^{J-1}{\eta_{j}}} \cdot \sum_{j=0}^{J-1}\left\{\eta_{j}\cdot \left( \sum_{l=0}^{j-1}\eta_{l}\right)^{2}\right\}& \leq \sum_{j=0}^{J-1}{\eta_{j}} \cdot \sum_{j=0}^{J-2}{\eta_{j}}.
    \end{align*}
\end{lem}
\begin{proof}
    For the first inequality,
    \begin{flalign}
         \frac{1}{\sum_{j=0}^{J-1}{\eta_{j}}} \cdot \sum_{j=0}^{J-1} \left\{ \eta_{j} \cdot \sum_{l=0}^{j-1}\eta_{l} \right\} &\leq \frac{1}{\sum_{j=0}^{J-1}{\eta_{j}}} \cdot \sum_{j=0}^{J-1} \left\{ \eta_{j} \cdot \sum_{l=0}^{J-2}\eta_{l} \right\} =  \sum_{l=0}^{J-2}\eta_{l}.
    \end{flalign}
    For the second inequality
    \begin{flalign}
         \frac{1}{\sum_{j=0}^{J-1}{\eta_{j}}} \cdot \sum_{j=0}^{J-1} \left\{ \eta_{j} \cdot \sum_{l=0}^{j-1}\eta_{l}^{2} \right\} &\leq \frac{1}{\sum_{j=0}^{J-1}{\eta_{j}}} \cdot \sum_{j=0}^{J-1} \left\{ \eta_{j} \cdot \sum_{l=0}^{J-2}\eta_{l}^{2} \right\} =  \sum_{l=0}^{J-2}\eta_{l}^{2}.
    \end{flalign}
    For the third inequality,
    \begin{flalign}
        \frac{1}{\sum_{j=0}^{J-1}{\eta_{j}}} \cdot \sum_{j=0}^{J-1}\left\{\eta_{j}\cdot \left( \sum_{l=0}^{j-1}\eta_{l}\right)^{2}\right\}& \leq \frac{1}{\sum_{j=0}^{J-1}{\eta_{j}}} \cdot \sum_{j=0}^{J-1}\left\{\eta_{j}\cdot \left( \sum_{l=0}^{J-2}\eta_{l}\right)^{2}\right\} &
        \\
        &\leq  \left(\sum_{j=0}^{J-2}{\eta_{j}}\right)^{2}
        \\
        &\leq \sum_{j=0}^{J-1}{\eta_{j}} \cdot \sum_{j=0}^{J-2}{\eta_{j}}.
    \end{flalign}
\end{proof}

\begin{lem}
    \label{lem:f_2l_smooth}
    Suppose that $g$ is a partial first-order surrogate of $f$, and that $g$ is $L$-smooth, where $L$ is the constant appearing in Definition~\ref{def:partial_first_order_surrogate}, then  $f$ is $2L$-smooth. 
\end{lem}
\begin{proof}
    The difference between $f$ and $g$ is $L$-smooth, and $g$ is $L$-smooth, thus $f$ is $2L$-smooth as the sum of two $L$-smooth functions.
\end{proof}

\begin{lem}
    \label{lem:bounded_dissimilarity_for_f}
    Consider $f = \sum_{t=1}^{T}\omega_{t} \cdot f_{t}$,~for weights $\omega \in \Delta^{T}$. Suppose that for all $\left(\vec{u}, \vec{v}\right) \in \mathbb{R}^{d_{u}} \times \mathcal{V}$, and $t\in[T]$, $f_{t}$ admits a partial first-order surrogate 
    $g^{\left\{\vec{u}, \vec{v}\right\}}_{t}$ 
    near $\left\{\vec{u}, \vec{v}\right\}$, and  that 
    $g^{{\left\{\vec{u}, \vec{v}\right\}}} =\sum_{t=1}^{T}\omega_{t}\cdot g^{{\left\{\vec{u}, \vec{v}\right\}}}_{t}$ verifies Assumption~\ref{assum:bounded_dissimilarity_bis} for $t\in[T]$. Then $f$ also verifies Assumption~\ref{assum:bounded_dissimilarity_bis}.
\end{lem}
\begin{proof}
    Consider arbitrary $
    \vec{u}, \vec{v} \in \mathbb{R}^{d_{u}} \times \mathcal{V}$, and for $t\in[T]$, consider $g^{\left\{\vec{u}, \vec{v}\right\}}$ to be a partial first-order surrogate of $f_{t}$ near $\left\{\vec{u}, \vec{v}\right\}$. We write Assumption~\ref{assum:bounded_dissimilarity_bis} for $g^{\left\{\vec{u}, \vec{v}\right\}}$,
    \begin{flalign}
        \sum_{t=1}^{T}\omega_{t} \cdot \Big\|\nabla_{\vec u}g^{\left\{u, v\right\}}_{t}(\vec u, \vec v)\Big\|^{2} \leq G^{2} + \beta^{2} \Big\| \sum_{t=1}^{T}\omega_{t} \cdot \nabla_{\vec u}g^{\left\{u, v\right\}}_{t}(\vec u, \vec v) \Big\|^{2}.
    \end{flalign}
    Since $g_{t}^{\left\{\vec{u}, \vec{v}\right\}}$ is a partial first-order surrogate of $f_{t}$ near $\left\{u, v\right\}$, it follows that 
    \begin{flalign}
        \sum_{t=1}^{T}\omega_{t} \cdot \Big\|\nabla_{\vec u}f_{t}(\vec u, \vec v)\Big\|^{2} \leq G^{2} + \beta^{2} \Big\| \sum_{t=1}^{T}\omega_{t} \cdot \nabla_{\vec u}f_{t}(\vec u, \vec v) \Big\|^{2}.
    \end{flalign}
\end{proof}
\begin{remark}
    Note that the assumption of Lemma~\ref{lem:bounded_dissimilarity_for_f} is implicitly verified in  Alg.~\ref{alg:fed_surrogate_opt} and Alg.~\ref{alg:decentralized_surrogate_opt}, where we assume that every client $t\in\mathcal{T}$ canfunction compute a partial first-order surrogate of its local objective $f_{t}$ near any iterate $\left(\vec{u}, \vec{v}\right) \in \mathbb{R}^{d_{u}} \times \mathcal{V}$.
\end{remark}

\begin{lem}
    \label{lem:bound_g_theta_bar}
    For $k>0$, the iterates of Alg.~\ref{alg:decentralized_surrogate_opt}, verify the following inequalities:
    \begin{equation*}
        g^{k}\left(\bar{\vec{u}}^{k-1}, \vec{v}_{1:T}^{k-1}\right) \leq f\left(\bar{\vec{u}}^{k-1}, \vec{v}_{1:T}^{k-1}\right)  + \frac{L}{2}\sum_{t=1}^{T}\omega_{t}\left\|\bar{\vec{u}}^{k-1} - \vec{u}_{t}^{k-1}\right\|^{2},
    \end{equation*}
    \begin{equation*}
        \left\|\nabla_{\vec{u}}f\left(\bar{\vec{u}}^{k-1}, \vec{v}_{1:T}^{k-1}\right)\right\|^{2} \leq 2\left\|\nabla_{\vec{u}}g^{k}\left(\bar{\vec{u}}^{k-1}, \vec{v}_{1:T}^{k-1}\right)\right\|^{2}  + 2L^{2}\sum_{t=1}^{T}\omega_{t}\left\|\bar{\vec{u}}^{k-1} + \vec{u}_{t}^{k-1}\right\|^{2},
    \end{equation*}
    and,
    \begin{equation*}
        \left\|\nabla_{\vec{u}}g^{k}\left(\bar{\vec{u}}^{k-1}, \vec{v}_{1:T}^{k-1}\right)\right\|^{2} \leq 2\left\|\nabla_{\vec{u}}f\left(\bar{\vec{u}}^{k-1}, \vec{v}_{1:T}^{k-1}\right)\right\|^{2}  + 2L^{2}\sum_{t=1}^{T}\omega_{t}\left\|\bar{\vec{u}}^{k-1} - \vec{u}_{t}^{k-1}\right\|^{2},
    \end{equation*}
\end{lem}
\begin{proof}
    For $k>0$ and $t\in[T]$, we have 
    \begin{flalign}
        g_{t}^{k}\Big(&\bar{\vec{u}}^{k-1} , \vec{v}_{t}^{k-1}\Big) = & \nonumber
        \\
        & g_{t}^{k}\left(\bar{\vec{u}}^{k-1}, \vec{v}_{t}^{k-1}\right) + f_{t}\left(\bar{\vec{u}}^{k-1}, \vec{v}_{t}^{k-1}\right) - f_{t}\left(\bar{\vec{u}}^{k-1}, \vec{v}_{t}^{k-1}\right)
        \\
        &= f_{t}\left(\bar{\vec{u}}^{k-1}, \vec{v}_{t}^{k-1}\right) + r_{t}^{k}\left(\bar{\vec{u}}^{k-1}, \vec{v}_{t}^{k-1}\right)
        \\
        &= f_{t}\left(\bar{\vec{u}}^{k-1}, \vec{v}_{t}^{k-1}\right) + r_{t}^{k}\left(\bar{\vec{u}}^{k-1}, \vec{v}_{t}^{k-1}\right) - r_{t}^{k}\left(\vec{u}_{t}^{k-1}, \vec{v}_{t}^{k-1}\right) + r_{t}^{k}\left(\vec{u}_{t}^{k-1}, \vec{v}_{t}^{k-1}\right).
    \end{flalign}
    Since $g_{t}^{k}\left(\vec{u}^{k}_{t}, \vec{v}_{t}^{k-1}\right) = f_{t}\left(\vec{u}^{k}_{t}, \vec{v}_{t}^{k-1}\right)$ (Definition~\ref{def:partial_first_order_surrogate}), it follows that 
    \begin{equation}
        \label{eq:g_f_r_relation}
        g_{t}^{k}\left(\bar{\vec{u}}^{k-1}, \vec{v}_{t}^{k-1}\right)  = f_{t}\left(\bar{\vec{u}}^{k-1}, \vec{v}_{t}^{k-1}\right) + r_{t}^{k}\left(\bar{\vec{u}}^{k-1}, \vec{v}_{t}^{k-1}\right) - r_{t}^{k}\left(\vec{u}_{t}^{k-1}, \vec{v}_{t}^{k-1}\right).
    \end{equation}
    Because $r^{k}_{t}$ is $L$-smooth in $\vec{u}$ (Definition~\ref{def:partial_first_order_surrogate}), we have
    \begin{flalign}
        r_{t}^{k}\left(\bar{\vec{u}}^{k-1}, \vec{v}_{t}^{k-1}\right) - & r_{t}^{k}\left(\vec{u}_{t}^{k-1}, \vec{v}_{t}^{k-1}\right) \leq \biggl<\nabla_{\vec{u}}r_{t}^{k}\left(\vec{u}_{t}^{k-1},\vec{v}^{k-1}_{t}\right), \bar{\vec{u}}^{k-1} - \vec{u}_{t}^{k-1}\biggr> & \nonumber
        \\
        & \qquad + \frac{L}{2}\left\|\bar{\vec{u}}^{k-1} - \vec{u}_{t}^{k-1}\right\|^{2}.
    \end{flalign}
    Since $g_{t}^{k}$ is a partial first order surrogate of We have $\nabla_{\vec{u}}r_{t}^{k}\left(\vec{u}_{t}^{k-1},\vec{v}^{k-1}_{t}\right) = 0$, thus
    \begin{equation}
        g_{t}^{k}\left(\bar{\vec{u}}^{k-1}, \vec{v}_{t}^{k-1}\right) \leq f_{t}\left(\bar{\vec{u}}^{k-1}, \vec{v}_{t}^{k-1}\right)  + \frac{L}{2}\left\|\bar{\vec{u}}^{k-1} - \vec{u}_{t}^{k-1}\right\|^{2}.
    \end{equation}
    Multiplying by $\omega_{t}$ and summing for $t\in[T]$, we have
    \begin{equation}
        g^{k}\left(\bar{\vec{u}}^{k-1}, \vec{v}_{1:T}^{k-1}\right) \leq f\left(\bar{\vec{u}}^{k-1}, \vec{v}_{1:T}^{k-1}\right)  + \frac{L}{2}\sum_{t=1}^{T}\omega_{t}\left\|\bar{\vec{u}}^{k-1} - \vec{u}_{t}^{k-1}\right\|^{2},
    \end{equation}
    and the first inequality is proved.
    
    Writing the gradient of Eq.~\eqref{eq:g_f_r_relation}, we have
    \begin{equation}
        \nabla_{\vec{u}}g_{t}^{k}\left(\bar{\vec{u}}^{k-1}, \vec{v}_{t}^{k-1}\right) = \nabla_{\vec{u}}f_{t}\left(\bar{\vec{u}}^{k-1}, \vec{v}_{t}^{k-1}\right) + \nabla_{\vec{u}}r_{t}^{k}\left(\bar{\vec{u}}^{k-1}, \vec{v}_{t}^{k-1}\right) - \nabla_{\vec{u}}r_{t}^{k}\left(\vec{u}_{t}^{k-1}, \vec{v}_{t}^{k-1}\right).
    \end{equation}
    Multiplying by $\omega_{t}$ and summing for $t\in[T]$, we have
    \begin{flalign}
        \nabla_{\vec{u}}g^{k}\left(\bar{\vec{u}}^{k-1}, \vec{v}_{1:T}^{k-1}\right) & = \nabla_{\vec{u}}f\left(\bar{\vec{u}}^{k-1}, \vec{v}_{1:T}^{k-1}\right) + & \nonumber\\
        & \qquad +\sum_{t=1}^{T}\omega_{t}\left[\nabla_{\vec{u}}r_{t}^{k}\left(\bar{\vec{u}}^{k-1}, \vec{v}_{t}^{k-1}\right) - \nabla_{\vec{u}}r_{t}^{k}\left(\vec{u}_{t}^{k-1}, \vec{v}_{t}^{k-1}\right)\right].
    \end{flalign}
    Thus,
    \begin{flalign}
        \Bigg\|& \nabla_{\vec{u}}  g^{k}\left(\bar{\vec{u}}^{k-1}, \vec{v}_{1:T}^{k-1}\right)\Bigg\|^{2} = & \nonumber
        \\
        & \left\|\nabla_{\vec{u}}f\left(\bar{\vec{u}}^{k-1}, \vec{v}_{1:T}^{k-1}\right) +\sum_{t=1}^{T}\omega_{t}\left[\nabla_{\vec{u}}r_{t}^{k}\left(\bar{\vec{u}}^{k-1}, \vec{v}_{t}^{k-1}\right) - \nabla_{\vec{u}}r_{t}^{k}\left(\vec{u}_{t}^{k-1}, \vec{v}_{t}^{k-1}\right)\right]\right\|^{2} \label{eq:bound_g_u_bar_step_1}
        \\
        \geq &  \frac{1}{2}\left\|\nabla_{\vec{u}}f\left(\bar{\vec{u}}^{k-1}, \vec{v}_{1:T}^{k-1}\right)\right\|^{2} - \left\| \sum_{t=1}^{T}\omega_{t}\left[\nabla_{\vec{u}}r_{t}^{k}\left(\bar{\vec{u}}^{k-1}, \vec{v}_{t}^{k-1}\right) - \nabla_{\vec{u}}r_{t}^{k}\left(\vec{u}_{t}^{k-1}, \vec{v}_{t}^{k-1}\right)\right]\right\|^{2}  \label{eq:bound_g_u_bar_step_2}
        \\
        \geq &  \frac{1}{2}\left\|\nabla_{\vec{u}}f\left(\bar{\vec{u}}^{k-1}, \vec{v}_{1:T}^{k-1}\right)\right\|^{2} - \sum_{t=1}^{T}\omega_{t}\left\|\nabla_{\vec{u}}r_{t}^{k}\left(\bar{\vec{u}}^{k-1}, \vec{v}_{t}^{k-1}\right) - \nabla_{\vec{u}}r_{t}^{k}\left(\vec{u}_{t}^{k-1}, \vec{v}_{t}^{k-1}\right)\right\|^{2}
        \\
        \geq & \frac{1}{2}\left\|\nabla_{\vec{u}}f\left(\bar{\vec{u}}^{k-1}, \vec{v}_{1:T}^{k-1}\right)\right\|^{2} - L^{2}\sum_{t=1}^{T}\omega_{t}\left\| \bar{\vec{u}}^{k-1} - \vec{u}_{t}^{k-1} \right\|^{2},
    \end{flalign}
    where \eqref{eq:bound_g_u_bar_step_2} follows from $\left\|a\right\|^{2} = \left\|a + b - b\right\|^{2} \leq 2\left\|a+b\right\|^{2} + 2\left\|b\right\|^{2}$.
    Thus, 
    \begin{equation}
        \left\|\nabla_{\vec{u}}f_{t}\left(\bar{\vec{u}}^{k-1}, \vec{v}_{t}^{k-1}\right)\right\|^{2} \leq 2\left\|\nabla_{\vec{u}}g_{t}^{k}\left(\bar{\vec{u}}^{k-1}, \vec{v}_{t}^{k-1}\right)\right\|^{2}  + 2L^{2}\sum_{t=1}^{T}\omega_{t}\left\|\bar{\vec{u}}^{k-1} - \vec{u}_{t}^{k-1}\right\|^{2}.
    \end{equation}
    The proof of the last inequality is similar, it leverages $\left\|a + b \right\|^{2} \leq 2\left\|a\right\|^{2} + 2\left\|a\right\|^{2}$ to upper bound \eqref{eq:bound_g_u_bar_step_1}.
\end{proof}

\begin{lem}
    \label{lem:hessian_is_negative}
    Consider $\vec{u}_{1},\dots, \vec{u}_{M} \in \mathbb{R}^{d}$ and $\vec{\alpha} = \left(\alpha_{1}, \dots, \alpha_{M}\right) \in \Delta^{M}$. Define the block matrix $\mathbf{H}$  with
    \begin{equation}
        \begin{cases}
            \begin{aligned}
                \mathbf{H}_{m, m} &= -\alpha_{m}\cdot \left(1 - \alpha_{m}\right) \cdot \vec{u}_{m} \cdot \vec{u}_{m}^{\intercal}  &
                \\
                \mathbf{H}_{m, m'} &=  \alpha_{m} \cdot \alpha_{m'} \cdot \vec{u}_{m} \cdot \vec{u}_{m'}^{\intercal}; & m'\neq m,
            \end{aligned}
        \end{cases}
    \end{equation}
    then $\mathbf{H}$ is a semi-definite negative matrix. Moreover, if there exists a constant $B>0$, such that $\left\|\vec{u}_{m}\right\|_{\mathbb{R}^{d}} \leq B$ for all $m\in[M]$, then $\vec{H} \succcurlyeq -B^{2} I_{d M}$.
\end{lem}
\begin{proof}
    Consider $\mathbf{x} = \left[\vec{x}_{1}, \dots, \vec{x}_{M}\right] \in \mathbb{R}^{dM}$, we have:
    \begin{flalign}
        \mathbf{x}^{\intercal} \cdot \mathbf{H} \cdot \mathbf{x} &= \sum_{m=1}^{M} \sum_{m'=1}^{M} \vec{x}_{m}^{\intercal}\cdot \mathbf{H}_{m, m'} \cdot \vec{x}_{m'} &
        \\
        &= \sum_{m=1}^{M}\left[ \vec{x}_{m}^{\intercal}\cdot \mathbf{H}_{m, m} \cdot \vec{x}_{m} + \sum_{\substack{m'=1 \\ m'\neq m}}^{M}\vec{x}_{m}^{\intercal}\cdot \mathbf{H}_{m, m} \cdot \vec{x}_{m'}\right]
        \\
        &= \sum_{m=1}^{M}\left( - \alpha_{m} \cdot \left(1-\alpha_{m}\right)\cdot \vec{x}_{m}^{\intercal}\cdot \vec{u}_{m} \cdot \vec{u}_{m}^{\intercal} \cdot \vec{x}_{m} \right)
        \\
        & \qquad + \sum_{m=1}^{M}\left[\sum_{\substack{m'=1 \\ m'\neq m}}^{M} \left( \alpha_{m} \cdot \alpha_{m'} \cdot \vec{x}_{m}^{\intercal}\cdot \vec{u}_{m} \cdot \vec{u}_{m'}^{\intercal} \cdot \vec{x}_{m'}\right)\right]
        \\
        &= \sum_{m=1}^{M}\left[ - \alpha_{m} \cdot \left(1-\alpha_{m}\right)\cdot \langle \vec{x}_{m},  \vec{u}_{m} \rangle ^{2} + \alpha_{m} \cdot \langle \vec{x}_{m},  \vec{u}_{m} \rangle  \sum_{\substack{m'=1 \\ m'\neq m}}^{M} \alpha_{m'} \cdot \langle \vec{x}_{m'}, \vec{u}_{m'} \rangle \right].
    \end{flalign}
    Since $\alpha \in \Delta^{M}$,
    \begin{equation}
        \forall m \in [M],~\sum_{\substack{m'=1 \\ m'\neq m}}^{M}\alpha_{m'}  = \left(1-\alpha_{m}\right),
    \end{equation}
    thus,
    \begin{flalign}
        \mathbf{x}^{\intercal} \cdot \mathbf{H} \cdot \mathbf{x} &= \sum_{m=1}^{M}\alpha_{m} \cdot \langle \vec{x}_{m}, \vec{u}_{m} \rangle \cdot \sum_{\substack{m'=1 \\ m'\neq m}}^{M}\alpha_{m'} \Big(\langle \vec{x}_{m'}, \vec{u}_{m'} \rangle - \langle \vec{x}_{m}, \vec{u}_{m} \rangle \Big)  &
        \\
        &= \sum_{m=1}^{M}\alpha_{m} \cdot \langle \vec{x}_{m}, \vec{u}_{m} \rangle \cdot \sum_{m'=1}^{M}\alpha_{m'} \Big(\langle \vec{x}_{m'}, \vec{u}_{m'} \rangle - \langle \vec{x}_{m}, \vec{u}_{m} \rangle \Big)
        \\
        &= \left(\sum_{m=1}^{M} \alpha_{m} \cdot \langle \vec{x}_{m}, \vec{u}_{m} \rangle\right)^{2} - \sum_{m=1}^{M} \alpha_{m} \cdot \langle \vec{x}_{m}, \vec{u}_{m} \rangle^{2}.
    \end{flalign}
    Using Jensen inequality, we have $\mathbf{x}^{\intercal} \cdot \mathbf{H} \cdot \mathbf{x} \leq 0$. It follows that $\vec{H}$ is a semi-definite negative matrix.
    
    In what follows, we suppose that $\left\|\vec{u}_{m}\right\| \leq B$ for all $m\in[M]$. Note that one can write
    \begin{equation}
        \mathbf{x}^{\intercal} \cdot \mathbf{H} \cdot \mathbf{x} = - \left( \sum_{m=1}^{M} \alpha_{m} \cdot \langle \vec{x}_{m}, \vec{u}_{m} \rangle^{2} - \left(\sum_{m=1}^{M} \alpha_{m} \cdot \langle \vec{x}_{m}, \vec{u}_{m} \rangle\right)^{2}\right).
    \end{equation}
    Thus, $\mathbf{x}^{\intercal} \cdot \mathbf{H} \cdot \mathbf{x}$ can be interpreted as  the opposite of the variance of the random variable taking the value $\langle \vec{x}_{m}, \vec{u}_{m} \rangle$ with probability $\alpha_{m}$. For $m\in[M]$, one can bound $\langle \vec{x}_{m}, \vec{u}_{m} \rangle$ using Cauchy-Schwarz inequality as follows,
    \begin{flalign}
        -  \left\|\vec{u}_{m}\right\|_{\mathbb{R}^{d}} \cdot \left\|\vec{x}_{m}\right\|_{\mathbb{R}^{d}} \leq\langle \vec{x}_{m}, \vec{u}_{m} \rangle \leq \left\|\vec{u}_{m}\right\|_{\mathbb{R}^{d}} \cdot \left\|\vec{x}_{m}\right\|_{\mathbb{R}^{d}} . 
    \end{flalign}
    Since $\left\|\vec{u}_{m}\right\|_{\mathbb{R}^{d}} \leq B$ and $\max_{m\in[M]}\left\|\vec{x}_{m}\right\|_{\mathbb{R}^{d}} \leq \left\|\vec{x}\right\|_{\mathbb{R}^{dM}}$, it follows that,
    \begin{flalign}
        -  B \left\|\vec{x}\right\|_{\mathbb{R}^{dM}} \leq\langle \vec{x}_{m}, \vec{u}_{m} \rangle \leq  B\left\|\vec{x}\right\|_{\mathbb{R}^{dM}}. 
    \end{flalign}
    Using Popoviciu's inequality~\cite{popoviciu1965certaines}, we have,
    \begin{equation}
        \sum_{m=1}^{M} \alpha_{m} \cdot \langle \vec{x}_{m}, \vec{u}_{m} \rangle^{2} - \left(\sum_{m=1}^{M} \alpha_{m} \cdot \langle \vec{x}_{m}, \vec{u}_{m} \rangle\right)^{2} \leq B^{2}\left\|\vec{x}\right\|_{\mathbb{R}^{dM}}^{2}.
    \end{equation}
    Thus, $ \mathbf{x}^{\intercal} \cdot \mathbf{H} \cdot \mathbf{x} \geq - B^{2} \left\|\vec{x}\right\|_{\mathbb{R}^{dM}}^{2}$. It follows that $\vec{H} \succcurlyeq - B^{2}I_{d M}$.
\end{proof}

\newpage
\section{Distributed Surrogate Optimization with Black-Box Solver}
\label{app:black_box_solver}
In this section, we cover the scenario where the local SGD solver used in our algorithms (Alg.~\ref{alg:fed_surrogate_opt} and Alg.~\ref{alg:decentralized_surrogate_opt})  is replaced by a (possibly non-iterative) black-box solver that is guaranteed to provide a \emph{local inexact solution} of 
\begin{equation}
    \forall m \in [M],~\minimize_{\theta \in \mathbb{R}^{d}} \sum_{i=1}^{n_{t}} q^{k}(z_{t}^{i}=m)\cdot l(h_{\theta}(\vec{x}_{t}^{(i)}), y_{t}^{(i)}),
\end{equation}
with the following approximation guarantee.
\begin{assumption}[Local $\alpha$-approximate solution]
    \label{assum:local_inexact_solution}
    There exists $0 < \alpha < 1$ such that for $t\in[T]$, $m\in[M]$ and $k>0$, 
    \begin{flalign}
        \sum_{i=1}^{n_{t}} q^{k}(z_{t}^{i}=m)\cdot &\left\{l(h_{\theta_{m, t}^{k}}(\vec{x}_{t}^{(i)}), y_{t}^{(i)})- l(h_{\theta_{m, t,*}^{k}}(\vec{x}_{t}^{(i)}), y_{t}^{(i)})\right\} \leq \nonumber &
        \\
        & \alpha \cdot \sum_{i=1}^{n_{t}} q^{k}(z_{t}^{i}=m)\cdot \left\{l(h_{\theta_{m}^{k-1}}(\vec{x}_{t}^{(i)}), y_{t}^{(i)})- l(h_{\theta_{m, t,*}^{k}}(\vec{x}_{t}^{(i)}), y_{t}^{(i)})\right\},
    \end{flalign}
    where $\theta_{m, t,*}^{k} \in \argmin_{\theta \in \mathbb{R}^{d}} \sum_{i=1}^{n_{t}} q^{k}(z_{t}^{i}=m)\cdot l(h_{\theta}(\vec{x}_{t}^{(i)}), y_{t}^{(i)})$, $\theta_{m, t}^{k}$ is the output of the local solver at client $t$ and $\theta_{m}^{k-1}$ is its starting point (see Alg.~\ref{alg:fed_em}). 
\end{assumption}
We further assume strong convexity.
\begin{assumption}
    \label{assum:strongly_convex}
    For $t\in[T]$ and $i\in[n_{t}]$, we suppose that $\theta \mapsto l\left(h_{\theta}\left(\vec{x}_{t}^{(i)}\right), y_{t}^{(i)}\right)$ is $\mu$-strongly convex. 
\end{assumption}

Assumption~\ref{assum:local_inexact_solution} is equivalent to the $\gamma$-inexact solution used in \cite{li2020federated} (Lemma.~\ref{lem:reformulate_locla_inexact_solution}), when local functions $\left(\Phi_{t}\right)_{1 \leq t\leq T}$ are assumed to be convex. We also need to have $G^{2}=0$ in Assumption~\ref{assum:bounded_dissimilarity} as in \cite[Definition~3]{Sahu2018OnTC}, in order to ensure the convergence of Alg.~\ref{alg:fed_em} and Alg.~\ref{alg:d_em} to a stationary point of $f$, as shown by \cite[Theorem.~2]{wang2020tackling}.\footnote{
    As shown by \cite[Theorem.~2]{wang2020tackling}, the convergence is guaranteed in two scenarios: 1) $G^{2}=0$, 2) All clients use take the same number of local steps using the same local solver. Note that we allow each client to use an arbitrary approximate local solver.
} 

\begin{thm}
    \label{thm:black_box}
    Suppose that Assumptions~\ref{assum:mixture}--\ref{assum:bounded_dissimilarity},~\ref{assum:local_inexact_solution} and~\ref{assum:strongly_convex} hold with $G^{2}=0$ and $\alpha < \frac{1}{\beta^{2}\kappa^{4}}$, then the updates of  federated surrogate optimization converge to a stationary point of $f$, i.e.,
    \begin{equation}
        \lim_{k\to+\infty}\left\|\nabla_{\Theta} f(\Theta^{k}, \Pi^{k})\right\|_{F}^{2} = 0,
    \end{equation}
    and
    \begin{equation}
        \lim_{k\to+\infty}\sum_{t=1}^{T}\frac{n_{t}}{n}\mathcal{KL}\left(\pi_{t}^{k}, \pi_{t}^{k-1}\right) = 0.
    \end{equation}
\end{thm}

As in App.~\ref{app:proofs}, we provide the analysis for the general case of federated surrogate optimization (Alg.~\ref{alg:fed_surrogate_opt}) before showing that \FedEM{} (Alg.~\ref{alg:fed_em}) is a particular case. 

We suppose that, at iteration $k>0$, the partial first-order surrogate functions $g_{t}^{k},~t\in[T]$ used in Alg.~\ref{alg:fed_surrogate_opt} verifies, in addition to Assumptions~\ref{assum:bounded_f_bis}--\ref{assum:bounded_dissimilarity_bis}, the following assumptions that generalize Assumptions~\ref{assum:local_inexact_solution} and~\ref{assum:strongly_convex},

\begin{assumptionbis}{assum:local_inexact_solution}[Local $\alpha$-inexact solution]
    \label{assum:local_inexact_solution_bis}
    There exists $0<\alpha <1$ such that for $t\in[T]$ and $k>0$,
    \begin{equation}
        \forall \vec{v} \in \mathcal{V},~g_{t}^{k}(\vec{u}_{t}^{k}, \vec{v}) - g_{t}^{k}(\vec{u}^{k}_{t,*}, \vec{v}) \leq \alpha \cdot \left\{g_{t}^{k}\left(\vec{u}^{k-1}, \vec{v}\right) - g_{t}^{k}\left(\vec{u}^{k}_{t,*}, \vec{v}\right)\right\},
    \end{equation}
    where $\vec{u}^{k}_{ t,*} \in \argmin_{\vec{u}\in \mathbb{R}^{d_{u}}}g^{k}_{t}\left(\vec{u}, \vec{v}_{t}^{k} \right)$.
\end{assumptionbis}

\begin{assumptionbis}{assum:strongly_convex}
    \label{assum:strongly_convex_bis}
    For $t\in[T]$ and $k>0$, $g_{t}^{k}$ is $\mu$-strongly convex in $\vec{u}$.
\end{assumptionbis}


Under these assumptions a parallel result to Theorem.~\ref{thm:black_box} holds.

\begin{thmbis}{thm:black_box}
    \label{thm:black_box_bis}
    Suppose that Assumptions~\ref{assum:bounded_f_bis}--\ref{assum:bounded_dissimilarity_bis}, Assumptions~\ref{assum:local_inexact_solution_bis} and~\ref{assum:strongly_convex_bis}  hold with $G^{2}=0$ and $\alpha < \frac{1}{\beta^{2}\kappa^{4}}$, then the updates of  federated surrogate optimization converges to a stationary point of $f$, i.e.,
    \begin{equation}
        \lim_{k\to+\infty}\left\|\nabla_{\vec{u}} f(\vec{u}^{k}, \vec{v}_{1:T}^{k})\right\|^{2} = 0,
    \end{equation}
    and
    \begin{equation}
        \lim_{k\to+\infty}\sum_{t=1}^{T}\omega_{t} \cdot d_{\mathcal{V}}\left(\vec{v}_{t}^{k}, \vec{v}_{t}^{k-1}\right) = 0.
    \end{equation}
\end{thmbis}

\subsection{Supporting Lemmas}

First, we prove the following result.

\begin{lem}
    \label{lem:reformulate_locla_inexact_solution}
    Under Assumptions~\ref{assum:smoothness_bis},~\ref{assum:local_inexact_solution_bis} and~\ref{assum:strongly_convex_bis}, the iterates of Alg.~\ref{alg:fed_em} verify for $k>0$ and $t\in[T]$,
    \begin{equation}
        \forall \vec{v} \in \mathcal{V},~\left\|\nabla_{\vec{u}} g_{t}^{k}\left(\vec{u}_{t}^{k}, \vec{v}\right)\right\| \leq \sqrt{\alpha\kappa} \cdot   \left\|\nabla_{\vec{u}} g_{t}^{k}\left(\vec{u}^{k-1}, \vec{v}\right)\right\|,
    \end{equation}
    where $\kappa = L/\mu$.
\end{lem}
\begin{proof}
    Consider $\vec{v} \in \mathcal{V}$. Since $g_{t}^{k}$ is $L$-smooth in $\vec{u}$ (Assumption~\ref{assum:smoothness_bis}), we have using Assumption~\ref{assum:local_inexact_solution_bis},
    \begin{equation}
        \left\|\nabla_{\vec{u}} g_{t}^{k}\left(\vec{u}_{t}^{k}, \vec{v}\right)\right\|^{2}_{F}  \leq 2L \left(g_{t}^{k}\left(\vec{u}_{t}^{k}, \vec{v}\right) - g_{t}^{k}\left(\vec{u}_{t,*}^{k}, \vec{v}\right)\right) \leq 2L\alpha \left(g_{t}^{k}\left(\vec{u}^{k-1}, \vec{v}\right) - g_{t}^{k}\left(\vec{u}_{t,*}^{k}, \vec{v}\right)\right).
        \label{eq:reformulate_locla_inexact_solution_1}
    \end{equation}
    Since $\Phi_{t}^{k}$ is $\mu$-strongly convex (Assumption~\ref{assum:strongly_convex_bis}), we can use Polyak-Lojasiewicz (PL) inequality, 
    \begin{flalign}
        g_{t}^{k}\left(\vec{u}_{t}^{k-1}, \vec{v}\right) - \frac{1}{2\mu} \left\|\nabla_{\vec{u}} g_{t}^{k}\left(\vec{u}^{k-1}, \vec{v}\right)\right\|^{2} \leq g_{t}^{k}\left(\vec{u}_{t,*}^{k-1}, \vec{v}\right),
    \end{flalign}
    thus,
    \begin{flalign}
        \label{eq:reformulate_locla_inexact_solution_2}
         2\mu \left(g_{t}^{k}\left(\vec{u}_{t}^{k-1}, \vec{v}\right) - g_{t}^{k}\left(\vec{u}_{t,*}^{k}, \vec{v}\right)\right) \leq \left\|\nabla_{\vec{u}} g_{t}^{k}\left(\vec{u}^{k-1}, \vec{v}\right)\right\|^{2}.
    \end{flalign}
    Combining Eq.~\eqref{eq:reformulate_locla_inexact_solution_1} and Eq.~\eqref{eq:reformulate_locla_inexact_solution_2}, we have
    \begin{equation}
        \left\|\nabla_{\vec{u}} g_{t}^{k}\left(\vec{u}^{k-1}, \vec{v}\right)\right\|^{2} \leq \frac{L}{\mu} \alpha \left\|\nabla_{\vec{u}} g_{t}^{k-1}\left(\vec{u}^{k-1}, \vec{v}\right)\right\|^{2},
    \end{equation}
    thus,
    \begin{equation}
        \left\|\nabla_{\vec{u}} g_{t}^{k}(\vec{u}_{t}^{k}, \vec{v})\right\| \leq \sqrt{\alpha \kappa} \left\|\nabla_{\vec{u}} g_{t}^{k}(\vec{u}^{k-1}, \vec{v})\right\|.
    \end{equation}
\end{proof}

\begin{lem}
    \label{lem:global_improvement_black_box}
    Suppose that Assumptions~\ref{assum:smoothness_bis},~\ref{assum:bounded_dissimilarity_bis},~\ref{assum:local_inexact_solution_bis} and ~\ref{assum:strongly_convex_bis} hold with $G^{2}=0$. Then,
    \begin{equation}
        g^{k}\left(\vec{u}^{k},\vec{v}^{k}\right) - g^{k}\left(\vec{u}_{*}^{k},\vec{v}^{k}\right) \leq \tilde{\alpha} \times  \left\{g^{k}\left(\vec{u}^{k-1},\vec{v}^{k-1}\right) - g^{k}\left(\vec{u}_{*}^{k},\vec{v}^{k}\right)\right\},
    \end{equation}
    where $\tilde{\alpha} = \beta^{2}\kappa^{4}\alpha$, and $\vec{u}_{*}^{k} \triangleq \argmin_{\vec{u}}g^{k}\left(\vec{u}, \vec{v}_{1:T}^{k}\right)$ where $g^{k}$ is defined in \eqref{eq:app_gk_def}
\end{lem}

\begin{proof}
    Consider $k>0$ and $t\in[T]$. Since $g_{t}$ is $\mu$-convex in $\vec{u}$ (Assumption~\ref{assum:strongly_convex_bis}), we write 
    \begin{flalign}
        \left\|\vec{u}_{t}^{k} - \vec{u}_{*}^{k}\right\|_{F} &\leq \frac{1}{\mu}\left\|\nabla_{\vec{u}} g_{t}^{k}\left(\vec{u}_{t}^{k}, \vec{v}_{t}^{k}\right) -\nabla_{\vec{u}} g_{t}^{k}\left(\vec{u}_{*}^{k}, \vec{v}_{t}^{k}\right) \right\| &
        \\
        &\leq \frac{1}{\mu}\left\|\nabla_{\vec{u}} g_{t}^{k}\left(\vec{u}_{t}^{k}, \vec{v}_{t}^{k}\right)  \right\| + \frac{1}{\mu} \left\|\nabla_{\vec{u}} g_{t}^{k}\left(\vec{u}_{*}^{k}, \vec{v}_{t}^{k}\right) \right\|
        \\
        &\leq \frac{\sqrt{\alpha\kappa}}{\mu}\left\|\nabla_{\vec{u}} g_{t}^{k}\left(\vec{u}^{k-1}, \vec{v}_{t}^{k}\right)\right\| + \frac{1}{\mu} \left\|\nabla_{\vec{u}} g_{t}^{k}\left(\vec{u}_{*}^{k}, \vec{v}_{t}^{k}\right) \right\|,
    \end{flalign}
    where the last inequality is a result of Lemma~\ref{lem:reformulate_locla_inexact_solution}.
    Using Jensen inequality, we have
    \begin{flalign}
        \left\|\vec{u}^{k} - \vec{u}_{*}^{k}\right\|_{F} &= \left\|\sum_{t=1}^{T}\omega_{t} \cdot \left(\vec{u}_{t}^{k} - \vec{u}_{*}^{k}\right)\right\| &
        \\
        &\leq \sum_{t=1}^{T}\omega_{t} \cdot\left\|\vec{u}_{t}^{k} - \vec{u}_{*}^{k}\right\|
        \\
        &\leq \sum_{t=1}^{T}\omega_{t} \cdot\left\{\frac{\sqrt{\alpha\kappa}}{\mu}\left\|\nabla_{\vec{u}} g_{t}^{k}\left(\vec{u}^{k-1}, \vec{v}_{t}^{k}\right)\right\| + \frac{1}{\mu} \left\|\nabla_{\vec{u}} g_{t}^{k}\left(\vec{u}_{*}^{k}, \vec{v}_{t}^{k}\right) \right\|\right\}.
    \end{flalign}
    Using Assumption~\ref{assum:bounded_dissimilarity_bis} and Jensen inequality with the "$\sqrt{\cdot}$" function, it follows that
    \begin{flalign}
         \left\|\vec{u}^{k} - \vec{u}_{*}^{k}\right\| &\leq \sqrt{\alpha\kappa}\frac{\beta}{\mu}\left\|\nabla_{\vec{u}} g^{k}\left(\vec{u}^{k}, \vec{v}_{1:T}^{k}\right)\right\|  +  \frac{\beta}{\mu} \left\|\nabla_{\vec{u}} g^{k}\left(\vec{u}_{*}^{k}, \vec{v}_{1:T}^{k}\right) \right\| &
         \\
         & = \sqrt{\alpha\kappa}\frac{\beta}{\mu}\left\|\nabla_{\vec{u}} g^{k}\left(\vec{u}^{k-1}, \vec{v}_{1:T}^{k}\right)\right\|. 
    \end{flalign}
    Since $g^{k}$ is $L$-smooth in $\vec{u}$ as a convex combination of $L$-smooth function, we have
    \begin{flalign}
        \left\|\nabla_{\vec{u}} g^{k}\left(\vec{u}^{k}, \vec{v}_{1:T}^{k}\right)  \right\| &= \left\|\nabla_{\vec{u}} g^{k}\left(\vec{u}^{k-1}, \vec{v}_{1:T}^{k}\right) - \nabla_{\vec{u}} g^{k}\left(\vec{u}_{*}^{k}, \vec{v}_{1:T}^{k}\right)\right\|\\
        &\leq L\left\|\vec{u}^{k} - \vec{u}_{*}^{k}\right\| &
        \\
        &\leq \beta\sqrt{\alpha\kappa^{3}}  \left\|\nabla_{\vec{u}} g^{k}\left(\vec{u}^{k-1}, \vec{v}_{1:T}^{k} \right)\right\|. 
    \end{flalign}
    Using Polyak-Lojasiewicz (PL), we have
    \begin{equation}
        g^{k}\left(\vec{u}^{k}, \vec{v}_{1:T}^{k}\right) -  g^{k}\left(\vec{u}_{*}^{k}, \vec{v}_{1:T}^{k}\right) \leq \frac{1}{2\mu} \left\|\nabla_{\vec{u}} g^{k}\left(\vec{u}^{k}, \vec{v}_{1:T}^{k}\right)  \right\|^{2} \leq \frac{\beta^{2}\alpha \kappa^{3}}{2\mu} \left\|\nabla_{\vec{u}} g^{k}\left(\vec{u}^{k-1}, \vec{v}_{1:T}^{k}\right)\right\|^{2}.
    \end{equation}
    Using the $L$-smoothness of $g^{k}$ in $\vec{u}$, we have
    \begin{equation}
        \left\|\nabla_{\vec{u}} g^{k}\left(\vec{u}^{k-1}, \vec{v}_{1:T}^{k}\right)\right\|^{2} \leq 2L\left[g^{k}\left(\vec{u}^{k-1}, \vec{v}_{1:T}^{k}\right) -  g^{k}\left(\vec{u}_{*}^{k}, \vec{v}_{1:T}^{k}\right)\right].
    \end{equation}
    Thus, 
    \begin{equation}
        g^{k}\left(\vec{u}^{k}, \vec{v}_{1:T}^{k}\right) -  g^{k}\left(\vec{u}_{*}^{k}, \vec{v}_{1:T}^{k}\right) \leq \underbrace{\beta^{2}\kappa^{4}\alpha}_{\triangleq \tilde{\alpha}}\left(g^{k}\left(\vec{u}^{k-1}, \vec{v}_{1:T}^{k}\right) -  g^{k}\left(\vec{u}_{*}^{k}, \vec{v}_{1:T}^{k}\right)\right).
    \end{equation}
    Since $\vec{v}_{t}^{k} = \argmin_{v\in\mathcal{V}}g_{t}^{k}\left(\vec{u}^{k-1}, \vec{v}\right)$, it follows that 
    \begin{equation}
        g^{k}_{t}\left(\vec{u}^{k-1}, \vec{v}_{t}^{k}\right) \leq g^{k}_{t}\left(\vec{u}^{k-1}, \vec{v}_{t}^{k-1}\right).
    \end{equation}
    Thus,
    \begin{equation}
        g^{k}\left(\vec{u}^{k},\vec{v}_{1:T}^{k}\right) - g^{k}\left(\vec{u}_{*}^{k},\vec{v}_{1:T}^{k}\right) \leq \tilde{\alpha} \times  \left\{g^{k}\left(\vec{u}^{k-1},\vec{v}_{1:T}^{k-1}\right) - g^{k}\left(\vec{u}_{*}^{k},\vec{v}_{1:T}^{k}\right)\right\}.
    \end{equation}
\end{proof}

For $t\in[T]$ and $k>0$, we introduce $r_{t}^{k} \triangleq g_{t}^{k} - f_{t}$ and $r^{k} \triangleq g^{k} - f = \sum_{t=1}^{T}\omega_{t}\left(g_{t}^{k} - f_{t}\right)$. Since $g^{k}_{t}$ is a partial first-order surrogate of $f_{t}$, it follows that $r_{t}^{k}\left(\vec{u}^{k-1}, \vec{v}_{t}^{k-1}\right) = 0$ and that $r_{t}^{k}$ is non-negative and $L$-smooth in $\vec{u}$.
\begin{lem}
    \label{lem:black_box_convergence}
    Suppose that Assumptions~\ref{assum:bounded_f_bis} and ~\ref{assum:smoothness_bis} hold and that 
    \begin{equation}
        g^{k}(\vec{u}^{k}, \vec{v}_{1:T}^{k}) \leq g^{k}(\vec{u}^{k-1}, \vec{v}_{1:T}^{k-1}),~\forall k>0,
    \end{equation}
    then 
    \begin{align}
        \lim_{k\to\infty}r^{k}(\vec{u}^{k}, \vec{v}_{1:T}^{k}) =& 0 \label{eq:convergence_of_r_black_box}
        \\
        \lim_{k\to\infty}\left\|\nabla_{\vec{u}} r^{k}(\vec{u}^{k}, \vec{v}_{1:T}^{k})\right\|^{2} =& 0     \label{eq:convergence_of_nabla_r_black_box}
    \end{align}
    If we moreover suppose that Assumption~\ref{assum:strongly_convex_bis} holds and that there exists $0 < \tilde{\alpha} < 1$ such that for all $k>0$, \begin{equation}
        g^{k}(\vec{u}^{k}, \vec{v}_{1:T}^{k}) - g^{k}(\vec{u}^{k}_{*}, \vec{v}_{1:T}^{k}) \leq \tilde{\alpha} \times \left(g^{k}(\vec{u}^{k-1}, \vec{v}_{1:T}^{k-1}) - g^{k}(\vec{u}^{k}_{*}, \vec{v}_{1:T}^{k}) \right),
    \end{equation}
    then,
    \begin{align}
        \lim_{k\to\infty} \left\|\vec{u}^{k} - \vec{u}_{*}^{k}\right\|^{2} = 0 \label{eq:convergenc_of_distance_to_optimum_black_box}
    \end{align}
    where $\vec{u}^{k}_{*}$ is the minimizer of $\vec{u} \mapsto g^{k}\left(\vec{u}, \vec{v}_{1:T}^{k}\right)$.
\end{lem}
\begin{proof}
    Since $g_{t}$ is a partial first-order surrogate of $f$ near $\left\{\vec{u}^{k-1}, \vec{v}_{t}^{k-1}\right\}$ for $t\in[T]$ and $k>0$, it follows that $g^{k}$ is a majorant of $f$ and that $g^{k}(\vec{u}^{k-1}, \vec{v}^{k-1}) = f(\vec{u}^{k-1}, \vec{v}^{k-1})$. Thus, the following holds,
    \begin{equation}
        f(\vec{u}^{k}, \vec{v}^{k}) \leq g^{k}(\vec{u}^{k}, \vec{v}^{k}) \leq g^{k}(\vec{u}^{k-1}, \vec{v}^{k-1}) = f(\vec{u}^{k-1}, \vec{v}^{k-1}),
    \end{equation}
    It follows that the sequence $\left(f\left(\vec{u}^{k}, \vec{v}^{k}\right)\right)_{k\geq0}$ is a non-increasing sequence. Since $f$ is bounded below (Assum.~\ref{assum:bounded_f_bis}), it follows that $\left(f\left(\vec{u}^{k}, \vec{v}^{k}\right)\right)_{k\geq0}$ is convergent. Denote by $f^{\infty}$ its limit. The sequence $\left(g^{k}(\vec{u}^{k}, \vec{v}^{k})\right)_{k\geq0}$ also  converges to $f^{\infty}$.
    
    \paragraph{Proof of Eq.~\ref{eq:convergence_of_r_black_box}} Using  the fact that $g^{k}(\vec{u}^{k}, \vec{v}^{k}) \leq g^{k}(\vec{u}^{k-1}, \vec{v}^{k})$, we write for $k>0$,
    \begin{equation}
        f(\vec{u}^{k}, \vec{v}_{1:T}^{k}) + r^{k}(\vec{u}^{k}, \vec{v}_{1:T}^{k}) = g^{k}(\vec{u}^{k}, \vec{v}_{1:T}^{k}) \leq g^{k}(\vec{u}^{k-1}, \vec{v}_{1:T}^{k-1}) = f(\vec{u}^{k-1}, \vec{v}_{1:T}^{k-1}),
    \end{equation}
    Thus,
    \begin{equation}
        r^{k}(\vec{u}^{k}, \vec{v}_{1:T}^{k}) \leq f(\vec{u}^{k-1}, \vec{v}_{1:T}^{k-1}) - f(\vec{u}^{k}, \vec{v}^{k}),
    \end{equation}
    By summing over $k$ then passing to the limit when $k \to + \infty$, we have
    \begin{equation}
        \sum_{k=1}^{\infty} r^{k}(\vec{u}^{k}, \vec{v}_{1:T}^{k}) \leq f(\vec{u}^{0}, \vec{v}_{1:T}^{0}) - f^{\infty},
        \end{equation}
    Finally since $r^{k}(\vec{u}^{k}, \vec{v}_{1:T}^{k})$ is non negative for $k>0$, the sequence $\left(r^{k}(\vec{u}^{k}, \vec{v}_{1:T}^{k})\right)_{k\geq 0}$ necessarily converges to zero, i.e.,
    \begin{equation}
        \lim_{k\to\infty}r^{k}(\vec{u}^{k}, \vec{v}_{1:T}^{k}) = 0.
    \end{equation}
    
    \paragraph{Proof of Eq. \ref{eq:convergence_of_nabla_r_black_box}} Because the $L$-smoothness of $\vec{u} \mapsto r^{k}\left(\vec{u}, \vec{v}_{1:T}^{k}\right)$, we have
    \begin{equation}
        r^{k}\left(\vec{u}^{k} - \frac{1}{L}\nabla_{\vec{u}} r^{k}\left(\vec{u}^{k}, \vec{v}_{1:T}^{k}\right)  ,\vec{v}_{1:T}^{k}\right) \leq r^{k}\left(\vec{u}^{k}, \vec{v}_{1:T}^{k}\right)  -\frac{1}{2L} \left\|\nabla_{\vec{u}} r^{k}\left(\vec{u}^{k}, \vec{v}_{1:T}^{k}\right)\right\|^{2}
    \end{equation}
    Thus,
    \begin{flalign}
        \left\|\nabla_{\vec{u}} r^{k}\left(\vec{u}^{k}, \vec{v}_{1:T}^{k}\right)\right\|_{F}^{2} &\leq 2L\left(r^{k}\left(\vec{u}^{k}, \vec{v}_{1:T}^{k}\right) - r^{k}\left(\vec{u}^{k} - \frac{1}{L}\nabla_{\vec{u}} r^{k}\left(\vec{u}^{k}, \vec{v}_{1:T}^{k}\right)  ,\vec{v}_{1:T}^{k}\right) \right) &
        \\
        &\leq 2L r^{k}\left(\vec{u}^{k}, \vec{v}_{1:T}^{k}\right),
    \end{flalign}
    because $r^{k}$ is a non-negative function (Definition~\ref{def:partial_first_order_surrogate}).
    Finally, using Eq.~\eqref{eq:convergence_of_r_black_box}, it follows that
    \begin{equation}
        \lim_{k\to\infty}\left\|\nabla_{\vec{u}} r^{k}(\vec{u}^{k}, \vec{v}_{1:T}^{k})\right\|^{2} = 0.
    \end{equation}
    
    \paragraph{Proof of Eq.~\ref{eq:convergenc_of_distance_to_optimum_black_box}} We suppose now that there exists $0 < \tilde{\alpha} < 1$ such that
    \begin{equation}
        \forall k>0,~~g^{k}(\vec{u}^{k}, \vec{v}_{1:T}^{k}) - g^{k}(\vec{u}^{k}_{*}, \vec{v}_{1:T}^{k}) \leq \tilde{\alpha} \left(g^{k}(\vec{u}^{k-1}, \vec{v}_{1:T}^{k-1}) - g^{k}(\vec{u}^{k}_{*}, \vec{v}_{1:T}^{k}) \right),
    \end{equation}
    It follows that, 
    \begin{equation}
        g^{k}(\vec{u}^{k}, \vec{v}_{1:T}^{k}) - \tilde{\alpha}  g^{k}(\vec{u}^{k-1}, \vec{v}_{1:T}^{k-1}) \leq (1-\tilde{\alpha}) g^{k}(\vec{u}^{k}_{*}, \vec{v}_{1:T}^{k}),
    \end{equation}
    then,
    \begin{equation}
        g^{k}(\vec{u}^{k}_{*}, \vec{v}_{1:T}^{k}) \geq \frac{1}{1-\tilde{\alpha}} \times \left[g^{k}(\vec{u}^{k}, \vec{v}_{1:T}^{k}) -\tilde{\alpha}\times g^{k}(\vec{u}^{k-1}, \vec{v}_{1:T}^{k-1}) \right],
    \end{equation}
    and by using the definition of $g^{k}$ we have,
    \begin{equation}
        \label{eq:lower_bound_g_k}
        g^{k}(\vec{u}^{k}_{*}, \vec{v}_{1:T}^{k}) \geq \frac{1}{1-\tilde{\alpha}} \times \left[g^{k}(\vec{u}^{k}, \vec{v}_{1:T}^{k}) -\tilde{\alpha}\times f(\vec{u}^{k-1}, \vec{v}_{1:T}^{k-1}) \right],
    \end{equation}
    Since $g^{k}\left(\vec{u}_{*}^{k}, \vec{v}_{1:T}^{k}\right) \leq g^{k}\left(\vec{u}^{k}, \vec{v}_{1:T}^{k}\right) \leq g^{k}\left(\vec{u}^{k-1}, \vec{v}_{1:T}^{k-1}\right)$, we have
    \begin{equation}
        \label{eq:upper_bound_g_k}
        g^{k}(\vec{u}^{k}_{*}, \vec{v}_{1:T}^{k}) \leq g^{k}(\vec{u}^{k-1}, \vec{v}_{1:T}^{k-1}) = f(\vec{u}^{k-1}, \vec{v}_{1:T}^{k-1}). 
    \end{equation}
        From Eq.~\eqref{eq:lower_bound_g_k} and Eq.~\eqref{eq:upper_bound_g_k}, it follows that,
    \begin{equation}
        \label{eq:bound_optimal_g_k}
        \frac{1}{1-\tilde{\alpha}} \times \left[g^{k}(\vec{u}^{k}, \vec{v}_{1:T}^{k}) -\tilde{\alpha}\times f(\vec{u}^{k-1}, \vec{v}_{1:T}^{k-1}) \right] \leq g^{k}(\vec{u}^{k}_{*}, \vec{v}_{1:T}^{k}) \leq f(\vec{u}^{k-1}, \vec{v}_{1:T}^{k-1}),
    \end{equation}
    Finally, since $f\left(\vec{u}^{k-1}, \vec{v}_{1:T}^{k-1}\right)  \xrightarrow[k \to +\infty]{} f^{\infty}$ and $g^{k}\left(\vec{u}^{k}, \vec{v}_{1:T}^{k}\right)  \xrightarrow[k \to +\infty]{} f^{\infty}$, it follows from Eq.~\eqref{eq:bound_optimal_g_k} that,
    \begin{equation}
        \lim_{k\to\infty}g^{k}\left(\vec{u}^{k}_{*}, \vec{v}_{1:T}^{k}\right) = f^{\infty}.
    \end{equation}
    Since $g^{k}$ is $\mu$-strongly convex in $\vec{u}$ (Assumption~\ref{assum:strongly_convex}), we write
    \begin{equation}
        \frac{\mu}{2} \left\|\vec{u}^{k} - \vec{u}^{k}_{*}\right\|^{2} \leq g^{k}\left(\vec{u}^{k}, \vec{v}_{1:T}^{k}\right) - g^{k}\left(\vec{u}^{k}_{*}, \vec{v}_{1:T}^{k}\right),
    \end{equation}
    It follows that,
    \begin{equation}
        \lim_{k\to+\infty} \left\|\vec{u}^{k} - \vec{u}^{k}_{*}\right\|^{2} = 0.
    \end{equation}
\end{proof}

\subsection{Proof of Theorem~\ref{thm:black_box_bis}}
Combining the previous lemmas we prove the convergence of Alg.~\ref{alg:fed_surrogate_opt} with a black box solver.

\begin{repthm}{thm:black_box_bis}
    Suppose that Assumptions~\ref{assum:bounded_f_bis}--\ref{assum:bounded_dissimilarity_bis}, Assumptions~\ref{assum:local_inexact_solution_bis} and~\ref{assum:strongly_convex_bis}  hold with $G^{2}=0$ and $\alpha \leq \frac{1}{\beta^{2}\kappa^{4}}$, then the updates of  federated surrogate optimization (Alg.~\ref{alg:fed_surrogate_opt}) converge to a stationary point of $f$, i.e.,
    \begin{equation}
        \label{eq:convergence_of_u_balck_box}
        \lim_{k\to+\infty}\left\|\nabla_{\vec{u}} f(\vec{u}^{k}, \vec{v}_{1:T}^{k})\right\|^{2} = 0,
    \end{equation}
    and, 
    \begin{equation}
        \label{eq:convergence_of_v_black_box}
        \lim_{k\to+\infty}\sum_{t=1}^{T}\omega_{t} \cdot d_{\mathcal{V}}\left(\vec{v}_{t}^{k}, \vec{v}_{t}^{k-1}\right) = 0.
    \end{equation}
\end{repthm}
\begin{proof}
    \begin{equation}
        f(\vec{u}^{k}, \vec{v}_{1:T}^{k}) = g^{k}(\vec{u}^{k}, \vec{v}_{1:T}^{k}) - r^{k}(\vec{u}^{k}, \vec{v}_{1:T}^{k}).
    \end{equation}
    Computing the gradient norm, we have,
    \begin{flalign}
        \left\|\nabla_{\vec{u}} f(\vec{u}^{k}, \vec{v}_{1:T}^{k})\right\| &= \left\|\nabla_{\vec{u}} g^{k}(\vec{u}^{k}, \vec{v}_{1:T}^{k}) - \nabla_{\vec{u}} r^{k}(\vec{u}^{k}, \vec{v}_{1:T}^{k})\right\| &
        \\
        &\leq \left\|\nabla_{\vec{u}} g^{k}(\vec{u}^{k}, \vec{v}_{1:T}^{k}) \right\| + \left\|\nabla_{\vec{u}} r^{k}(\vec{u}^{k}, \vec{v}_{1:T}^{k}) \right\|. \label{eq:bound_nabla_fk_part1}
    \end{flalign}
    Since $g^{k}$ is $L$-smooth in $\vec{u}$, we write
    \begin{flalign}
        \left\|\nabla_{\vec{u}} g^{k}(\vec{u}^{k}, \vec{v}_{1:T}^{k})\right\| &= \left\|\nabla_{\vec{u}} g^{k}(\vec{u}^{k}, \vec{v}^{k}) - \nabla_{\vec{u}} g^{k}(\vec{u}^{k}_{*}, \vec{v}_{1:T}^{k}) \right\| &
        \\
        &\leq L \left\|\vec{u}^{k} - \vec{u}^{k}_{*} \right\|. \label{eq:nabla_bound_gk}
    \end{flalign}
    Thus by replacing Eq.~\eqref{eq:nabla_bound_gk} in Eq.~\eqref{eq:bound_nabla_fk_part1}, we have
    \begin{equation}
        \label{eq:bound_nabla_fk_part2}
        \left\|\nabla_{\vec{u}} f(\vec{u}^{k}, \vec{v}_{1:T}^{k})\right\|  \leq L^{2}  \left\|\vec{u}^{k} - \vec{u}^{k}_{*} \right\|^{2} + \left\|\nabla_{\vec{u}} r^{k}(\vec{u}^{k}, \vec{v}_{1:T}^{k}) \right\|.
    \end{equation}
    Using Lemma~\ref{lem:global_improvement_black_box}, there exists $0 <\tilde{\alpha} < 1$, such that
    \begin{equation}
        \left[g^{k}(\vec{u}^{k}, \vec{v}_{1:T}^{k}) - g^{k}(\vec{u}^{k}_{*}, \vec{v}_{1:T}^{k})\right] \leq \tilde{\alpha} \times \left[g^{k}(\vec{u}^{k-1}, \vec{v}_{1:T}^{k-1}) - g^{k}(\vec{u}^{k}_{*}, \vec{v}_{1:T}^{k}) \right].
    \end{equation}
    Thus, the conditions of Lemma~\ref{lem:black_box_convergence} hold, and we can use Eq.~\eqref{eq:convergence_of_nabla_r_black_box} and ~\eqref{eq:convergenc_of_distance_to_optimum_black_box}, i.e.
    \begin{align}
        \left\|\nabla_{\vec{u}} r^{k}(\vec{u}^{k}, \vec{v}_{1:T}^{k})\right\|^{2} & \xrightarrow[k \to +\infty]{}  0\\
        \left\|\vec{u}^{k} - \vec{u}_{*}^{k}\right\|^{2} & \xrightarrow[k \to +\infty]{}  0.
    \end{align}
    Finally, combining this with Eq.~\eqref{eq:bound_nabla_fk_part2}, we get the final result
    \begin{equation}
        \lim_{k\to+\infty} \left\|\nabla_{\vec{u}} f(\vec{u}^{k}, \vec{v}_{1:T}^{k})\right\| = 0.
    \end{equation}
    
    Since $g_{t}^{k}$ is a partial first-order surrogate of $f_{t}$ near $\left\{\vec{u}^{k-1}, \vec{v}_{t}^{k-1}\right\}$ for $k>0$ and $t\in[T]$, it follows that
    \begin{flalign}
        \sum_{t=1}^{T}\omega \cdot d_{\mathcal{V}}\left(\vec{v}_{t}^{k}, \vec{v}_{t}^{k-1}\right) &= g^{k}\left(\vec{u}^{k-1}, \vec{v}_{1:T}^{k-1}\right) - g^{k}\left(\vec{u}^{k-1}, \vec{v}_{1:T}^{k}\right) &
        \\
        &\leq g^{k}\left(\vec{u}^{k-1}, \vec{v}_{1:T}^{k-1}\right) - g^{k}\left(\vec{u}^{k}, \vec{v}_{1:T}^{k}\right)
    \end{flalign}
    Thus,
    \begin{equation}
        \sum_{t=1}^{T}\omega_{t}\cdot d_{\mathcal{V}}\left(\vec{v}_{t}^{k}, \vec{v}_{t}^{k-1}\right) \leq  f\left(\vec{u}^{k-1}, \vec{v}_{1:T}^{k-1}\right) - f\left(\vec{u}^{k}, \vec{v}_{1:T}^{k}\right)
    \end{equation}
    Since  $d_{\mathcal{V}}\left(\vec{v}_{t}^{k}, \vec{v}_{t}^{k-1}\right)$ is non-negative for $k>0$ and $t\in[T]$, it follows that 
    \begin{equation}
        \lim_{k\to+\infty}\sum_{t=1}^{T}\omega_{t} \cdot d_{\mathcal{V}}\left(\vec{v}_{t}^{k}, \vec{v}_{t}^{k-1}\right) = 0
    \end{equation}
\end{proof}

\subsection{Proof of Theorem~\ref{thm:black_box}}

\begin{repthm}{thm:black_box}
    Suppose that Assumptions~\ref{assum:mixture}--\ref{assum:bounded_dissimilarity} and Assumptions~\ref{assum:local_inexact_solution},~\ref{assum:strongly_convex} hold with $G^{2}=0$ and $\alpha \leq \frac{1}{\beta^{2}\kappa^{5}}$, then the updates of \FedEM{} (Alg.~\ref{alg:fed_em}) converge to a stationary point of $f$, i.e.,
    \begin{equation}
        \lim_{k\to+\infty}\left\|\nabla_{\Theta} f(\Theta^{k}, \Pi^{k})\right\|_{F}^{2} = 0,
    \end{equation}
    and, 
    \begin{equation}
        \lim_{k\to+\infty}\sum_{t=1}^{T}\frac{n_{t}}{n}\mathcal{KL}\left(\pi_{t}^{k}, \pi_{t}^{k-1}\right) = 0.
    \end{equation}
\end{repthm}
\begin{proof}
    We prove this result as a particular case of Theorem~\ref{thm:black_box_bis}. To this purpose, we consider that $\mathcal{V} \triangleq \Delta^{M}$, ${u}=\Theta \in \mathbb{R}^{d M}$, ${v}_{t} = \pi_{t}$, and $\omega_t = n_t/n$ for $t\in[T]$. For $k>0$, we define $g^{k}_{t}$ as follow,
    \begin{flalign}
        g^{k}_{t}\Big (\Theta, \pi_{t}\Big) =   \frac{1}{n_{t}}\sum_{i=1}^{n_{t}}\sum_{m=1}^{M}q_{t}^{k}\left(z^{(i)}_{t}=m\right) \cdot &\bigg( l\left(h_{\theta_m}(\vec{x}_{t}^{(i)}), y_{t}^{(i)}\right)  -  \log p_{m}(\vec{x}_{t}^{(i)}) - \log \pi_{t} \nonumber &
        \\
        & \qquad \qquad +  \log q_{t}^{k}\left(z_{t}^{(i)}=m\right) - c \bigg), 
    \end{flalign}
    where $c$ is the same constant appearing in Assumption~\ref{assum:logloss}, Eq.~\eqref{eq:log_loss}. With this definition, it is easy to check that the federated surrogate optimization algorithm (Alg.~\ref{alg:fed_surrogate_opt}) reduces to \FedEM{} (Alg.~\ref{alg:fed_em}). Theorem~\ref{thm:black_box} then follows immediately  from Theorem~\ref{thm:black_box_bis}, once we verify that $\left(g_{t}^{k}\right)_{1\leq t\leq T}$ satisfy the assumptions of Theorem~\ref{thm:black_box_bis}.
    
    Assumption~\ref{assum:bounded_f_bis}, Assumption~\ref{assum:finite_variance_bis},~Assumption~\ref{assum:bounded_dissimilarity_bis}, Assumption~\ref{assum:local_inexact_solution_bis} and Assumption~\ref{assum:strongly_convex_bis} follow directly from  Assumption~\ref{assum:bounded_f}, Assumption~\ref{assum:finitie_variance}, Assumption~\ref{assum:bounded_dissimilarity}, Assumption~\ref{assum:local_inexact_solution} and Assumption~\ref{assum:strongly_convex}, respectively.
    Lemma~\ref{lem:g_is_smooth} shows that for $k>0$, $g^{k}$ is smooth w.r.t. $\Theta$ and then Assumption~\ref{assum:smoothness_bis} is satisfied.
    Finally, Lemmas~\ref{lem:em_gap_is_kl_divergence}--\ref{lem:compute_kl_pi_decentralized} show that for $t\in[T]$ $g_{t}^{k}$ is a partial first-order surrogate of $f_{t}$ w.r.t.~$\Theta$ near $\left\{\Theta^{k-1}, \pi_{t}\right\}$ with $d_{\mathcal{V}}(\cdot, \cdot) = \mathcal{KL}(\cdot\|\cdot)$. 
\end{proof}

\newpage
\section{Details on Experimental Setup}
\label{app:experiments_details}
\subsection{Datasets and Models}
\label{app:datasets_models}
In this section we provide detailed description of the datasets and models used in our experiments. We used a synthetic dataset, verifying Assumptions~\ref{assum:mixture}-\ref{assum:logloss}, and five "real" datasets (CIFAR-10/CIFAR-100~\cite{Krizhevsky09learningmultiple}, sub part of EMNIST~\cite{cohen2017emnist}, sub part of FEMNIST~\cite{caldas2018leaf, mcmahan2017communication} and Shakespeare \cite{caldas2018leaf, mcmahan2017communication}) from which, two (FEMNIST and Shakespeare) has natural client partitioning. Below, we give a detailed description of the datasets and the models / tasks considered for each of them.

\subsubsection{CIFAR-10 / CIFAR-100}
CIFAR-10 and CIFAR-100 are labeled subsets of the 80 million tiny images dataset. They both share the same $60,000$ input images. CIFAR-100 has a finer labeling, with $100$ unique labels, in comparison to CIFAR-10, having $10$ unique label. We used Dirichlet allocation~\cite{Wang2020Federated}, with parameter $\alpha=0.4$ to partition CIFAR-10 among $80$ clients. We used Pachinko allocation~\cite{reddi2021adaptive} with parameters $\alpha=0.4$ and $\beta=10$ to partition CIFAR-100 on $100$ clients. For both of them we train  MobileNet-v2~\cite{sandler2018mobilenetv2} architecture with an additional linear layer. We used TorchVision~\cite{marcel2010torchvision} implementation of MobileNet-v2. 

\subsubsection{EMNIST}
EMNIST (Extended MNIST) is a 62-class image classification dataset, extending the classic MNIST dataset. In our experiments, we consider $10\%$ of the EMNIST dataset, that we partition using Dirichlet allocation of parameter $\alpha=0.4$ over $100$ clients. 
We train the same convolutional network as in~\cite{reddi2021adaptive}. The network has two convolutional layers (with $3 \times 3$ kernels), max pooling, and dropout, followed by a 128 unit dense layer.

\subsubsection{FEMNIST}
FEMNIST (Federated Extended MNIST) is a 62-class image classification dataset built by
partitioning the data of Extended MNIST based on the writer of the digits/characters. In our
experiments, we used a subset with $15\%$ of the total number of writers in FEMNIST.
We train the same convolutional network as in~\cite{reddi2021adaptive}. The network has two convolutional layers (with $3 \times 3$ kernels), max pooling, and dropout, followed by a 128 unit dense layer.

\subsubsection{Shakespeare}
This dataset is built from The Complete Works of William Shakespeare and is partitioned by the speaking roles~\cite{mcmahan2017communication}. In our experiments, we discarded roles with less than two sentences.
We consider character-level based language modeling on this dataset. The model takes
as input a sequence of 200 English characters and predicts the next character. The model
embeds the $80$ characters into a learnable $8$-dimensional embedding space, and uses two
stacked-LSTM layers with $256$ hidden units, followed by a densely-connected layer. We also normalized each character by its frequency of appearance.

\subsubsection{Synthetic dataset}
Our synthetic dataset has been generated according to Assumptions~\ref{assum:mixture}--\ref{assum:logloss} as follows:
\begin{enumerate}
  \item Sample weight $\pi_{t}\sim\text{Dir}\left(\alpha\right),~t\in[T]$ from a symmetric Dirichlet distribution of parameter $\alpha \in \mathbb{R}^{+}$
  \item Sample $\theta_{m}\in\mathbb{R}^{d}\sim\mathcal{U}\left(\left[-1, 1\right]^{d}\right),~m\in[M]$ for uniform distribution over $\left[-1, 1\right]^{d}$.
  \item Sample $m_{t},~t\in[T]$ from a log-normal distribution with mean $4$ and sigma $2$, then set $n_{t}=\min\left(50+m_{t}, 1000\right)$.
  \item For $t\in[T]$ and $i\in[n_{t}]$,~draw $x_{t}^{(i)}\sim\mathcal{U}\left(\left[-1, 1\right]^{d}\right)$ and $\epsilon_{t}^{(i)}\sim \mathcal{N}\left(0, I_{d}\right)$.
  \item For $t\in[T]$ and $i\in[n_{t}]$,~draw $z_{t}^{(i)}\sim\mathcal{M}\left(\pi_{t}\right)$.
  \item For $\in[T]$ and $i\in[n_{t}]$,~draw $y_{t}^{(i)}\sim \mathcal{B}\left(\text{sigmoid}\left(\langle x_{t}^{(i)}, \theta_{z_{t}^{(i)}}\rangle + \epsilon_{t}^{(i)}\right)\right)$.
\end{enumerate}

\subsection{Implementation Details}
\label{app:implementation_details}

\subsubsection{Machines} We ran the experiments on a CPU/GPU cluster, with different GPUs available (e.g., Nvidia Tesla V100, GeForce GTX 1080 Ti, Titan X, Quadro RTX 6000, and Quadro RTX 8000). Most experiments with CIFAR10/CIFAR-100 and EMNIST were run on GeForce GTX 1080 Ti cards, while most experiments with Shakespeare and FEMNIST were run on the Quadro RTX 8000 cards. For each dataset, we ran around $30$ experiments (not counting the development/debugging time). Table~\ref{tab:computation_time} gives the average amount of time needed to run one simulation for each dataset. The time needed per simulation was extremely long for Shakespeare dataset, because we used a batch size of $128$. We remarked that increasing the batch size beyond $128$ caused the model to converge to poor local minima, where the model keeps predicting a white space as next character. 

\begin{table}
    \caption{Average computation time and used GPU for each dataset.}
    \label{tab:computation_time}
    \centering
    \begin{tabular}{l  r  r}
        \toprule
        Dataset & GPU & Simulation time    \\
        \midrule
            Shakespeare \cite{caldas2018leaf, mcmahan2017communication} & Quadro RTX 8000 & 4h42min
            \\
            FEMNIST \cite{caldas2018leaf} & Quadro RTX 8000 & 1h14min \\
            EMNIST \cite{cohen2017emnist} & GeForce GTX 1080 Ti&  46min\\
            CIFAR10 \cite{Krizhevsky09learningmultiple} & GeForce GTX 1080 Ti & 2h37min
            \\
            CIFAR100 \cite{Krizhevsky09learningmultiple} & GeForce GTX 1080 Ti & 3h9min
            \\
            Synthetic & GeForce GTX 1080 Ti & 20min
            \\
        \bottomrule
    \end{tabular}
\end{table}

\subsubsection{Libraries}
We used PyTorch \cite{paszke2019pytorch} to build and train our models. We also used Torchvision~\cite{marcel2010torchvision} implementation of MobileNet-v2~\cite{sandler2018mobilenetv2}, and for image datasets preprossessing. We used LEAF~\cite{caldas2018leaf} to build FEMNIST dataset and the federated version of Shakespeare dataset.  

\subsubsection{Hyperparameters}
For each method and each task, the learning rate was set via grid search on the set $\left\{10^{-0.5}, 10^{-1}, 10^{-1.5}, 10^{-2}, 10^{-2.5}, 10^{-3}\right\}$.
$\FedProx$ and $\pFedMe$'s penalization parameter $\mu$ was tuned via grid search on  $\left\{10^{1}, 10^{0}, 10^{-1}, 10^{-2}, 10^{-3}\right\}$. For Clustered FL, we used the same values of tolerance as the ones used in its official implementation~\cite{sattler2020clustered}. 
We found tuning $\texttt{tol}_{1}$ and $\texttt{tol}_{2}$ particularly hard: no empirical rule is provided in \cite{sattler2020clustered}, and the few random setting we tried did not show any improvement in comparison to the default ones. 
For each dataset and each method, Table~\ref{tab:hyperparams_summary} reports the learning rate $\eta$ that achieved the corresponding result in Table~\ref{tab:results_summary}. 
\begin{table}
    \caption{Learning rates $\eta$ used for the experiments in Table~\ref{tab:results_summary}.  Base-10 logarithms are reported.}
    \label{tab:hyperparams_summary}
    \centering
    \scriptsize
    \begin{tabular}{l   r   r  r  r r  r  r}
        \toprule
        Dataset & \FedAvg~\cite{mcmahan2017communication}  & \FedProx~\cite{Sahu2018OnTC}  & \FedAvg+~\cite{jiang2019improving} & \texttt{Clustered FL}~\cite{sattler2020clustered} &  \texttt{pFedMe}~\cite{dinh2020personalized} & \FedEM~(Ours)   \\
        \midrule
            FEMNIST  & $-1.5$ & $-1.5$ & $-1.5$ & $-1.5$ & $-1.5$ & $-1.0$
            \\
            EMNIST  &  $-1.5$ & $-1.5$ & $-1.5$ & $-1.5$ & $-1.5$ & $-1.0$
            \\
            CIFAR10   &  $-1.5$ & $-1.5$ & $-1.5$ & $-1.5$ & $-1.0$ & $-1.0$
            \\
            CIFAR100  &  $-1.0$ & $-1.0$ & $-1.0$ & $-1.0$ & $-1.0$ & $-0.5$
            \\
            Shakespeare  &  $-1.0$ & $-1.0$ & $-1.0$ & $-1.0$ & $-1.0$ & $-0.5$
            \\
            Synthetic  & $-1.0$ & $-1.0$ & $-1.0$ & $-1.0$ & $-1.0$ & $-1.0$
            \\
        \bottomrule
    \end{tabular}
\end{table}

\newpage
\section{Additional Experimental Results}
\label{app:full_experiments}
\subsection{Fully Decentralized Federated Expectation-Maximization}
\label{app:d_fedem_results}
\DEM{} considers the scenario where clients communicate directly in a peer-to-peer fashion instead of relying on the central server mediation. In order to simulate \DEM{}, we consider a binomial  Erd\H{o}s-R\'enyi graph \cite{erdos59a} with parameter $p=0.5$, and we set the mixing weight using \emph{Fast Mixing Markov Chain} \cite{Boyd03fastestmixing} rule. We report the result of this experiment in Table~\ref{tab:results_summary_decentralized}, showing the average weighted accuracy with weight proportional to local dataset sizes. We observe that \DEM{} often performs better than other FL approaches and slightly worst than \FedEM{}, except on CIFAR-10 where it has low performances.

\begin{table}
    \caption{Test accuracy: average across clients.}
    \label{tab:results_summary_decentralized}
    \centering
    \scriptsize
    \begin{tabular*}{\textwidth}{l @{\extracolsep{\fill}} r @{\extracolsep{\fill}} r @{\extracolsep{\fill}} r @{\extracolsep{\fill}} r @{\extracolsep{\fill}} r @{\extracolsep{\fill}} r @{\extracolsep{\fill}} r}
        \toprule
        Dataset & Local & \FedAvg~\cite{mcmahan2017communication}   & \FedAvg+~\cite{jiang2019improving}  & \texttt{Clustered FL}~\cite{sattler2020clustered}  & \pFedMe~\cite{dinh2020personalized} &  \FedEM~(Ours) & \DEM~(Ours)
        \\
        \midrule
            FEMNIST & $71.0 $ & $78.6 $ & $75.3 $ & $73.5 $ & $74.9 $ & $\mathbf{79.9} $  & $
            77.2$
            \\
            EMNIST & $71.9$ & $82.6 $ &  $83.1 $ & $82.7 $ &  $83.3$ & $\mathbf{83.5} $ & $\mathbf{83.5}$
            \\
            CIFAR10  & $70.2 $ & $78.2 $ &  $82.3 $ & $78.6 $  & $81.7 $ & $\mathbf{84.3}$ & $77.0$
            \\
            CIFAR100 & $31.5 $ & $40.9 $ & $39.0 $ & $41.5 $ & $41.8 $ & $\mathbf{44.1}$ & $43.9$
            \\
            Shakespeare & $32.0 $ & $\mathbf{46.7}$  &  $40.0 $ & $46.6 $ & $41.2 $ &  $\mathbf{46.7} $ & $45.4$
            \\
            Synthetic & $65.7 $ & $68.2 $ &  $68.9 $ & $69.1 $ & $69.2 $ & $\mathbf{74.7} $  & $73.8$
            \\
        \bottomrule
    \end{tabular*}
\end{table}

\subsection{Comparison with \texttt{MOCHA}}
\label{app:mocha}

In the case of synthetic dataset, for which train a linear model, we compare \FedEM{} with \texttt{MOCHA}~\cite{smith2017federated}. We implemented \texttt{MOCHA} in Python following the official implementation
\footnote{
    \url{https://github.com/gingsmith/fmtl}
}
in MATLAB. We tuned the parameter $\lambda$ of \texttt{MOCHA on a holdout validation set via grid search in $\{10^{1}, 10^{0}, 10^{-1}, 10^{-2}, 10^{-3}\}$}, and we found that the optimal value of $\lambda$ is $10^{0}$. For this value, we ran \texttt{MOCHA} on the synthetic dataset with three different seeds, and we found that the average accuracy is $73.4 \pm 0.05$ in comparison to $74.7 \pm 0.01$ achieved by \FedEM{}. Note that \texttt{MOCHA} is the second best method after \FedEM{} on this dataset. Unfortunately, \texttt{MOCHA} only works for linear models.

\subsection{Generalization to Unseen Clients}
\label{app:generalization_to_unseen_clients}

Table~\ref{tab:results_summary_zero_shot} shows that \FedEM{} allows new clients to learn a personalized model at least as good as \FedAvg's global one and always better than \FedAvg+'s one.
Unexpectedly, new clients achieve sometimes  a significantly higher test accuracy than old clients (e.g., 47.5\% against 44.1\% on CIFAR100). 

In order to better understand this difference, we looked at the distribution of \FedEM{} personalized weights for the old clients and new ones. 
The average distribution entropy equals $0.27$ and $0.92$ for old and new clients, respectively.
This difference shows that old clients tend to have more skewed distributions, suggesting that some components may be overfitting the local training dataset leading the old clients to give them a high weight.

We also considered a setting where unseen clients progressively collect their own dataset. We investigate the effect of the number of samples on the average test accuracy across unseen clients, starting from no local data (and therefore using uniform weights to mix the $M$ components) and progressively adding more labeled examples until the full local labeled training set is assumed to be available. Figure~\ref{fig:progressive_pi} shows that \FedEM{} achieves a significant level of personalization as soon as clients collect a labeled dataset whose size is about  $20\%$ of what the original clients used for training.

As we mentioned in the main text, it is not clear how the other personalized FL algorithms (e.g., \texttt{pFedMe} and Clustered FL)  should be extended to handle unseen clients.
For example, the  global model learned by pFedMe during training can then be used to perform some ``fine-tuning'' at the new clients, but how exactly? The original \texttt{pFedMe} paper~\cite{dinh2020personalized} does not even mention this issue.
For example, the client could use the global model as initial vector for some local SGD steps (similarly to what done in \texttt{FedAvg+} or the MAML approaches) or it could perform a local \texttt{pFedMe} update (lines 6-9 in \cite[Alg. 1]{dinh2020personalized}). 
The problem is even more complex for Clustered FL (and again not discussed in~\cite{sattler2020clustered}). 
The new client should be assigned to one of the clusters identified. One can think to compute the cosine distances of the new client from those who participated in training, but this would require the server to maintain not only the  model learned, but also the last-iteration gradients of all clients that participated in the training.
Moreover, it is not clear which metric should be considered to assign the new client to a given cluster (perhaps the average cosine similarity from all clients in the cluster?). 
This is an arbitrary choice as \cite{sattler2020clustered} does not provide a criterion to assign clients to a cluster, but only to decide if a given cluster should be split in two new ones.
It appears that many options are possible and they deserve separate investigation. 
Despite these considerations, 
we performed an additional experiment extending \texttt{pFedMe} to unseen clients as described in the second option above on CIFAR-100 dataset with a sampling rate of  $20\%$. \texttt{pFedMe} achieves a test accuracy of  $40.5\% \pm 1.66\%$, in comparison to  $38.9\% \pm 0.97\%$  for \FedAvg{} and  $42.7\% \pm 0.33\%$  for \FedEM. \FedEM{} thus performs better on unseen clients, and \texttt{pFedMe}'s accuracy shows a much larger variability.

\begin{figure}
    \centering
    \includegraphics[scale=0.3]{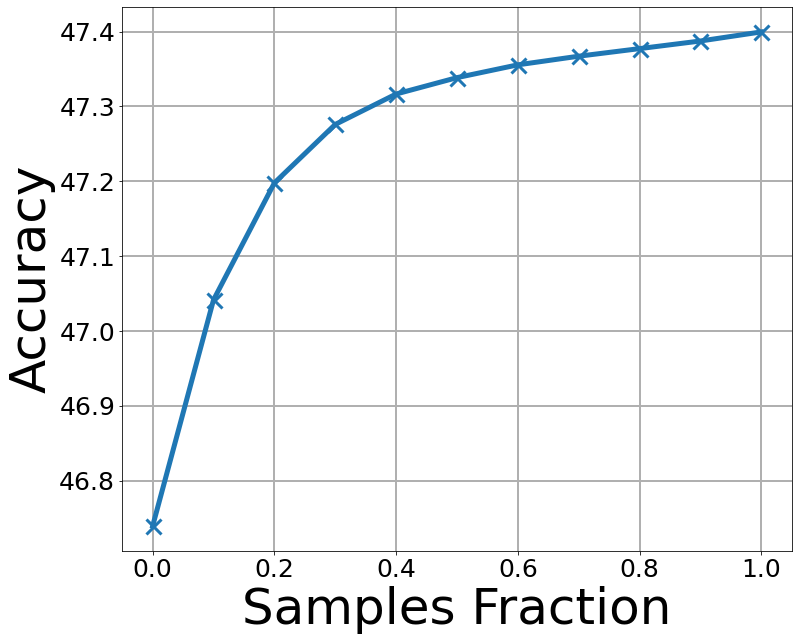}
    \caption{Effect of the number of samples on the average test accuracy across clients unseen at training on CIFAR100 dataset.}
    \label{fig:progressive_pi}
\end{figure}

\subsection{\FedEM{} and Clustering} We performed additional experiments with synthetic datasets to check if \FedEM{} recovers clusters in practice. 
We modified the synthetic dataset generation so that the mixture weight vector $\pi_{t}$ of each client $t$ has a single entry equal to $1$ that is selected uniformly at random.
We consider two scenarios both with $T=300$ client, the first with $M=2$ component and the second with $M=3$ components.
In both cases \FedEM{} recovered almost the correct $\Pi^{*}$ and $\Theta^{*}$: we have $\texttt{cosine\_distance}\left(\Theta^{*}, \breve{\Theta}\right) \leq 10^{-2}$ and $\texttt{cosine\_distance}\left(\Pi^{*}, \breve{\Pi}\right)\leq 10^{-8}$. 
A simple clustering algorithm that assigns each client to the component with the largest mixture weight achieves $100\%$ accuracy, i.e., it partitions the clients in sets coinciding with the original clusters.

\subsection{Effect of $M$ in Time-Constrained Setting}
\label{app:effect_of_M}

Recall that in \FedEM{}, each client needs to update and transmit $M$ components at each round, requiring roughly $M$ times more computation and $M$ times larger messages than the competitors in our study.
In this experiment, we considered a challenging time-constrained setting, where \FedEM{} is limited to run one third ($= 1/M$) of the rounds of the other methods. 
The results in Table~\ref{tab:results_summary_less_rounds} show that even if \FedEM{} does not reach its maximum accuracy, it still outperforms the other methods on 3 datasets.

\begin{table}
    \caption{Test and train accuracy comparison across different tasks. For each method, the best test accuracy is reported. For \FedEM{} we run only $\frac{K}{M}$ rounds, where $K$ is the total number of rounds for other methods--$K=80$ for Shakespeare and $K=200$ for all other datasets--and $M=3$ is the number of components used in \FedEM{}.}
    \label{tab:results_summary_less_rounds}
    \centering
    \scriptsize
    \resizebox{\textwidth}{!}{ 
    \begin{tabular}{l  c   c  c  c   c    c   c}
        \toprule
        \multirow{2}{*}{Dataset} & \multirow{2}{*}{Local} & \multirow{2}{*}{\FedAvg~\cite{mcmahan2017communication}}  & \multirow{2}{*}{\FedProx~\cite{Sahu2018OnTC}}  & \multirow{2}{*}{\FedAvg+~\cite{jiang2019improving}} & \texttt{Clustered}  &  \multirow{2}{*}{\texttt{pFedMe}~\cite{dinh2020personalized}} & \multirow{2}{*}{\FedEM{}~(Ours)} \\
        &  &   &   &  & \texttt{FL}~\cite{sattler2020clustered}  &   & 
        \\
        \midrule
            FEMNIST \cite{caldas2018leaf} & $71.0$ ($99.2$) & $\mathbf{78.6}$ ($79.5$) & $78.6$ ($79.6$) & $75.3$ ($86.0$) & $73.5$ ($74.3$) & $74.9$ ($91.9$) & $74.0$ ($80.9$)  \\
            EMNIST \cite{cohen2017emnist} & $71.9$ ($99.9$) & $82.6$ ($86.5$) & $82.7$ ($86.6$) & $83.1$ ($93.5$) & $82.7$ ($86.6$) &  $\mathbf{83.3}$ ($91.1$) & $82.7$ ($89.4$)  \\
            CIFAR10 \cite{Krizhevsky09learningmultiple} & $70.2$ ($99.9$) & $78.2$ ($96.8$) & $78.0$ ($96.7$) & $82.3$ ($98.9$) & $78.6$ ($96.8$)  & $81.7$ ($99.8$) & $\mathbf{82.5}$ ($92.2$)  \\
            CIFAR100 \cite{Krizhevsky09learningmultiple} & $31.5$ ($99.9$) & $41.0$ ($78.5$) & $40.9$ ($78.6$) & $39.0$ ($76.7$) & $41.5$ ($78.9$) & $41.8$ ($99.6$) & $\mathbf{42.0}$ ($72.9$)  \\
            Shakespeare \cite{caldas2018leaf} & $32.0$ ($95.3$) & $\mathbf{46.7}$ ($48.7$)  & $45.7$ ($47.3$) & $40.0$ ($93.1$) & $46.6$ ($48.7$) & $41.2$ ($42.1$) & $43.8$ ($44.6$)  \\
            Synthetic & $65.7$ ($91.0$) & $68.2$ ($68.7$) & $68.2$ ($68.7$) & $68.9$ ($71.0$) & $69.1$ ($85.1$) & $69.2$ ($72.8$) & $\mathbf{73.2}$ ($74.7$)  \\
        \bottomrule
    \end{tabular}
    }
\end{table}

We additionally compared \FedEM{} with a model having the same number of parameters in order to check if \FedEM{}'s advantage comes from the additional model parameters rather than by its specific formulation.
To this purpose, we trained Resnet-18 and Resnet-34 on CIFAR10. The first one has about $3$ times more parameters than MobileNet-v2 and then roughly as many parameters as FedEM with $M=3$. 
The second one has about $6$ times more parameters than FedEM with $M=3$.
We observed that both architectures perform even worse than MobileNet-v2, so the comparison with these larger models does not suggest that \FedEM{}'s advantage comes from the larger number of parameters. 

We note that there are many possible choices of (more complex) model architectures, and finding one that works well for the task at hand is quite challenging due to the large search space, the bias-variance trade-off, and the specificities of the FL setting.

\newpage
\subsection{Additional Results under Client Sampling}
\label{app:client_sampling}

In our experiments, except for Figure~\ref{fig:sample_rate_and_M}, we considered that all clients participate at each round. 
We run extra experiments with client sampling, by allowing only $20\%$ of the clients to participate at each round.
We also incorporate \texttt{APFL}~\cite{deng2020adaptive} into the comparison.
Table~\ref{tab:results_summary_partial_participation} summarizes our findings, giving the average and standard deviation of the test accuracy across 3 independent runs.

\begin{table}
    \caption{Test accuracy under $20\%$ client sampling: average across clients with +/- standard deviation over 3 independent runs. All experiments with $1200$ communication rounds.}
    \label{tab:results_summary_partial_participation}
    \centering
    \scriptsize
    \begin{tabular}{l   r  r  r   r    r  }
        \toprule
        Dataset & \FedAvg~\cite{mcmahan2017communication}  & \FedAvg+~\cite{jiang2019improving}  & \texttt{pFedMe}~\cite{dinh2020personalized} & \texttt{APFL}~\cite{deng2020adaptive}  & FedEM~(Ours)   
        \\
        \midrule
        CIFAR10 \cite{Krizhevsky09learningmultiple} & $73.1 \pm 0.14$ & $77.7 \pm 0.16$ &  $77.8 \pm 0.07$ &  $78.2 \pm 0.27$ & $\mathbf{82.1} \pm 0.13$
        \\
        CIFAR100 \cite{Krizhevsky09learningmultiple} & $40.6 \pm 0.17$ & $39.7 \pm 0.75$ &  $39.9 \pm 0.08$ &  $40.3 \pm 0.71$ & $\mathbf{43.2} \pm 0.23$
        \\
        Synthetic & $68.2 \pm 0.02$ & $69.0 \pm 0.03$ &  $69.1 \pm 0.03$ &  $69.1 \pm 0.04$ & $\mathbf{74.7} \pm 0.01$
        \\
        \bottomrule
    \end{tabular}
\end{table}

\subsection{Convergence Plots}

Figures~\ref{fig:cifar10} to \ref{fig:synthetic} show the evolution of  average train loss, train accuracy, test loss, and test accuracy over time for each experiment shown in Table~\ref{tab:results_summary}.

\begin{figure}[t] 
    \begin{minipage}{0.5\linewidth}
        \centering
        \includegraphics[width=.9\linewidth]{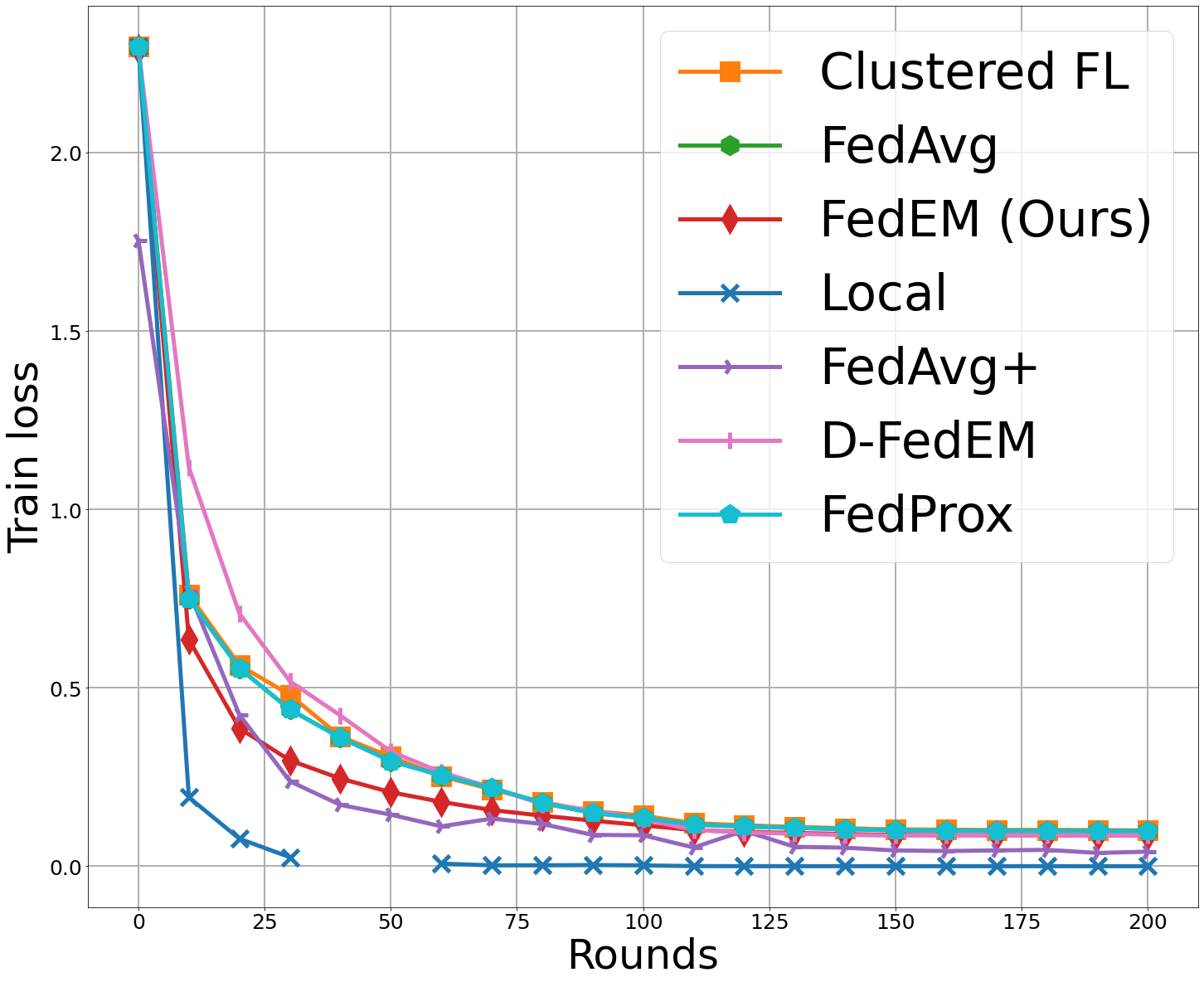} 
    \end{minipage}
    \begin{minipage}{0.5\linewidth}
        \centering
        \includegraphics[width=.9\linewidth]{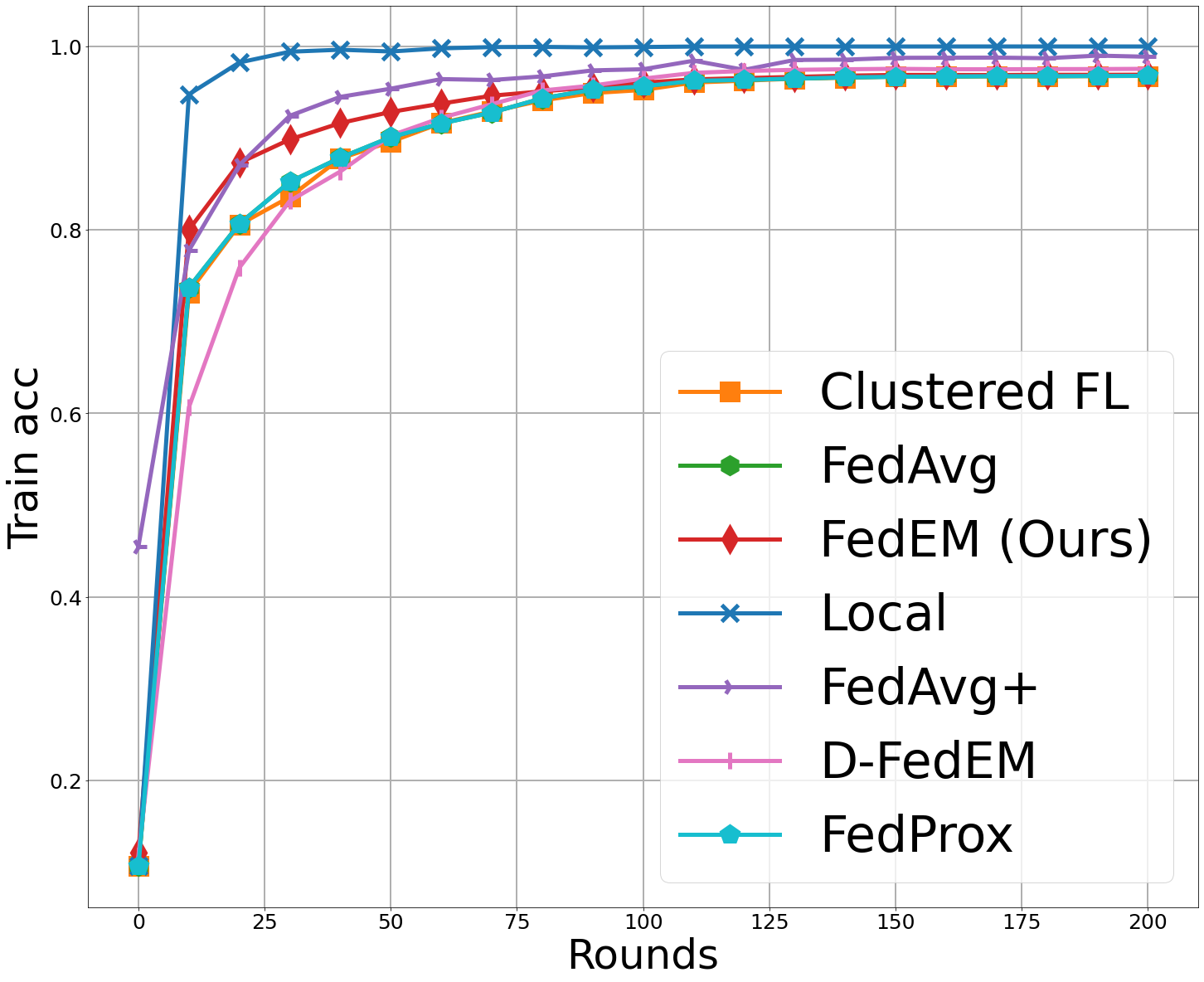} 
    \end{minipage} 
    \\
    \begin{minipage}{0.5\linewidth}
        \centering
        \includegraphics[width=.9\linewidth]{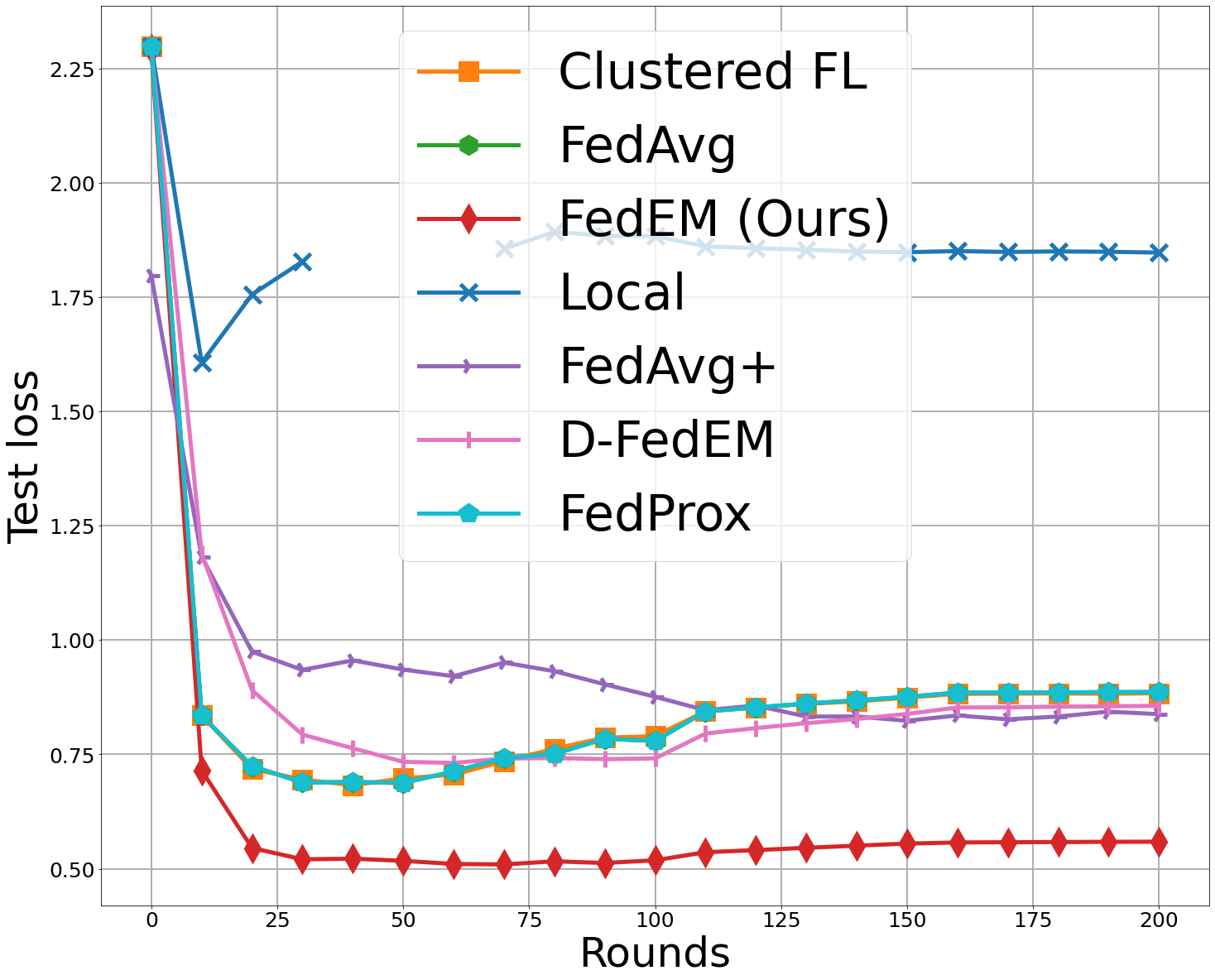} 
    \end{minipage}
    \begin{minipage}{0.5\linewidth}
        \centering
        \includegraphics[width=.9\linewidth]{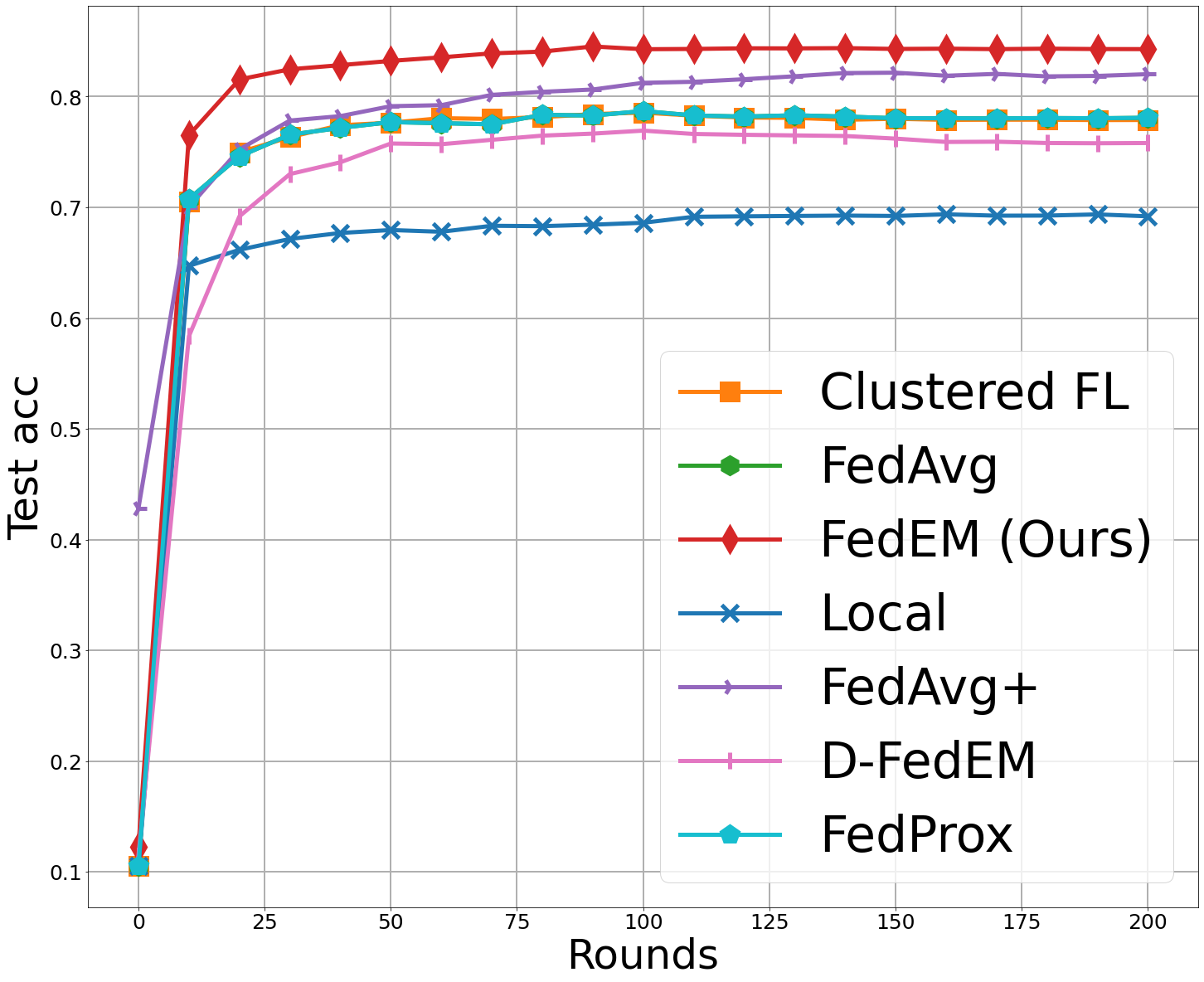} 
    \end{minipage} 
  \caption{Train loss, train accuracy, test loss, and test accuracy for   CIFAR10~\cite{Krizhevsky09learningmultiple}. .
  }
  \label{fig:cifar10} 
\end{figure}

\begin{figure}[t] 
    \begin{minipage}{0.5\linewidth}
        \centering
        \includegraphics[width=.9\linewidth]{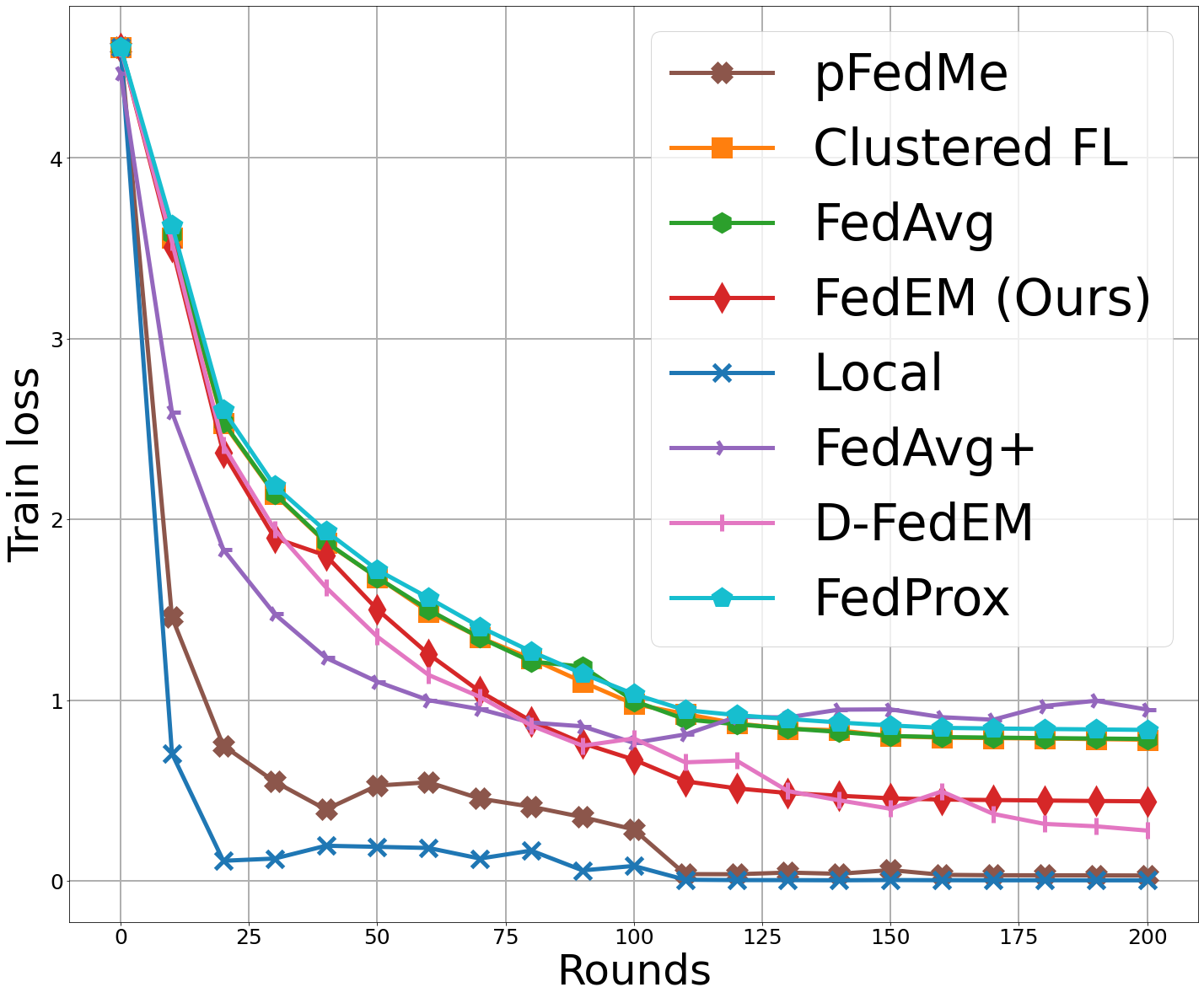} 
    \end{minipage}
    \begin{minipage}{0.5\linewidth}
        \centering
        \includegraphics[width=.9\linewidth]{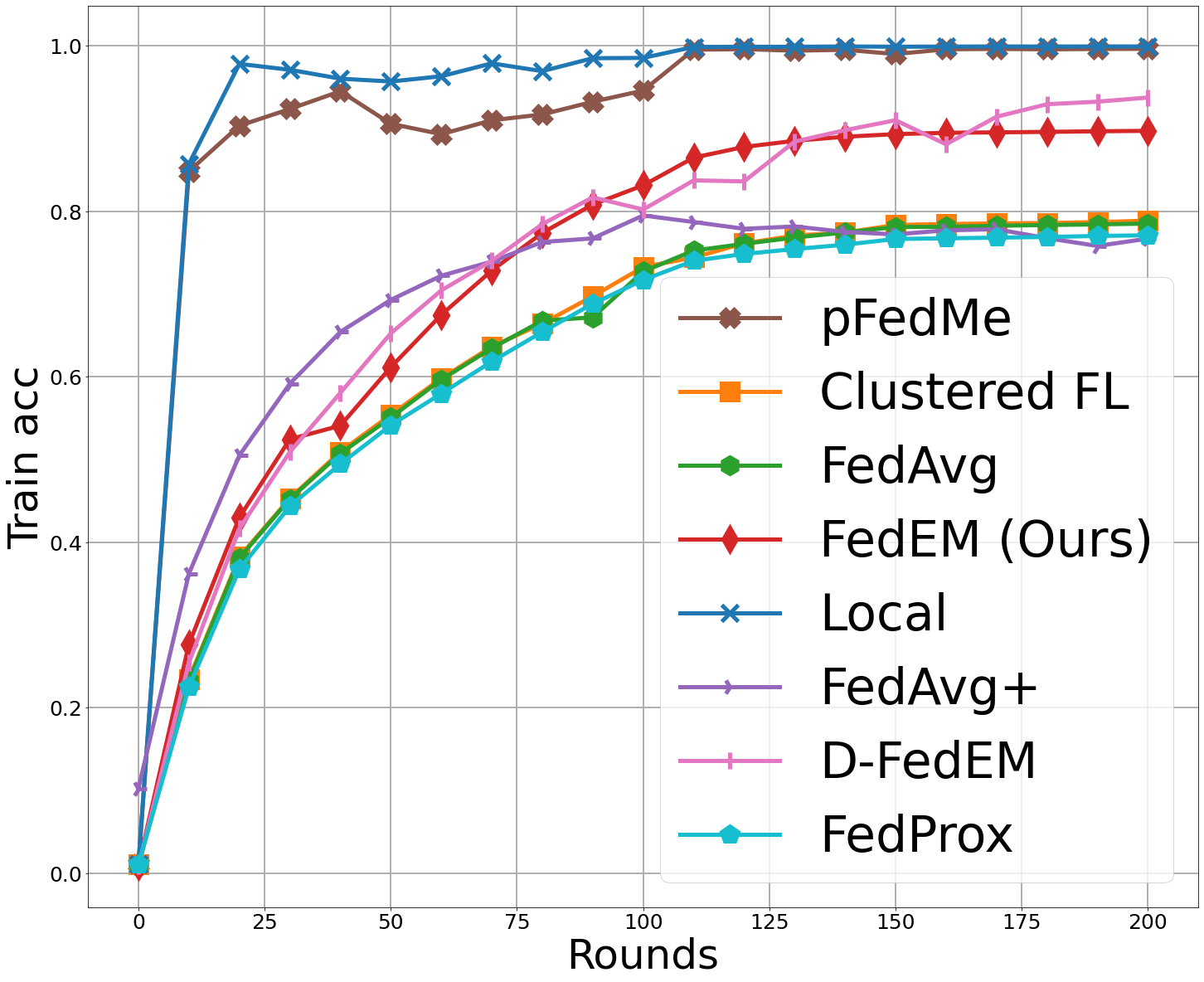} 
    \end{minipage} 
    \\
    \begin{minipage}{0.5\linewidth}
        \centering
        \includegraphics[width=.9\linewidth]{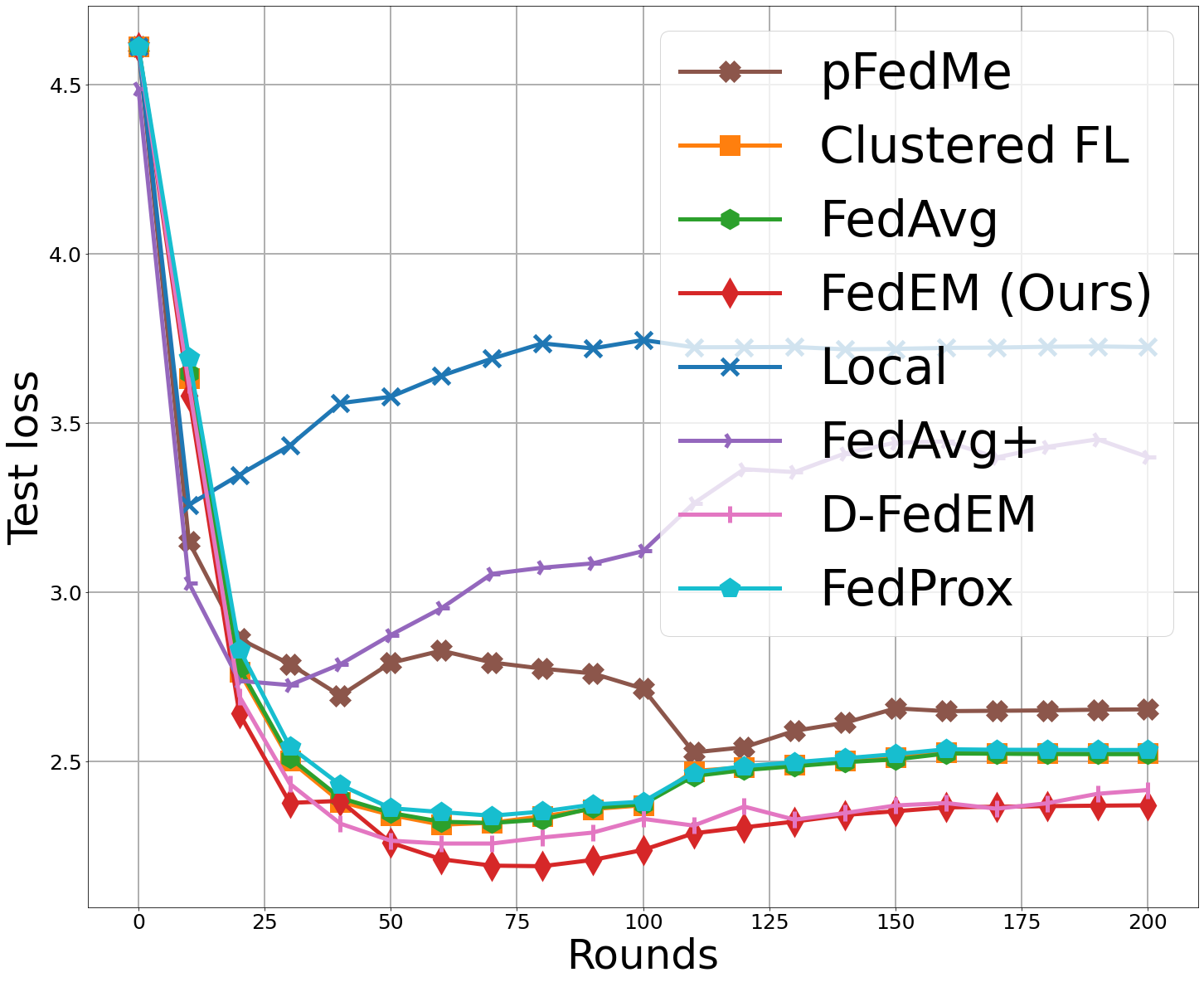} 
    \end{minipage}
    \begin{minipage}{0.5\linewidth}
        \centering
        \includegraphics[width=.9\linewidth]{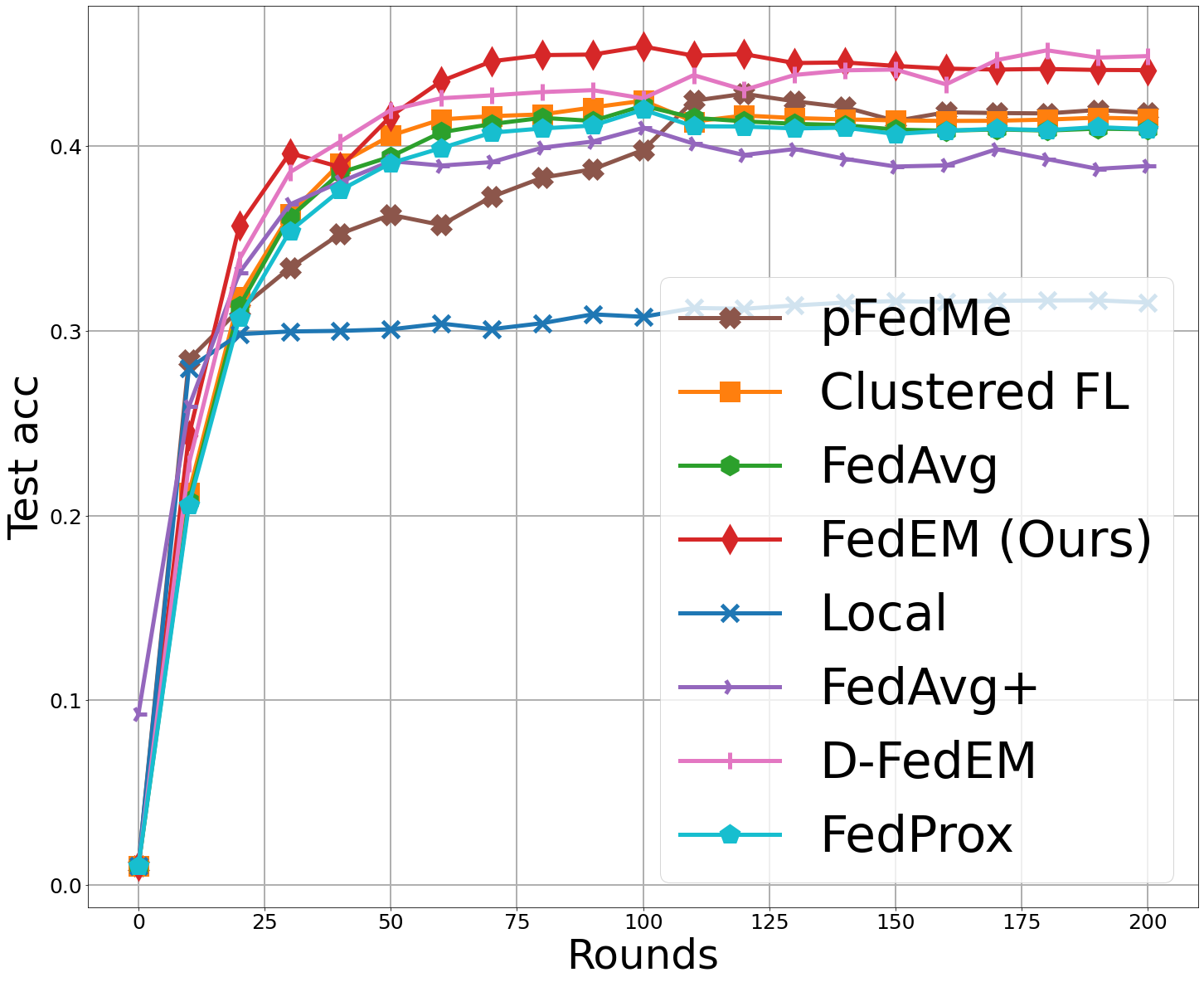} 
    \end{minipage} 
  \caption{Train loss, train accuracy, test loss, and test accuracy for   CIFAR100~\cite{Krizhevsky09learningmultiple}.}
\end{figure}

\begin{figure}[t] 
    \begin{minipage}{0.5\linewidth}
        \centering
        \includegraphics[width=.9\linewidth]{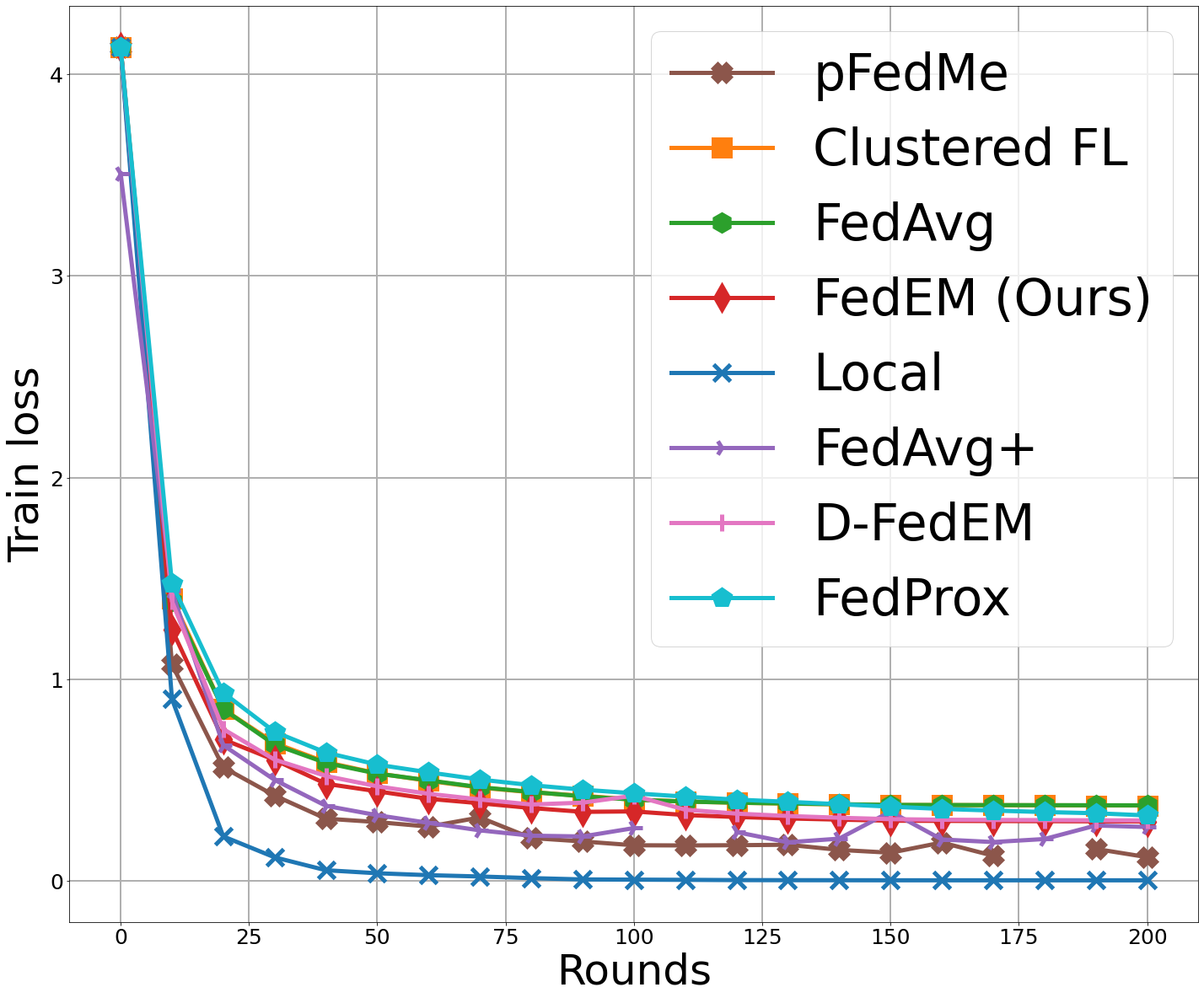} 
    \end{minipage}
    \begin{minipage}{0.5\linewidth}
        \centering
        \includegraphics[width=.9\linewidth]{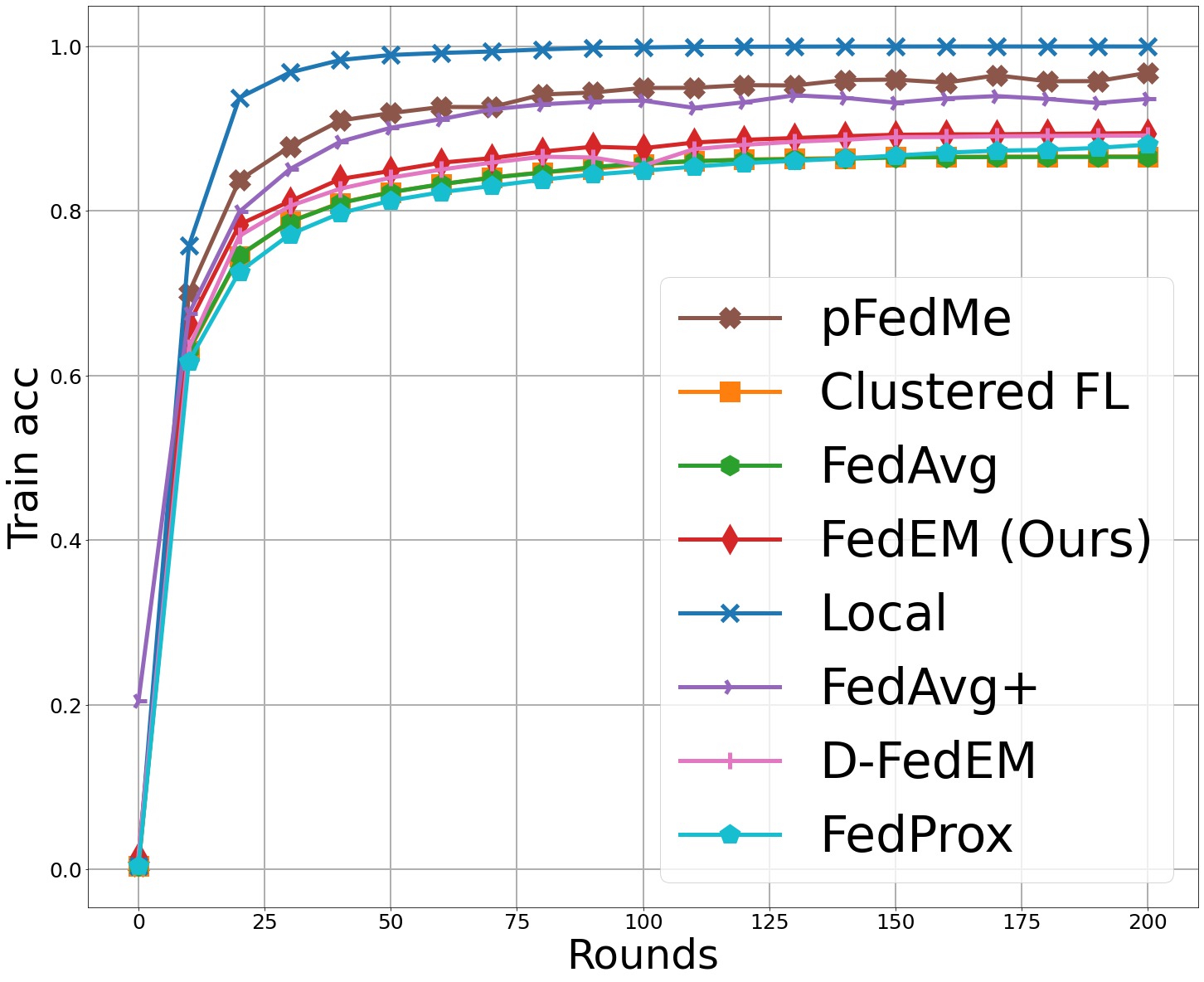} 
    \end{minipage} 
    \\
    \begin{minipage}{0.5\linewidth}
        \centering
        \includegraphics[width=.9\linewidth]{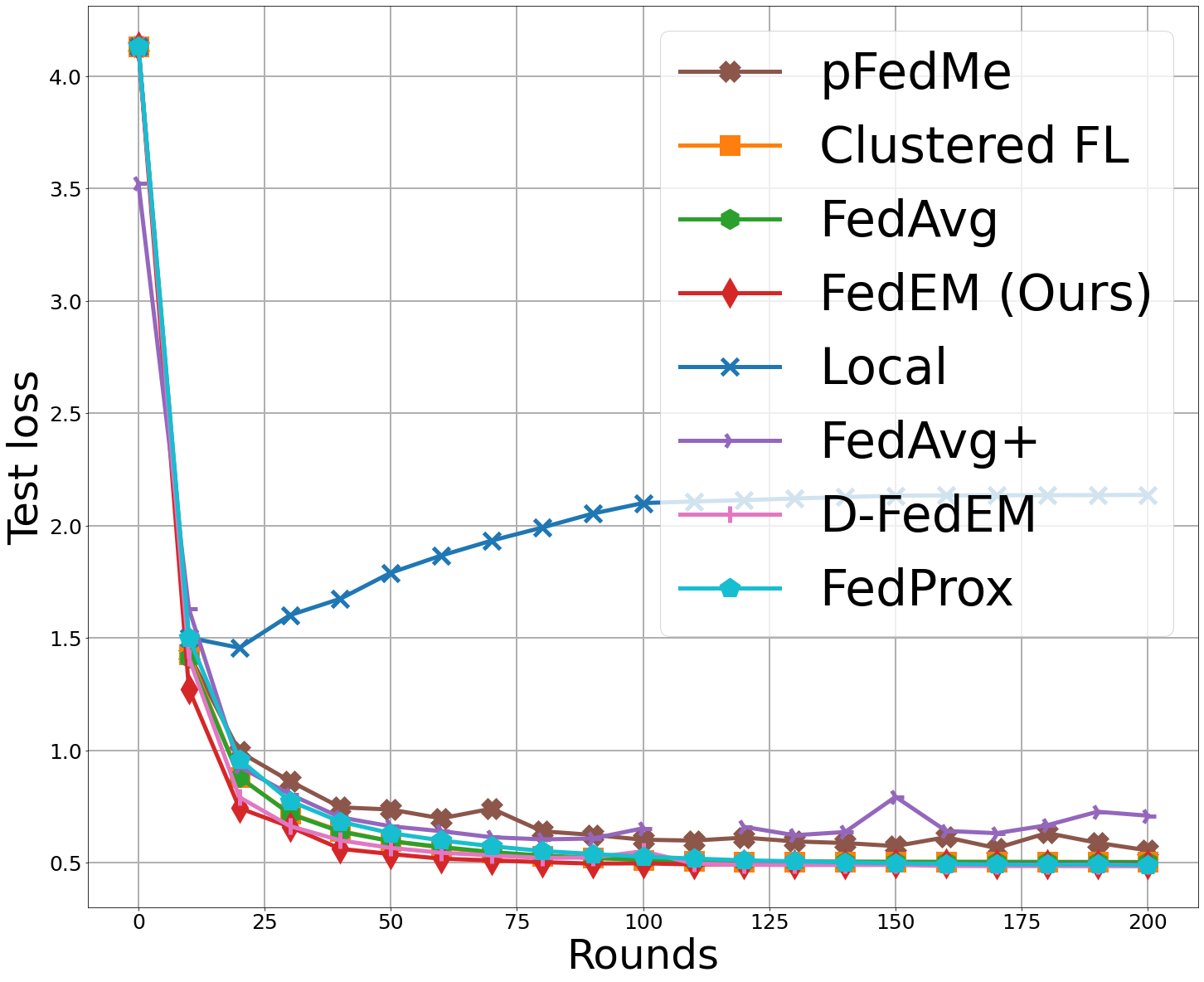} 
    \end{minipage}
    \begin{minipage}{0.5\linewidth}
        \centering
        \includegraphics[width=.9\linewidth]{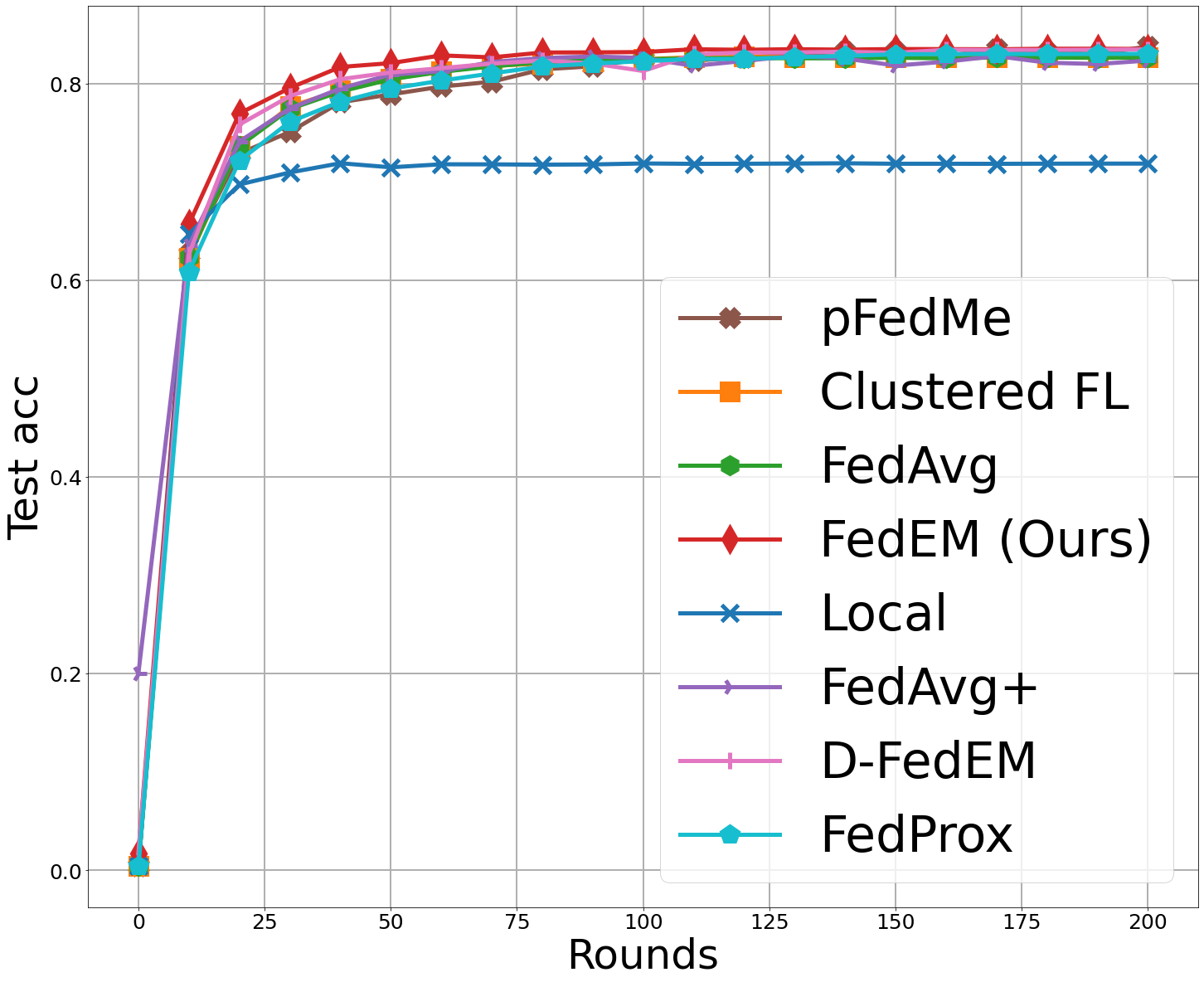} 
    \end{minipage} 
  \caption{Train loss, train accuracy, test loss, and test accuracy for   EMNIST~\cite{cohen2017emnist}.}
\end{figure}

\begin{figure}[t] 
    \begin{minipage}{0.5\linewidth}
        \centering
        \includegraphics[width=.9\linewidth]{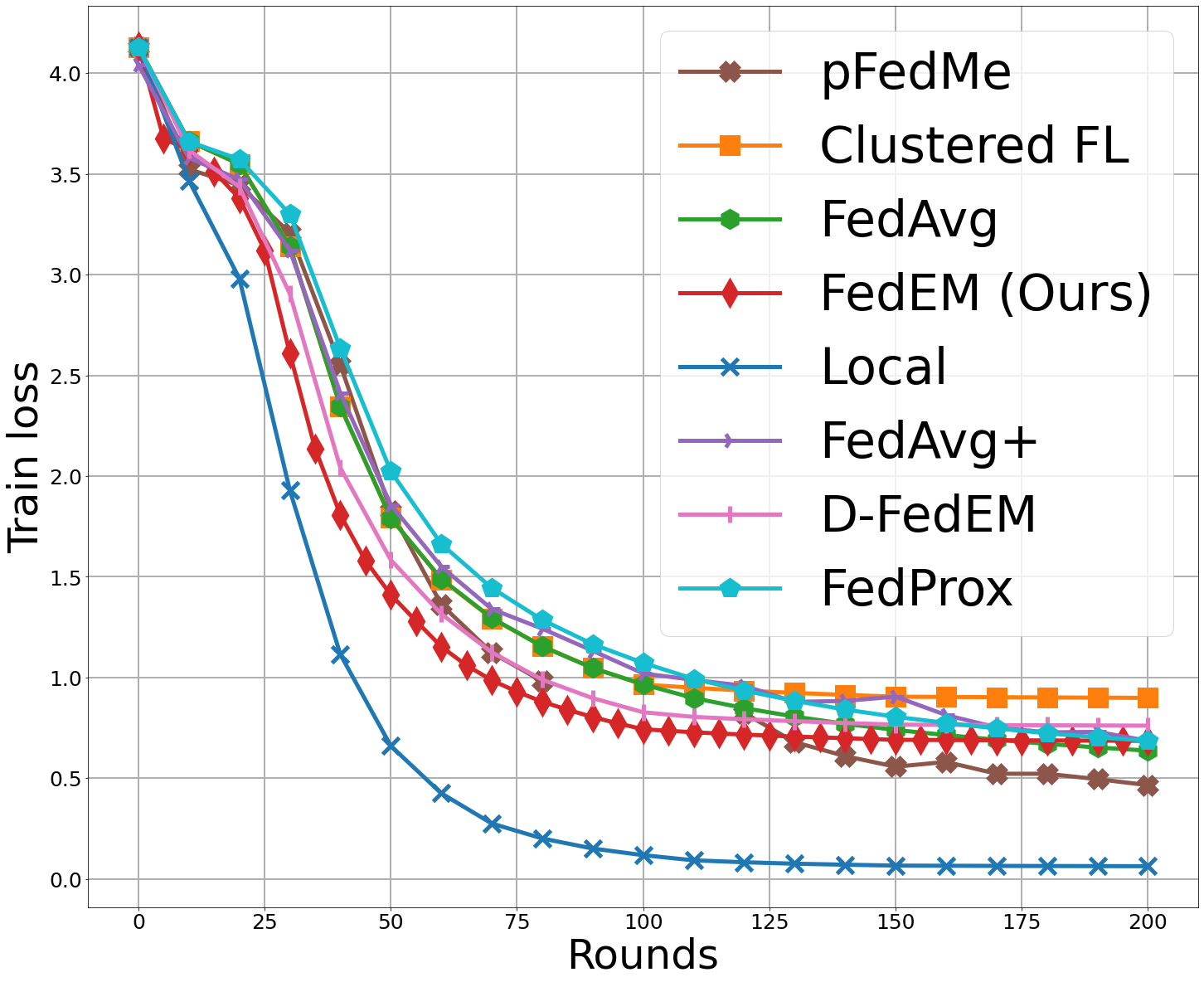} 
    \end{minipage}
    \begin{minipage}{0.5\linewidth}
        \centering
        \includegraphics[width=.9\linewidth]{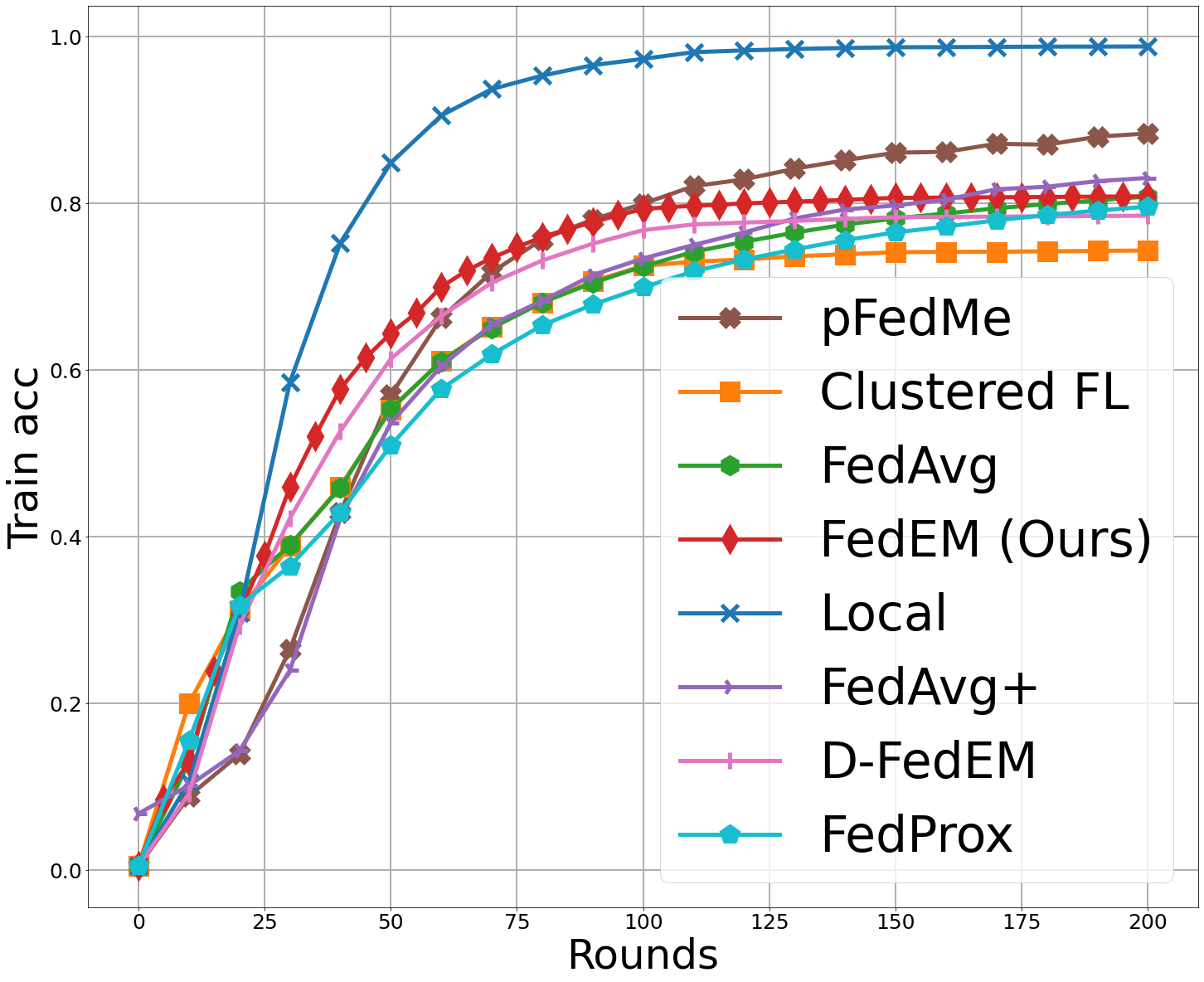} 
    \end{minipage} 
    \\
    \begin{minipage}{0.5\linewidth}
        \centering
        \includegraphics[width=.9\linewidth]{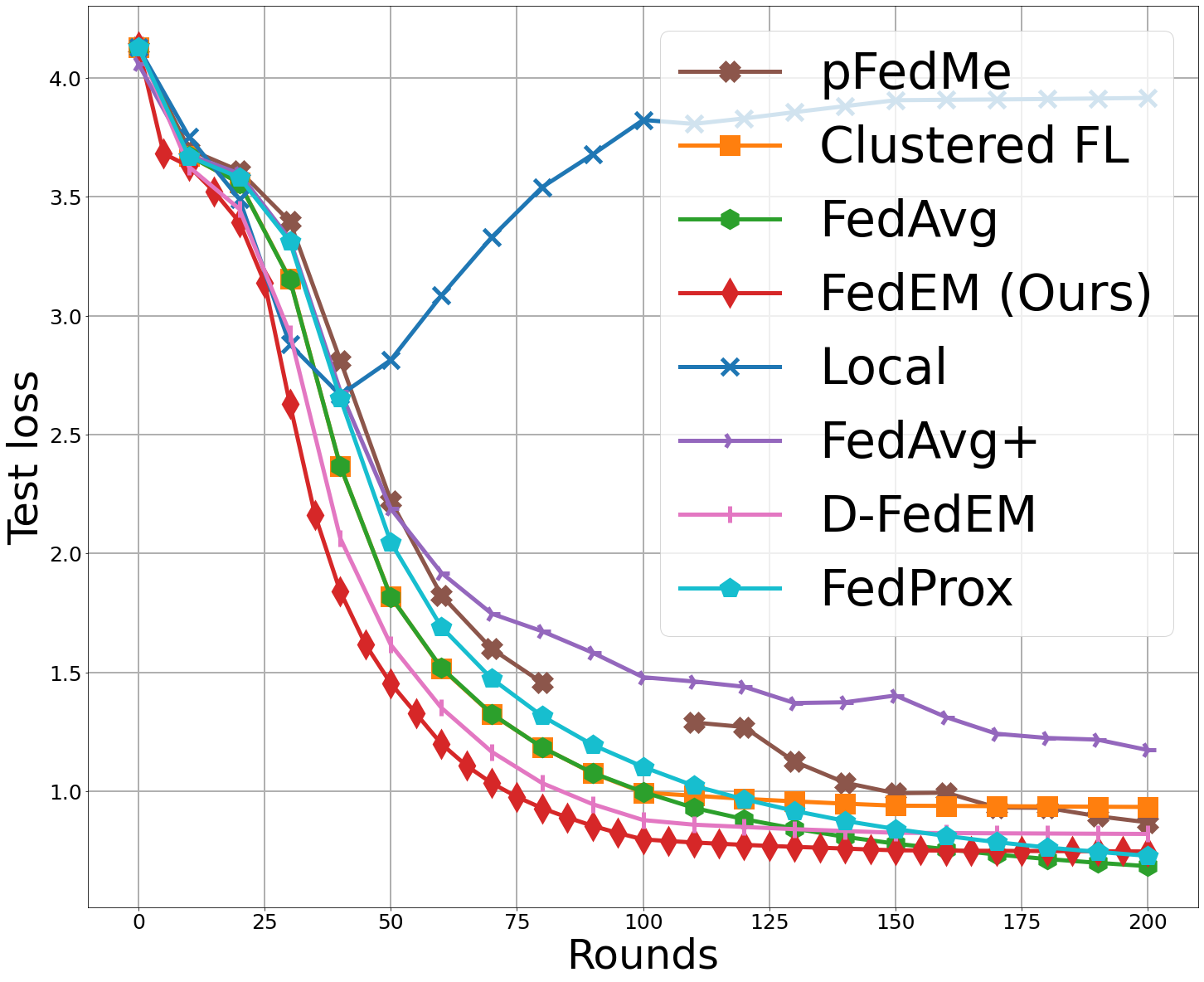} 
    \end{minipage}
    \begin{minipage}{0.5\linewidth}
        \centering
        \includegraphics[width=.9\linewidth]{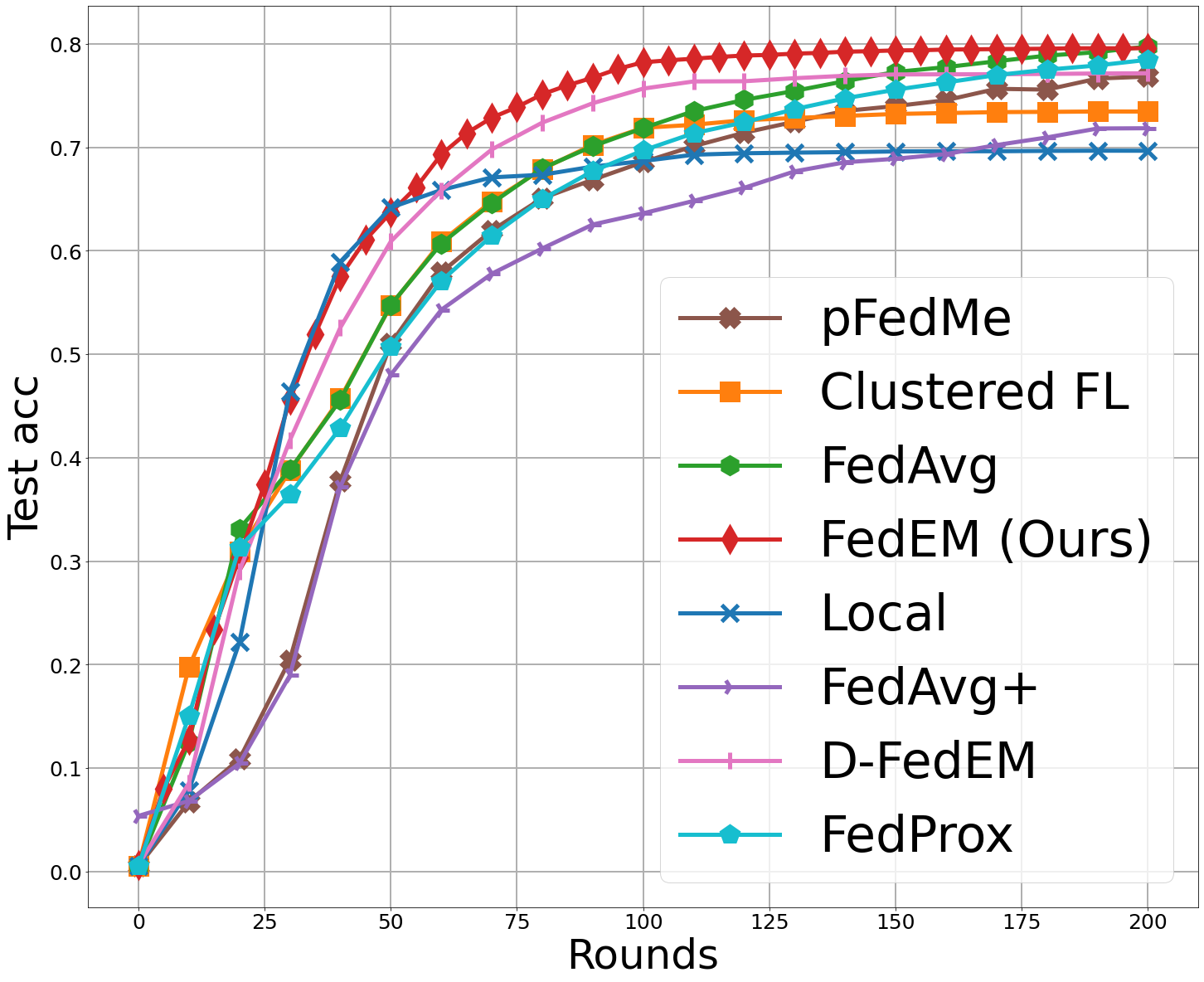} 
    \end{minipage} 
  \caption{Train loss, train accuracy, test loss, and test accuracy for   FEMNIST~\cite{caldas2018leaf, mcmahan2017communication}.}
\end{figure}

\begin{figure}[t] 
    \begin{minipage}{0.5\linewidth}
        \centering
        \includegraphics[width=.9\linewidth]{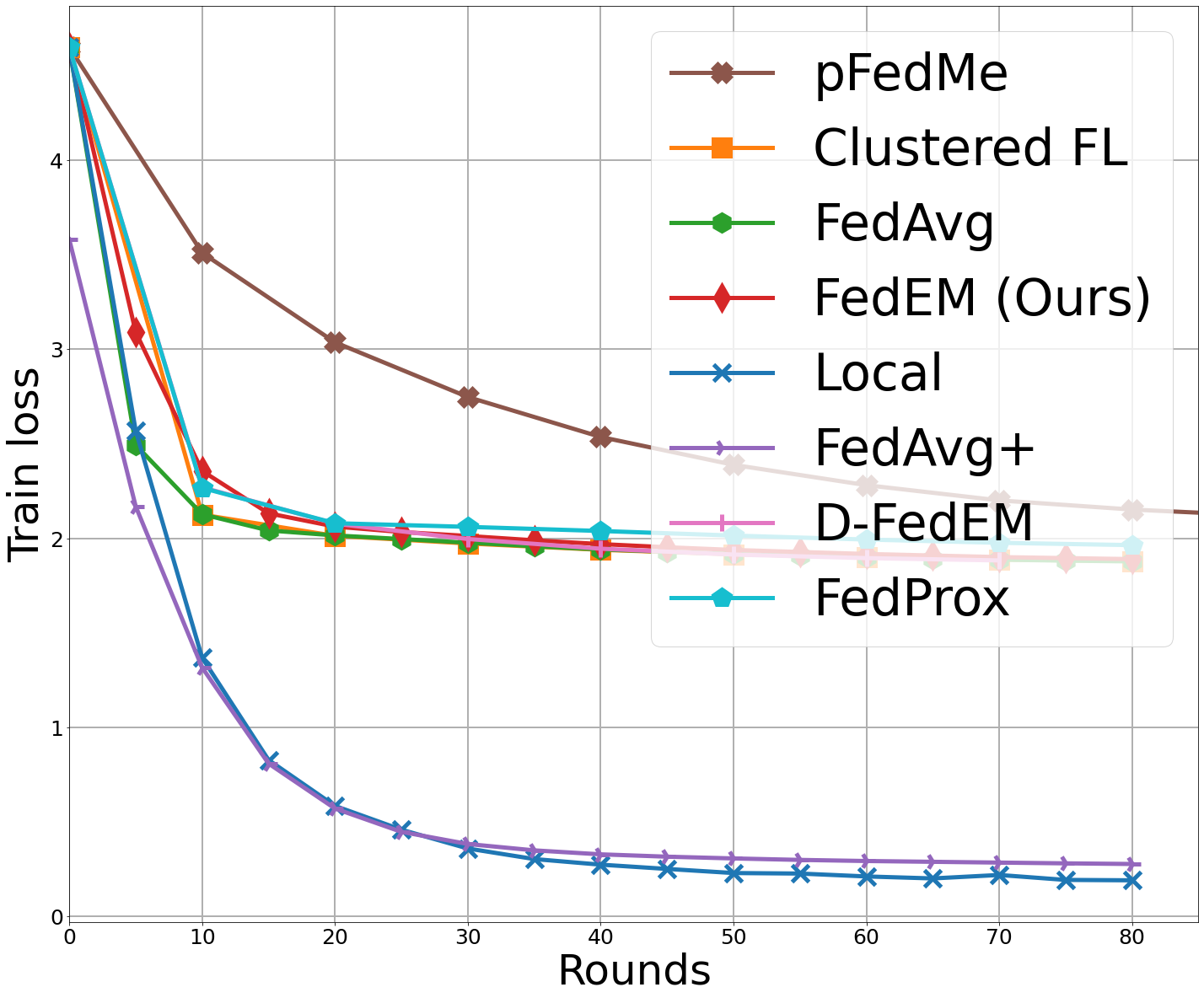} 
    \end{minipage}
    \begin{minipage}{0.5\linewidth}
        \centering
        \includegraphics[width=.9\linewidth]{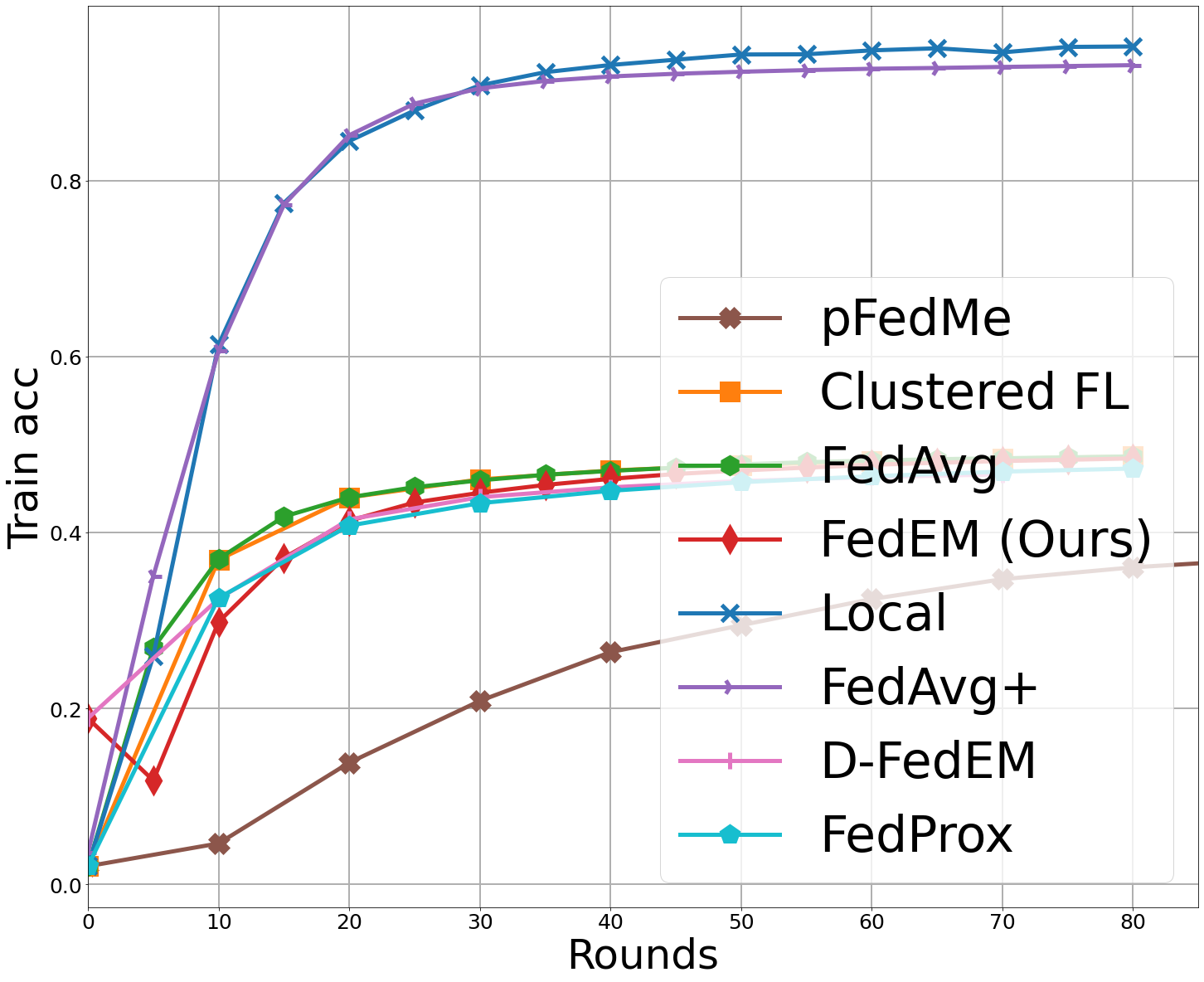} 
    \end{minipage} 
    \\
    \begin{minipage}{0.5\linewidth}
        \centering
        \includegraphics[width=.9\linewidth]{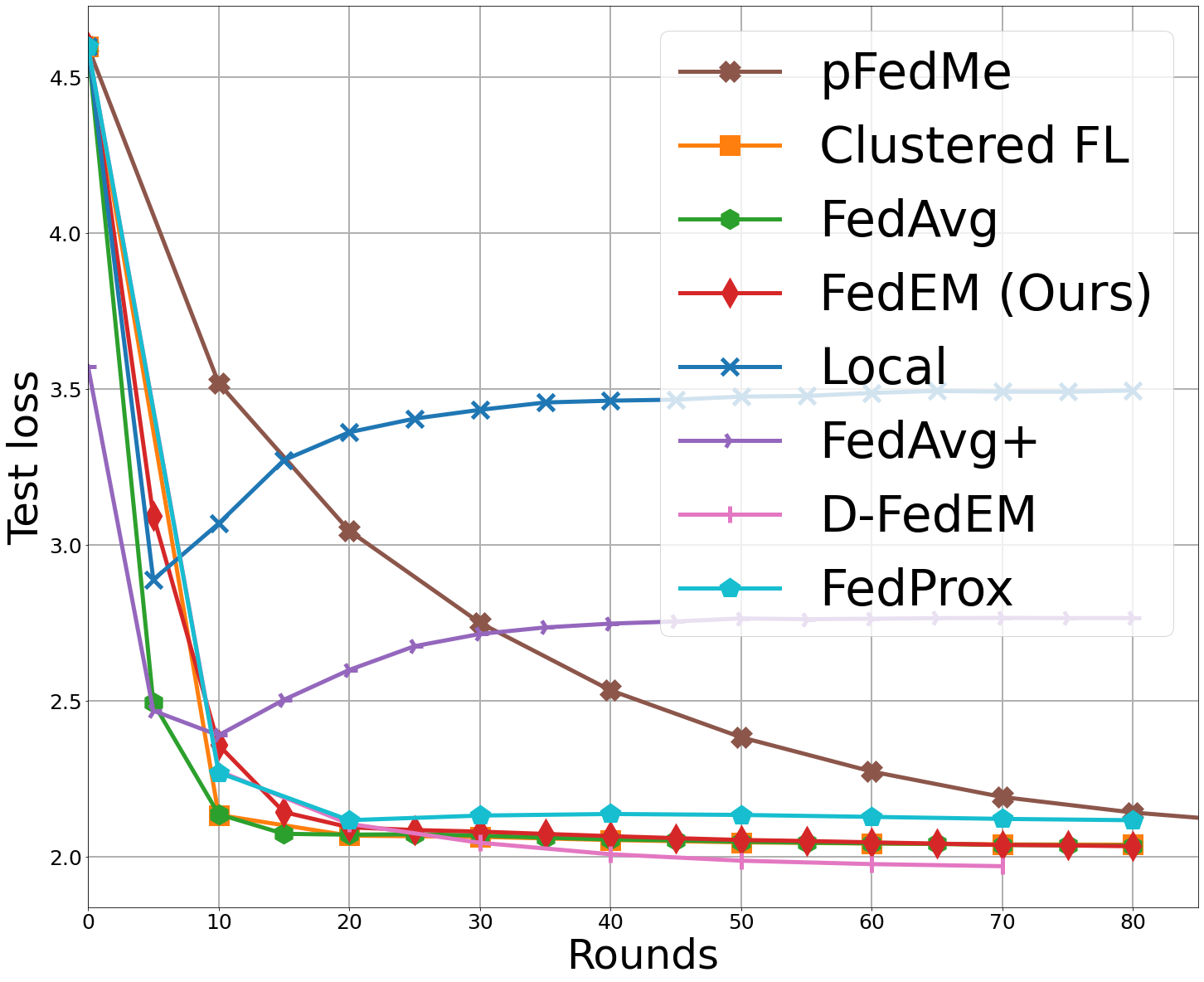} 
    \end{minipage}
    \begin{minipage}{0.5\linewidth}
        \centering
        \includegraphics[width=.9\linewidth]{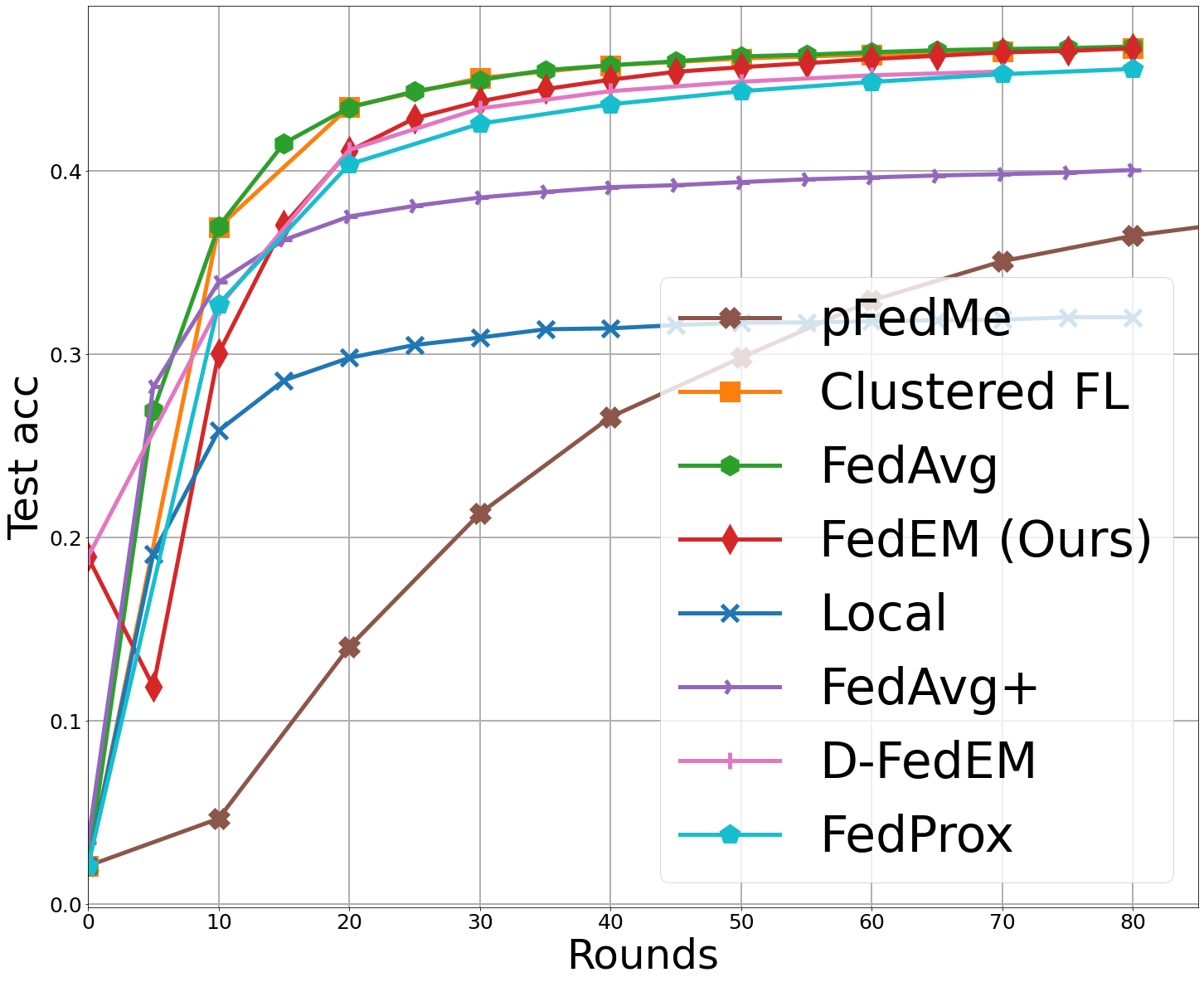} 
    \end{minipage} 
  \caption{Train loss, train accuracy, test loss, and test accuracy for   Shakespeare~\cite{caldas2018leaf, mcmahan2017communication}.}
\end{figure}

\begin{figure}[t] 
    \begin{minipage}{0.5\linewidth}
        \centering
        \includegraphics[width=.9\linewidth]{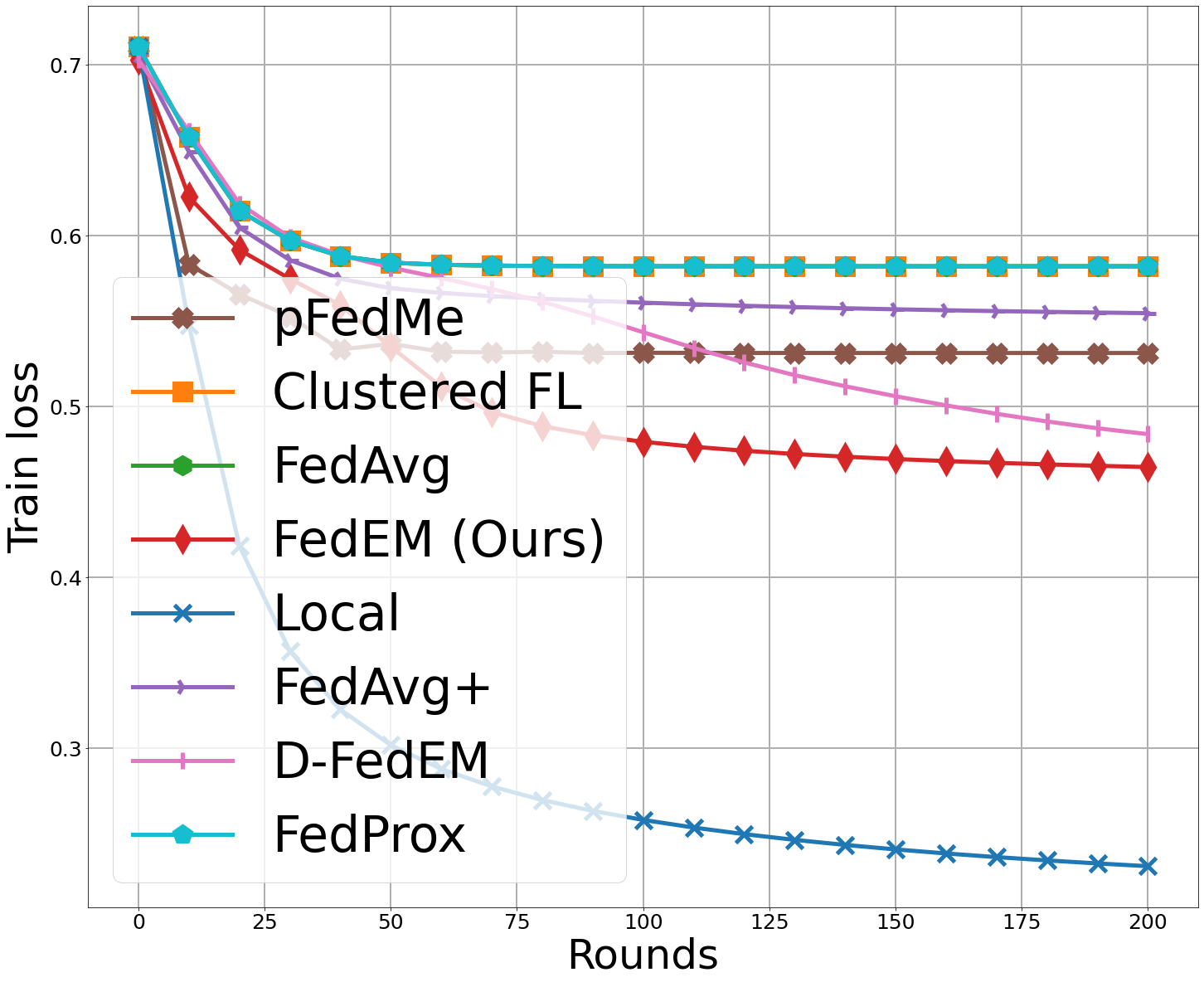} 
    \end{minipage}
    \begin{minipage}{0.5\linewidth}
        \centering
        \includegraphics[width=.9\linewidth]{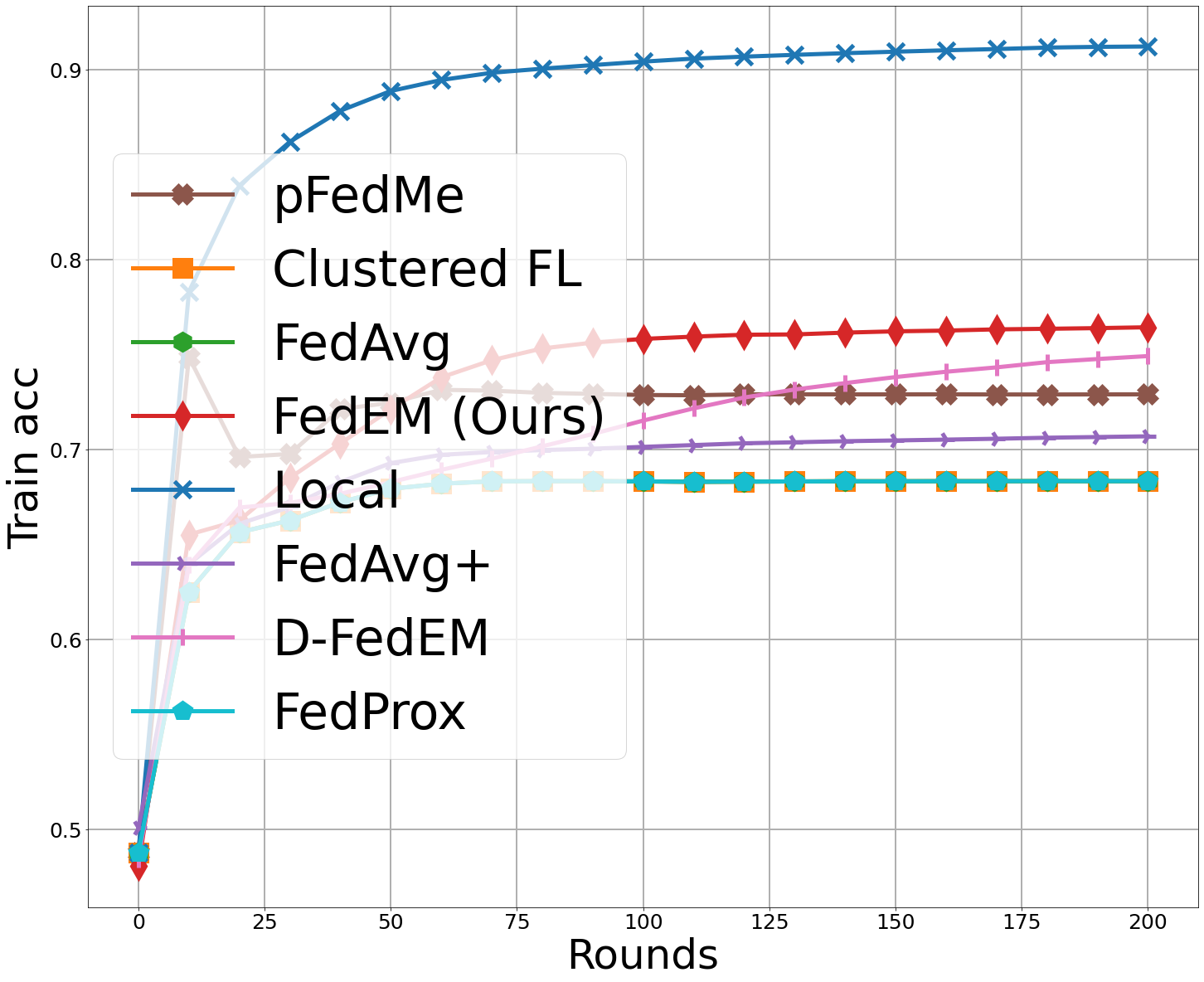} 
    \end{minipage} 
    \\
    \begin{minipage}{0.5\linewidth}
        \centering
        \includegraphics[width=.9\linewidth]{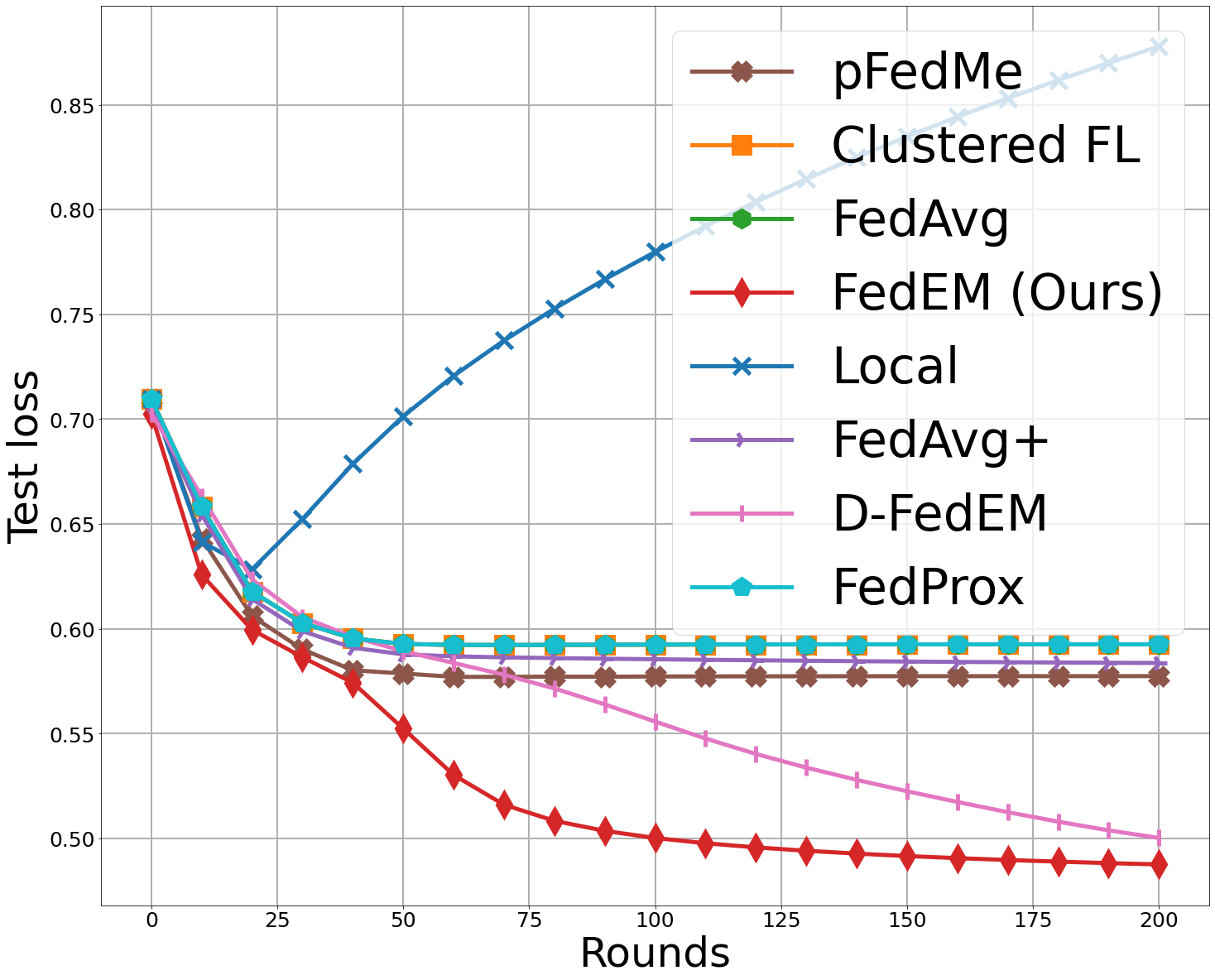} 
    \end{minipage}
    \begin{minipage}{0.5\linewidth}
        \centering
        \includegraphics[width=.9\linewidth]{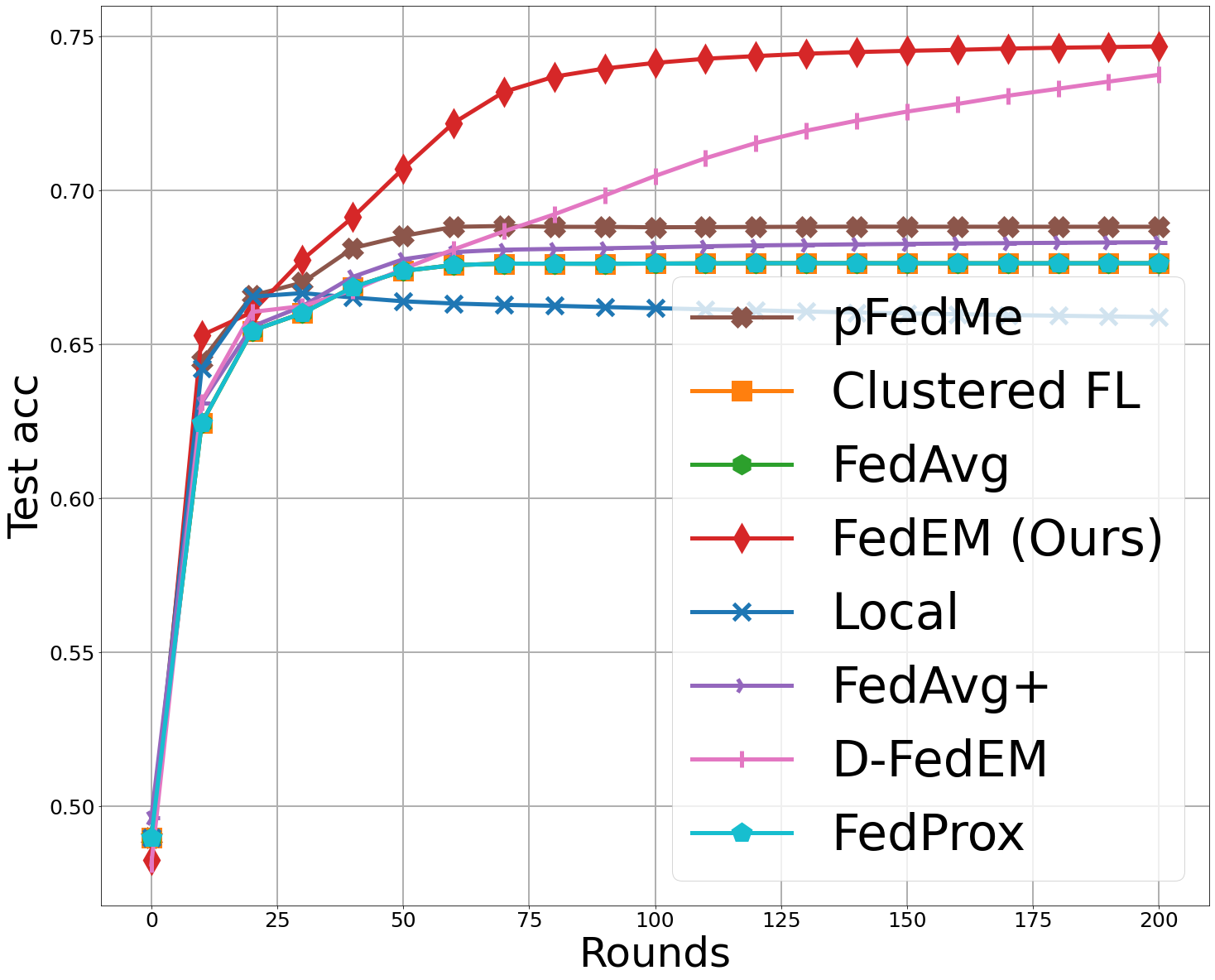} 
    \end{minipage} 
  \caption{Train loss, train accuracy, test loss, and test accuracy for   synthetic dataset.}
  \label{fig:synthetic}
\end{figure}

\end{document}